\documentclass[letterpaper]{article} % DO NOT CHANGE THIS
\usepackage{aaai24}  % DO NOT CHANGE THIS
\usepackage{times}  % DO NOT CHANGE THIS
\usepackage{helvet}  % DO NOT CHANGE THIS
\usepackage{courier}  % DO NOT CHANGE THIS
\usepackage[hyphens]{url}  % DO NOT CHANGE THIS
\usepackage{graphicx} % DO NOT CHANGE THIS
\usepackage{natbib}  % DO NOT CHANGE THIS AND DO NOT ADD ANY OPTIONS TO IT
\usepackage{caption} % DO NOT CHANGE THIS AND DO NOT ADD ANY OPTIONS TO IT
\usepackage{amsmath}
\usepackage{algorithm}
\usepackage{algorithmic}

\usepackage{caption} 
\usepackage{microtype}
\usepackage{graphicx}
\usepackage{subfigure}
\usepackage{booktabs} % for professional tables

\usepackage{amsmath}
\usepackage{amssymb}
\usepackage{mathtools}
\usepackage{amsthm}
\usepackage{tabularx}
\usepackage{multirow}

\theoremstyle{plain}
\newtheorem{theorem}{Theorem}[section]

\newtheorem{lemma}[theorem]{Lemma}

\theoremstyle{definition}

\theoremstyle{remark}

\renewcommand{\vec}[1]{\mathbf{#1}}

\usepackage[utf8]{inputenc} % allow utf-8 input
\usepackage[T1]{fontenc}    % use 8-bit T1 fonts
\usepackage{url}            % simple URL typesetting
\usepackage{booktabs}       % professional-quality tables
\usepackage{amsfonts}       % blackboard math symbols
\usepackage{nicefrac}       % compact symbols for 1/2, etc.
\usepackage{microtype}      % microtypography
\usepackage{xcolor}         % colors
\usepackage{algorithm}
\usepackage{algorithmic}
\usepackage{multirow}

\usepackage{newfloat}
\usepackage{listings}
\DeclareCaptionStyle{ruled}{labelfont=normalfont,labelsep=colon,strut=off} % DO NOT CHANGE THIS
\floatstyle{ruled}
\newfloat{listing}{tb}{lst}{}
\floatname{listing}{Listing}
\title{Pareto Front-Diverse Batch Multi-Objective Bayesian Optimization}

\author {
    Alaleh Ahmadianshalchi \equalcontrib \textsuperscript{\rm 1},
    Syrine Belakaria\equalcontrib \textsuperscript{\rm 2},
    Janardhan Rao Doppa \textsuperscript{\rm 1}
}

\affiliations {
\textsuperscript{\rm 1} School of EECS, Washington State University\\ \textsuperscript{\rm 2} Computer Science Department, Stanford University\\
a.ahmadianshalchi@wsu.edu, syrineb@stanford.edu, jana.doppa@wsu.edu
}

\begin{document}

\maketitle

\begin{abstract}
We consider the problem of multi-objective optimization (MOO) of expensive black-box functions with the goal of discovering high-quality and diverse Pareto fronts where we are allowed to evaluate a batch of inputs. This problem arises in many real-world applications including penicillin production where diversity of solutions is critical. We solve this problem in the framework of Bayesian optimization (BO) and propose a novel approach referred to as {\em {\bf P}areto front-{\bf D}iverse Batch Multi-Objective {\bf BO} (PDBO)}. PDBO tackles two important challenges: 1) How to automatically select the best acquisition function in each BO iteration, and 2) How to select a diverse batch of inputs by considering multiple objectives. We propose principled solutions to address these two challenges. First, PDBO employs a multi-armed bandit approach to select one acquisition function from a given library. We solve a cheap MOO problem by assigning the selected acquisition function for each expensive objective function to obtain a candidate set of inputs for evaluation. Second, it utilizes Determinantal Point Processes (DPPs) to choose a Pareto-front-diverse batch of inputs for evaluation from the candidate set obtained from the first step. The key parameters for the methods behind these two steps are updated after each round of function evaluations. Experiments on multiple MOO benchmarks demonstrate that PDBO outperforms prior methods in terms of both the quality and diversity of Pareto solutions.
\end{abstract}
\section{Introduction}

A wide range of science and engineering applications, including materials design \cite{ashby2000multi}, biological sequence design \cite{taneda2015multi}, and drug/vaccine design \cite{nicolaou2013multi} involves optimizing multiple {\em expensive-to-evaluate} objective functions. For example, in nanoporous materials design \cite{deshwal2021bayesian}, the goal is to optimize the adsorption property and cost of synthesis guided by physical lab experiments. Since the experiments are expensive in terms of the consumed resources, our goal is to approximate the {\em optimal Pareto set} of solutions. In many of the aforementioned applications, there are two important considerations. First, we can perform multiple parallel experiments which should be leveraged to accelerate the discovery of high-quality solutions. Second, practitioners care about diversity in solutions and their outcomes. For example, in penicillin production application, diverse solutions for objectives including penicillin production, the time to ferment, and the $CO_2$ byproduct \cite{birol2002modular}.

We consider the problem of multi-objective optimization (MOO) over expensive-to-evaluate functions to find high-quality and diverse Pareto fronts when we are allowed to perform a batch of experiments. We solve this problem using Bayesian optimization (BO) \cite{shahriari2015taking} which has been shown to be highly effective for such problems. The key idea behind BO is to learn a surrogate model (e.g., Gaussian process)from past experiments and use it to intelligently select the sequence of experiments guided by an acquisition function (e.g., expected improvement).
There is prior BO work on selecting a batch of experiments to find diverse high-performing solutions for single-objective optimization. However, there is very limited work on batch BO for MOO problems and to produce diverse MOO solutions. A key drawback of the existing BO methods for MOO is that they evaluate diversity in terms of input space, which is not appropriate for MOO (diversity in input space $\nRightarrow$ diversity in output space). To overcome this drawback and to measure the  diversity of MOO solution in the output space, we define a new metric referred to as {\em Diversity of the Pareto Front (DPF)}.

We propose a novel approach referred to as {\em {\bf P}areto front-{\bf D}iverse Batch Multi-Objective {\bf BO} (PDBO)}. PDBO selects a batch $B$ of inputs for evaluations in each iteration using two main steps. First, it employs a principled multi-arm bandit strategy to dynamically select one acquisition function (AF) from a given library of AFs within the single-objective BO literature. A cheap MOO problem is solved by assigning the selected AF for each expensive objective function to obtain a Pareto set. Second, a principled configuration of determinantal point processes (DPPs) \cite{borodin2009determinantal,kulesza2012determinantal} for multiple objectives is used to select $B$ inputs for evaluation from the Pareto set of the first step to improve the diversity of the uncovered Pareto front. PDBO updates the key parameters of the algorithms for these two steps (dynamic selection of AF and selecting $B$ Pareto-front diverse inputs from a candidate Pareto set) after obtaining the objective function evaluations. Our experiments on multiple benchmarks with varying input dimensions and number of objective functions demonstrate that PDBO outperforms prior methods in finding high-quality and diverse Pareto fronts.

\vspace{1.0ex}

\noindent {\bf Contributions.} The key contribution of this paper is developing and evaluating the PDBO approach for solving MOO problems to find high-quality and diverse Pareto fronts. 

Specific contributions are as follows:
\begin{itemize}
    \item A multi-arm bandit strategy with a novel reward function to dynamically select one acquisition function from a given library for MOO problems.
    \item A novel DPP method for selecting a batch of inputs from a given Pareto set to maximize Pareto front diversity using a new mechanism to generate the DPP kernel for MOO. 
    \item To demonstrate the effectiveness of PDBO and study the Pareto front diversity, we propose a new metric to measure the diversity of \emph{Pareto fronts}. To the best of our knowledge, this is the first attempt at extensive experimental evaluation of Pareto front diversity within BO.
    \item Theoretical analysis of our PDBO algorithm in terms of asymptotic regret bounds.
    \item Experimental evaluation of PDBO and baselines on multiple benchmark MOO problems. The code for PDBO is publicly available at \url{https://github.com/Alaleh/PDBO}.
\end{itemize}
\section{Problem Setup and Background} \label{setup}

We first formally define the batch MOO problem along with the metrics to evaluate the quality and diversity of solutions. Next, we provide an overview of the BO framework.

\vspace{1.0ex}

\noindent {\bf Batch Multi-Objective Optimization.} We consider a MOO problem where the goal is to 
optimize multiple conflicting functions. Let $\mathfrak{X} \subset \mathbb{R}^d$ be the input space of $d$ design variables, where each candidate input $\vec{x} \in \mathfrak{X}$ is a $d$-dimensional input vector. And let $\{f_1, \cdots, f_{K}\}$ with $K \geq 2$ be the objective functions defined over the input space $\mathfrak{X}$ where $f_1(\vec{x}), \cdots, f_K(\vec{x})  : \mathfrak{X} \rightarrow \mathbb{R}$. We denote the functions evaluation at an input $\vec{x}$ as  $\vec{y}=[y_1, \cdots, y_K]$, where $y_i = f_i(\vec{x})$ for all $i \in \{1, \cdots, K\}$. Without loss of generality, we assume minimization for all $K$ objective functions. The optimal solution of the MOO problem is a set of inputs $\mathcal{X}^* \subset \mathfrak{X}$ such that no input $\vec{x'} \in \mathfrak{X} \setminus \mathcal{X}^*$ Pareto-dominates another input $\vec{x} \in \mathcal{X}^*$. An input $\vec{x}$ Pareto-dominates another point $\vec{x}'$ if and only if $ \forall j : f_j(\vec{x}) \leq f_j(\vec{x}')$ and $\exists  j : f_j(\vec{x})<f_j(\vec{x}')$. The set of input solutions $\mathcal{X}^*$ is called the optimal {\em Pareto set} and the corresponding set of function values $\mathcal{Y}^*$  is called the optimal {\em Pareto front}. We can select $B$ inputs for parallel evaluation in each iteration, and our goal is to uncover a high-quality and diverse Pareto front while minimizing the total number of expensive function evaluations.

\vspace{1.0ex}

\noindent {\bf Metrics for Quality and Diversity of Pareto Front.} Our goal is to find high-quality and diverse Pareto fronts. The diversity of the Pareto front has not been formally \textit{evaluated} in any previous work. We introduce an appropriate evaluation metric to measure the diversity of the Pareto front and discuss an existing measure of Pareto front quality below.

\vspace{0.5ex}

{\em Diversity of Pareto Front.} Diversity is an important criterion for many optimization problems. Prior work on batch BO, both in the single-objective and MO settings focused on evaluating diversity with respect to the input space \cite{jain2022multi}. However, in most real-world MOO problems, diversity in the input space does not necessarily reflect diversity in the output space. In MOO, practitioners might care more about the diversity of the Pareto front rather than the Pareto set. Yet, little work has gone into understanding, formalizing, and measuring Pareto front diversity in MOO. In most cases, finding a more diverse set of points in the output space leads to a higher hypervolume \cite{zitzler1999multiobjective}. However, a higher hypervolume does not necessarily correspond to a more diverse Pareto front. \cite{konakovic2020diversity} is the only prior work that proposed a diversity-guided approach for batch MOO. However, the diversity of the produced Pareto front \textit{was not evaluated}. To fill this gap, we propose an evaluation metric to fill this gap to assess the {\em Diversity of the Pareto Front (DPF)}. Given a Pareto front $\mathcal{Y}_t$, $DPF (\mathcal{Y}_t)$ is the average pairwise distance between points (i.e., output vectors) in Pareto front $\mathcal{Y}_t$. It is important to clarify that the pairwise distances are computed in the output space between different vector pairs $(\vec{y}, \vec{y}') \in \mathcal{Y}_t$, unlike previously used metrics to assess input space diversity in the single-objective setting \cite{angermueller2020population}.
\begin{align*}
& DPF(\mathcal{Y}_t) = \frac{\sum_{(i,j) \in \mathcal{I}} || \vec{y_i} - \vec{y_j}||}{ |\mathcal{I}|} \\
& \text{with } \mathcal{I} = \{(i,j) \forall i,j \in \{1 \cdots t\}, i<j\}
\end{align*}
We provide a more detailed discussion of existing metrics and some illustrative results on previous metrics and their utility in evaluating diversity in the Appendix.

\vspace{0.5ex}

{\em Hypervolume Indicator.} The hypervolume indicator \cite{zitzler1999multiobjective} is the most commonly employed measure to evaluate the quality of a given Pareto front. Given a set of functions evaluations $\mathcal{Y}_t = \{\vec{y}_0, \cdots, \vec{y}_t\}$, the Pareto hypervolume (PHV) indicator is the volume between a predefined reference point $\vec{r}$ and the given Pareto front.

\subsection{Bayesian Optimization Framework}

BO \cite{shahriari2015taking} is a general framework for solving expensive black-box optimization problems in a sample-efficient manner. 
BO algorithms iterate through three steps. 

\textbf{1)} Build a probabilistic {\em surrogate model} of the true expensive objective function. Gaussian process (GPs) \cite{williams2006gaussian}) are widely employed in BO. 

\textbf{2)} Define an {\em acquisition function (AF)} to score the utility of unevaluated inputs. It uses the surrogate model's predictions as a cheap substitute for expensive function evaluations and strategically trades off exploitation and exploration. 

Some examples of widely used acquisition functions in single-objective optimization include expected improvement (EI) \cite{mockus1978application}, upper confidence bound (UCB)  \cite{auer2002using}, Thompson sampling (TS) \cite{thompson1933likelihood} and Identity (ID).
\begin{align}
    & EI(\vec{x}) =\sigma(\vec{x})(\alpha \Phi(\alpha) + \phi(\alpha)), \alpha = (\tau -  \frac{\mu(\vec{x})}{\sigma(\vec{x})})\\
    & UCB(\vec{x}) = \mu(\vec{x}) + \sqrt{\beta} \sigma(\vec{x})\\
    & TS(\vec{x}) = \hat{f}(\vec{x}) ~ with ~ \hat{f} \sim \textsc{Gp}\\
    & ID(\vec{x}) = \mu(\vec{x})
\end{align}

Here, $\beta$ is a parameter to balance exploration and exploitation in the UCB acquisition function \cite{srinivas2009Gaussian}, $\tau$ is the best-uncovered function value; and $\Phi$ and $\phi$ are the CDF and PDF of a standard normal distribution respectively.

\textbf{3)} Select the input with the highest utility score for function evaluation by optimizing the acquisition function. \looseness=-1

\section{Related Work}

\noindent {\bf Multi-Objective BO.} There is relatively less work on MOO in comparison with single-objective BO. Prior work builds on the insights from single-objective methods \cite{shahriari2015taking,PES,MES,hvarfner2022joint} for MOO. Recent work on MOBO includes Predictive Entropy Search \cite{PESMO}, Max-value Entropy Search \cite{belakaria2019max}, Multi-Objective Regionalized BO \cite{daulton2022multiobjective}, Uncertainty-aware Search  \cite{Usemo}, Pareto-Frontier Entropy Search \cite{suzuki2020multi}, and Expected Hypervolume Improvement \cite{daulton2020differentiable,emmerich2008computation}. These methods are shown to perform well on a variety of MOO problems.

\vspace{1.0ex}

\noindent {\bf Batch Multi-objective BO.} The Batch BO problem in the multi-objective setting is even much less studied. Diversity-Guided Efficient Multi-Objective Optimization (DGEMO) \cite{konakovic2020diversity} approximates and analyzes a piecewise-continuous Pareto set representation which allows the algorithm to introduce a batch selection strategy that optimizes for both hypervolume improvement and diversity of selected samples. However, DGEMO work did not study or evaluate the Pareto-front diversity of the produced solutions. qEHVI \cite{daulton2020differentiable} is an exact computation of the joint EHVI of $q$ new candidate points (up to Monte-Carlo integration error). 

qPAREGO is a novel extension of ParEGO \cite{knowles2006parego, daulton2020differentiable} that supports parallel evaluation and constraints. More recent work \cite{lin2022pareto} proposed an approach to address MOO problems with continuous/infinite Pareto fronts by approximating the whole Pareto set via a continuous manifold. This approach enables a better preference-based exploration strategy for practitioners compared to prior work \cite{abdolshah2019multi,paria2020flexible,astudillo2020multi}. However, it is typically unknown to the user if the Pareto front is dense/continuous, especially in expensive function settings where the data is limited. {\em It is not known whether any of these batch methods produce diverse Pareto fronts or not, as they were not evaluated on diversity metrics}. We perform an experimental evaluation to answer this question.

\vspace{1.0ex}

\noindent {\bf DPPs for Batch Single-Objective BO.} 
DPPs are elegant probabilistic models \cite{borodin2005harmonic,borodin2009determinantal} that characterize the property of repulsion in a set of vectors and are well-suited for the selection of a diverse subset of inputs from a predefined set. Prior work used DPPs for selecting batches for evaluation in the single-objective BO literature \cite{kathuria2016batched,nava2022diversified,wang2017batched}. In the context of MOO for cheap objective functions using evolutionary algorithms, DPP was deployed using non-learning-based kernels such as cosine function applied to points in the Pareto front while disregarding the input space \cite{wang2022many,zhang2020new,okoth2022large}. However, {\em to the best of our knowledge, there is no work on using DPPs for multi-objective BO to uncover diverse Pareto fronts using a learned kernel that captures the trade-off between multiple objectives, similarity in the input space, and diversity of the Pareto front}.

\vspace{1.0ex}

\noindent {\bf Adaptive Acquisition Function Selection.} There has been a plethora of research on finding efficient and reliable acquisition functions (AFs). However, prior work has shown that no single acquisition function is universally efficient and consistently outperforms all others. GP-Hedge \cite{hoffman2011portfolio} proposed to use a portfolio of acquisition functions. The optimization of each AF will nominate an input, and the algorithm will select one of them for evaluation using the selection probabilities. The GP-Hedge method uses the Hedge strategy \cite{freund1997decision}, a multi-arm bandit method designed to choose one action amongst a set of different possibilities using selection probabilities calculated based on the reward (performance given by function values) collected from previous evaluations.
\citeauthor{nopastvasconcelos2019no} extended \citeauthor{hoffman2011portfolio} by proposing to use discounted cumulative reward and \citeauthor{setupvasconcelos2022self} suggested using Thompson sampling to automatically set the hedge hyperparameter $\eta$.

\section{Proposed PDBO Algorithm}\label{overview_pdbo}
We start by providing an overview of the proposed PDBO algorithm illustrated in Figure \ref{fig:overview_pdbo}. 
Next, we explain our algorithms for two key components of PDBO, namely, adaptive acquisition function selection via a multi-arm bandit strategy and diverse batch selection via determinantal point processes for multi-objective output space diversity.

\begin{figure*}[t]
    \centering
    \includegraphics[width=\textwidth]{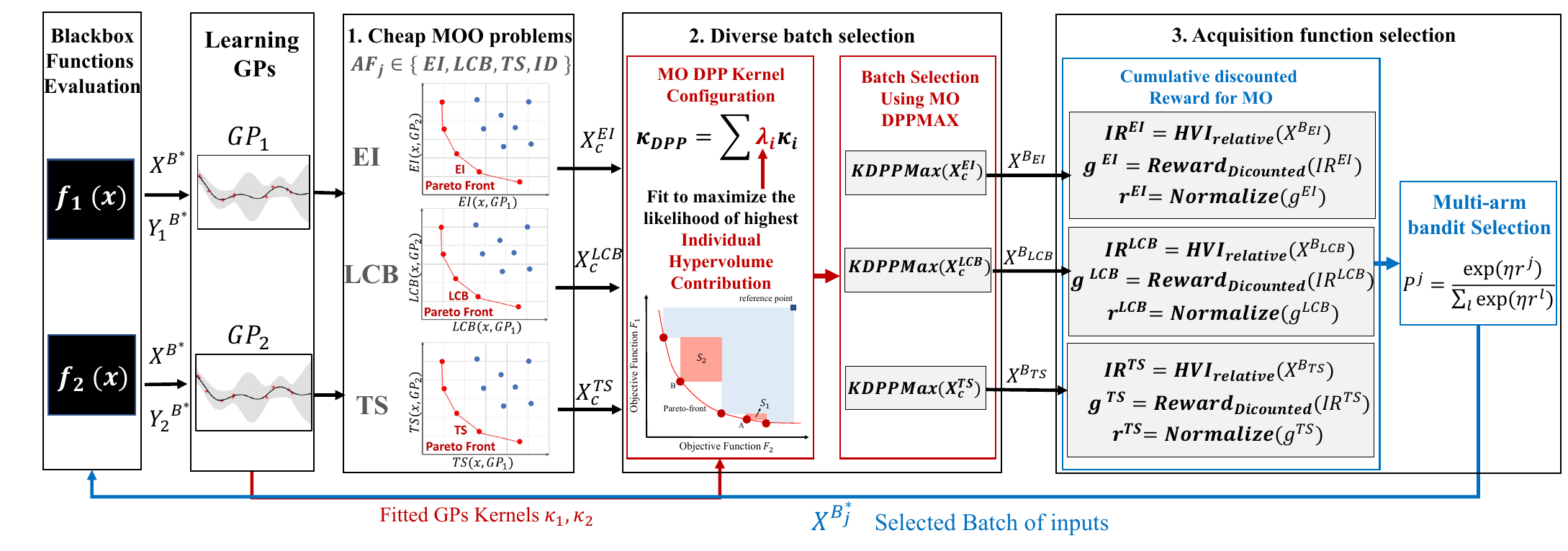}
    \caption{High-level overview of PDBO algorithm illustrating its three key components.}
    \label{fig:overview_pdbo}
\end{figure*}

\vspace{1.0ex}

\noindent \textbf{Overview of PDBO.} PDBO is an iterative algorithm. It introduces novel methods for selecting varying acquisition functions and for bringing the diversity of inputs into the multi-objective BO setting. The method builds $K$ independent Gaussian processes $\textsc{Gp}_1, \cdots, \textsc{Gp}_K$ as surrogates for each of the objective functions. Its three key steps at each iteration $t$ to select $B$ inputs for evaluation are:

{\em 1. Solving multiple cheap MOO problems:} PDBO takes as input a portfolio of acquisition functions, $\mathcal{P}=\{\textsc{Af}_1, \cdots, \textsc{Af}_M\}$ for single-objective BO. It constructs $M$ cheap MOO problems, each corresponding to one AF. The multiple objectives defining the cheap MOO problems are acquisition functions respectively corresponding to the $K$ objective functions. Solving cheap MOO problems will generate $M$ cheap Pareto-sets of solutions $\mathcal{X}_c^{1} \cdots \mathcal{X}_c^{M}$. 

{\em 2. Diverse batch selection:} From each cheap Pareto set $\mathcal{X}_c^{j}$, a batch $X^{B_j}_t \subset \mathcal{X}_c^{j}$ of $B$ inputs is selected using a diversity-aware approach based on DPPs. Importantly, the adapted DPP is configured to favor the diversity in the output space and to handle multiple objective settings by using a principally fitted convex combination of the kernels of the $K$ Gaussian processes. The convex combination scalars are strategically set to maximize the likelihood of selecting a diverse subset of inputs with respect to the Pareto front. 

{\em 3. Acquisition function selection:}
From $\{X^{B_j}_t ; ~ \forall ~ j\in[1, \cdots, M]\}$, only one nominated subset would be selected using a multi-arm bandit strategy. The keys to this selection are probabilities $p^1, \cdots, p^j$, one for each acquisition function to capture their performance based on past iterations. $p^j$ is the probability of selecting the batch generated by the acquisition function $\textsc{Af}_j$, defined in equation \ref{prob_eq}. These probabilities are updated based on the discounted cumulative reward $r^j$ of each of the respective acquisition functions. The reward values $r^j$ are updated based on the quality of the batches nominated by the respective  acquisition functions.

Algorithm \ref{our_method} provides a pseudocode with high-level steps of the PDBO approach. The $\textsc{DPP-Select}$ and \textsc{Adaptive-Af-Select} represent the second and third key steps. The details of these methods and their corresponding pseudocodes are provided in Sections \ref{adaptive_acq} and \ref{DPP}, respectively.

\subsection{Multi-arm Bandit Strategy for Adaptive Acquisition Function Selection} \label{adaptive_acq}
In this section, we propose a multi-arm bandit (MAB) approach to {\em adaptively} select one acquisition function (AF) from a given library of AFs in each iteration of PDBO.

\begin{algorithm}[t]
\caption{Pareto front-Diverse Batch Multi-Objective BO}
\footnotesize
\textbf{Input}: $\mathfrak{X}$ input space; \{$f_1, \cdots, f_K$\}, $K$ black-box objective functions; $\mathcal{P}$ = \{$\textsc{Af}_1, \cdots, \textsc{Af}_M$\} portfolio of acquisition functions; $B$ batch size; and  $T_{max}$ number of iterations
\begin{algorithmic}[1] \label{our_method}
\STATE Initialize data $\mathcal{D}_0=\{\mathcal{X}_{0},\mathcal{Y}_{0}\}$ with $N_0$ initial points
\FOR{each iteration $t \in [1,T_{max}]$}
\STATE Fit statistical models $\textsc{Gp}_1, \cdots, \textsc{Gp}_k$ using $\mathcal{D}_{t-1}$ 
\FOR{each acquisition function $\textsc{Af}_j \in \mathcal{P} $}
\STATE $\mathcal{X}_{c}^j \leftarrow \arg min_{\vec{x} \in \mathfrak{X}} \; (\textsc{Af}_j(\textsc{Gp}_1,\vec{x}),\cdots, \textsc{Af}_j(\textsc{Gp}_k,\vec{x}))$ // Solve cheap MOO problem
\STATE $X^{B_j}_{t} \leftarrow \textsc{DPP-Select}(\mathcal{X}_{c}^j,\{\textsc{Gp}_1, \cdots, \textsc{Gp}_k\},\mathcal{D}_{t-1})$ \\ // Select a batch of inputs from $\mathcal{X}_{c}^j$ using DPPs
\ENDFOR
\STATE $ \textsc{Af}_{j^*} \leftarrow \textsc{Adaptive-AF-Select}(\mathcal{D}_{t-1},\{X^{B_j}_{t-1}\}_{j=1}^M )$ \\ // Select an AF using previously aggregated data
\STATE $X^{B}_{t}=X^{B_{j^*}}_{t}$ // Choose the batch nominated by $\textsc{Af}_{j^*}$ 
\STATE $Y^B_{t} \leftarrow \{ [f_1(\vec{x}),\cdots,f_k(\vec{x})] ; ~ \forall ~ \vec{x} \in X^B_{t} \}$ \\ // Evaluate objective functions for batch of inputs $X^B_{t}$
\STATE $ \mathcal{D}_t=\{\mathcal{X}_{t},\mathcal{Y}_{t}\} \leftarrow \{\mathcal{X}_{t-1},\mathcal{Y}_{t-1}\} \cup \{(X^B_{t}, Y^B_{t})\}$ 
\ENDFOR
\STATE \textbf{return} Pareto set $\mathcal{X}_{T_{max}}$ and Pareto front  $\mathcal{Y}_{T_{max}}$
\end{algorithmic}
\end{algorithm}

\noindent {\bf Multi-arm Bandit Formulation.} We are given a portfolio of $M$ acquisition functions $\mathcal{P}$ = $\{\textsc{Af}_1, \cdots, \textsc{Af}_M\}$ and our goal is to adaptively select one AF in each iteration of PDBO. Each acquisition function in $\mathcal{P}$ corresponds to one arm, and we need to select an arm based on the performance of past selections for solving the MOO problem. Inspired by the previous work on AF selection and algorithm selection in the single objective setting \cite{hoffman2011portfolio,nopastvasconcelos2019no,setupvasconcelos2022self}, we propose an adaptive acquisition function selection approach for the MOO setting (see Algorithm 2). We explain the two main steps of this approach below.

\vspace{1.0ex}

\noindent {\bf Nominating Promising Candidates via Cheap MOO.}\label{cheapmo} In each PDBO iteration, we employ the updated statistical models $\{\textsc{Gp}_1 \cdots \textsc{Gp}_K \}$ and the portfolio $\mathcal{P}$ to generate $M$ sets of candidate points. For each $\textsc{Af}_j$, the algorithm constructs a cheap MOO problem with the objectives defined as $\textsc{Af}_j(\textsc{Gp}_1,\vec{x}) \cdots \textsc{Af}_j(\textsc{Gp}_K,\vec{x})$. Assuming minimization, the cheap MOO generate $M$ Pareto-sets of solutions $\mathcal{X}_c^{1} \cdots \mathcal{X}_c^{M}$ (one for each acquisition function) defined as:
\begin{align}
\label{eq2}
\mathcal{X}_c^{j}\leftarrow \arg \min_{\vec{x} \in \mathfrak{X}} \; \{\textsc{Af}_j(\textsc{Gp}_1,\vec{x}),\cdots, \textsc{Af}_j(\textsc{Gp}_K,\vec{x})\}
\end{align}
We employ the algorithm proposed by \citeauthor{deb2002nsga} to solve cheap MOO problems defined in \ref{eq2}.
From each $\mathcal{X}_c^{j}$, a batch $X^{B_j}_t \subset \mathcal{X}_c^{j}$ of $B$ points is selected in iteration $t$ using a diversity-aware approach described in Section \ref{DPP}. We denote the function evaluations of inputs in $X^{B_j}_t$  by $Y^{B_j}_t$.

\vspace{1.0ex}

\noindent {\bf Multi-Objective Reward Update.} We employ the relative hypervolume improvement as the quality metric to define our reward. The Pareto hypervolume captures the quality of nominated batches from a Pareto-dominance perspective and carries information about the represented trade-off between the multiple objectives. Defining the immediate reward as the raw Pareto hypervolume of the nominated batch $IR_t^j= HV(Y_t^j)$ can lead to an undesirable assessment of the suitable acquisition function since a batch of points can have a large hypervolume value at iteration $t$ but does not provide a significant improvement over the previous Pareto front $\mathcal{Y}_{t-1}$ while another batch nominated in iteration $t-1$ may have a smaller hypervolume $HV(Y_{t-1}^j)$ yet provides a higher improvement over $\mathcal{Y}_{t-2}$. Additionally, initial iterations might provide drastic hypervolume improvements even if the selected points are not optimal. To mitigate these issues, we use the relative hypervolume improvement as the immediate reward instead of the hypervolume. In each BO iteration $t$, the immediate reward $IR_t^j$ for each acquisition function $\textsc{Af}_j;  ~ \forall ~ j \in \{1\cdots M\}$ is defined as follows. 

\begin{equation} \label{immediate_reward}
    IR_t^j= \frac{HV(\tilde{\mathcal{Y}}_{t-1} \cup \tilde{Y}_t^j)-HV(\tilde{\mathcal{Y}}_{t-1})}{HV(\tilde{\mathcal{Y}}_{t-1})}
\end{equation}
where $\tilde{\mathcal{Y}}_{t-1}$  is the Pareto front at iteration $t-1$ and $\tilde{Y}_t^j$ is the evaluation of the batch of points $X_t^j$ nominated by $\textsc{Af}_j$ computed using the {\em predictive mean of the updated GP based statistical models}. As optimization progresses, statistical models provide a better representation of the objective functions, and the batches nominated by each AF become more informative about the quality of its selections. Hence, the impact of the early iterations may become irrelevant later. Consequently, we employ a {\em discounted cumulative reward} for each acquisition function $\textsc{Af}_j$ at iteration $t$ (denoted $g_{t}^j$).
\begin{equation} \label{cumulative_discounted_reward}
    g_{t}^j = \gamma g_{t-1}^j + IR_t^j = \sum_{t' \leq t} \gamma^{t'-1} IR_{t'}^j
\end{equation}
Where $\gamma$ is a decay rate that trades off past and recent improvements. The use of the decay rate can lead to equal or comparable rewards in advanced iterations, causing the algorithm to select the acquisition function randomly. To address this problem, the discounted cumulative reward $g_{t}^j$ should be normalized \cite{nopastvasconcelos2019no}. The rewards at the first iteration are all initialized to zero and then updated at each iteration $t$ using the following expression:
\begin{equation} \label{normalized_reward}
    r_{t}^j= \frac{ g_{t}^j - g_{max}^j }{g_{max}^j-g_{min}^j}
\end{equation}
where $g_{max}^j$ = $max(\{g^j_{t'}, t'\in [1,t] \})$ and  $g_{min}^j$ = $\min(\{g^j_{t'}, t'\in [1,t] \})$. Finally, the probability of selection of each AF at each iteration $t$ is calculated using equation \ref{prob_eq}.
\begin{equation}  \label{prob_eq}
p_t^j=\frac{exp(\eta r_{t}^j)}{\sum_{l=1}^M exp(\eta r_{t}^l)} ~ for ~ j=1 \dots M 
\end{equation}
The proposed MAB approach is an extension of the Hedge algorithm but is fundamentally distinct in its methodology and applicability. While it does generalize certain aspects of Hedge, it introduces critical variations that make it unique within the context of this study. Our approach diverges from the original Hedge algorithm by incorporating two key modifications: {\bf 1)} The use of discounted rewards and the application of normalization. Unlike the conventional Hedge, where rewards are typically used without any discounting or normalization, our strategy accounts for these factors, enhancing its adaptability to the specific problem domain. {\bf 2)} Another significant departure lies in the problem setting itself. Hedge was initially designed for single-objective optimization while our proposed approach solves the more challenging problem of multi-objective optimization. This shift in focus has substantial implications, as it requires an entirely different set of considerations and techniques to address the complexities introduced by multiple conflicting objectives.

\begin{algorithm}[t]
\caption{\textsc{Adaptive-AF-Select}}
\footnotesize
\textbf{Input}: training data $\mathcal{D}_t$; Batches nominated by different AFs $\{X^{B_j}_{t}, \forall j \in \{1 \cdots M\}\}$
\begin{algorithmic}[1] \label{afselect}
\IF{$t==0$}
\STATE $r_{t}^j=0 ~ \text{for} ~ j=1 \cdots M$ 
\ELSE
\STATE Compute rewards: $r_{t}^j ~ \text{for} ~ j=1 \cdots M$  using Equation \ref{normalized_reward}
\STATE Update probabilities: $p_t^j=\frac{exp(\eta r_{t}^j)}{\sum_{l=1}^M exp(\eta r_{t}^l)} ~ \text{for} ~ j=1 \cdots M$
\ENDIF
\STATE Select $\textsc{Af}_{j^*}$ according to the probabilities $\{p_t^j\}_{j=1}^{M}$ 
\STATE \textbf{return} acquisition function $\textsc{Af}_{j^*}$
\end{algorithmic}
\end{algorithm}

\vspace{0.5ex}

{\bf Remark.} It is important to note that we are using a full information multi-arm bandit strategy that requires the reward to be updated for all possible actions (i.e., for all acquisition functions) at each iteration.  
Since we evaluate only the batch nominated by the selected AF, we achieve this by computing the reward using the predictive mean functions of the {\em updated} surrogate models. For this reason, we solve a cheap MOO problem for each $\textsc{Af}_j \in \mathcal{P}$ even though the acquisition function is selected based on the data from the previous iterations. Algorithm \ref{afselect} provides the pseudocode of the adaptive AF selection based on the estimated rewards and probabilities. 

\subsection{DPPs for Batch Selection} \label{DPP}
We explain our approach to select a batch of diverse inputs by configuring DPPs to promote {\em output space diversity}.

\vspace{1.0ex}

\noindent \textbf{Determinantal Point Processes.} (DPPs) \cite{kulesza2012determinantal} are well-suited to model samples of a diverse subset of $k$ points from a predefined set of $n$ of points. Given a similarity function over a pair of points, DPPs assign a high probability of selection to the most diverse subsets according to the similarity function. The similarity function is typically defined as a kernel. 
Formally, given a DPP kernel defined over a set $S$ of $n$ elements, the $k$-DPP distribution is defined as selecting a subset $S'$ of size $k$ with $S'\subset S$ with probability proportional to the determinant of the kernel:
\begin{align}
    Pr(S')=\frac{det(\kappa(S'))}{\sum_{|s|=k} det(\kappa(s))}
\end{align}
\cite{kathuria2016batched} introduced the use of DPP for batch BO in the single-objective setting. Given the surrogate GP of the objective function, the covariance of the GP is used as the similarity function for the DPP. The approach selects the first point in the batch by maximizing the UCB acquisition function. Next, it creates a set of points referred to as relevance region by bounding the search space with the maximizer of the LCB acquisition function and manually discretizing the bounded space into a grid of $n$ points. DPP selects the remaining $(k-1)$ points out of the $n$ points in the relevance region. \cite{wang2017batched, oh2021batch,nava2022diversified} used similar techniques to apply DPPs to high-dimensional and discrete spaces. 

There exist two approaches to selecting a diverse subset with a fixed size via DPP: 1) Choosing the subset that maximizes the determinant referred to as {\em DPP-max}; and 2) Sampling with the determinantal probability measure referred to as {\em DPP-sample}. In this paper, we will focus on DPP-max. Although selecting the subset that maximizes the determinant is an NP-Hard problem, several approximations were proposed \cite{nikolov2015randomized}. A greedy strategy \cite{kathuria2016batched} provides an approximate solution and was adopted in several BO papers \cite{wang2017batched}

\vspace{1.0ex}

\noindent {\bf Limitations of Prior Work and Challenges for MOO.} We list the key limitations of prior methods for DPP-based batch selection in the single-objective setting as they are applicable to the multi-objective setting too. {\bf L1)} How can we overcome the limitation of selecting the first point separately regardless of the DPP diversity? {\bf L2)} How can we prevent the potential limitation of under-explored search space caused by the discretization of the space to create the relevance region set?

The key challenges to employing DPPs for batch selection in the MOO setting include {\bf C1)} How to define a kernel that captures the diversity for multiple objectives given that we have $K$ separate surrogate models and their corresponding kernel? {\bf C2)} How can the DPP kernel capture the Pareto front diversity and the trade-off between the objectives without compromising the Pareto quality of selected points?

\vspace{1.0ex}

\noindent \textbf{DPPs for Multi-objective BO.} We propose principled methods to address the limitations of prior work on DPPs for batch BO ({\bf L1} and {\bf L2}) and the challenges {\bf C1} and {\bf C2} for MOO.

\vspace{0.5ex}

\textit{Multi-objective Relevance Region.} Our proposed algorithm naturally mitigates the two limitations of the single-objective DPP approach. Recall that the first step of PDBO algorithm (Section \ref{cheapmo}) proposes to generate cheap approximate Pareto-sets which capture the trade-offs between the objectives in the utility space and might include, with high probability, optimal points \cite{Usemo, konakovic2020diversity}. We consider the cheap Pareto sets as the multi-objective relevance region. Our approach allows for generating the relevance region without manually discretizing the search space. Also, the full batch is selected from the multi-objective relevance region leading to a better diversity among all the points in the batch. 

\begin{algorithm}[t]
\caption{\textsc{DPP-Select}}
\footnotesize
\textbf{Input}: cheap Pareto set $\mathcal{X}_{c}$; surrogate models $\textsc{Gp}_1,\cdots ,\textsc{Gp}_k$; data $\mathcal{D}_{t}$.
\begin{algorithmic}[1] \label{alg:DPP}
\STATE $\vec{C}_t=[ HVC(\vec{y}), \; ~\forall ~ \vec{y} \in \mathcal{Y}_{t} ] $ // Calculate the individual hypervolume contribution for each input $\in \mathcal{D}_{t}$ 
\STATE $\kappa_{DPP} = \sum_{i = 1}^{K} \lambda_i \cdot \kappa_i ~ st. ~ \sum_{i=1}^{K} \lambda_i =1$ \\ // Construct $\kappa_{DPP}$ as a convex combination of function kernels 
\STATE $\Lambda^* =  \arg min_{\Lambda \in [0,1]^K} \log p(\vec{C}_t|\mathcal{X}_t) \quad s.t \sum_{i=1}^K \lambda_i = 1 $ \\// Select $\lambda_i$ values by maximizing the LML
\STATE Use the fitted $\kappa_{DPP}$ kernel to select the most diverse points $X^B_{t}$ from the cheap Pareto set $\mathcal{X}_{c}$ via DPP-max \\
\STATE \textbf{return} the selected $B$ inputs, $X^B_{t}$
\end{algorithmic}
\end{algorithm}

\vspace{0.5ex}

\textit{Multi-objective DPP Kernel Fitting.} To overcome the challenges of using DPPs in the MOO setting, we build a new kernel $\kappa_{dpp}$ that is defined as a convex combination of the $K$ kernels of the statistical models (GPs) representing each of the black-box objective functions. Let $\Lambda = [\lambda_1, \cdots, \lambda_K]$ be a vector of size $K$ where each $\lambda_i$ corresponds to the convex combination scalar associated with kernel $\kappa_i$ of the objective function $f_i$. The DPP kernel $\kappa_{DPP}$ is defined as: 
\begin{equation} \label{eq_kdpp}
    \kappa_{DPP} = \sum_{i = 1}^{K} \lambda_i \cdot \kappa_i ~~ st. ~ \sum_{i=1}^{i=K} \lambda_i =1
\end{equation}
The hyperparameters of the kernels $\kappa_i ~ \forall~ i \in \{1, 2, \cdots K\}$ are fixed during the fitting of the GPs. To set the convex combination scalars $\Lambda$ in a principled manner that promotes diverse batch selection, we propose to set the $\Lambda$ by maximizing the log marginal likelihood of selecting points with the highest individual hypervolume contribution. 

The individual hypervolume contribution (HVC) of each point in the evaluated Pareto front $\mathcal{Y}_t$ (via evaluated training data $D_{t}$) is the reduction in hypervolume if the point is removed from the Pareto front. HVC is considered a Pareto front (PF) diversity indicator \cite{daulton2022multiobjective}. Points in crowded regions of the Pareto front have smaller HVC values. Therefore, more Pareto front points with high HVC indicate more output space coverage and consequently, higher PF diversity. \looseness=-1
 \begin{equation} \label{contributionhv}
        HVC(\vec{y}) = HV(\mathcal{Y}_{t}) - HV(\mathcal{Y}_{t} \setminus \{\vec{y}\}) ~\forall ~ \vec{y} \in \mathcal{Y}_{t}
\end{equation}
Given equation \ref{contributionhv}, we can construct the training set for the fitting of $\Lambda$. Let $\vec{C}_t$ = $[ HVC(\vec{y}); ~\forall ~ \vec{y} \in \mathcal{Y}_{t} ] $ be a vector of the individual HV contributions of evaluated points $\mathcal{X}_t \in \mathcal{D}_t$. 
\begin{align}
    &\Lambda^* =  \text{argmin}_{\Lambda \in [0,1]^K} \log p(\vec{C}_t|\mathcal{X}_t) \label{LMLopt} \quad \text{s.t.} \sum_{i=1}^K \lambda_i = 1
\end{align}
\begin{align}
   \text{where} ~ \log p(\vec{C}_t|\mathcal{X}_t) = & -\frac{1}{2}\vec{C}_t^T\mathcal{K_{DPP}}^{-1}\vec{C}_t \nonumber \\
    & - \frac{1}{2}\log|\mathcal{K_{DPP}}| -\frac{n}{2}\log2\pi \nonumber
\end{align}
Algorithm \ref{alg:DPP} provides the pseudo-code for selecting the diverse batch from a given candidate/cheap Pareto set $\mathcal{X}_c$.

\section{Theoretical Analysis}

Prior work developed regret bound for different single objective acquisition functions including UCB \cite{srinivas2009Gaussian,Usemo}. However, the theoretical analysis of our proposed MAB algorithm is more challenging as it involves input selection based on different acquisition functions at each iteration $t$. 
The choices made at any given iteration $t$ influence the state and subsequent rewards of all future iterations \cite{hoffman2011portfolio} and therefore there is a need to adapt prior theoretical analysis. Additionally, regret bounds for Hedge MAB strategies have been developed independently outside the context of acquisition function selection \cite{cesa2006prediction}. We follow similar steps suggested by \citet{hoffman2011portfolio} to derive a suitable regret bound for our MOO setting.

We assume maximization of objectives and further assume that the UCB acquisition function is in the portfolio of acquisition functions used by PDBO. We make this choice for the sake of clarity and ease of readability as we build our theoretical analysis on prior seminal work \cite{srinivas2009Gaussian,hoffman2011portfolio}. Notably, this is not a restrictive assumption, and with minimal mathematical manipulations, the same derived regret bound holds for the case of minimization with the LCB acquisition function being in the portfolio instead.  

To simplify the proof and solely for the sake of theoretical regret bound, we consider the instant reward at iteration $t$ to be the sum of predictive means of the Gaussian processes
\begin{equation}\label{eq:IRtheo1}
 IR_t= \sum_{i=1}^k \mu_{i,t-1}(\vec{x}_t)   
\end{equation}
where $\mu_{i,t-1}$ is the posterior mean of function $i$. The cumulative reward over ${T_{max}}$ iterations that would have been obtained using acquisition function $AF_j$ is defined as: 
\begin{equation}
	r_{T_{max}}^j = \sum_{t=1}^{T_{max}} IR_t = \sum_{t=1}^{T_{max}} \sum_{i=1}^k \mu_{i,t-1}(\vec{x}_t^j). \label{new_IR}
\end{equation}
It is important to note that in our proposed PDBO algorithm, we use a different and better instant reward $IR_t$ and cumulative reward $r_{T_{max}}^i$. The rewards in equation \ref{eq:IRtheo1} is a design choice to achieve the following regret bound. In Section \ref{ablation_theo}, we provide a discussion accompanied by an experimental ablation study comparing the reward function used in theory to the reward function used in PDBO.
\begin{theorem}
\label{theorem}
 Let $\vec{x}^*$ be a point in the optimal Pareto set $\mathcal{X^*}$. Let $\vec{x}$ be a point in the Pareto set $\mathcal{X}_t$ estimated by PDBO via solving cheap MOO problem at the $t^{th}$ iteration. Let the cumulative regret for the multiple objectives be defined as $R_{T_{max}}(\vec{x}^*) = \sum_{t=1}^{T_{max}} \sum_{i=1}^k  f_i(\vec{x}^*) - f_i(\vec{x}_t)$
 
Assuming maximization of objectives and that UCB is in the portfolio of acquisition functions, let $\beta_t$ be the UCB parameter and $\gamma^i_{T_{max}}$ be the bound on the information gained for function $i$ at points selected by PDBO after ${T_{max}}$ iterations, then with probability at least $1-\delta$ the cumulative regret is bounded by
\begin{align*}
	R_{T_{max}}(\vec{x}^*) 	
	&\leq
	\mathcal O(\sqrt{{T_{max}}}) + 
	\sum_{t=1}^{T_{max}} \sum_{i=1}^k \sqrt{\beta_t}\sigma_{i,t-1}(\vec{x}_t^\mathrm{UCB})\\
 &+ 
	\sqrt{C_i{T_{max}}\beta_{T_{max}}\gamma^i_{T_{max}}}.
\end{align*}
\end{theorem}
We provide complete proof in the Appendix. The theorem suggests that our regret is bounded by two sublinear terms and another term that might include points suggested by UCB but not necessarily selected by the Hedge strategy. Additionally, the theoretical proof accounts only for sequential input selection. Extension to batch selection using DPP is possible by carefully accounting for results introduced by \citet{kathuria2016batched}.

\begin{figure*}[t]
\begin{subfigure}
     \centering
     \begin{subfigure}
     \centering
     \begin{subfigure}
         \centering
         \includegraphics[width=0.24\textwidth]{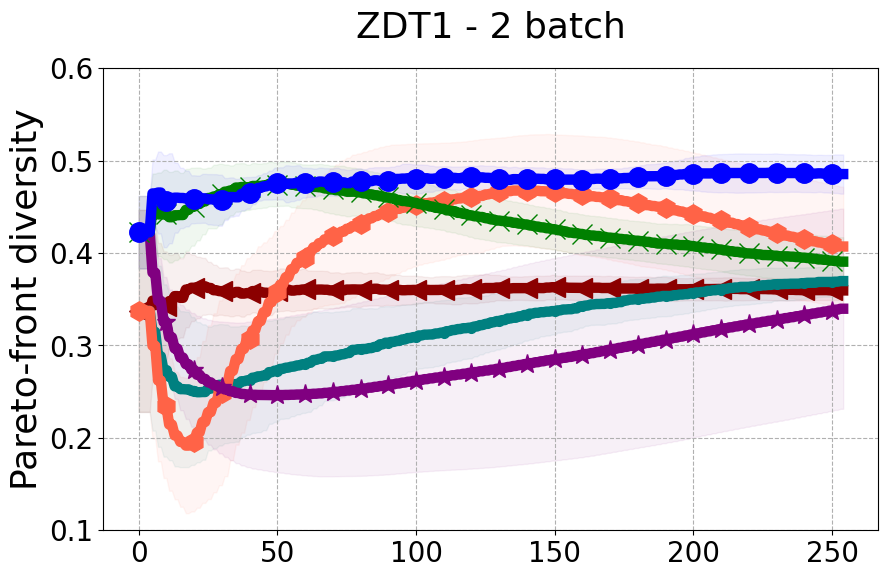}
     \end{subfigure}
     \hfill
     \begin{subfigure}
         \centering
         \includegraphics[width=0.24\textwidth]{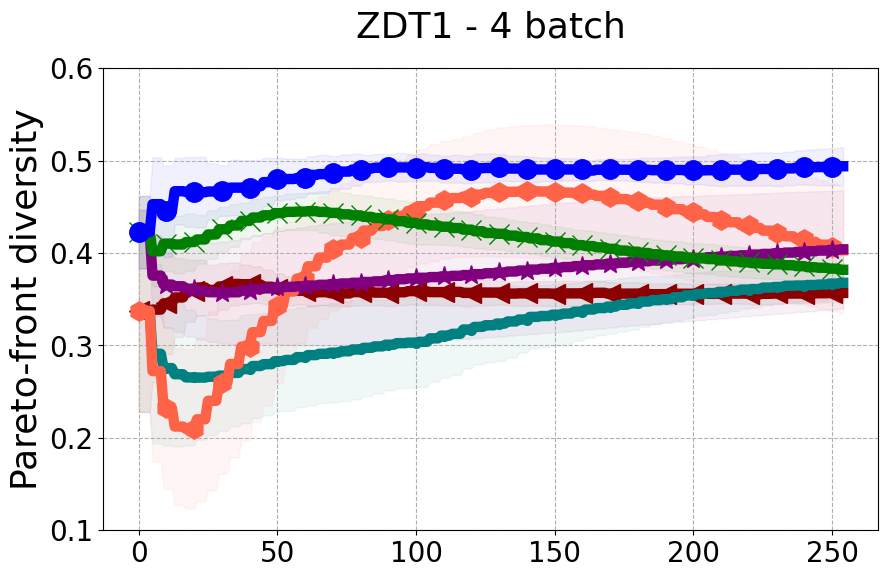}
     \end{subfigure}
     \hfill
     \begin{subfigure}
         \centering
         \includegraphics[width=0.24\textwidth]{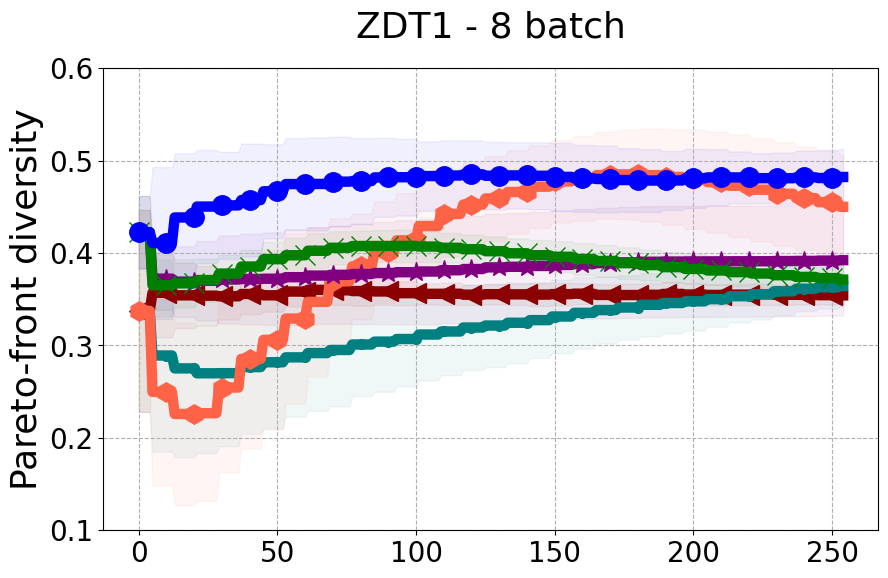}
     \end{subfigure}
     \hfill
     \begin{subfigure}
         \centering
         \includegraphics[width=0.24\textwidth]{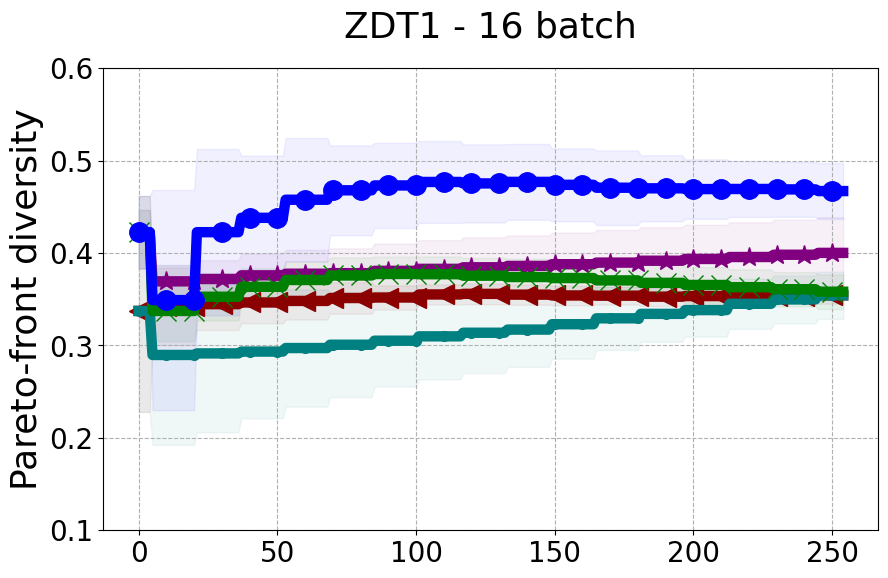}
     \end{subfigure}
        \label{fig:paper-div-zdt1}
    \end{subfigure}
        \centering
\centering
     \begin{subfigure}
         \centering
         \includegraphics[width=0.24\textwidth]{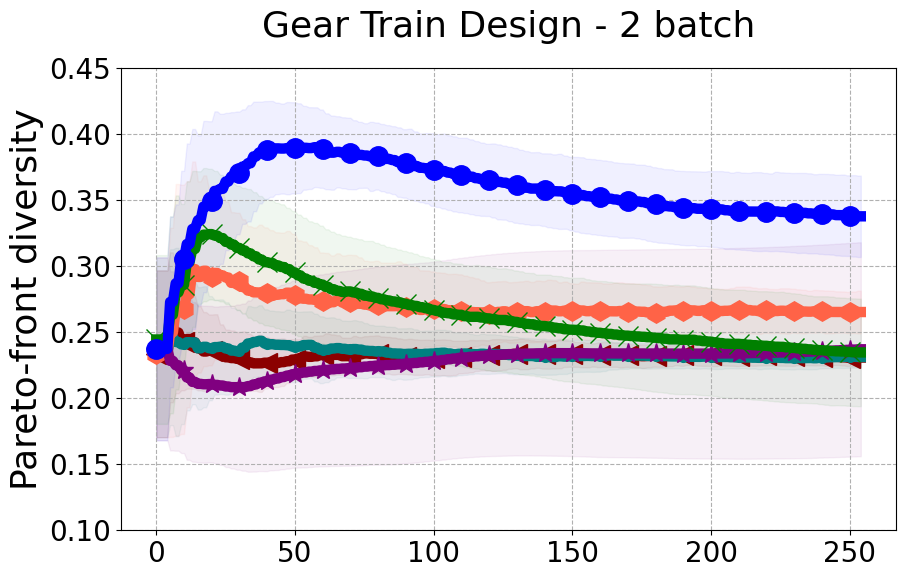}
     \end{subfigure}
     \hfill
     \begin{subfigure}
         \centering
         \includegraphics[width=0.24\textwidth]{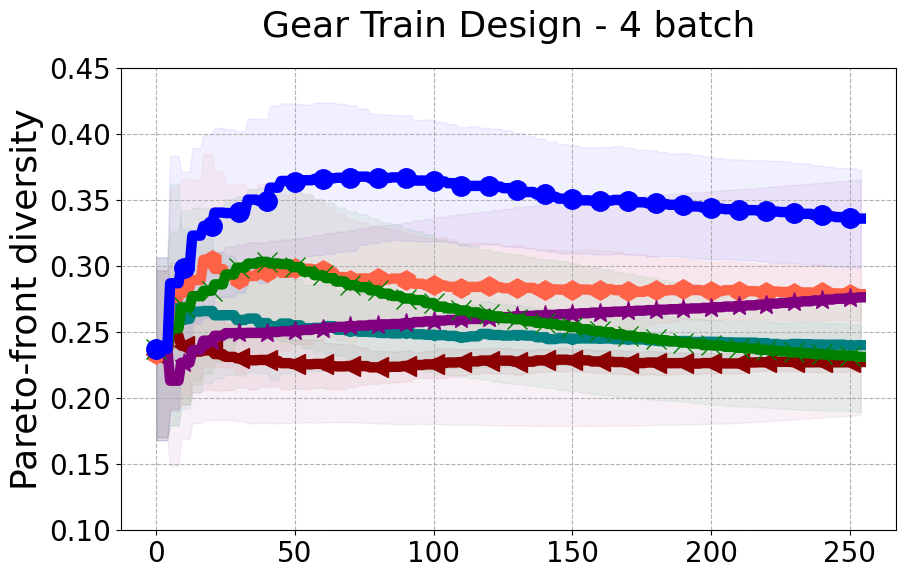}
     \end{subfigure}
     \hfill
     \begin{subfigure}
         \centering
         \includegraphics[width=0.24\textwidth]{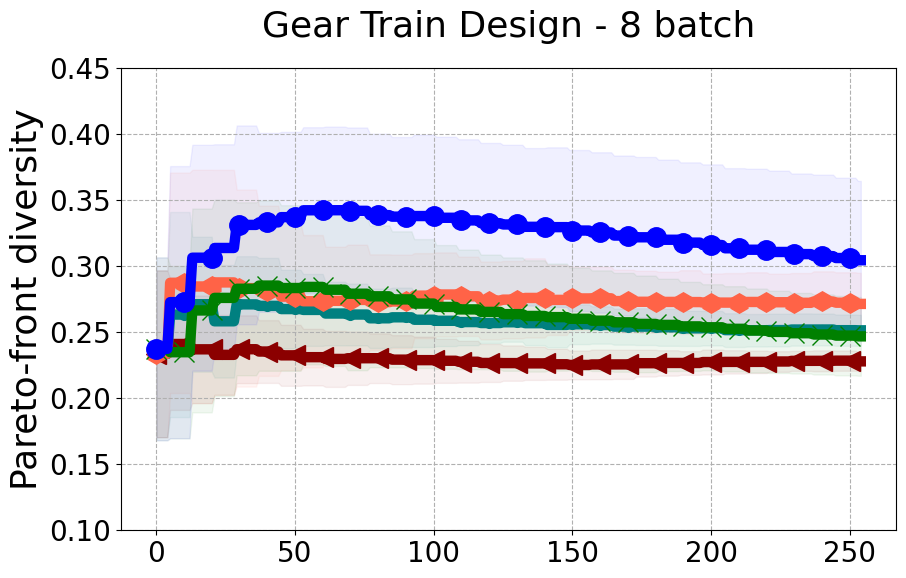}
     \end{subfigure}
     \hfill
     \begin{subfigure}
         \centering
         \includegraphics[width=0.24\textwidth]{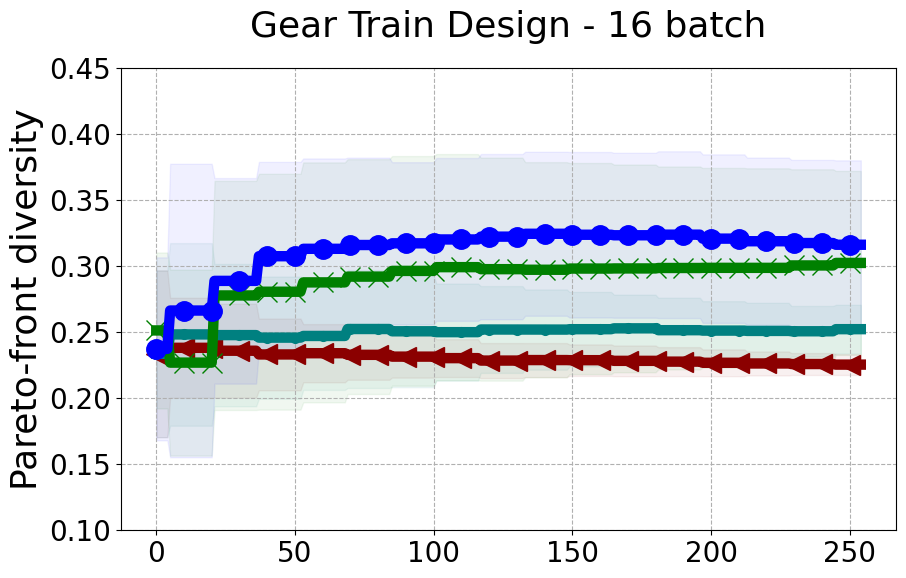}
     \end{subfigure}
        \label{fig:paper-div-re6}
    \end{subfigure}
    \centering
    \begin{subfigure}
    \centering
     \begin{subfigure}
         \centering
         \includegraphics[width=0.24\textwidth]{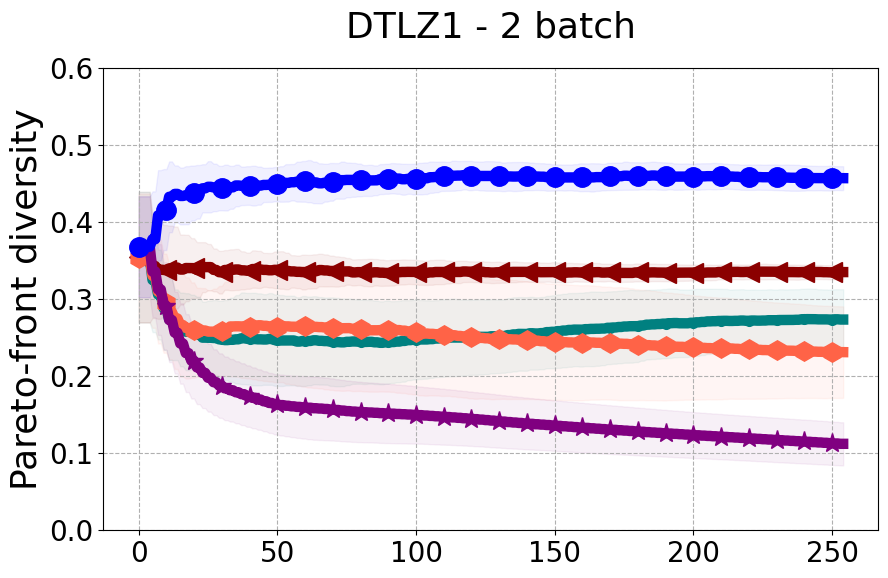}
     \end{subfigure}
     \hfill
     \begin{subfigure}
         \centering
         \includegraphics[width=0.24\textwidth]{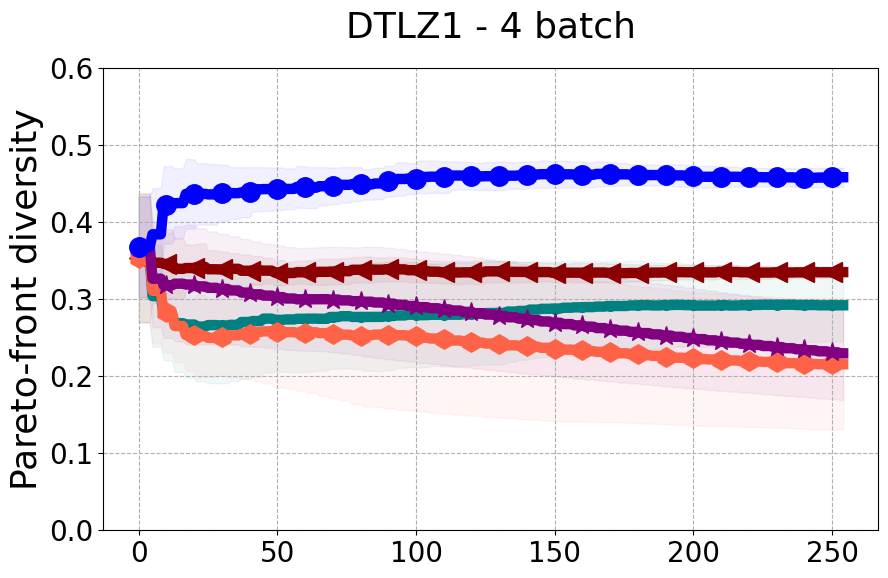}
     \end{subfigure}
     \hfill
     \begin{subfigure}
         \centering
         \includegraphics[width=0.24\textwidth]{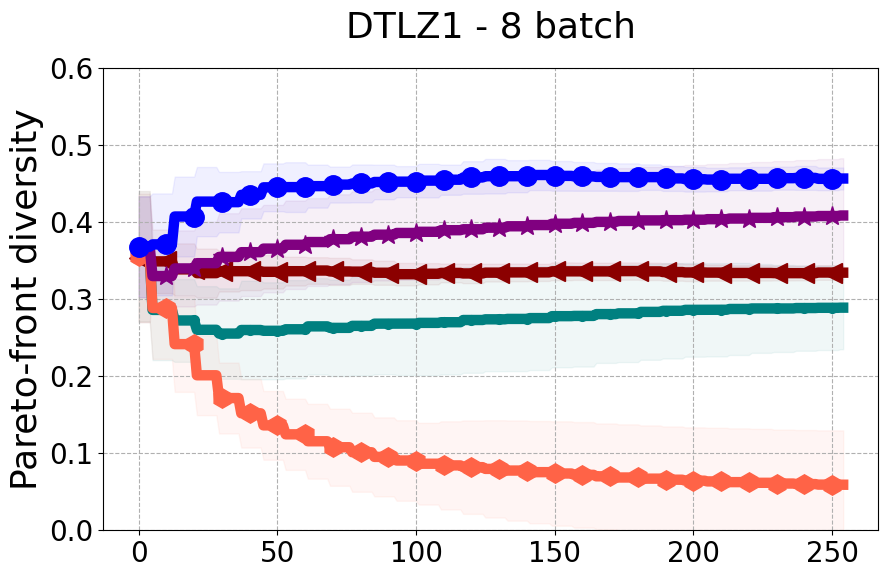}
     \end{subfigure}
     \hfill
     \begin{subfigure}
         \centering
         \includegraphics[width=0.24\textwidth]{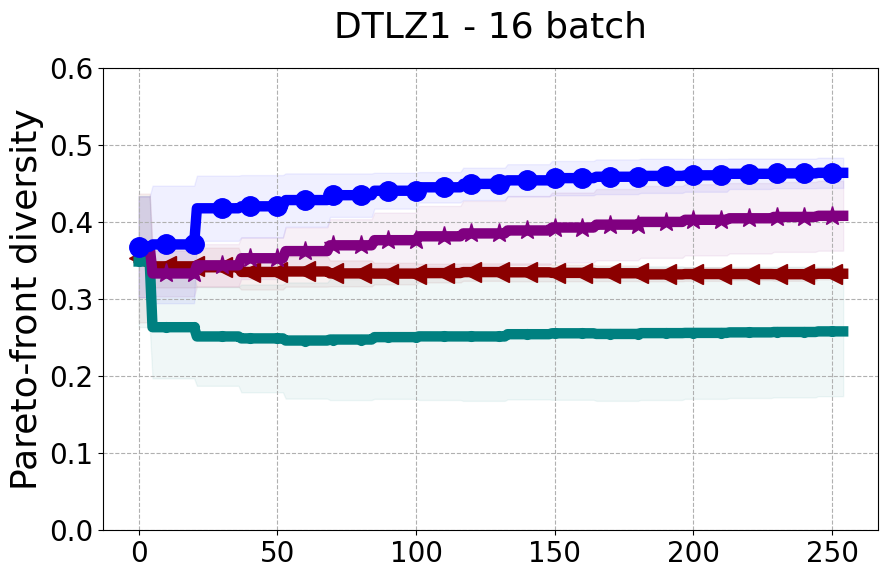}
     \end{subfigure}
        \label{fig:paper-div-dtlz1}
    \end{subfigure}
    \centering
    \begin{subfigure}
         \centering
     \begin{subfigure}
         \centering
         \includegraphics[width=0.24\textwidth]{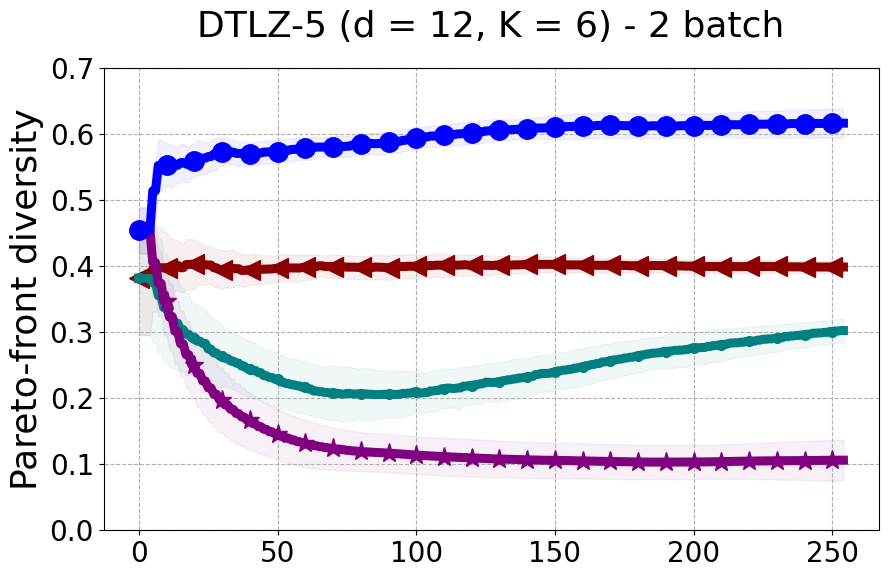}
     \end{subfigure}
     \hfill
     \begin{subfigure}
         \centering
         \includegraphics[width=0.24\textwidth]{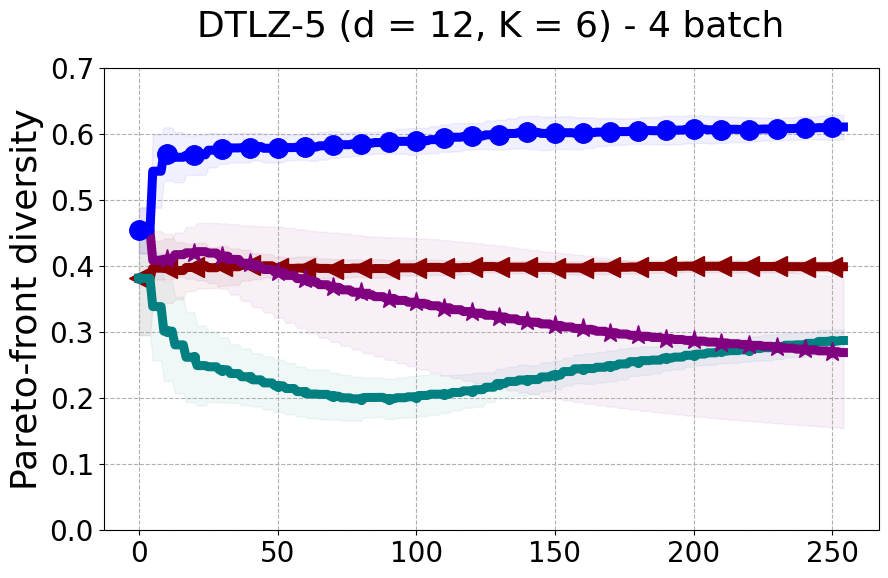}
     \end{subfigure}
     \hfill
     \begin{subfigure}
         \centering
         \includegraphics[width=0.24\textwidth]{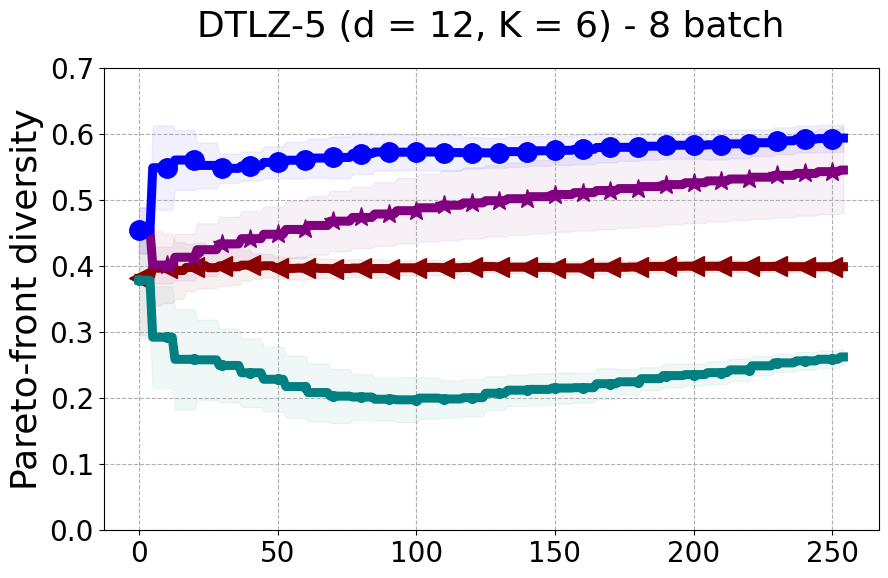}
     \end{subfigure}
     \hfill
     \begin{subfigure}
         \centering
         \includegraphics[width=0.24\textwidth]{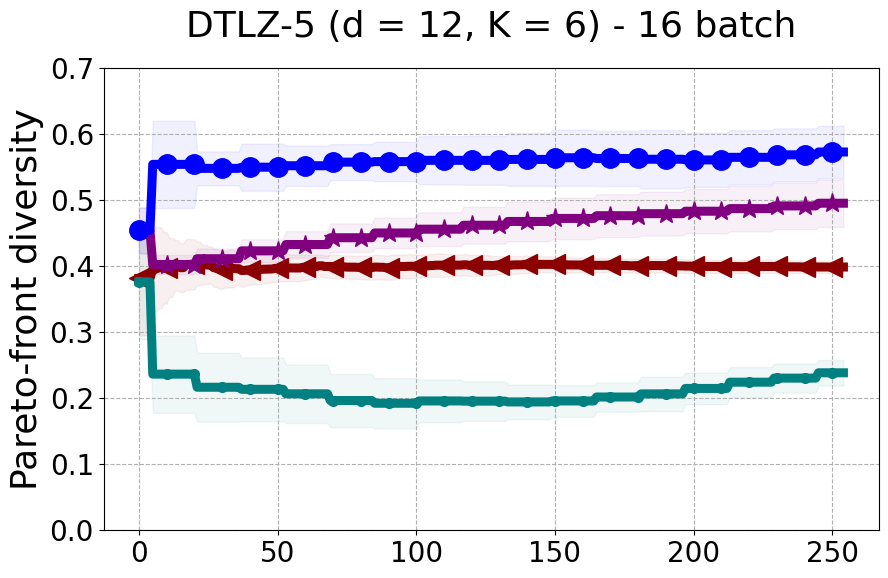}
     \end{subfigure}
    \label{fig:paper-div-dtlz5}
    \end{subfigure}
    \centering
      \begin{subfigure}
         \centering
         \includegraphics[width=0.98\textwidth]{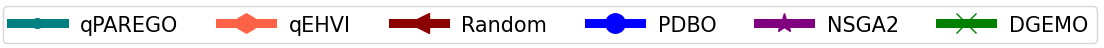}
     \end{subfigure}
     \caption{Diverse Pareto front (DPF) results evaluated on multiple benchmarks and batch sizes.}
     \label{fig:paper-dpf}
     \vspace{-2mm}
\end{figure*}

\subsection{Analysis and Ablation Study}\label{ablation_theo}

To simplify the proof for regret bound, we defined a new instant reward and cumulative reward in equations \ref{eq:IRtheo1} and \ref{new_IR} that are different from the rewards we use for PDBO in equations  \ref{immediate_reward}, \ref{cumulative_discounted_reward}, and \ref{normalized_reward}. The goal of this section was to define a reward that allows tractable theoretical analysis. However, this reward is not intuitive and has several practical issues: {\bf 1)} It is defined as a summation over predictive means of the functions and does not provide any insight on the quality of the selected points; {\bf 2)} It does not account for improvement with respect to previous iterations which is uninformative in terms of the quality of points selected by different acquisition functions at different iterations; {\bf 3)} It is non-discounted and does not account for the importance of the iterative progress of the input selection; and  {\bf 4)} It is not normalized and therefore can lead to random selection in the advanced iterations. All stated issues have been addressed by our proposed reward function which we carefully designed to be intuitive and informative about the different acquisition strategies and to mitigate potential numerical issues. We performed an ablation study to compare the performance of PDBO when using our proposed reward function and the reward function in the theoretical proof. Our results in the Appendix show superior performance for our proposed reward strategy.

\section{Experiments and Results}
\label{experimental-section}
We provide experimental details and compare PDBO to baseline methods on multiple MOO benchmarks and varying batch sizes. We evaluate all methods using the hypervolume indicator and diversity of Pareto front (DPF) measure. 

\begin{figure*}[t]
        \centering
     \begin{subfigure}
         \centering
         \includegraphics[width=0.23\textwidth]{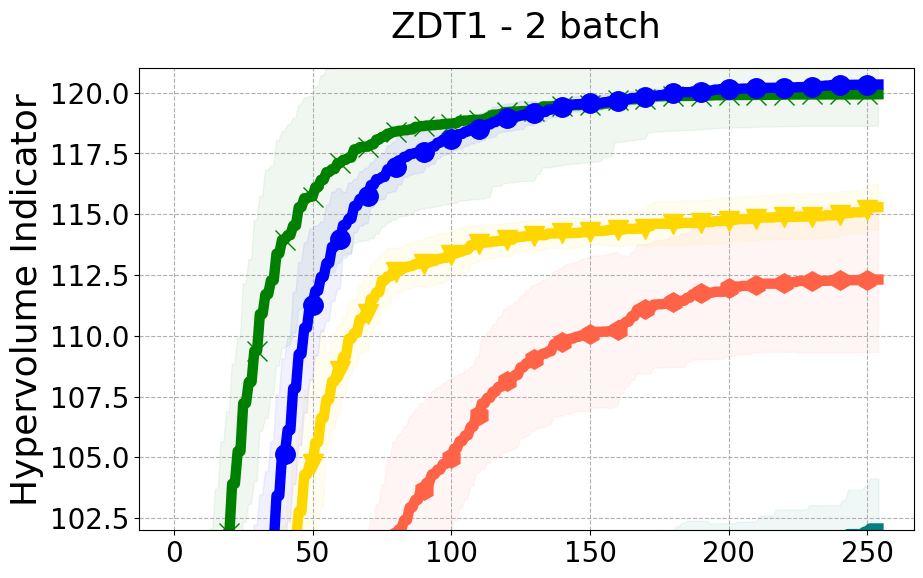}
     \end{subfigure}
     \hfill
     \begin{subfigure}
         \centering
         \includegraphics[width=0.23\textwidth]{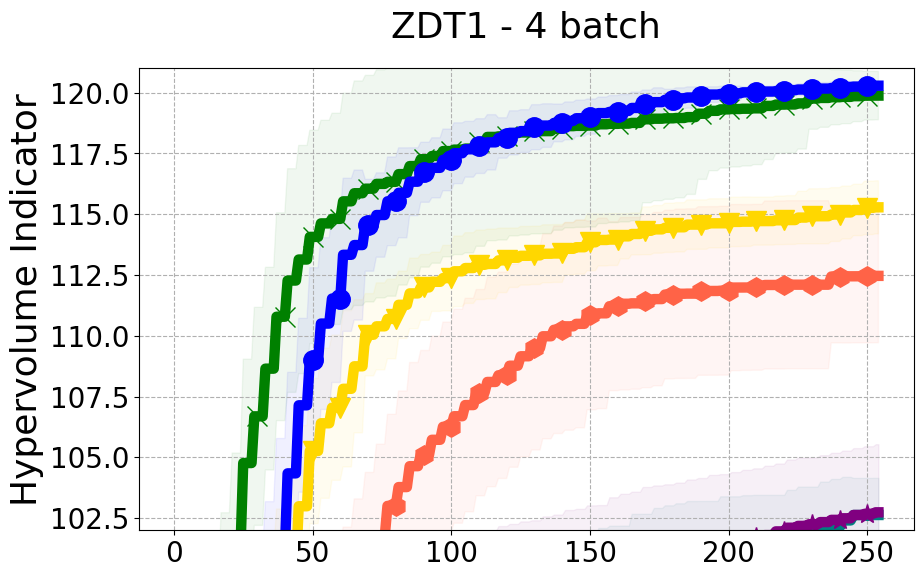}
     \end{subfigure}
     \hfill
     \begin{subfigure}
         \centering
         \includegraphics[width=0.23\textwidth]{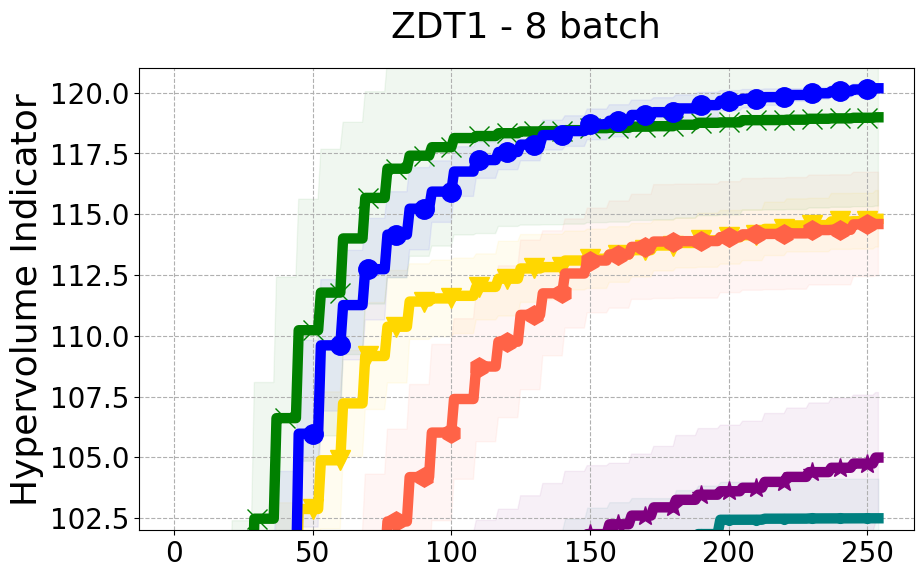}
     \end{subfigure}
     \hfill
     \begin{subfigure}
         \centering
         \includegraphics[width=0.23\textwidth]{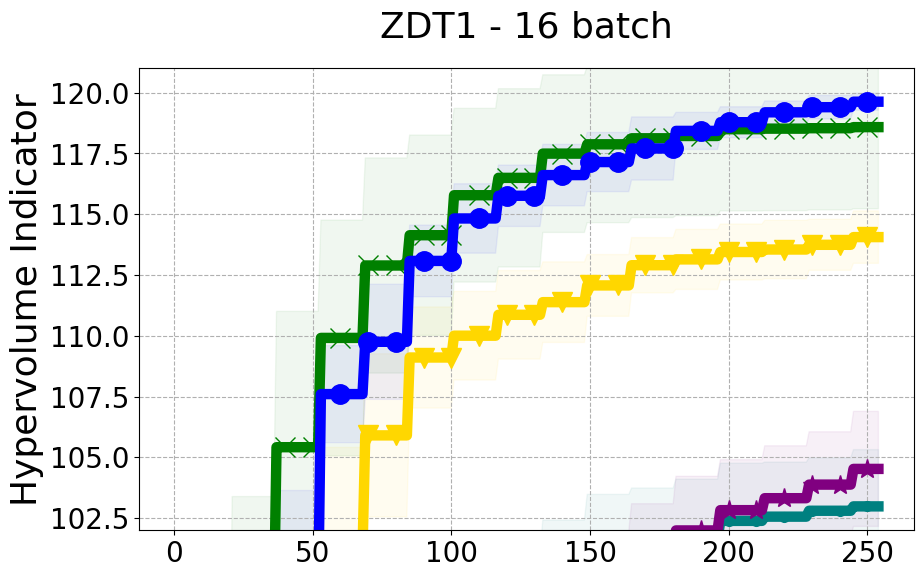}
     \end{subfigure}
        \label{fig:paper-zdt1}
        \centering
     \begin{subfigure}
         \centering
         \includegraphics[width=0.23\textwidth]{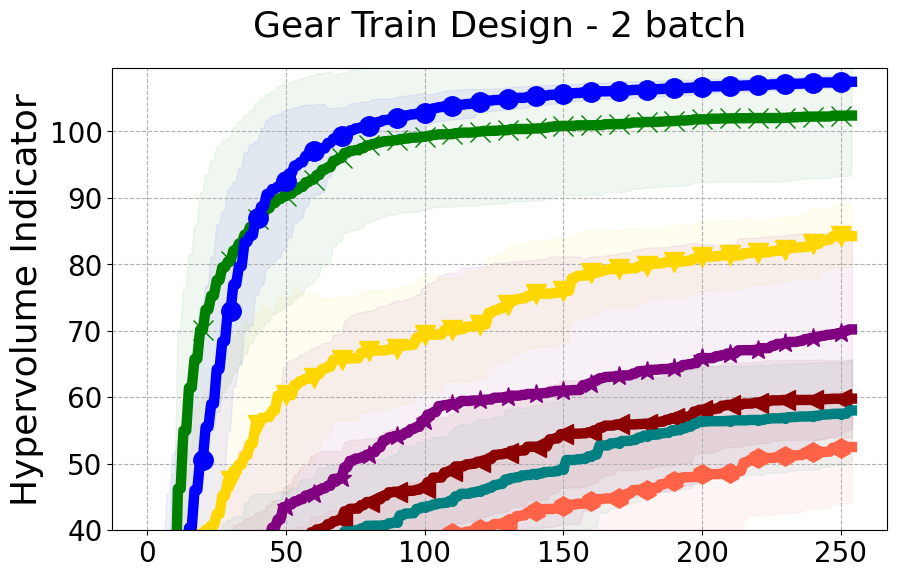}
     \end{subfigure}
     \hfill
     \begin{subfigure}
         \centering
         \includegraphics[width=0.23\textwidth]{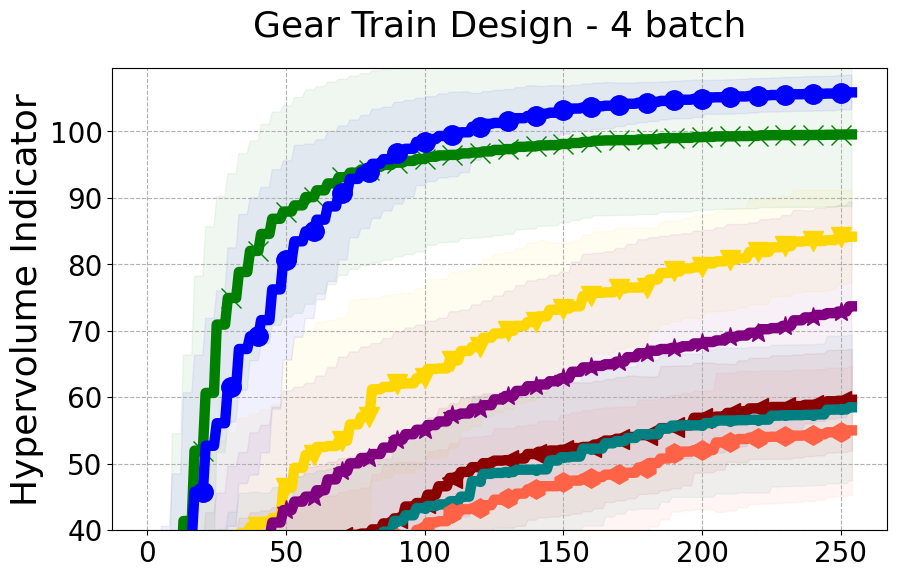}
     \end{subfigure}
     \hfill
     \begin{subfigure}
         \centering
         \includegraphics[width=0.23\textwidth]{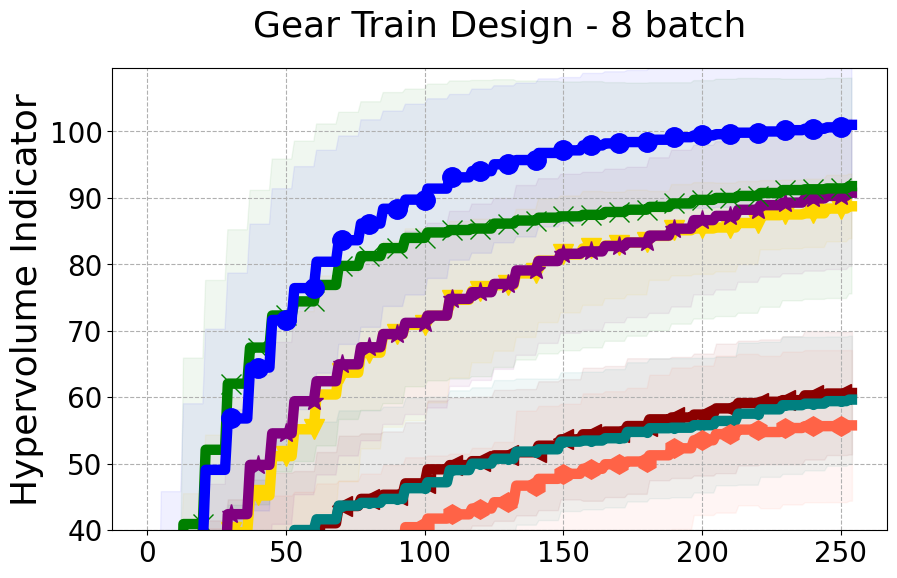}
     \end{subfigure}
     \hfill
     \begin{subfigure}
         \centering
         \includegraphics[width=0.23\textwidth]{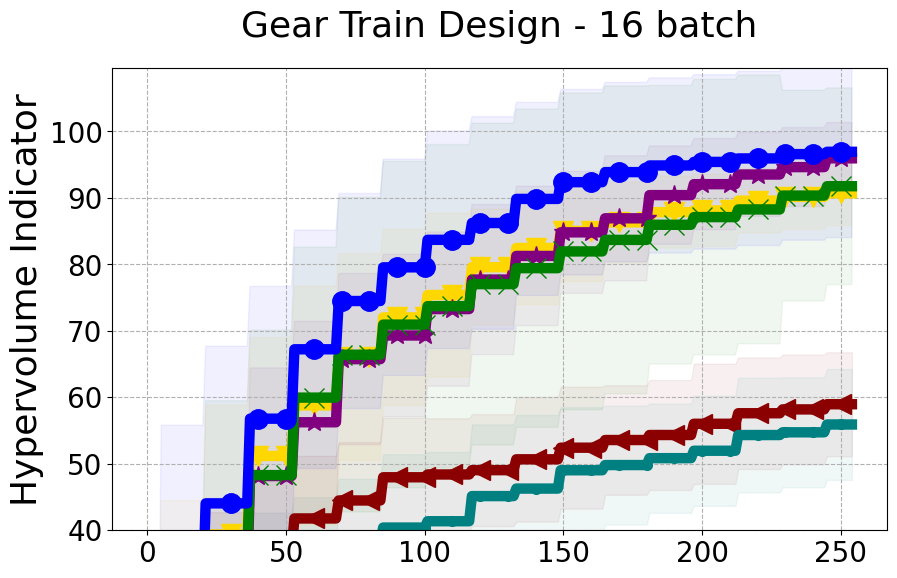}
     \end{subfigure}
        \label{fig:paper-re6}
        \centering
     \begin{subfigure}
         \centering
         \includegraphics[width=0.23\textwidth]{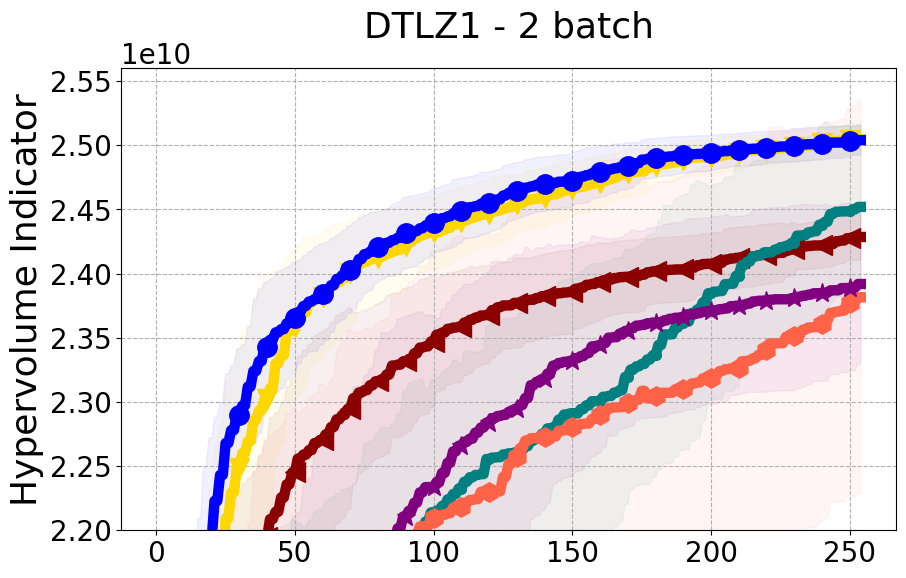}
     \end{subfigure}
     \hfill
     \begin{subfigure}
         \centering
         \includegraphics[width=0.23\textwidth]{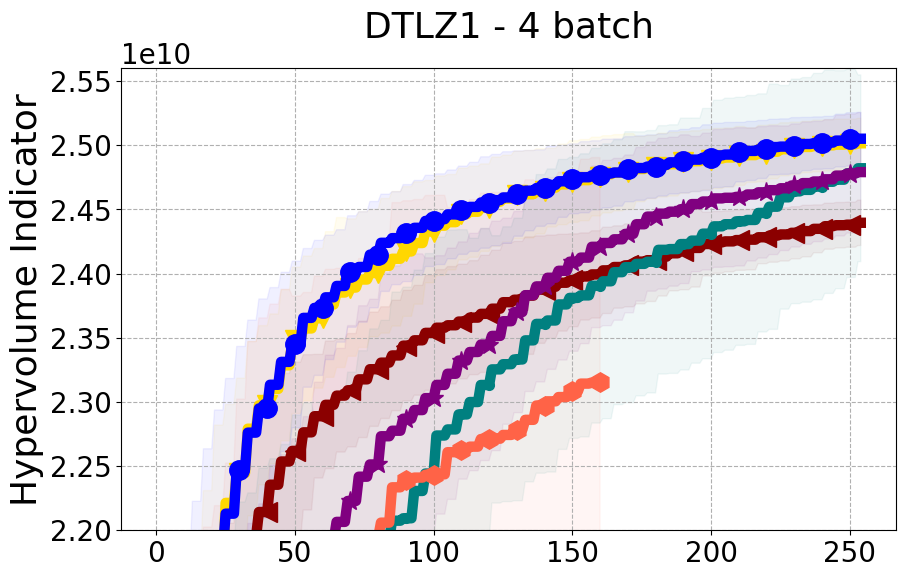}
     \end{subfigure}
     \hfill
     \begin{subfigure}
         \centering
         \includegraphics[width=0.23\textwidth]{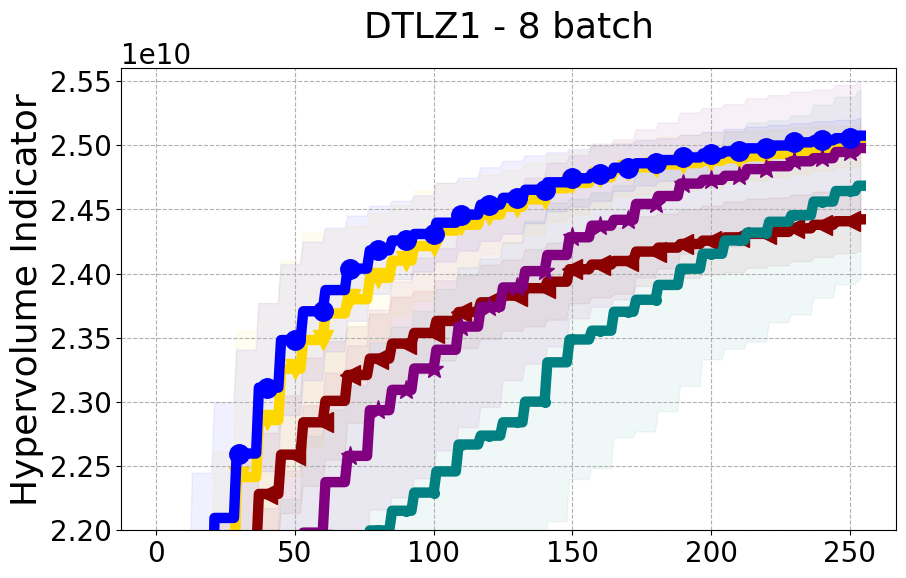}
     \end{subfigure}
     \hfill
     \begin{subfigure}
         \centering
         \includegraphics[width=0.23\textwidth]{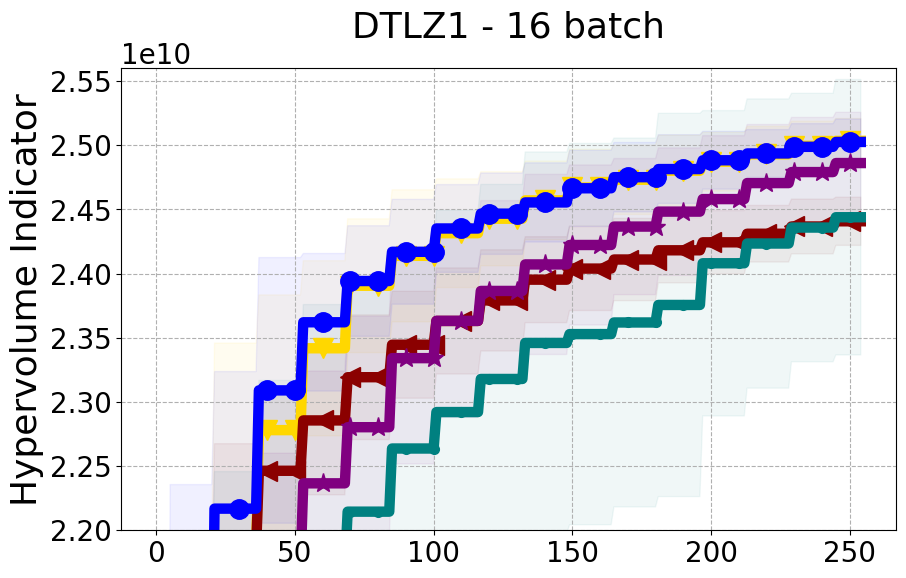}
     \end{subfigure}
        \label{fig:paper-dtlz1}
        \centering
     \begin{subfigure}
         \centering
         \includegraphics[width=0.23\textwidth]{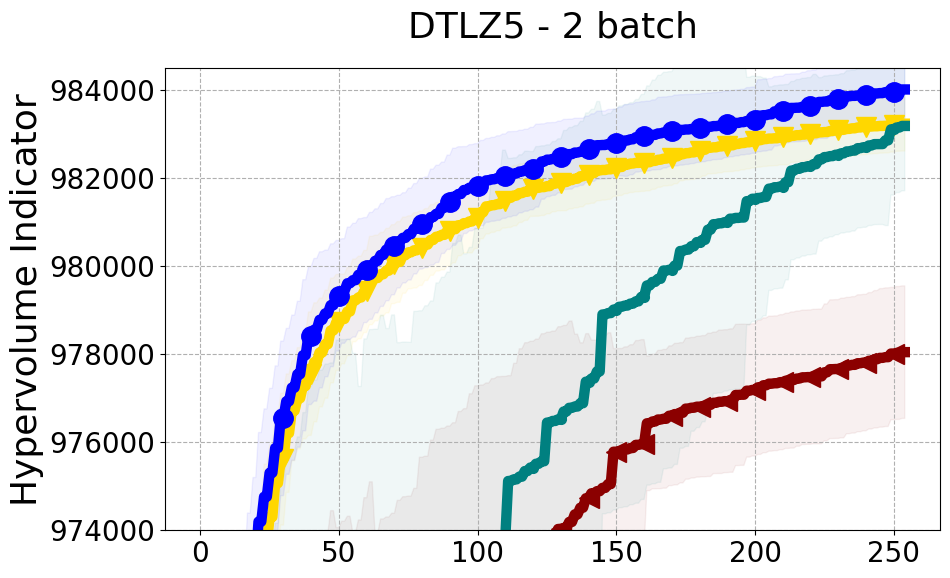}
     \end{subfigure}
     \hfill
     \begin{subfigure}
         \centering
    \includegraphics[width=0.23\textwidth]{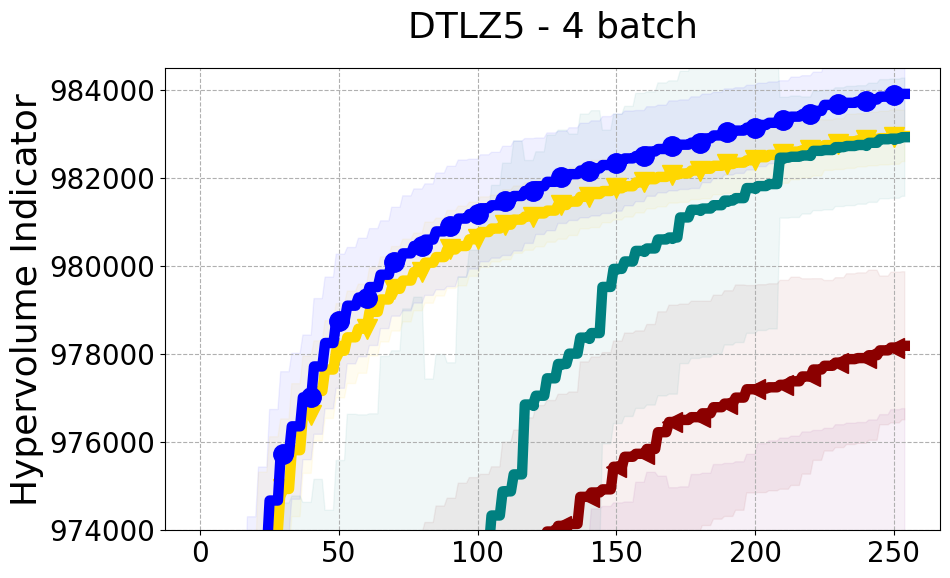}
     \end{subfigure}
     \hfill
     \begin{subfigure}
         \centering
    \includegraphics[width=0.23\textwidth]{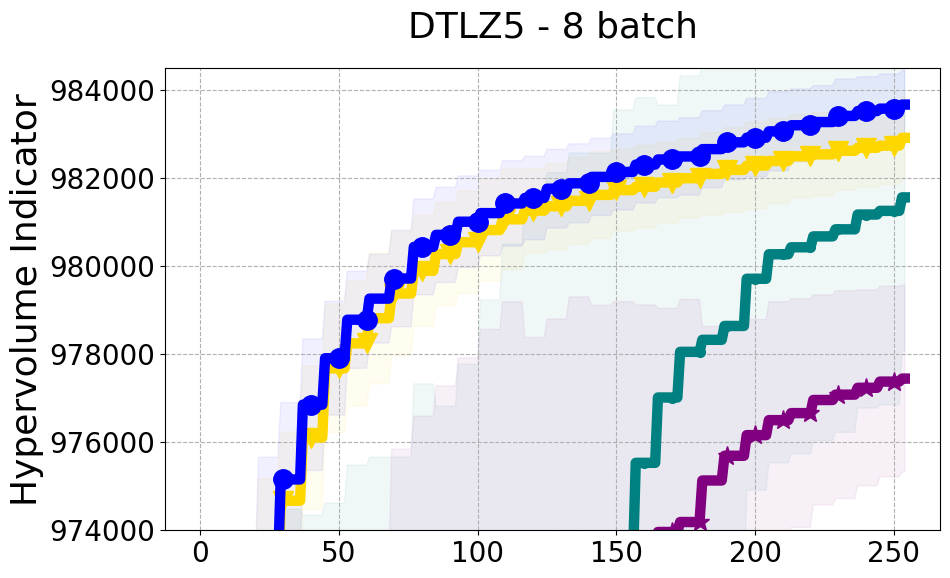}
     \end{subfigure}
     \hfill
     \begin{subfigure}
         \centering
    \includegraphics[width=0.23\textwidth]{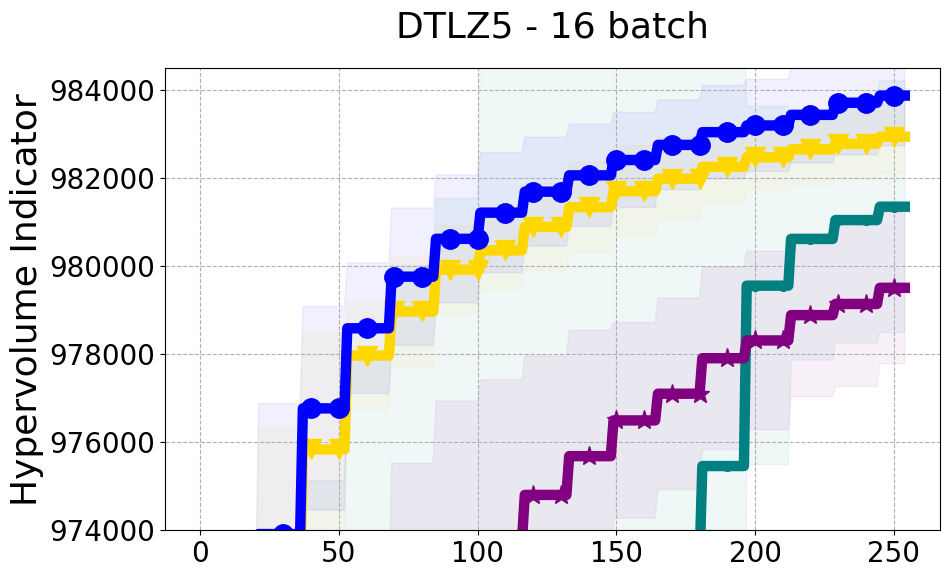}
     \end{subfigure}
     \label{fig:appendix-paper-dtlz5}
      \begin{subfigure}
         \centering
         \includegraphics[width=\textwidth]{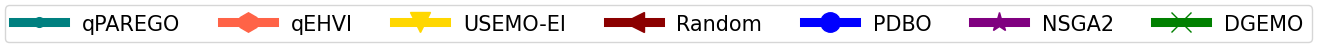}
         \label{fig:paper-legend}
     \end{subfigure}
    \caption{Hypervolume results evaluated on multiple benchmarks and batch sizes.}
    \label{fig:mainpaper-all-results}
    \vspace{-2.5mm}
\end{figure*}

\noindent \textbf{Benchmarks.} We conduct experiments on benchmarks with varying numbers of input and output dimensions to show the versatility and flexibility of our method. We use several synthetic problems: ZDT-1, ZDT-2, ZDT-3 \cite{zitzler2000comparison}, DTLZ-1, DTLZ-3, DTLZ-5 \cite{deb2005scalable} and three real wold problems: the gear train design problem \cite{deb2006innovization,konakovic2020diversity}, SWLLVM \cite{siegmund2012predicting} and Unmanned aerial vehicle power system design \cite{belakaria2020machine}. More details and descriptions of problem settings are included in the Appendix. In the ablation studies, we provide additional experiments with synthetic benchmarks where we vary the input and output dimensions. 

\vspace{0.5ex}

\noindent \textbf{Baselines.} We compare our PDBO method to state-of-the-art batch MOO methods: DGEMO, qEHVI, qPAREGO, and  USEMO-EI. We also include NSGA-II as the evolutionary algorithm baseline and random input selection. We set the hyperparameters of PDBO to $\gamma$ = 0.7 and $\tau$ = 4 as recommended by \cite{hoffman2011portfolio,nopastvasconcelos2019no}. We define the AF portfolio as $\mathcal{P} = \{ EI, TS, UCB, ID\}$.

\vspace{0.5ex}

\noindent \textbf{Experimental Setup.} 
All experiments are initialized with five random inputs/evaluations and run for at least 250 function evaluations. We conduct experiments with four different batch sizes $B \in \{2, 4, 8, 16\}$ and adjust the number of iterations accordingly. For instance, when using a batch size of two, we run the algorithm for 125 iterations. Each experiment is repeated 25 times, and we report the average and standard deviation of the hypervolume indicator and the DPF metric. To solve the constrained optimization problem in the DPP algorithm, we utilize an implementation of the SQP method \cite{lalee1998implementation,nocedal2006numerical} from the Python SciPy library \cite{2020SciPy-NMeth}. For baselines, we use the codes and hyperparameters provided in the open source repositories of DGEMO \footnote{https://github.com/yunshengtian/DGEMO} and Botorch \footnote{https://github.com/pytorch/botorch}. We provide the details for the NSGA-II baseline and cheap MO solver, and more details about the setup for fitting the hyperparameters of GP models in the Appendix. 

\vspace{0.5ex}

\noindent \textbf{Results and Discussion.} Figure \ref{fig:paper-dpf} demonstrates that PDBO outperforms other baselines with respect to the Pareto-front diversity metric. Additionally, Figure \ref{fig:mainpaper-all-results} demonstrates that PDBO outperforms all baseline methods in most experiments with respect to the Hypervolume indicator and provides a competitive performance on the others. 

In the Appendix, we present a comprehensive set of additional results and analyses. This includes the evaluation of hypervolume and DPF on various benchmarks. We also introduce results using other metrics, notably the Inverted Generational Distance (IGD) and a modified version of DPF, accompanied by a relevant discussion. Additionally, we compare the run-time of all baseline methods. For visual insight into the diversity of solutions, we include scatter plots representing the Pareto front for problems with two objectives. Lastly, we provide statistics on the selection of AF.

\vspace{0.25ex}

\textbf{PDBO Advantages.} PDBO is fast and effective in producing high-quality and diverse Pareto fronts. While outperforming the baseline methods, it can also be used with any number of input and output dimensions as well as being flexible to run with any batch size. The two state-of-the-art methods are DGEMO and qEHVI. The DGEMO method fails to run for experiments with more than three objective functions as the graph cut algorithm consistently crashes (same observation was made by \cite{daulton2021parallel}). qEHVI fails to run with batch sizes higher than eight as the method becomes extremely memory-consuming even with GPUs. We provide a more detailed discussion about these limitations in the Appendix. Therefore, PDBO's ability to easily run with any input and output dimensions as well as any batch size is an advantage for practitioners. PDBO is capable of proactively creating a diverse Pareto front while improving or maintaining the quality of the Pareto front. 

Given that PDBO incorporates two key contributions, namely adaptive acquisition function selection and multi-objective batch selection using DPPs, we examine the individual contributions of each component to the overall performance by conducting ablation experiments.

\vspace{0.25ex}

{\bf Merits of Adaptive AF Selection.} We demonstrate the superiority of the adaptive AF selection method, as outlined in Section \ref{adaptive_acq}, compared to using a static AF from the portfolio. To isolate the impact of this component from the batch selection process, we conduct an ablation study using the USEMO baseline. With a batch size of one, we consider USEMO with UCB, TS, ID, and EI as baselines. We then evaluate the efficacy of our MAB method by incorporating the adaptive AF selection approach into USEMO. Results shown in the Appendix consistently demonstrate the superior performance of the MAB strategy over using a static AF. 

\vspace{0.25ex}

{\bf Merits of DPP-Based Batch Selection for MOO.} Following a similar ablation approach, we employ the USEMO-EI baseline to examine the impact of the DPP-based batch selection. USEMO-EI selects the next input for evaluation from the cheap Pareto set based on an uncertainty metric. To perform this ablation, we replace the input selection mechanism utilized in USEMO with our proposed DPP selection strategy and compare their performance. The ablation is conducted across different batch sizes $B \in \{2, 4, 8, 16\}$. The results presented in the Appendix reveal that the proposed DPP selection strategy, referred to as DPP-EI, surpasses the USEMO selection strategy in terms of diversity while simultaneously improving the quality of hypervolume.

\vspace{-0.5mm}
\section{Summary}
We studied the Pareto front-Diverse Batch Multi-Objective BO (PDBO) method based on the BO framework. It employs a full information multi-arm bandit algorithm with discounted reward to adaptively select the most suitable acquisition function in each iteration. We also proposed an appropriate reward based on the relative hypervolume contribution of each acquisition function and a multi-objective DPP approach configured to select a batch of Pareto-diverse inputs for evaluation. Experiments on multiple benchmarks demonstrate that PDBO outperforms prior methods in terms of both diversity and quality of Pareto-front solutions.\\

\vspace{-1.5mm}

\noindent \textbf{Acknowledgements} The authors gratefully acknowledge the in part support from National Science Foundation (NSF) grants IIS-1845922, SII-2030159, and CNS-2308530. The views expressed are those of the authors and do not reflect the official policy or position of the NSF.

\bibliographystyle{plainnat}
\bibliography{aaai24}

\begin{thebibliography}{62}
\providecommand{\natexlab}[1]{#1}
\providecommand{\url}[1]{\texttt{#1}}
\expandafter\ifx\csname urlstyle\endcsname\relax
  \providecommand{\doi}[1]{doi: #1}\else
  \providecommand{\doi}{doi: \begingroup \urlstyle{rm}\Url}\fi

\bibitem[Abdolshah et~al.(2019)Abdolshah, Shilton, Rana, Gupta, and Venkatesh]{abdolshah2019multi}
Majid Abdolshah, Alistair Shilton, Santu Rana, Sunil Gupta, and Svetha Venkatesh.
\newblock Multi-objective bayesian optimisation with preferences over objectives.
\newblock \emph{Advances in neural information processing systems}, 32, 2019.

\bibitem[Angermueller et~al.(2020)Angermueller, Belanger, Gane, Mariet, Dohan, Murphy, Colwell, and Sculley]{angermueller2020population}
Christof Angermueller, David Belanger, Andreea Gane, Zelda Mariet, David Dohan, Kevin Murphy, Lucy Colwell, and D~Sculley.
\newblock Population-based black-box optimization for biological sequence design.
\newblock In \emph{International Conference on Machine Learning}, pages 324--334. PMLR, 2020.

\bibitem[Ashby(2000)]{ashby2000multi}
MF~Ashby.
\newblock Multi-objective optimization in material design and selection.
\newblock \emph{Acta materialia}, 2000.

\bibitem[Astudillo and Frazier(2020)]{astudillo2020multi}
Raul Astudillo and Peter Frazier.
\newblock Multi-attribute bayesian optimization with interactive preference learning.
\newblock In \emph{International Conference on Artificial Intelligence and Statistics}, pages 4496--4507. PMLR, 2020.

\bibitem[Auer(2002)]{auer2002using}
Peter Auer.
\newblock Using confidence bounds for exploitation-exploration trade-offs.
\newblock \emph{JMLR}, 2002.

\bibitem[Belakaria et~al.(2019)Belakaria, Deshwal, and Doppa]{belakaria2019max}
Syrine Belakaria, Aryan Deshwal, and Janardhan~Rao Doppa.
\newblock Max-value entropy search for multi-objective {B}ayesian optimization.
\newblock In \emph{Conference on Neural Information Processing Systems}, 2019.

\bibitem[Belakaria et~al.(2020{\natexlab{a}})Belakaria, Deshwal, Jayakodi, and Doppa]{Usemo}
Syrine Belakaria, Aryan Deshwal, Nitthilan~Kannappan Jayakodi, and Janardhan~Rao Doppa.
\newblock Uncertainty-aware search framework for multi-objective bayesian optimization.
\newblock In \emph{AAAI}, 2020{\natexlab{a}}.

\bibitem[Belakaria et~al.(2020{\natexlab{b}})Belakaria, Jackson, Cao, Doppa, and Lu]{belakaria2020machine}
Syrine Belakaria, Derek Jackson, Yue Cao, Janardhan~Rao Doppa, and Xiaonan Lu.
\newblock Machine learning enabled fast multi-objective optimization for electrified aviation power system design.
\newblock In \emph{IEEE Energy Conversion Congress and Exposition (ECCE)}, 2020{\natexlab{b}}.

\bibitem[Birol et~al.(2002)Birol, Undey, and Cinar]{birol2002modular}
Gulnur Birol, Cenk Undey, and Ali Cinar.
\newblock A modular simulation package for fed-batch fermentation: penicillin production.
\newblock \emph{Computers and chemical engineering}, 2002.

\bibitem[Borodin(2009)]{borodin2009determinantal}
Alexei Borodin.
\newblock Determinantal point processes.
\newblock \emph{arXiv preprint arXiv:0911.1153}, 2009.

\bibitem[Borodin and Olshanski(2005)]{borodin2005harmonic}
Alexei Borodin and Grigori Olshanski.
\newblock Harmonic analysis on the infinite-dimensional unitary group and determinantal point processes.
\newblock \emph{Annals of mathematics}, pages 1319--1422, 2005.

\bibitem[Cesa-Bianchi and Lugosi(2006)]{cesa2006prediction}
Nicolo Cesa-Bianchi and G{\'a}bor Lugosi.
\newblock \emph{Prediction, learning, and games}.
\newblock Cambridge university press, 2006.

\bibitem[Coello~Coello and Reyes~Sierra(2004)]{coello2004study}
Carlos~A Coello~Coello and Margarita Reyes~Sierra.
\newblock A study of the parallelization of a coevolutionary multi-objective evolutionary algorithm.
\newblock In \emph{MICAI 2004: Third Mexican International Conference on Artificial Intelligence, Mexico City, Mexico, April 26-30, 2004. Proceedings 3}. Springer, 2004.

\bibitem[Daulton et~al.(2020)Daulton, Balandat, and Bakshy]{daulton2020differentiable}
Samuel Daulton, Maximilian Balandat, and Eytan Bakshy.
\newblock Differentiable expected hypervolume improvement for parallel multi-objective bayesian optimization.
\newblock \emph{Advances in Neural Information Processing Systems}, 33:\penalty0 9851--9864, 2020.

\bibitem[Daulton et~al.(2021)Daulton, Balandat, and Bakshy]{daulton2021parallel}
Samuel Daulton, Maximilian Balandat, and Eytan Bakshy.
\newblock Parallel bayesian optimization of multiple noisy objectives with expected hypervolume improvement.
\newblock \emph{NeurIPS}, 34, 2021.

\bibitem[Daulton et~al.(2022)Daulton, Eriksson, Balandat, and Bakshy]{daulton2022multiobjective}
Samuel Daulton, David Eriksson, Maximilian Balandat, and Eytan Bakshy.
\newblock Multi-objective bayesian optimization over high-dimensional search spaces.
\newblock In \emph{The 38th Conference on Uncertainty in Artificial Intelligence}, 2022.
\newblock URL \url{https://openreview.net/forum?id=r5IEvvIs9xq}.

\bibitem[Deb and Srinivasan(2006)]{deb2006innovization}
Kalyanmoy Deb and Aravind Srinivasan.
\newblock Innovization: Innovating design principles through optimization.
\newblock In \emph{Proceedings of the 8th annual conference on Genetic and evolutionary computation}, pages 1629--1636, 2006.

\bibitem[Deb et~al.(1995)Deb, Agrawal, et~al.]{deb1995simulated}
Kalyanmoy Deb, Ram~Bhushan Agrawal, et~al.
\newblock Simulated binary crossover for continuous search space.
\newblock \emph{Complex systems}, 1995.

\bibitem[Deb et~al.(1996)Deb, Goyal, et~al.]{deb1996combined}
Kalyanmoy Deb, Mayank Goyal, et~al.
\newblock A combined genetic adaptive search (geneas) for engineering design.
\newblock \emph{Computer Science and informatics}, 26:\penalty0 30--45, 1996.

\bibitem[Deb et~al.(2002{\natexlab{a}})Deb, Pratap, Agarwal, Meyarivan, and Fast]{deb2002nsga}
Kalyanmoy Deb, Amrit Pratap, Sameer Agarwal, T~Meyarivan, and A~Fast.
\newblock Nsga-ii.
\newblock \emph{IEEE Transactions on Evolutionary Computation}, 6\penalty0 (2):\penalty0 182--197, 2002{\natexlab{a}}.

\bibitem[Deb et~al.(2002{\natexlab{b}})Deb, Pratap, Agarwal, and Meyarivan]{deb2002fast}
Kalyanmoy Deb, Amrit Pratap, Sameer Agarwal, and TAMT Meyarivan.
\newblock A fast and elitist multiobjective genetic algorithm: Nsga-ii.
\newblock \emph{IEEE transactions on evolutionary computation}, 6, 2002{\natexlab{b}}.

\bibitem[Deb et~al.(2005)Deb, Thiele, Laumanns, and Zitzler]{deb2005scalable}
Kalyanmoy Deb, Lothar Thiele, Marco Laumanns, and Eckart Zitzler.
\newblock Scalable test problems for evolutionary multiobjective optimization.
\newblock In \emph{Evolutionary multiobjective optimization}. Springer, 2005.

\bibitem[Deshwal et~al.(2021)Deshwal, Simon, and Doppa]{deshwal2021bayesian}
Aryan Deshwal, Cory~M Simon, and Janardhan~Rao Doppa.
\newblock Bayesian optimization of nanoporous materials.
\newblock \emph{Molecular Systems Design and Engineering}, 2021.

\bibitem[Emmerich and Klinkenberg(2008)]{emmerich2008computation}
Michael Emmerich and Jan-willem Klinkenberg.
\newblock The computation of the expected improvement in dominated hypervolume of pareto front approximations.
\newblock \emph{Leiden University}, 2008.

\bibitem[Eriksson et~al.(2019)Eriksson, Pearce, Gardner, Turner, and Poloczek]{eriksson2019scalable}
David Eriksson, Michael Pearce, Jacob Gardner, Ryan~D Turner, and Matthias Poloczek.
\newblock Scalable global optimization via local bayesian optimization.
\newblock \emph{NeurIPS}, 2019.

\bibitem[Freund and Schapire(1997)]{freund1997decision}
Yoav Freund and Robert~E Schapire.
\newblock A decision-theoretic generalization of on-line learning and an application to boosting.
\newblock \emph{Journal of computer and system sciences}, 55\penalty0 (1), 1997.

\bibitem[Hern{\'a}ndez-Lobato et~al.(2016)Hern{\'a}ndez-Lobato, Hernandez-Lobato, Shah, and Adams]{PESMO}
Daniel Hern{\'a}ndez-Lobato, Jose Hernandez-Lobato, Amar Shah, and Ryan Adams.
\newblock Predictive entropy search for multi-objective {B}ayesian optimization.
\newblock In \emph{Proceedings of International Conference on Machine Learning (ICML)}, pages 1492--1501, 2016.

\bibitem[Hern{\'a}ndez-Lobato et~al.(2014)Hern{\'a}ndez-Lobato, Hoffman, and Ghahramani]{PES}
Jos{\'e}~Miguel Hern{\'a}ndez-Lobato, Matthew~W Hoffman, and Zoubin Ghahramani.
\newblock Predictive entropy search for efficient global optimization of black-box functions.
\newblock In \emph{NeurIPS}, 2014.

\bibitem[Hoffman et~al.(2011)Hoffman, Brochu, De~Freitas, et~al.]{hoffman2011portfolio}
Matthew Hoffman, Eric Brochu, Nando De~Freitas, et~al.
\newblock Portfolio allocation for bayesian optimization.
\newblock In \emph{UAI}, pages 327--336, 2011.

\bibitem[Hvarfner et~al.(2022)Hvarfner, Hutter, and Nardi]{hvarfner2022joint}
Carl Hvarfner, Frank Hutter, and Luigi Nardi.
\newblock Joint entropy search for maximally-informed bayesian optimization.
\newblock \emph{arXiv preprint arXiv:2206.04771}, 2022.

\bibitem[Jain et~al.(2022)Jain, Raparthy, Hernandez-Garcia, Rector-Brooks, Bengio, Miret, and Bengio]{jain2022multi}
Moksh Jain, Sharath~Chandra Raparthy, Alex Hernandez-Garcia, Jarrid Rector-Brooks, Yoshua Bengio, Santiago Miret, and Emmanuel Bengio.
\newblock Multi-objective gflownets.
\newblock \emph{arXiv preprint arXiv:2210.12765}, 2022.

\bibitem[Kathuria et~al.(2016)Kathuria, Deshpande, and Kohli]{kathuria2016batched}
Tarun Kathuria, Amit Deshpande, and Pushmeet Kohli.
\newblock Batched gaussian process bandit optimization via determinantal point processes.
\newblock \emph{NeurIPS}, 29, 2016.

\bibitem[Knowles(2006)]{knowles2006parego}
Joshua Knowles.
\newblock Parego: A hybrid algorithm with on-line landscape approximation for expensive multiobjective optimization problems.
\newblock \emph{IEEE Transactions on Evolutionary Computation}, 10\penalty0 (1):\penalty0 50--66, 2006.

\bibitem[Konakovic~Lukovic et~al.(2020)Konakovic~Lukovic, Tian, and Matusik]{konakovic2020diversity}
Mina Konakovic~Lukovic, Yunsheng Tian, and Wojciech Matusik.
\newblock Diversity-guided multi-objective bayesian optimization with batch evaluations.
\newblock \emph{Advances in Neural Information Processing Systems}, 2020.

\bibitem[Kulesza et~al.(2012)Kulesza, Taskar, et~al.]{kulesza2012determinantal}
Alex Kulesza, Ben Taskar, et~al.
\newblock Determinantal point processes for machine learning.
\newblock \emph{Foundations and Trends in Machine Learning}, 2012.

\bibitem[Lalee et~al.(1998)Lalee, Nocedal, and Plantenga]{lalee1998implementation}
Marucha Lalee, Jorge Nocedal, and Todd Plantenga.
\newblock On the implementation of an algorithm for large-scale equality constrained optimization.
\newblock \emph{SIAM Journal on Optimization}, 8\penalty0 (3):\penalty0 682--706, 1998.

\bibitem[Lin et~al.(2022)Lin, Yang, Zhang, and Zhang]{lin2022pareto}
Xi~Lin, Zhiyuan Yang, Xiaoyuan Zhang, and Qingfu Zhang.
\newblock Pareto set learning for expensive multi-objective optimization.
\newblock \emph{arXiv preprint arXiv:2210.08495}, 2022.

\bibitem[Mockus et~al.(1978)Mockus, Tiesis, and Zilinskas]{mockus1978application}
Jonas Mockus, Vytautas Tiesis, and Antanas Zilinskas.
\newblock The application of bayesian methods for seeking the extremum.
\newblock \emph{Towards global optimization}, 1978.

\bibitem[Nava et~al.(2022)Nava, Mutny, and Krause]{nava2022diversified}
Elvis Nava, Mojmir Mutny, and Andreas Krause.
\newblock Diversified sampling for batched bayesian optimization with determinantal point processes.
\newblock In \emph{International Conference on Artificial Intelligence and Statistics}, pages 7031--7054. PMLR, 2022.

\bibitem[Nicolaou and Brown(2013)]{nicolaou2013multi}
Christos~A Nicolaou and Nathan Brown.
\newblock Multi-objective optimization methods in drug design.
\newblock \emph{Drug Discovery Today: Technologies}, 2013.

\bibitem[Nikolov(2015)]{nikolov2015randomized}
Aleksandar Nikolov.
\newblock Randomized rounding for the largest simplex problem.
\newblock In \emph{Proceedings of the forty-seventh annual ACM symposium on Theory of computing}, pages 861--870, 2015.

\bibitem[Nocedal and Wright(2006)]{nocedal2006numerical}
J~Nocedal and SJ~Wright.
\newblock Numerical optimization (springer, new york, 1999)., 2006.

\bibitem[Oh et~al.(2021)Oh, Bondesan, Gavves, and Welling]{oh2021batch}
Changyong Oh, Roberto Bondesan, Efstratios Gavves, and Max Welling.
\newblock Batch bayesian optimization on permutations using acquisition weighted kernels.
\newblock \emph{arXiv preprint arXiv:2102.13382}, 2021.

\bibitem[Okoth et~al.(2022)Okoth, Shang, Jiao, Arshad, Rehman, and Hamam]{okoth2022large}
Michael~Aggrey Okoth, Ronghua Shang, Licheng Jiao, Jehangir Arshad, Ateeq~Ur Rehman, and Habib Hamam.
\newblock A large scale evolutionary algorithm based on determinantal point processes for large scale multi-objective optimization problems.
\newblock \emph{Electronics}, 11\penalty0 (20):\penalty0 3317, 2022.

\bibitem[Paria et~al.(2020)Paria, Kandasamy, and P{\'o}czos]{paria2020flexible}
Biswajit Paria, Kirthevasan Kandasamy, and Barnab{\'a}s P{\'o}czos.
\newblock A flexible framework for multi-objective bayesian optimization using random scalarizations.
\newblock In \emph{Uncertainty in Artificial Intelligence}, pages 766--776. PMLR, 2020.

\bibitem[Pierrot et~al.(2022)Pierrot, Richard, Beguir, and Cully]{pierrot2022multi}
Thomas Pierrot, Guillaume Richard, Karim Beguir, and Antoine Cully.
\newblock Multi-objective quality diversity optimization.
\newblock In \emph{Proceedings of the Genetic and Evolutionary Computation Conference}, 2022.

\bibitem[Shahriari et~al.(2015)Shahriari, Swersky, Wang, Adams, and De~Freitas]{shahriari2015taking}
Bobak Shahriari, Kevin Swersky, Ziyu Wang, Ryan~P Adams, and Nando De~Freitas.
\newblock Taking the human out of the loop: A review of bayesian optimization.
\newblock \emph{Proceedings of the IEEE}, 2015.

\bibitem[Siegmund et~al.(2012)Siegmund, Kolesnikov, K{\"a}stner, Apel, Batory, Rosenm{\"u}ller, and Saake]{siegmund2012predicting}
Norbert Siegmund, Sergiy~S Kolesnikov, Christian K{\"a}stner, Sven Apel, Don Batory, Marko Rosenm{\"u}ller, and Gunter Saake.
\newblock Predicting performance via automated feature-interaction detection.
\newblock In \emph{Proceedings of the 34th International Conference on Software Engineering (ICSE)}, pages 167--177, 2012.

\bibitem[Srinivas et~al.(2009)Srinivas, Krause, Kakade, and Seeger]{srinivas2009Gaussian}
Niranjan Srinivas, Andreas Krause, Sham~M Kakade, and Matthias Seeger.
\newblock {G}aussian process optimization in the bandit setting: No regret and experimental design.
\newblock \emph{arXiv preprint arXiv:0912.3995}, 2009.

\bibitem[Suzuki et~al.(2020)Suzuki, Takeno, Tamura, Shitara, and Karasuyama]{suzuki2020multi}
Shinya Suzuki, Shion Takeno, Tomoyuki Tamura, Kazuki Shitara, and Masayuki Karasuyama.
\newblock Multi-objective bayesian optimization using pareto-frontier entropy.
\newblock In \emph{International Conference on Machine Learning}, pages 9279--9288. PMLR, 2020.

\bibitem[Taneda(2015)]{taneda2015multi}
Akito Taneda.
\newblock Multi-objective optimization for rna design with multiple target secondary structures.
\newblock \emph{BMC bioinformatics}, 2015.

\bibitem[Thompson(1933)]{thompson1933likelihood}
William~R Thompson.
\newblock On the likelihood that one unknown probability exceeds another in view of the evidence of two samples.
\newblock \emph{Biometrika}, 25:\penalty0 285--294, 1933.

\bibitem[Vasconcelos et~al.(2019)Vasconcelos, de~Souza, Mattos, and Gomes]{nopastvasconcelos2019no}
Thiago de~P Vasconcelos, Daniel~ARMA de~Souza, C{\'e}sar~LC Mattos, and Jo{\~a}o~PP Gomes.
\newblock No-past-bo: Normalized portfolio allocation strategy for bayesian optimization.
\newblock In \emph{International Conference on Tools with Artificial Intelligence (ICTAI)}. IEEE, 2019.

\bibitem[Vasconcelos et~al.(2022)Vasconcelos, de~Souza, Virgolino, Mattos, and Gomes]{setupvasconcelos2022self}
Thiago de~P Vasconcelos, Daniel Augusto~RMA de~Souza, Gustavo C de~M Virgolino, C{\'e}sar~LC Mattos, and Jo{\~a}o~PP Gomes.
\newblock Self-tuning portfolio-based bayesian optimization.
\newblock \emph{Expert Systems with Applications}, 188:\penalty0 115847, 2022.

\bibitem[Virtanen et~al.(2020)Virtanen, Gommers, Oliphant, Haberland, Reddy, Cournapeau, Burovski, Peterson, Weckesser, Bright, {van der Walt}, Brett, Wilson, Millman, Mayorov, Nelson, Jones, Kern, Larson, Carey, Polat, Feng, Moore, {VanderPlas}, Laxalde, Perktold, Cimrman, Henriksen, Quintero, Harris, Archibald, Ribeiro, Pedregosa, {van Mulbregt}, and {SciPy 1.0 Contributors}]{2020SciPy-NMeth}
Pauli Virtanen, Ralf Gommers, Travis~E. Oliphant, Matt Haberland, Tyler Reddy, David Cournapeau, Evgeni Burovski, Pearu Peterson, Warren Weckesser, Jonathan Bright, St{\'e}fan~J. {van der Walt}, Matthew Brett, Joshua Wilson, K.~Jarrod Millman, Nikolay Mayorov, Andrew R.~J. Nelson, Eric Jones, Robert Kern, Eric Larson, C~J Carey, {\.I}lhan Polat, Yu~Feng, Eric~W. Moore, Jake {VanderPlas}, Denis Laxalde, Josef Perktold, Robert Cimrman, Ian Henriksen, E.~A. Quintero, Charles~R. Harris, Anne~M. Archibald, Ant{\^o}nio~H. Ribeiro, Fabian Pedregosa, Paul {van Mulbregt}, and {SciPy 1.0 Contributors}.
\newblock {{SciPy} 1.0: Fundamental Algorithms for Scientific Computing in Python}.
\newblock \emph{Nature Methods}, 17:\penalty0 261--272, 2020.
\newblock \doi{10.1038/s41592-019-0686-2}.

\bibitem[Wang et~al.(2022)Wang, Ge, Chen, and Liu]{wang2022many}
Mengzhen Wang, Fangzhen Ge, Debao Chen, and Huaiyu Liu.
\newblock A many-objective evolutionary algorithm using determinantal point process in potential region.
\newblock In \emph{Proceedings of the 6th International Conference on Control Engineering and Artificial Intelligence}, pages 83--91, 2022.

\bibitem[Wang and Jegelka(2017)]{MES}
Zi~Wang and Stefanie Jegelka.
\newblock Max-value entropy search for efficient {B}ayesian optimization.
\newblock In \emph{ICML}, 2017.

\bibitem[Wang et~al.(2017)Wang, Li, Jegelka, and Kohli]{wang2017batched}
Zi~Wang, Chengtao Li, Stefanie Jegelka, and Pushmeet Kohli.
\newblock Batched high-dimensional bayesian optimization via structural kernel learning.
\newblock In \emph{International Conference on Machine Learning}, pages 3656--3664. PMLR, 2017.

\bibitem[Williams and Rasmussen(2006)]{williams2006gaussian}
Christopher~KI Williams and Carl~Edward Rasmussen.
\newblock \emph{Gaussian processes for machine learning}.
\newblock MIT press Cambridge, MA, 2006.

\bibitem[Zhang et~al.(2020)Zhang, Li, Li, and Chen]{zhang2020new}
Peng Zhang, Jinlong Li, Tengfei Li, and Huanhuan Chen.
\newblock A new many-objective evolutionary algorithm based on determinantal point processes.
\newblock \emph{IEEE Transactions on Evolutionary Computation}, 2020.

\bibitem[Zitzler and Thiele(1999)]{zitzler1999multiobjective}
Eckart Zitzler and Lothar Thiele.
\newblock Multiobjective evolutionary algorithms: a comparative case study and the strength pareto approach.
\newblock \emph{IEEE transactions on Evolutionary Computation}, 1999.

\bibitem[Zitzler et~al.(2000)Zitzler, Deb, and Thiele]{zitzler2000comparison}
Eckart Zitzler, Kalyanmoy Deb, and Lothar Thiele.
\newblock Comparison of multiobjective evolutionary algorithms: Empirical results.
\newblock \emph{Evolutionary computation}, 8\penalty0 (2):\penalty0 173--195, 2000.

\end{thebibliography}

\clearpage

\section{Appendix}

\subsection{Theoretical Analysis}

In this section, we assume maximization and assume that the UCB acquisition function is in the portfolio of acquisition functions used during the multi-objective Bayesian optimization process. We make this choice for the sake of clarity and ease of readability as we build our theoretical analysis on prior seminal work \cite{srinivas2009Gaussian,hoffman2011portfolio}. It is important to clarify that this is not a restrictive assumption and that with minimal mathematical transformations, the same derived regret bound holds for the case of minimization with the LCB acquisition function being in the portfolio instead.  

In order to simplify the proof and solely for the sake of theoretical regret bound, we consider the instant reward of iteration $t$ to be the sum of predictive means of the Gaussian processes.
\begin{equation}\label{eq:IRtheo}
 IR_t= \sum_{i=1}^k \mu_{i,t-1}(\vec{x}_t)   
\end{equation}
 with $\mu_{i,t-1}$ is the posterior mean of function $i$.

The cumulative reward over ${T_{max}}$ iterations that would have been obtained using acquisition function $j$ is defined as: 
\begin{equation}
	r_{T_{max}}^j = \sum_{t=1}^{T_{max}} IR_t = \sum_{t=1}^{T_{max}} \sum_{i=1}^k \mu_{i,t-1}(\vec{x}_t^j).
\end{equation}

It is important to note that in our proposed algorithm, we use different and better-designed instant reward $IR_t$ and cumulative reward $r_{T_{max}}^i$ functions. The rewards in equation \ref{eq:IRtheo} is a design choice to achieve the following regret bound. In section \ref{ablation_theo}, we provide a discussion accompanied with an ablation study comparing the reward function used in theory to the reward function used in our proposed approach.

\vspace{1.0ex}

To bound the regret of Hedge with respect to the gain, we define the maximum strategy as: $r_{T_{max}}^\mathrm{max}=\max_jr_{T_{max}}^j$
\begin{lemma}
\label{lemma:exp3}
With probability at least $1-\delta_1$ and $\eta=\sqrt{8\ln k/{T_{max}}}$, the regret is bounded by
\begin{equation}
	r_{T_{max}}^\mathrm{max} - r_{T_{max}}^\mathrm{Hedge} \leq \mathcal O(\sqrt{{T_{max}}}).
\end{equation}
\end{lemma}

This result follows directly from \cite[Section 4.2]{cesa2006prediction} for rewards in the range $[0,1]$. This result is also used in the proof of regret bound in \cite{hoffman2011portfolio}. \cite{hoffman2011portfolio} discussed possible generalizations that come at the cost of worsening the bound with a multiplicative or additive constant. These relaxations include a non-restrictive reward bound and a time variant $\eta$ term. These generalizations also hold in the context of our proof. We refer the reader to \cite{hoffman2011portfolio} for more details. It is important to note that this lemma holds for any choice of $IR_t$ where the reward is a result of the action taken by the Hedge algorithm.

\vspace{1.0ex}

The next lemmas are defined in \cite{srinivas2009Gaussian} and \cite{hoffman2011portfolio}. We will refer the reader to \cite[Lemma 5.1 and 5.3]{srinivas2009Gaussian} and \cite[Lemma 4 and 5]{hoffman2011portfolio} for proofs. We provide them here for the sake of completeness. It is important to clarify that these lemmas only depend on the surrogate Gaussian process models and can be used regardless of the UCB acquisition function.
\begin{lemma}
	\label{lemma:concentration}
	Assume $\delta_2\in(0,1)$, a finite sample space $|\mathcal{X}|<\infty$, and $\beta_t=2\log(|\mathcal{X}|\pi_t/\delta_2)$ where $\sum_{t}\pi_t^{-1}=1$ and $\pi_t>0$. Then with probability at least $1-\delta_2$, the absolute deviation of the mean is bounded by
	\begin{equation*}
		|f_i(\vec{x}) - \mu_{i,t-1}(\vec{x})| \leq \sqrt{\beta_t}\sigma_{i,t-1}(\vec{x})
		\quad\forall \vec{x}\in \mathcal{X}, \forall t\geq 1.
	\end{equation*}
\end{lemma}

\begin{lemma}
	\label{lemma:sum}
	Given points $\vec{x}_t$ selected by the algorithm, the following bound holds for the sum of variances:
	\begin{equation*}
		\sum_{t=1}^{T_{max}} \beta_t \sigma_{i,t}^2(\vec{x}_t) \leq C_i\beta_{T_{max}}\gamma^i_{T_{max}},
	\end{equation*}
	where $C_i=2/\log(1+\sigma_i^{-2})$.
\end{lemma}

The next lemma is defined in \cite{hoffman2011portfolio} and follows directly from \cite[Lemma 5.2]{srinivas2009Gaussian}. This lemma depends only on the definition of the UCB acquisition function, and does not require that points at any previous iteration were selected via UCB acquisition function.
\begin{lemma}
	\label{lemma:ucb}

The cheap multi-objective optimization (MOO) problem when the UCB acquisition function is selected is defined as follows:
\begin{align}
     \max_{x \in \mathcal{X}}\, (UCB_{1,t}(\vec{x}),\cdots, UCB_{k,t}(\vec{x}))
\end{align}
Assuming that the cheap MOO solver achieves optimality, leading to the optimal Pareto set for the above defined problem, with the Pareto set defined as $\mathcal{X}_{t}$, either  there exists a $\vec{x}_t^\mathrm{UCB} \in \mathcal{X}_{t}$ such that 
\begin{align}
UCB_{i,t} (\vec{x}^*) \leq UCB_{i,t} (\vec{x}_t^\mathrm{UCB}), \forall{i \in \{1,\cdots,k\}}
\end{align}
or $\vec{x}^*$ is in the optimal Pareto set  $\mathcal{X}_{t}$ generated by cheap MOO solver (i.e., $\vec{x}_t^\mathrm{UCB}= \vec{x}^*$).
\vspace{1ex}

\noindent Now, if the bound from Lemma~\ref{lemma:concentration} holds, then for a point $\vec{x}_t^\mathrm{UCB}$ proposed by UCB with parameters $\beta_t$, the following bound holds for any function $f_i$, 
\begin{align}
 f_i(\vec{x}^*) \leq UCB_{i,t} (\vec{x}^*) \leq UCB_{i,t} (\vec{x}_t^\mathrm{UCB}) 
\end{align}
Leading to the following bound
\begin{equation}
    f_i(\vec{x}^*) - \mu_{i,t-1}(\vec{x}_t^\mathrm{UCB}) \leq 
    \sqrt{\beta_t}\sigma_{i,t-1}(\vec{x}_t^\mathrm{UCB}).
\end{equation}
\end{lemma}

\vspace{1.0ex}

\noindent We can now combine these results to construct the proof of Theorem~\ref{theorem}.
\begin{proof}[Proof of Theorem~\ref{theorem}]
With probability at least $1-\delta_1$, the result of Lemma~\ref{lemma:exp3} holds. If we assume that UCB acquisition function is included in the portfolio of acquisition functions, we have
\begin{equation*}
	-r_{T_{max}}^\mathrm{Hedge} \leq \mathcal O(\sqrt{{T_{max}}}) - r_{T_{max}}^\mathrm{UCB}
\end{equation*}
and by adding $\sum_{t=1}^{T_{max}} \sum_{i=1}^k f_i(\vec{x}^*)$ to both sides of the inequality, we have:
\begin{align*}
	 & \sum_{t=1}^{T_{max}} \sum_{i=1}^k f_i(\vec{x}^*) - \mu_{i,t-1}(\vec{x}_t) \\ & \leq  \mathcal O(\sqrt{{T_{max}}}) +
	\sum_{t=1}^{T_{max}} \sum_{i=1}^k f_i(\vec{x}^*) - \mu_{i,t-1}(\vec{x}_t^\mathrm{UCB})   
\end{align*}
With probability at least $1-\delta_2$ the bound from Lemma~\ref{lemma:concentration} can be applied to the left-hand-side and the result of Lemma~\ref{lemma:ucb} can be applied to the right side of the inequality leading to the following inequality
\begin{align*}
	&\sum_{t=1}^{T_{max}} \sum_{i=1}^k f_i(\vec{x}^*) - f_i(\vec{x}_t) - \sqrt{\beta_t}\sigma_{i,t-1}(\vec{x}_t) \\
	& \leq
	\mathcal O(\sqrt{{T_{max}}}) + 
	\sum_{t=1}^{T_{max}} \sum_{i=1}^k \sqrt{\beta_t}\sigma_{i,t-1}(\vec{x}_t^\mathrm{UCB})
\end{align*}
which means that the regret is bounded by
\begin{align*}
	R_{T_{max}}(\vec{x}^*)
	&= \sum_{t=1}^{T_{max}} \sum_{i=1}^k  f_i(\vec{x}^*) - f_i(\vec{x}_t) \\
	&\leq
	\mathcal O(\sqrt{{T_{max}}}) + 
	\sum_{t=1}^{T_{max}} \sum_{i=1}^k\sqrt{\beta_t}\sigma_{i,t-1}(i,\vec{x}_t^\mathrm{UCB}) \\ 
 & + 
	\sum_{t=1}^{T_{max}} \sum_{i=1}^k \sqrt{\beta_t}\sigma_{i,t-1}(\vec{x}_t) \\
	&\leq
	\mathcal O(\sqrt{{T_{max}}}) + 
	\sum_{t=1}^{T_{max}} \sum_{i=1}^k \sqrt{\beta_t}\sigma_{i,t-1}(\vec{x}_t^\mathrm{UCB}) \\ 
 & + 
	\sqrt{C_i{T_{max}}\beta_{T_{max}}\gamma^i_{T_{max}}}.
\end{align*}
We should note that we cannot use Lemma~\ref{lemma:sum} to further simplify the terms involving  $\vec{x}_t^\mathrm{UCB}$. This is because the lemma only holds for points that are sampled by the algorithm, which may not include those proposed by UCB acquisition function.

Finally, this derivation requires Lemmas~\ref{lemma:exp3} and~\ref{lemma:ucb} to hold. Using a union bound argument, we can see that both lemmas hold with probability at least $1-\delta_1-\delta_2$. By setting $\delta_1=\delta_2=\delta/2$, we recover our result \cite{hoffman2011portfolio}.
\end{proof}

\subsection{PDBO Hypervolume Experiments}

In this section, we present supplementary experiments that focus on the comparison of hypervolume between our proposed method, PDBO, and other existing baselines. We refer the interested reader to Figure \ref{fig:appendix-main-results} for the additional results. 

\subsection{PDBO Diversity Experiments}

In terms of Pareto front diversity, we conduct a comprehensive comparison between PDBO and state-of-the-art methods introduced in Section \ref{experimental-section}. It is worth noting that, apart from DGEMO, none of the other baseline methods explicitly address Pareto front diversity. The results depicted in Figure \ref{fig:appendix-dpf} demonstrate the superior performance of PDBO, outperforming all existing methods in terms of Pareto front diversity measure.

\begin{figure*}[!ht]
\centering
\begin{subfigure}{\includegraphics[width=0.17\textwidth]{ 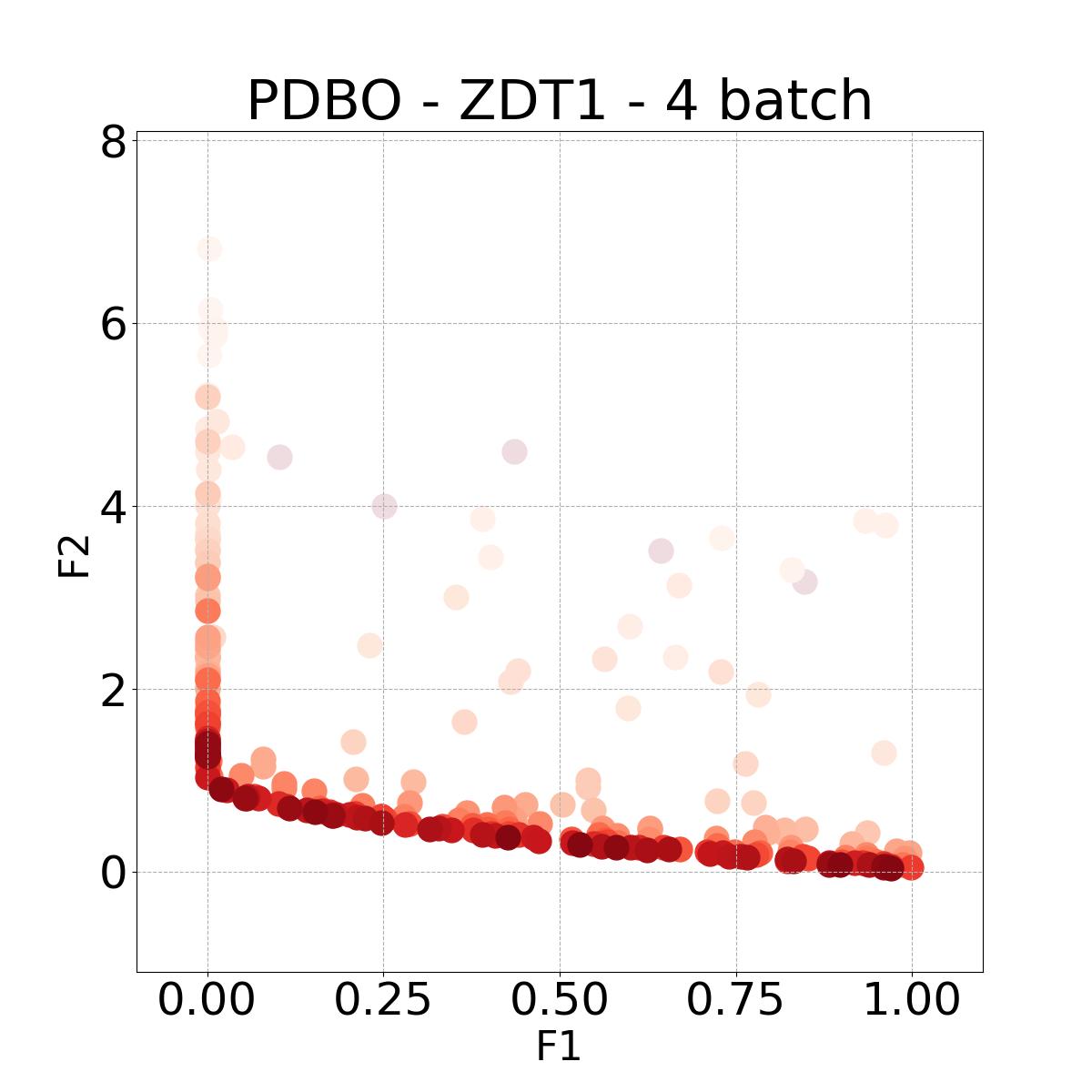}}\hspace{-3mm}
\end{subfigure}
\begin{subfigure}{\includegraphics[width=0.17\textwidth]{ 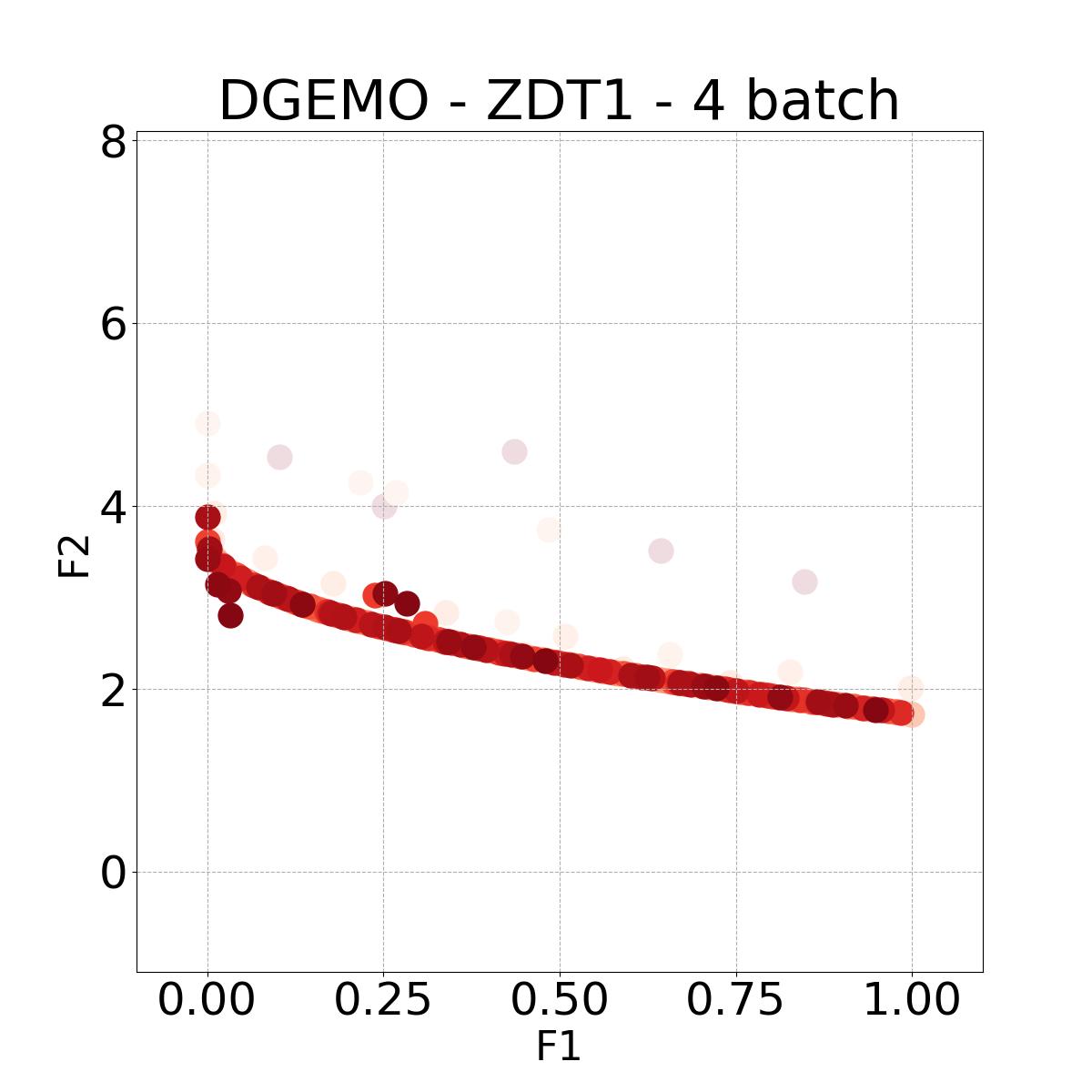}}\hspace{-3mm}
\end{subfigure}
\begin{subfigure}{\includegraphics[width=0.17\textwidth]{ 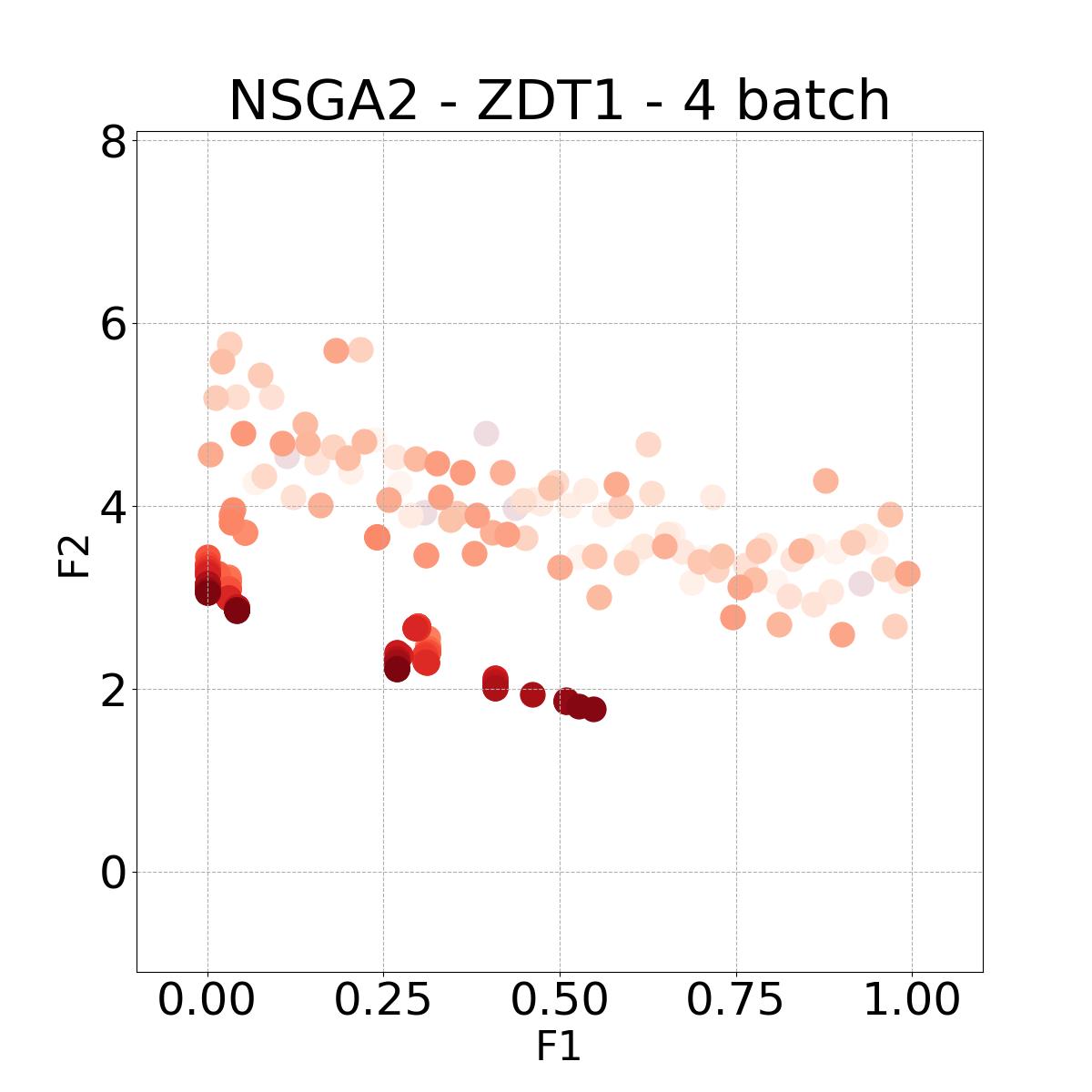}}\hspace{-3mm}
\end{subfigure}
\begin{subfigure}{\includegraphics[width=0.17\textwidth]{ 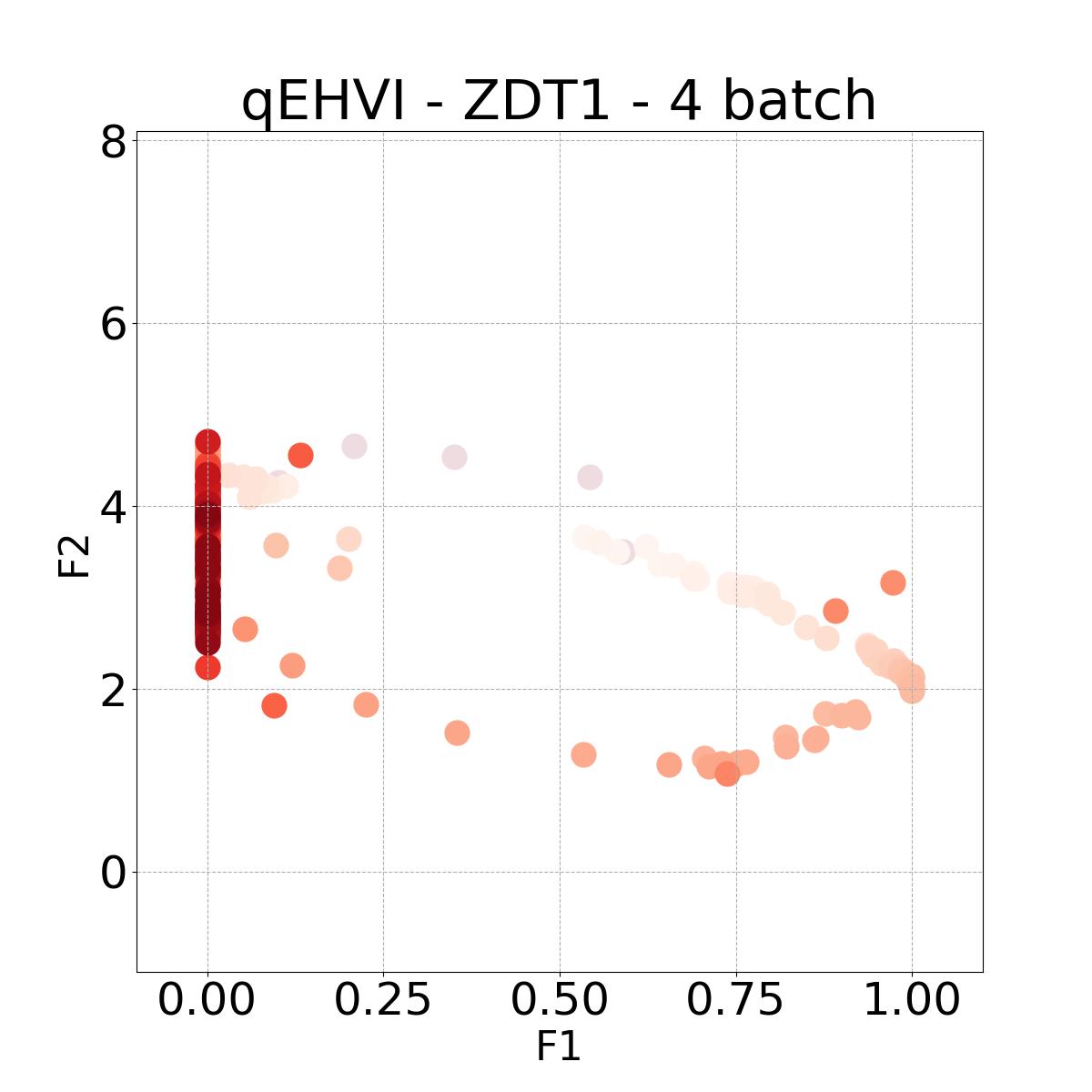}}\hspace{-3mm}
\end{subfigure}
\begin{subfigure}{\includegraphics[width=0.17\textwidth]{ 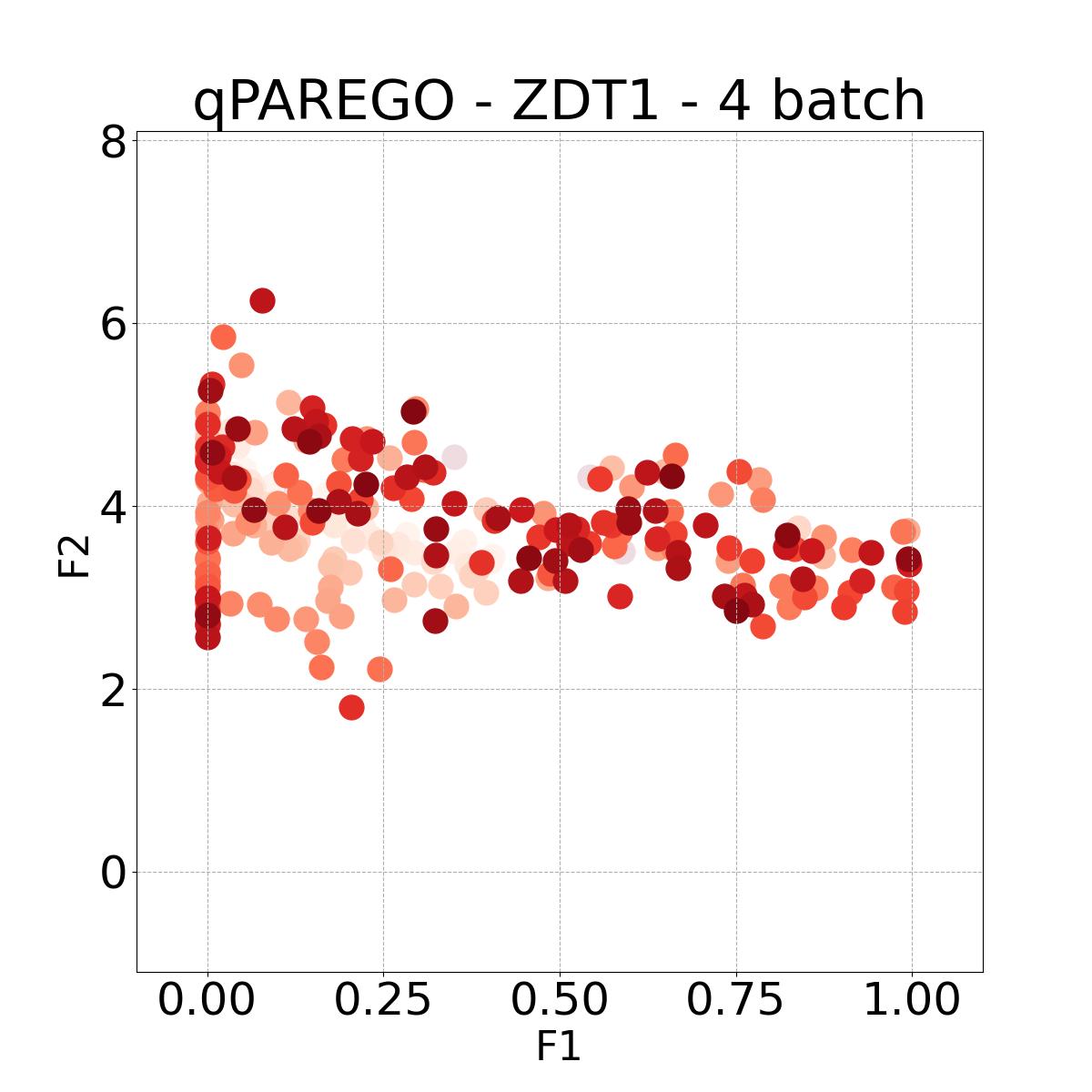}}\hspace{-3mm}
\end{subfigure}
\begin{subfigure}{\includegraphics[width=0.17\textwidth]{ 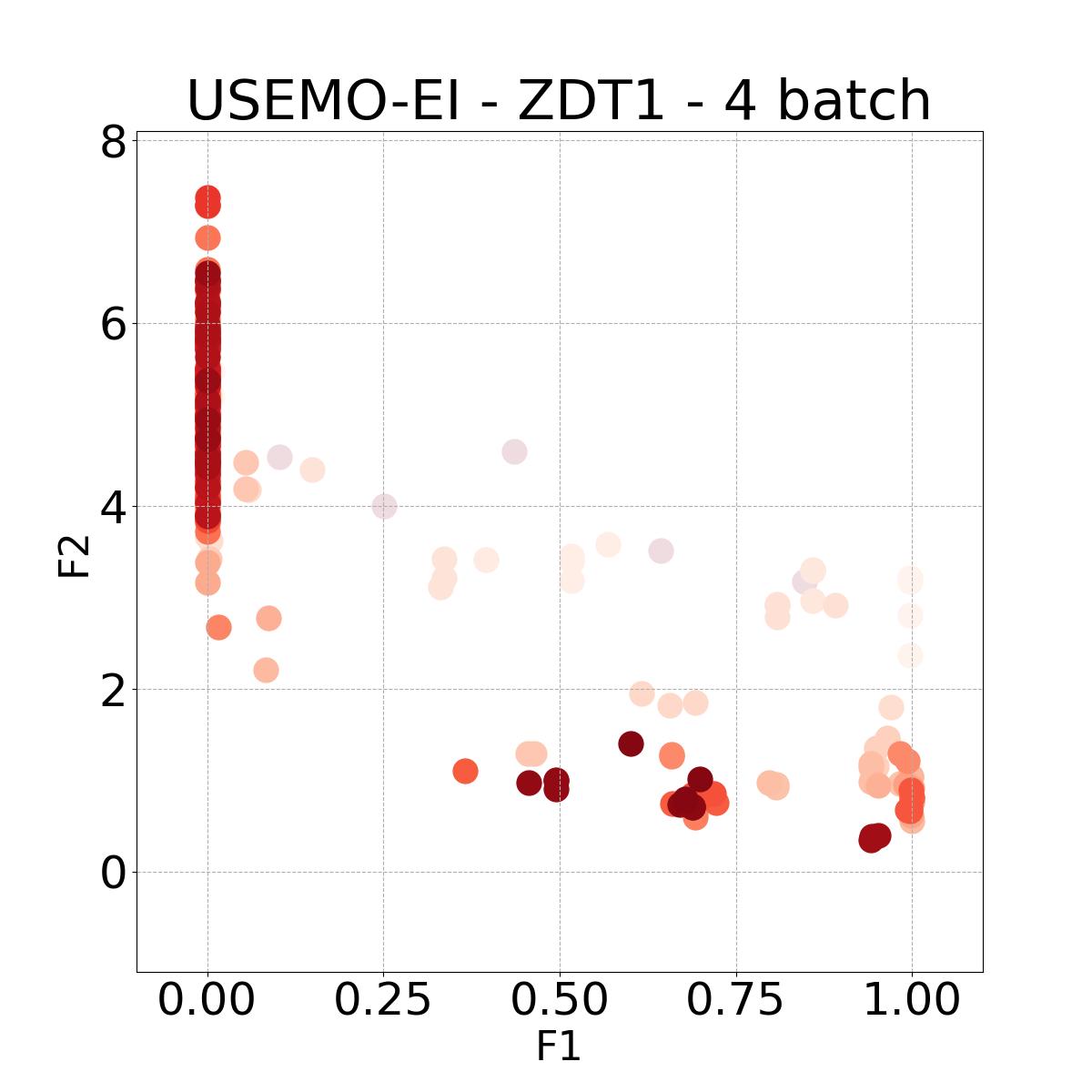}}\hspace{-3mm}
\end{subfigure}\\

\begin{subfigure}{\includegraphics[width=0.17\textwidth]{ 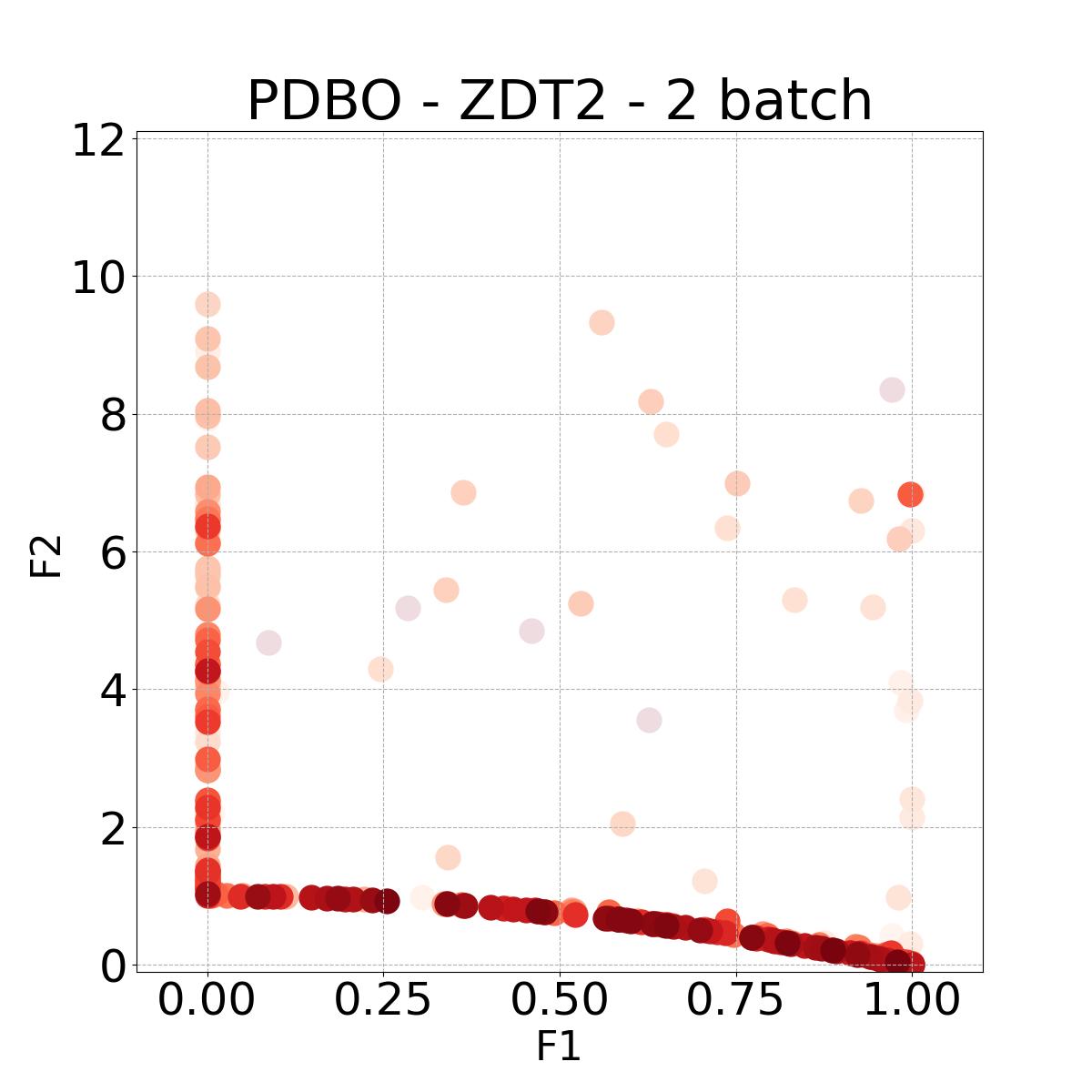}}\hspace{-3mm}
\end{subfigure}
\begin{subfigure}{\includegraphics[width=0.17\textwidth]{ 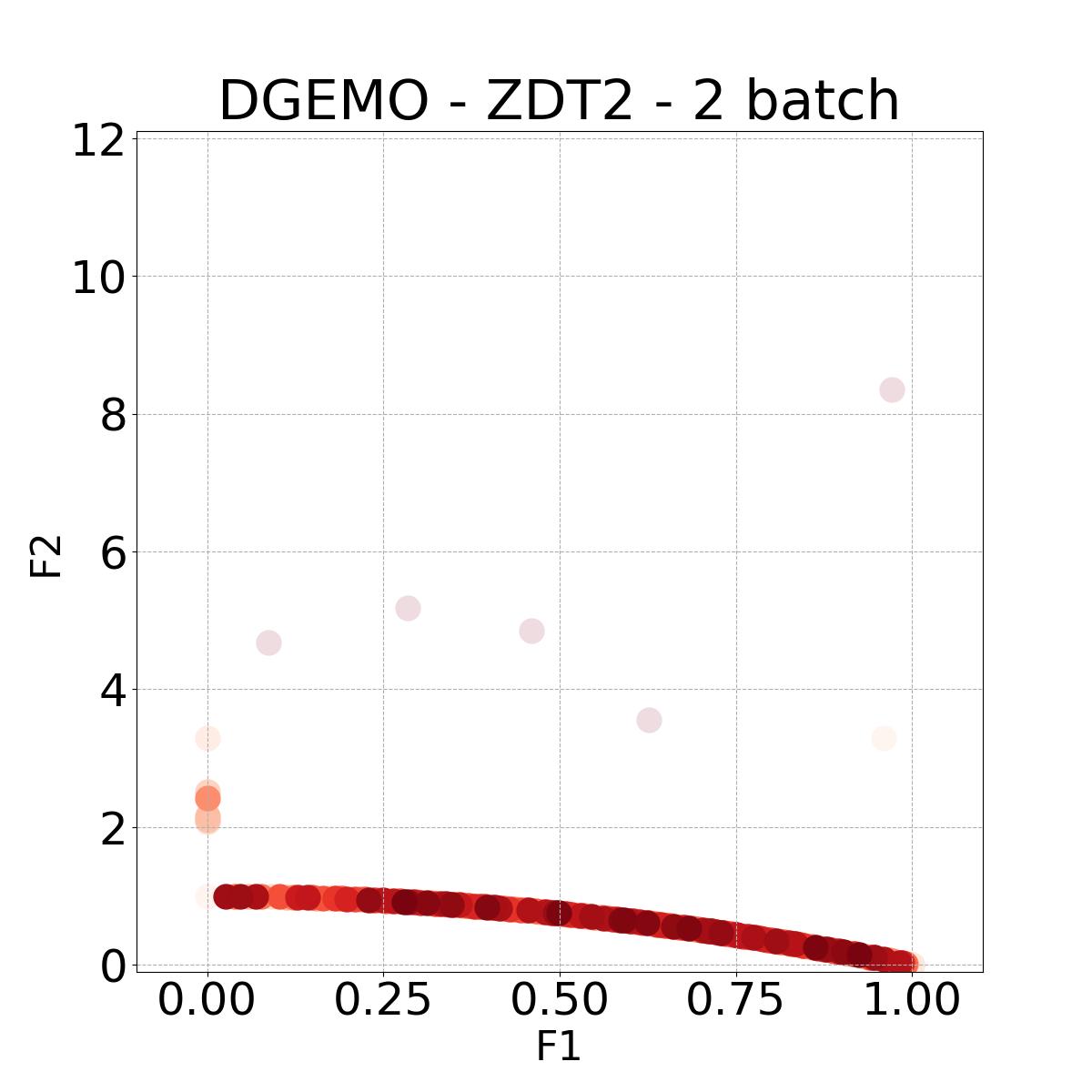}}\hspace{-3mm}
\end{subfigure}
\begin{subfigure}{\includegraphics[width=0.17\textwidth]{ 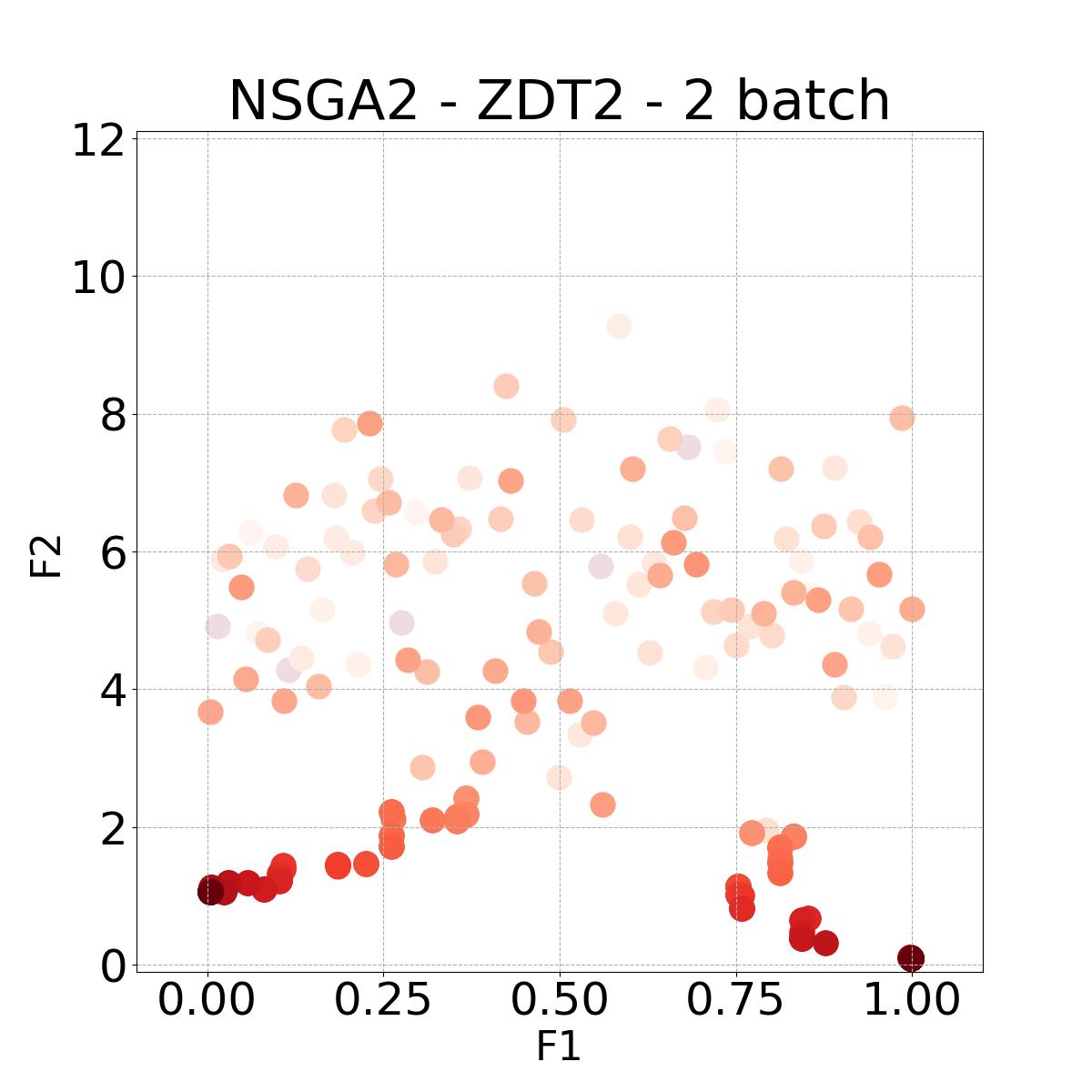}}\hspace{-3mm}
\end{subfigure}
\begin{subfigure}{\includegraphics[width=0.17\textwidth]{ 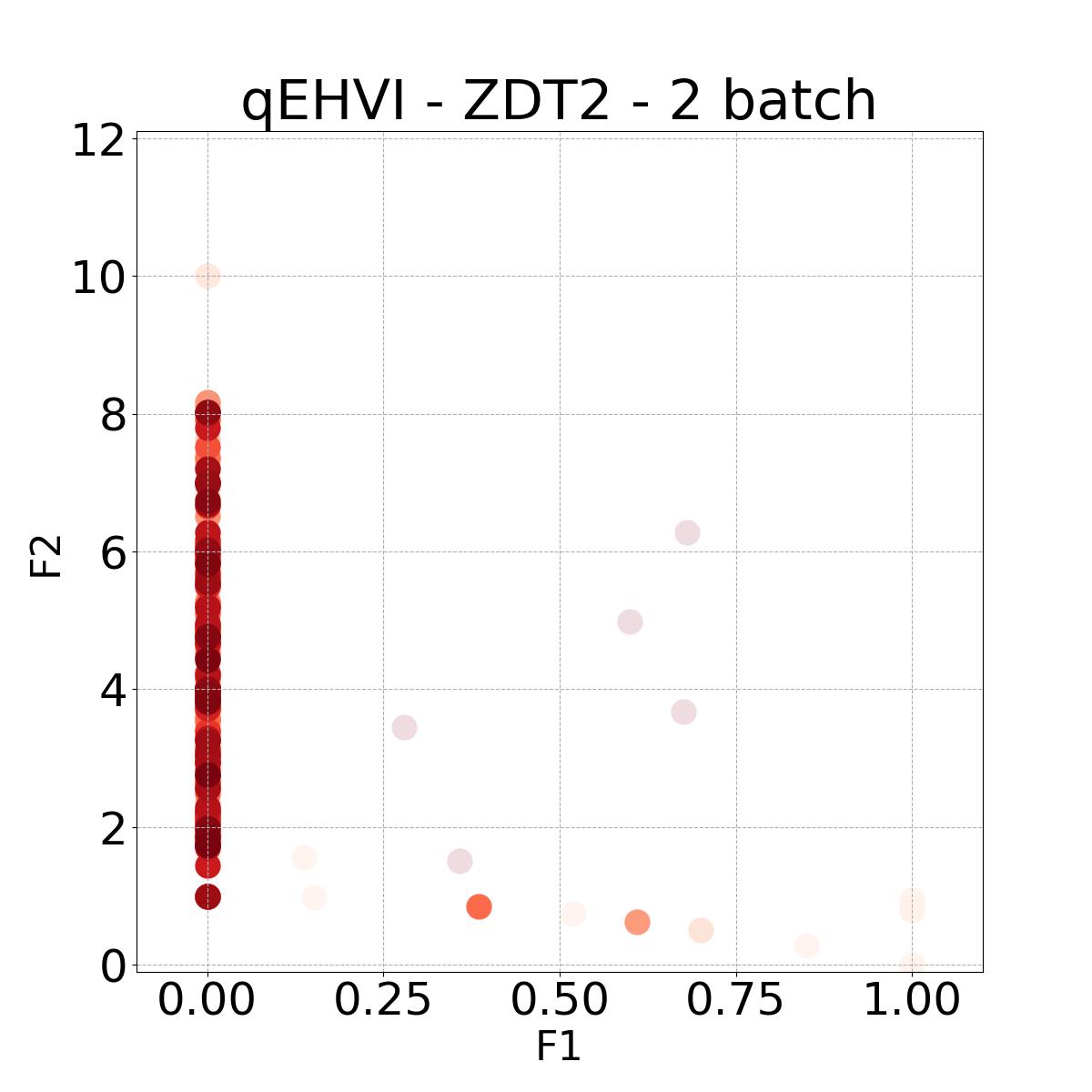}}\hspace{-3mm}
\end{subfigure}
\begin{subfigure}{\includegraphics[width=0.17\textwidth]{ 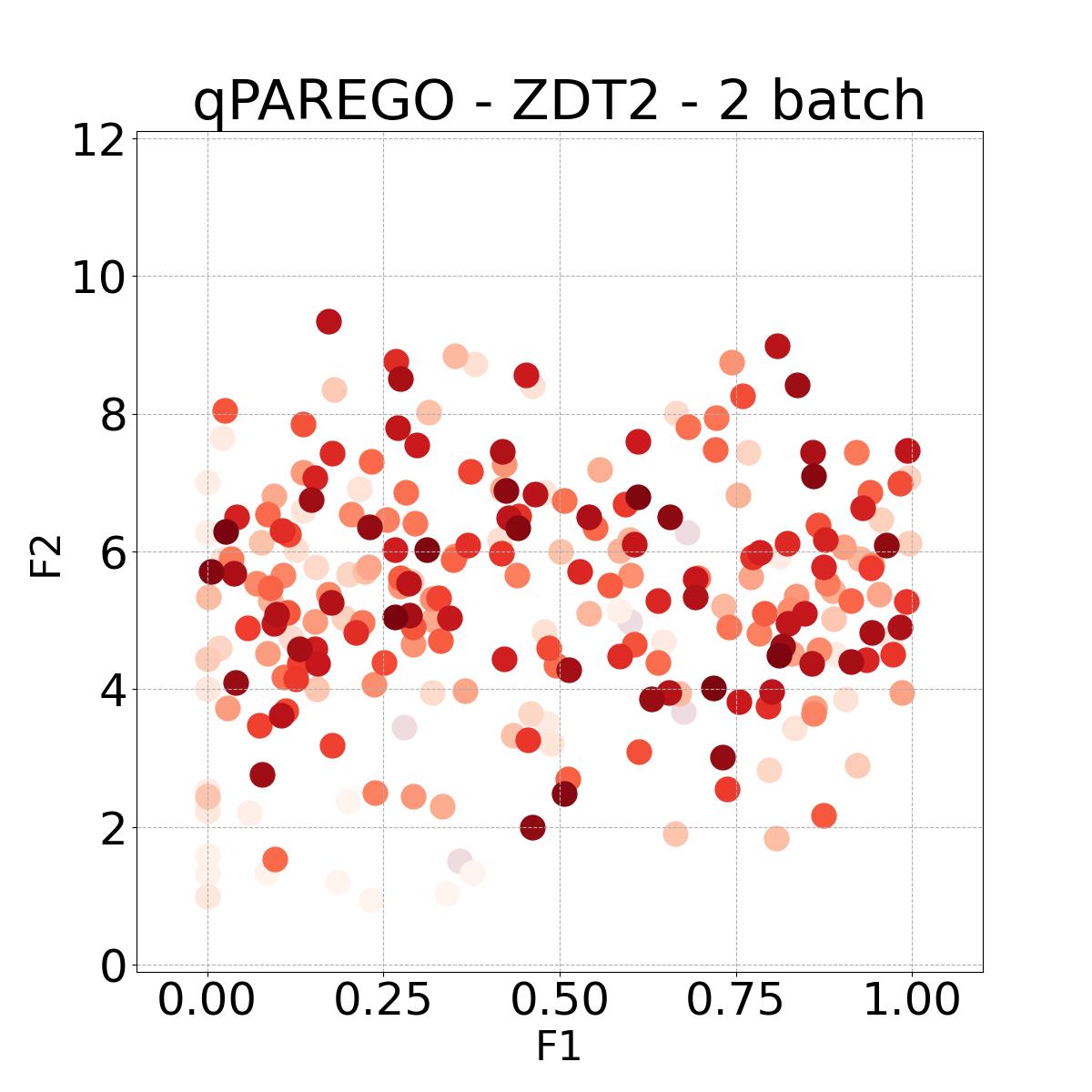}}\hspace{-3mm}
\end{subfigure}
\begin{subfigure}{\includegraphics[width=0.17\textwidth]{ 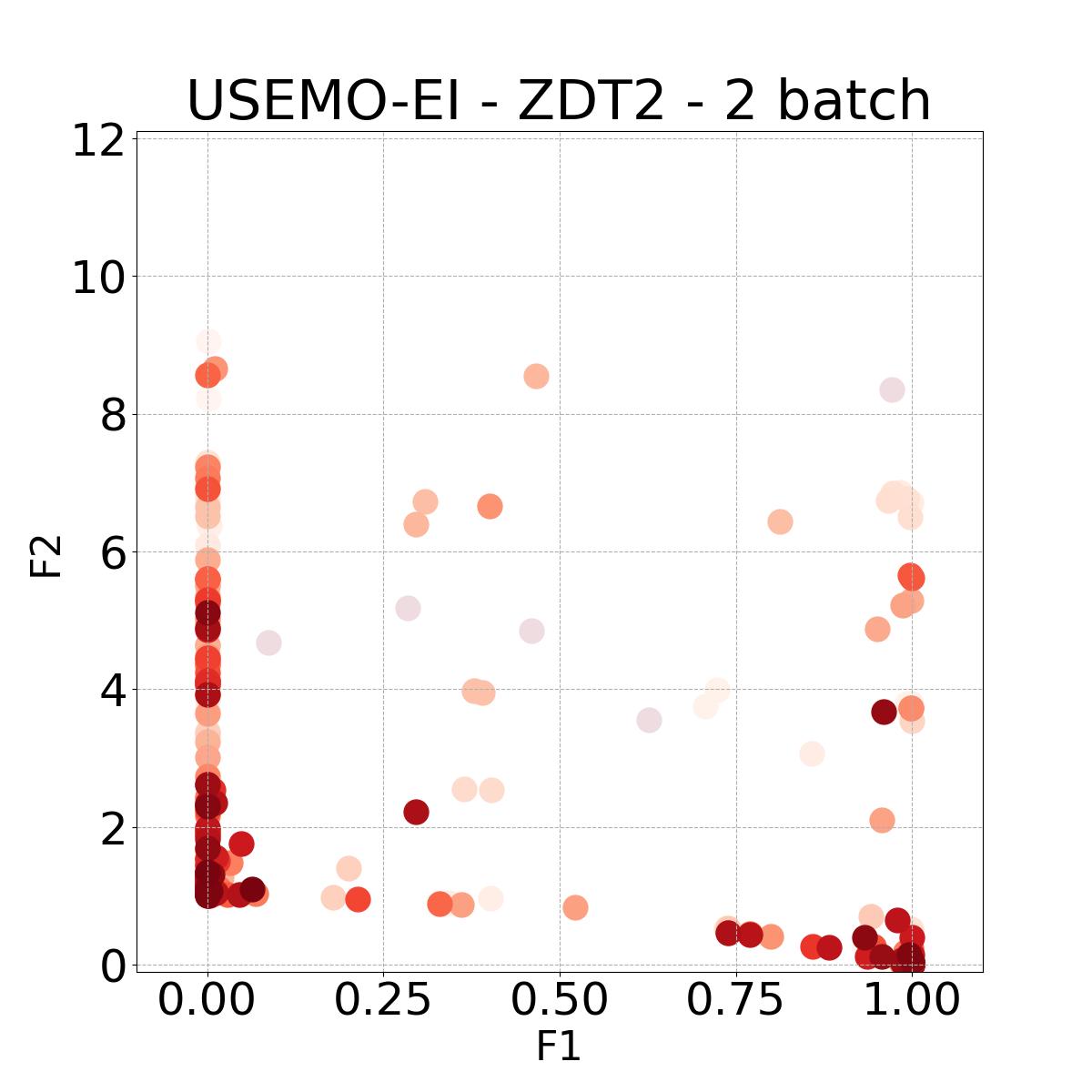}}\hspace{-3mm}\\
\end{subfigure}

\begin{subfigure}{\includegraphics[width=0.17\textwidth]{ 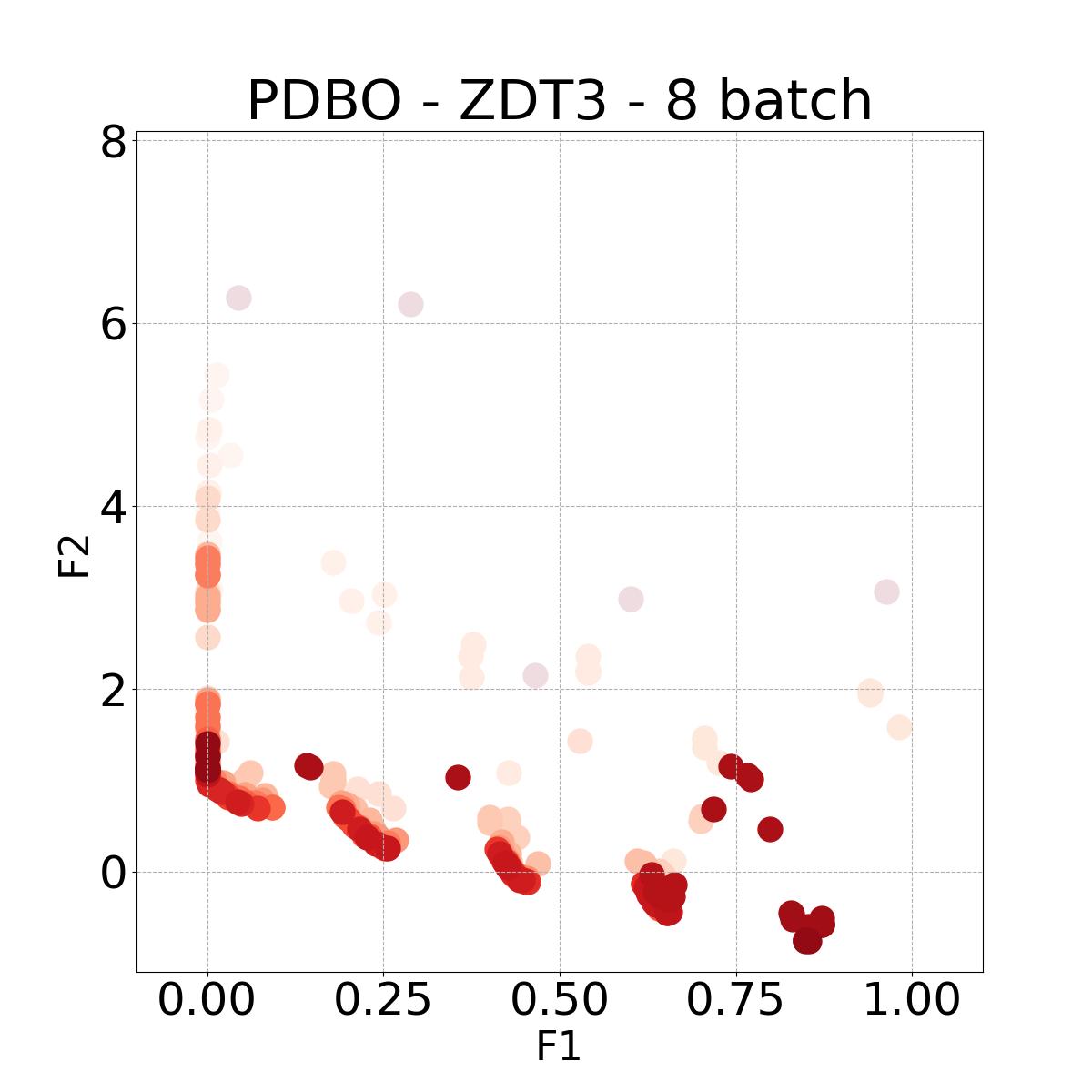}}\hspace{-3mm}
\end{subfigure}
\begin{subfigure}{\includegraphics[width=0.17\textwidth]{ 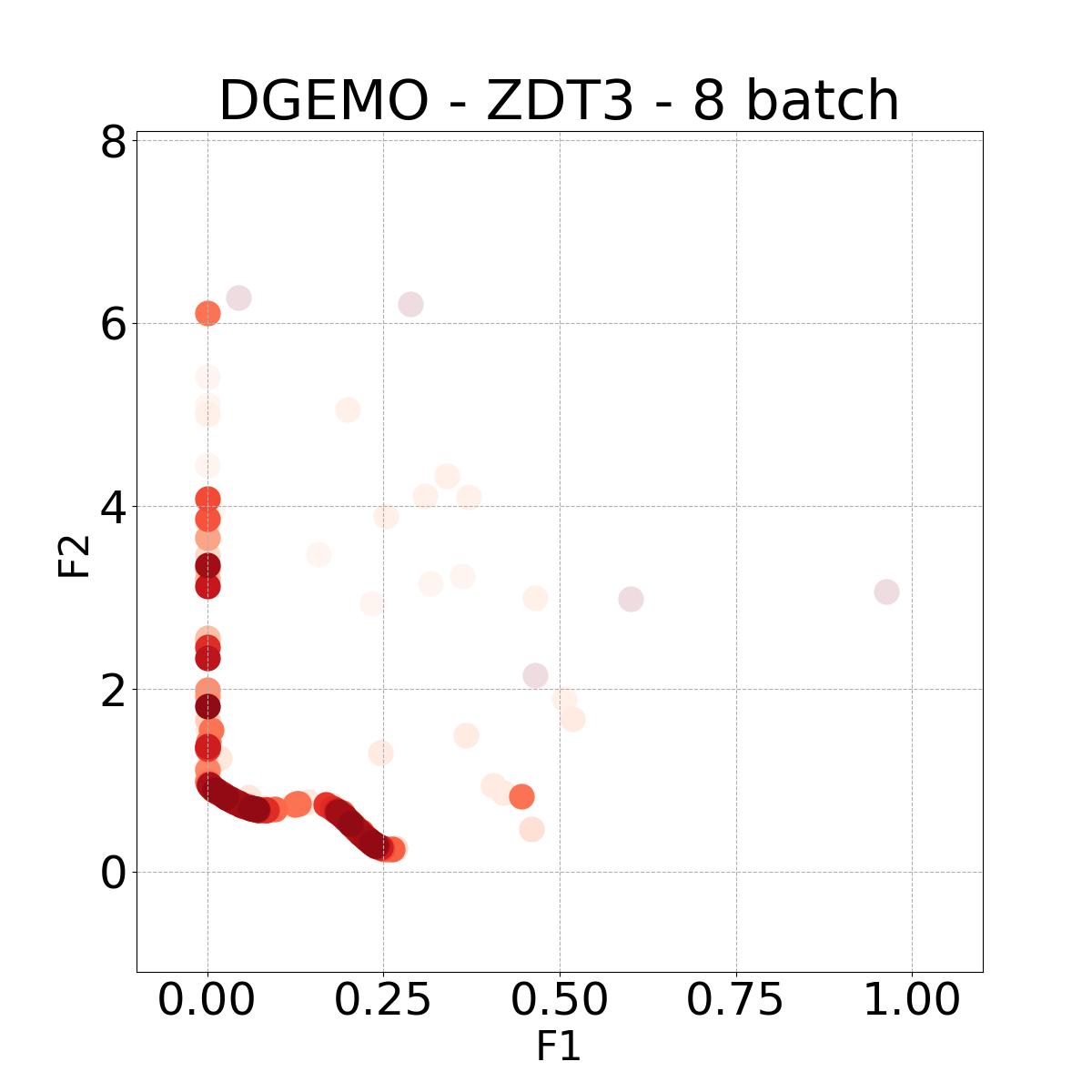}}\hspace{-3mm}
\end{subfigure}
\begin{subfigure}{\includegraphics[width=0.17\textwidth]{ 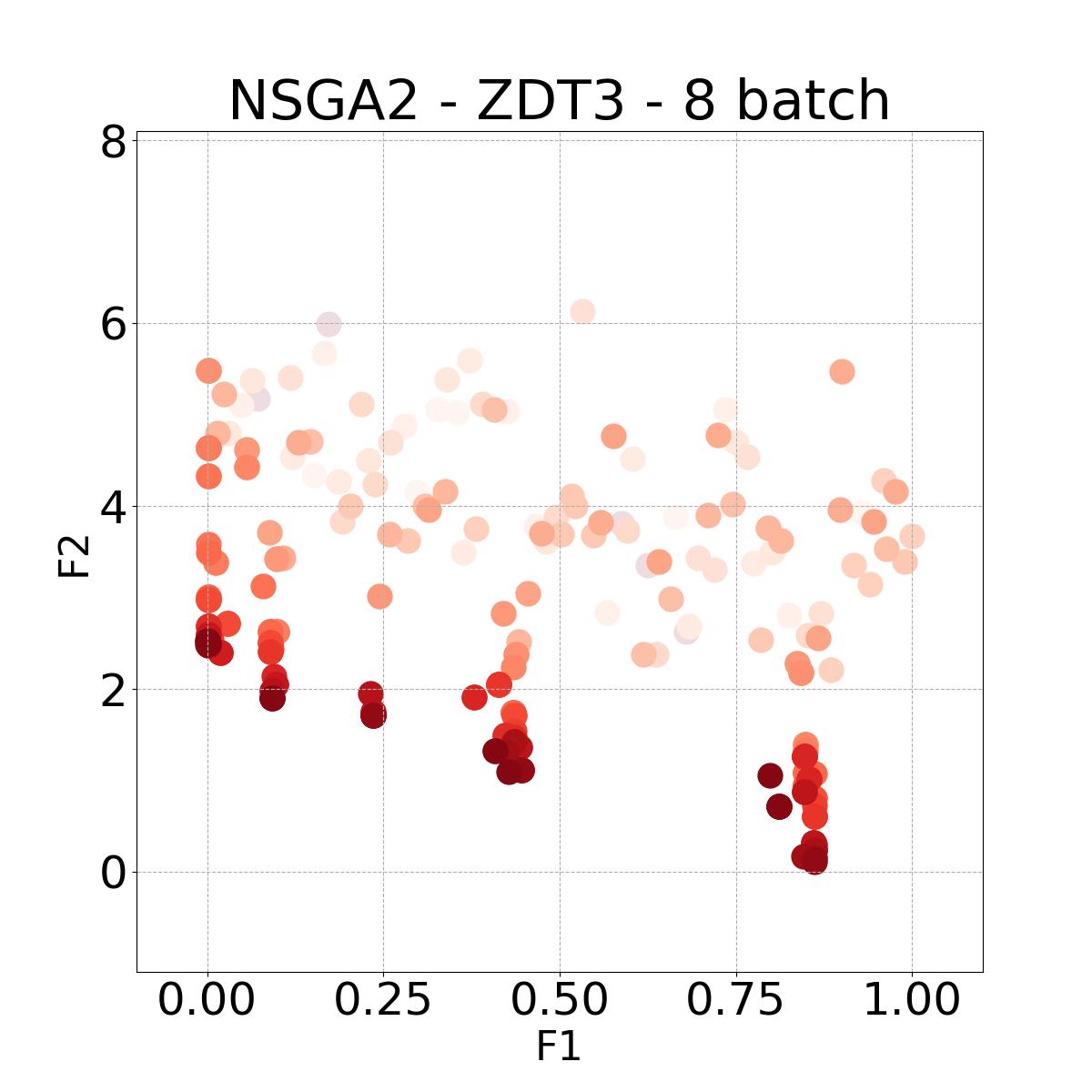}}\hspace{-3mm}
\end{subfigure}
\begin{subfigure}{\includegraphics[width=0.17\textwidth]{ 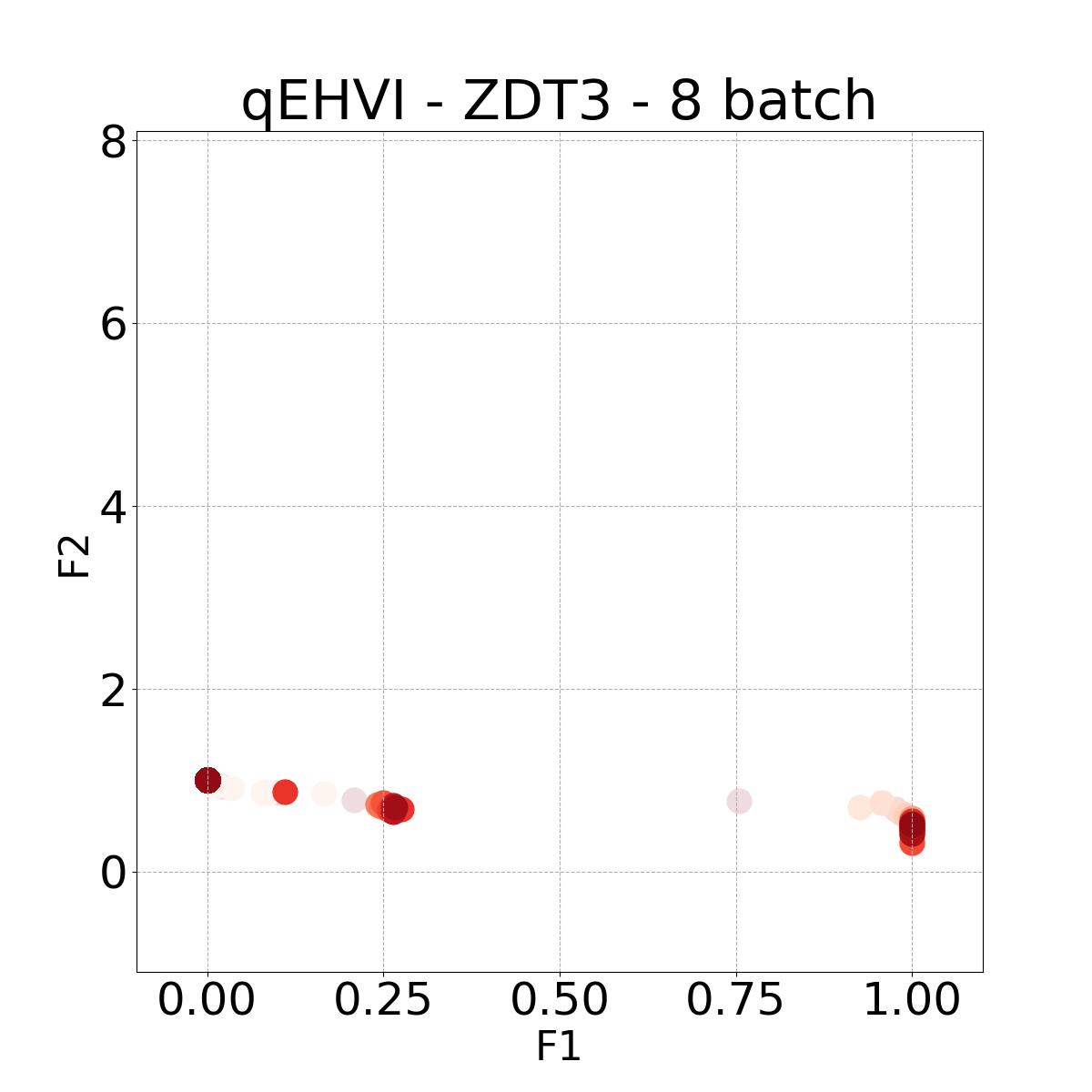}}\hspace{-3mm}
\end{subfigure}
\begin{subfigure}{\includegraphics[width=0.17\textwidth]{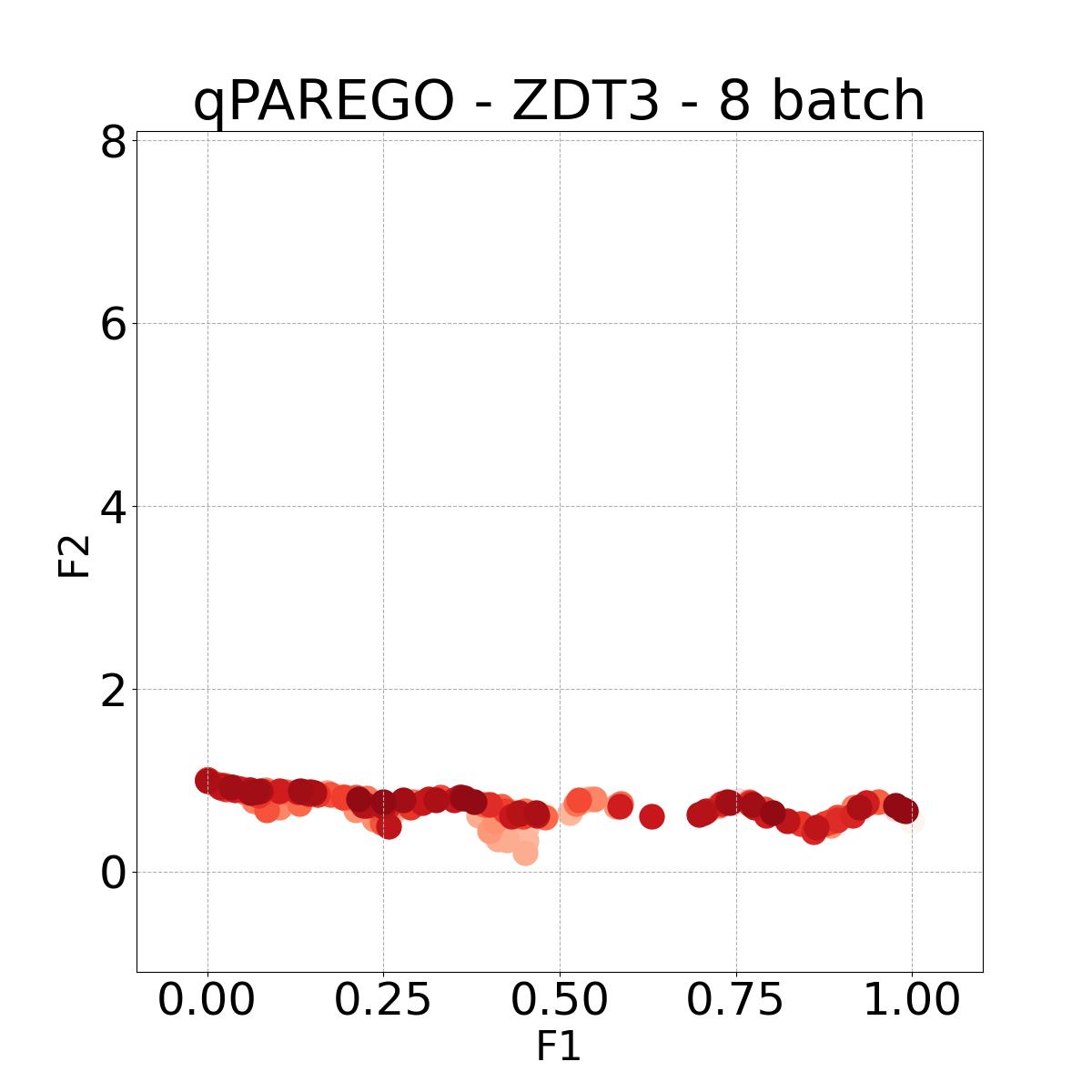}}\hspace{-3mm}
\end{subfigure}
\begin{subfigure}{\includegraphics[width=0.17\textwidth]{ 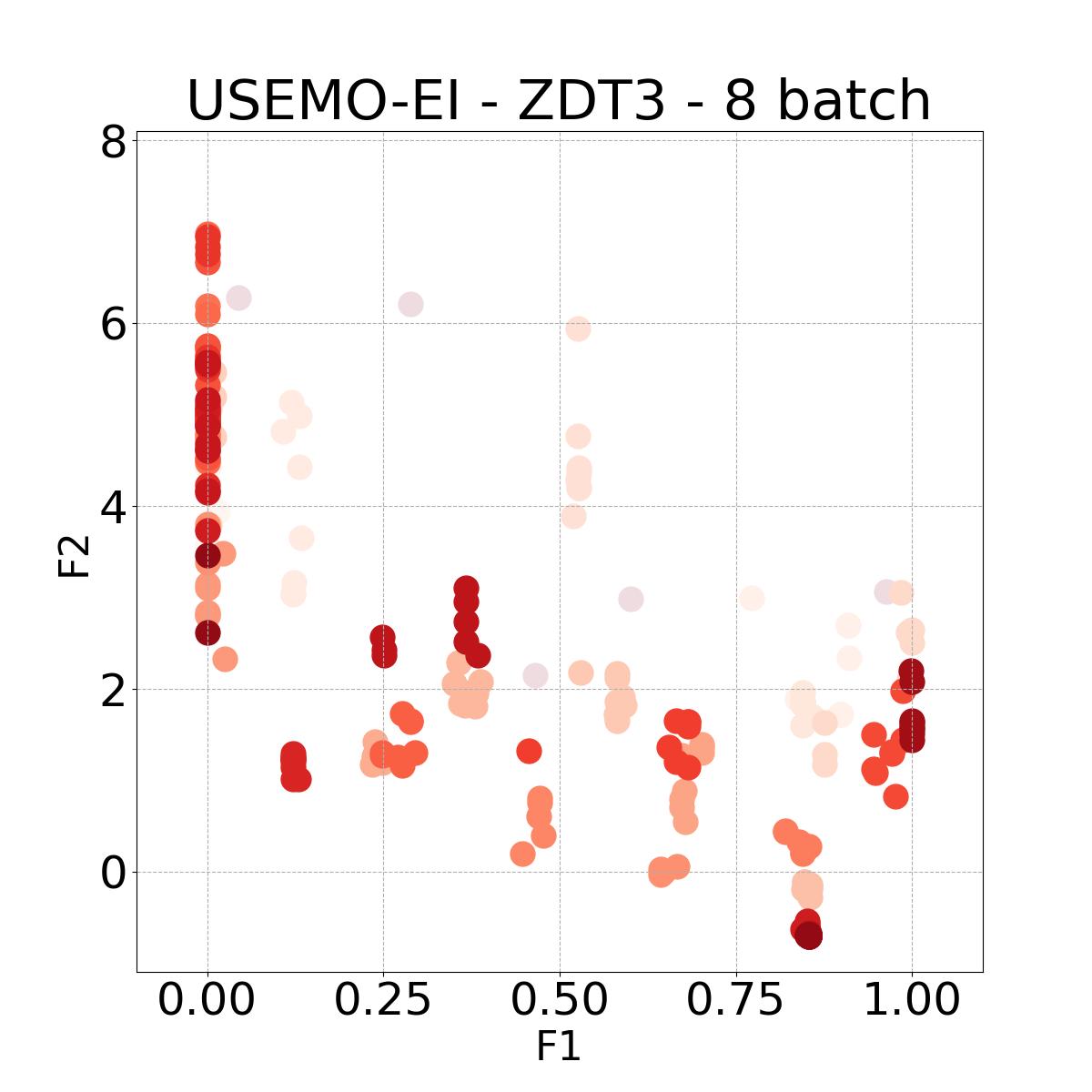}}\hspace{-3mm}\\
\end{subfigure}
\vspace{-0.4cm}
\caption{The Pareto front of the inputs selected by each algorithm. Points from later batches are depicted in darker colors.}
\label{appendix:PFs}
\end{figure*}

\subsection{Visualization of Diversity of Pareto front} \label{appendix-section:PF-scatterplots}
In Figure \ref{appendix:PFs}, we present illustrative examples of Pareto front figures achieved by each baseline. In each experiment, we deliberately vary the batch size to showcase the consistent diversity of results, irrespective of the batch size chosen. It's worth noting that several of benchmarks featured in this paper involve more than two objective functions, rendering it impractical to effectively visualize the Pareto front points in a coherent and legible manner.

\subsection{Acquisition Function Selection Statistics}\label{AF_stats}

Table \ref{tab:acq_histogram} contains statistics regarding the selection of each acquisition function by the MAB algorithm. It is important to clarify that the goal of the acquisition function selection multi-armed bandit algorithm is not solely to converge to a single acquisition function. Instead, its purpose is to dynamically choose the most appropriate acquisition function at each iteration to maximize overall performance improvement. Therefore, we believe that these statistics are interesting but do not necessarily lead to the conclusion that the most selected acquisition function is necessarily the best for the experiment. \looseness=-1

\subsection{Additional Real-World Experiments}\label{add_real_exp}

This section offers insights into two supplementary real-world experiments carried out to further demonstrate the effectiveness of PDBO when compared to existing baselines in terms of the hypervolume metric and diversity of the Pareto front. \looseness=-1

\vspace{1.0ex}

\noindent \textbf{Unmanned Aerial Vehicle (UAV) power system design:} This real-world benchmark \cite{belakaria2020machine} has two objective functions: minimizing the energy consumption and minimizing the UAV’s total mass. The search space has five input variables defined as the battery cells in series (ranging between 10 and 18), the battery cells in parallel (ranging between 16 and 70), the motor quantity (ranging between 6 and 10), the height of the stator structure (ranging between 80 and 260), and the motor stator winding turns (ranging between 100 and 550). We use a high-fidelity simulator to evaluate the two objective functions for any candidate input configuration. \looseness=-1

\vspace{1.0ex}

\noindent \textbf{SW-LLVM compiler settings optimization:} SW-LLVM \cite{siegmund2012predicting} is a compiler settings problem determined by d=10 compiler configurations. The goal of this experiment is to find a setting of the LLVM compiler that optimizes the memory footprint and performance on a given set of software programs.

In Figure \ref{fig:additional-real-world-hv} and Figure \ref{fig:additional-real-world-div}, we provide hypervolume and diverse Pareto front (DPF) results for these two additional real-world problems. 

\begin{table*}[ht]
\centering
\footnotesize
\setlength{\tabcolsep}{3pt}
\begin{tabularx}{\textwidth}{@{}p{1.5cm}*{16}{X}@{}}
\toprule
\centering
\multirow{3}{*}{\shortstack{Problem\\ Name}} & \multicolumn{4}{c}{2-Batch} & \multicolumn{4}{c}{4-Batch} & \multicolumn{4}{c}{8-Batch} & \multicolumn{4}{c}{16-Batch} \\ \cmidrule(lr){2-5} \cmidrule(lr){6-9} \cmidrule(lr){10-13} \cmidrule(lr){14-17}
      & TS & EI & UCB & ID & TS & EI & UCB & ID & TS & EI & UCB & ID & TS & EI & UCB & ID \\ \midrule
ZDT1  & 12$\pm$12 & 9$\pm$7 & 1$\pm$1 & 75$\pm$15 & 13$\pm$18 & 14$\pm$12 & 4$\pm$3 & 67$\pm$17 & 16$\pm$28 & 10$\pm$12 & 5$\pm$5 & 67$\pm$30 & 10$\pm$21 & 5$\pm$7 & 11$\pm$10 & 72$\pm$21 \\
ZDT2  & 23$\pm$17 & 22$\pm$11 & 19$\pm$10 & 33$\pm$18 & 17$\pm$18 & 25$\pm$18 & 16$\pm$17 & 40$\pm$30 & 9$\pm$19 & 31$\pm$27 & 21$\pm$24 & 39$\pm$27 & 7$\pm$19 & 40$\pm$38 & 24$\pm$29 & 28$\pm$34 \\
ZDT3  & 1$\pm$5 & 8$\pm$6 & 4$\pm$5 & 85$\pm$10 & 1$\pm$2 & 12$\pm$11 & 2$\pm$3 & 83$\pm$12 & 1$\pm$2 & 16$\pm$14 & 7$\pm$8 & 74$\pm$17 & 1$\pm$2 & 28$\pm$33 & 7$\pm$3 & 62$\pm$35 \\
DTLZ1 & 13$\pm$10 & 13$\pm$9 &  7$\pm$5 & 65$\pm$16 & 11$\pm$12 & 13$\pm$11 & 7$\pm$6 & 67$\pm$17 & 10$\pm$13 & 12$\pm$17 & 10$\pm$10 & 66$\pm$27 & 14$\pm$24 & 20$\pm$26 & 16$\pm$14 & 48$\pm$35 \\
DTLZ3 & 7$\pm$6 & 9$\pm$11 & 3$\pm$4 & 79$\pm$15 & 16$\pm$19 & 9$\pm$10 & 5$\pm$6 & 69$\pm$23 & 13$\pm$20 & 17$\pm$24 & 9$\pm$10 & 60$\pm$32 & 26$\pm$36 & 15$\pm$32 & 21$\pm$24 & 37$\pm$38 \\
DTLZ5 & 4$\pm$3 & 85$\pm$8 & 3$\pm$3 & 5$\pm$4 & 5$\pm$6 & 82$\pm$8 & 3$\pm$2 & 8$\pm$6 & 6$\pm$8 & 68$\pm$20 & 10$\pm$10 & 14$\pm$19  & 7$\pm$17 & 61$\pm$29 & 14$\pm$9 & 17$\pm$22 \\
Gear Train Design & 1$\pm$4 & 2$\pm$3 & 1$\pm$1 & 94$\pm$5 & 1$\pm$4 & 2$\pm$3 & 2$\pm$2 & 93$\pm$6 & 2$\pm$7 & 5$\pm$6 & 3$\pm$1 & 88$\pm$8 & 2$\pm$6 & 10$\pm$15 & 8$\pm$8 & 79$\pm$18 \\
\bottomrule
\end{tabularx}
\vspace{-0.15cm} 
\caption{The percentage of each acquisition function used for each problem and batch size, averaged among 25 separate runs}
\label{tab:acq_histogram}
\end{table*}

\begin{figure*}[h!]
\begin{subfigure}{\includegraphics[width=0.24\textwidth]{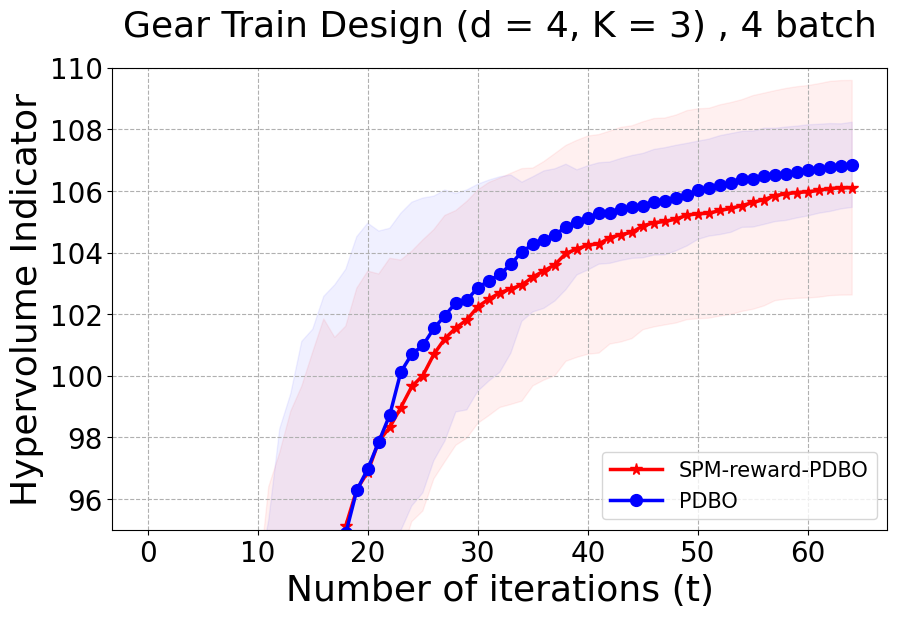}}
\end{subfigure}
\begin{subfigure}{\includegraphics[width=0.24\textwidth]{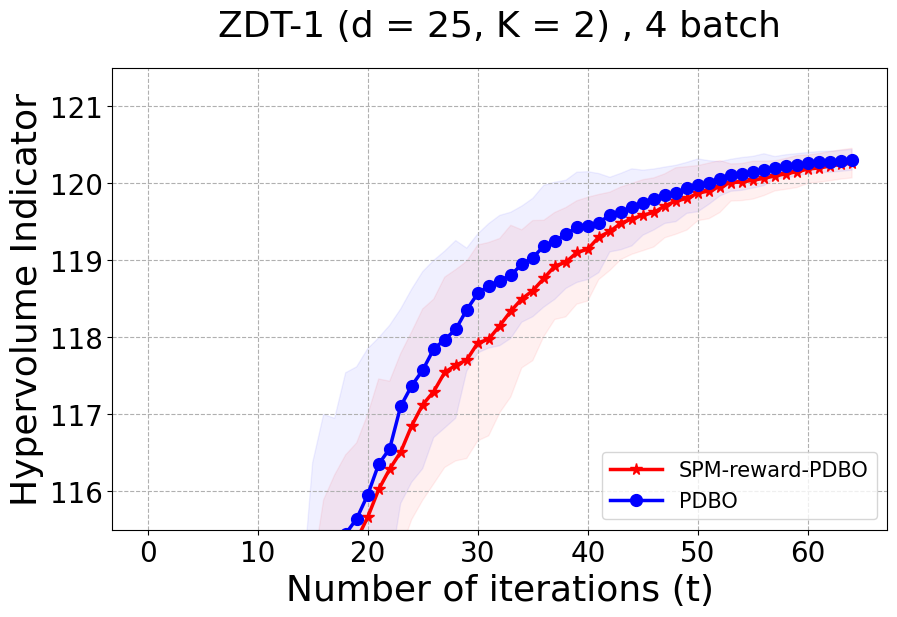}}
\end{subfigure}
\begin{subfigure}{\includegraphics[width=0.24\textwidth]{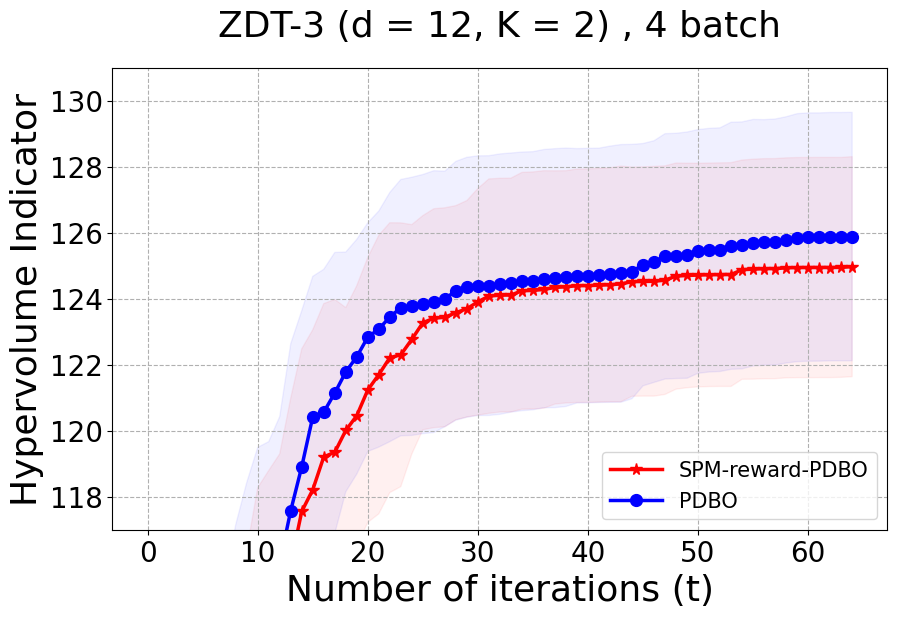}}
\end{subfigure}
\begin{subfigure}{\includegraphics[width=0.24\textwidth]{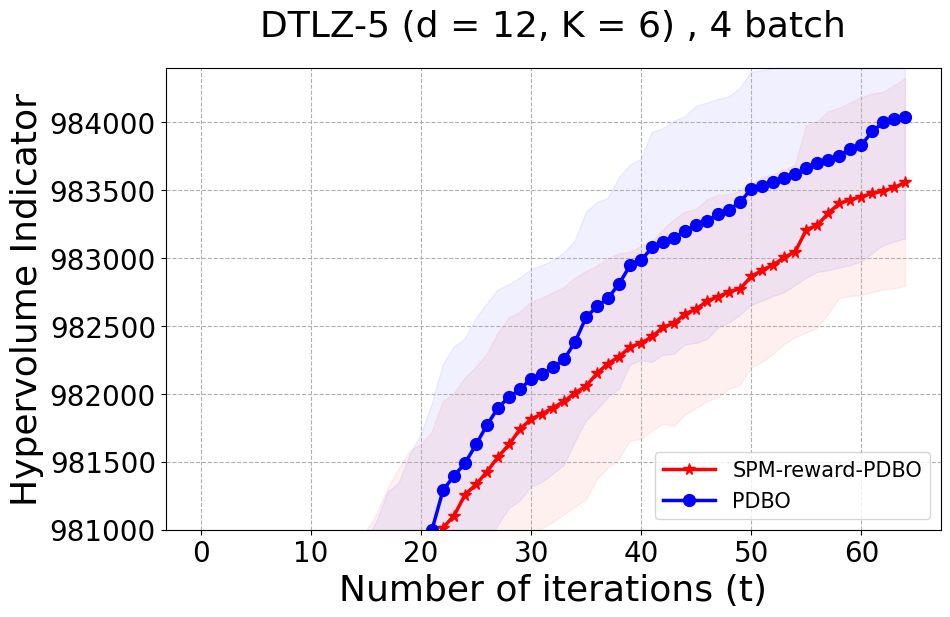}}
\end{subfigure}
\caption{Comparing PDBO against PDBO with the sum of posterior means (SPM) as the reward.}
\label{fig:pdbo_vs_spm}
\end{figure*}

\subsection{Ablation Studies} \label{ablations}

\subsubsection{Merits of DPP-Based Batch Selection compared to Greedy HVC selection.}

As detailed in Section 4.2 of our paper, we have established that Hypervolume alone may not effectively capture the diversity of individual data points. For a more nuanced assessment of each new point in relation to previously evaluated designs, we turn to Individual Hypervolume Contributions (HVCs). However, selecting points solely based on the highest HVC values presents several potential issues: 1. It may lead to exploitative behavior, as it relies solely on predictive mean information; 2. It may undermine diversity considerations, especially when multiple points exhibit closely ranked HVC values; 3. It may overlook diversity in the input domain, as DPP's utilization of a kernel allows for a more comprehensive representation of diversity across the input space. Figure \ref{fig:rebuttal_hvc_vs_dpp} presents the results of our ablation study on both a real-world problem and a synthetic problem offering a comparative evaluation of batch selection methods: one based on the highest hypervolume contribution (HVC) and the other based on our proposed multi-objective DPP-based batch selection approach.

\begin{figure}[ht!]
\centering
\includegraphics[width=0.3\textwidth]{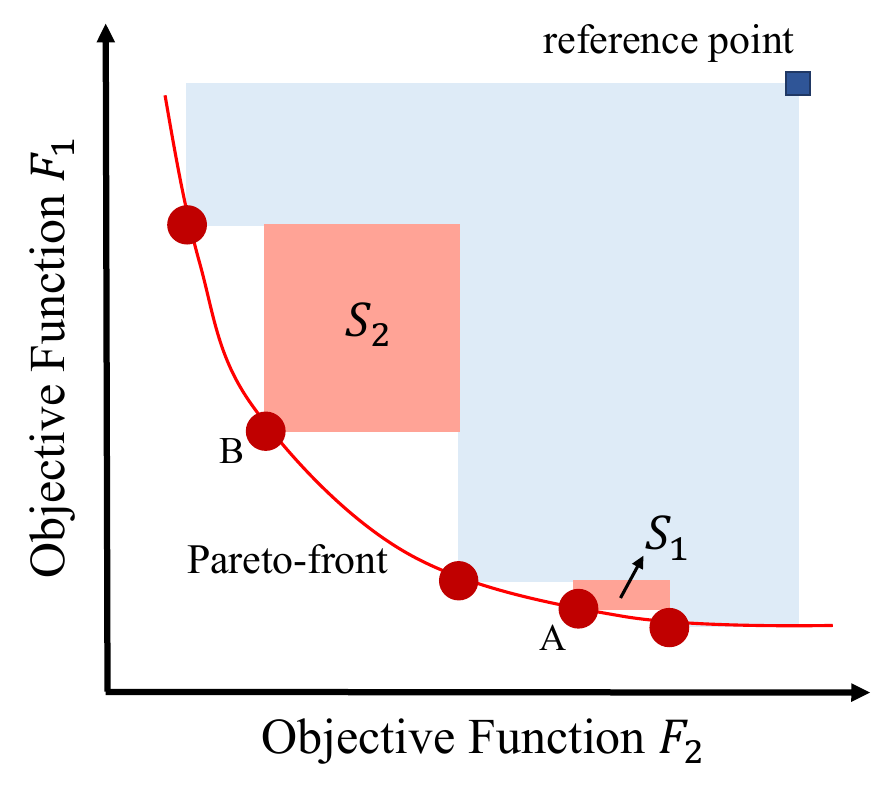}
\caption{Sections $S_1$ and $S_2$ exhibit the specific hypervolume contributions of points A and B, respectively. This figure indicates a positive correlation between higher individual hypervolume contributions and increased diversity among the points.}\label{fig:hv-contribution}
\end{figure}

\begin{figure*}[h!]
     \centering
     \begin{subfigure}
         \centering
         \includegraphics[width=0.23\textwidth]{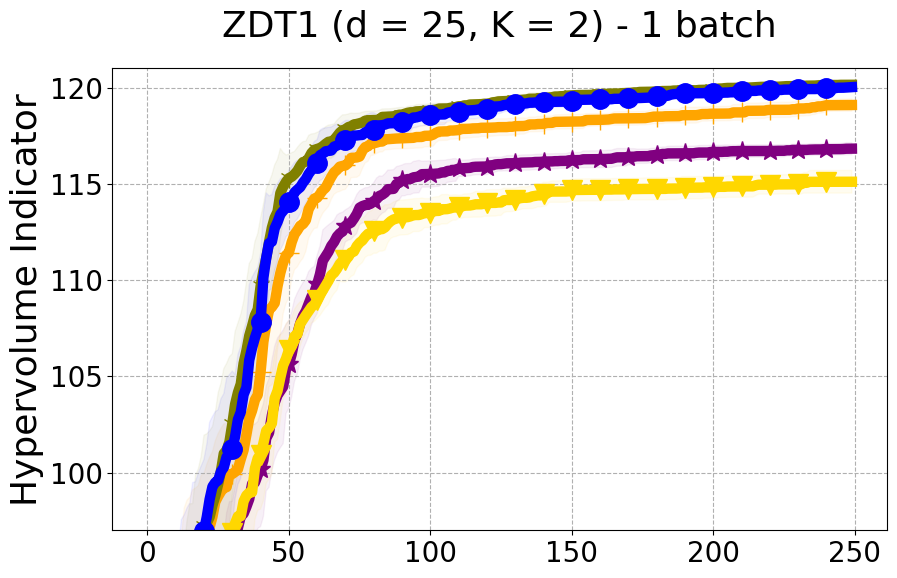}
     \end{subfigure}
     \hfill
     \begin{subfigure}
         \centering
         \includegraphics[width=0.23\textwidth]{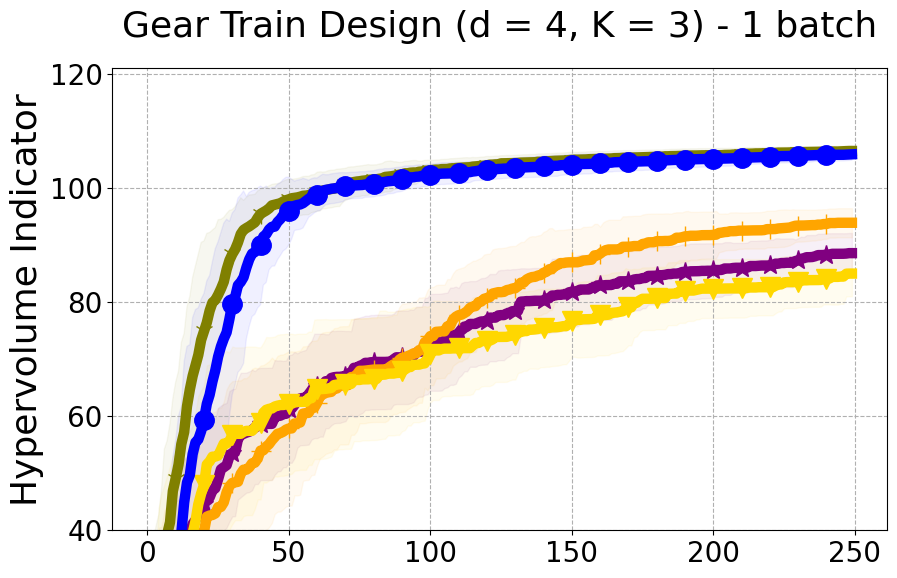}
     \end{subfigure}
     \hfill
     \begin{subfigure}
         \centering
         \includegraphics[width=0.23\textwidth]{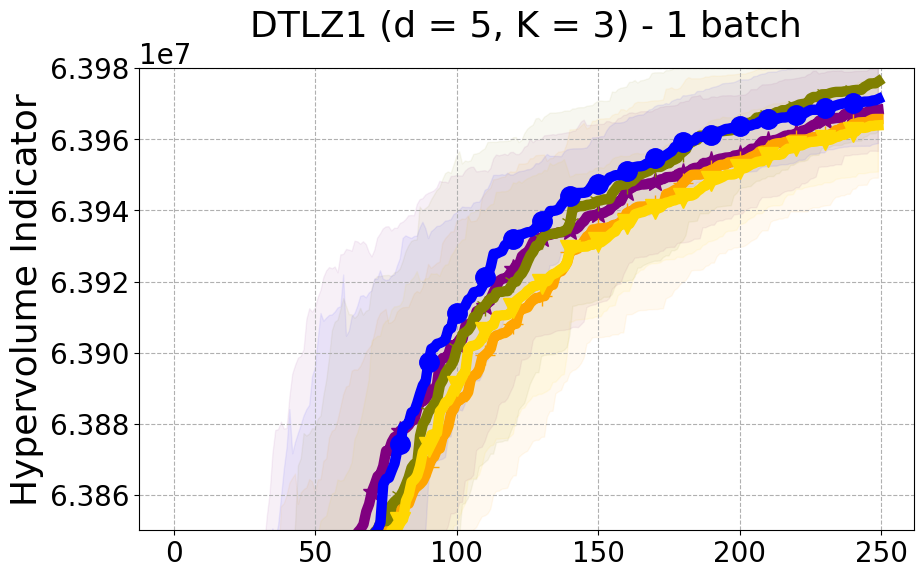}
     \end{subfigure}
     \centering
     \begin{subfigure}
         \centering
         \includegraphics[width=0.23\textwidth]{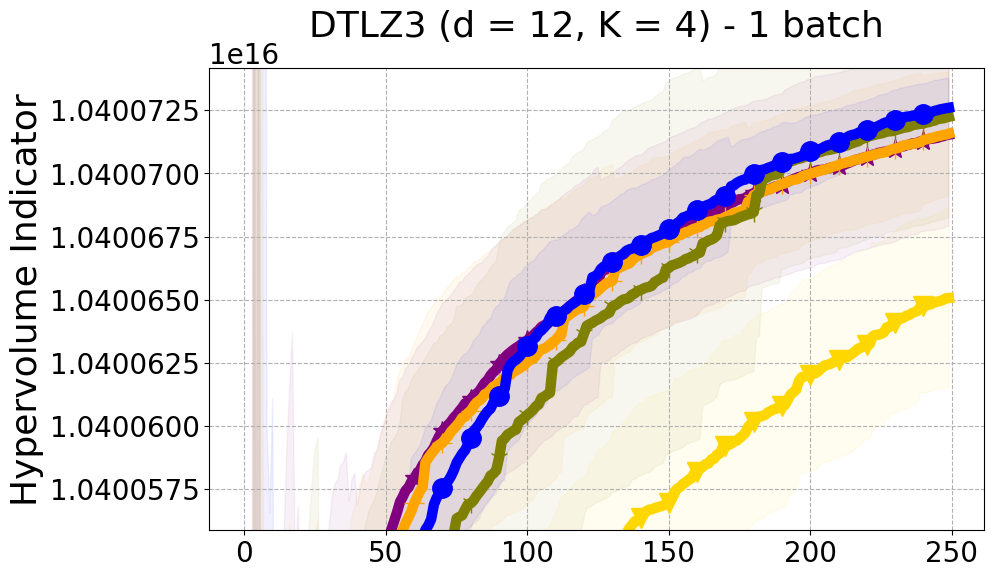}
     \end{subfigure}
     \hfill
     \begin{subfigure}
         \centering
         \includegraphics[width=0.24\textwidth]{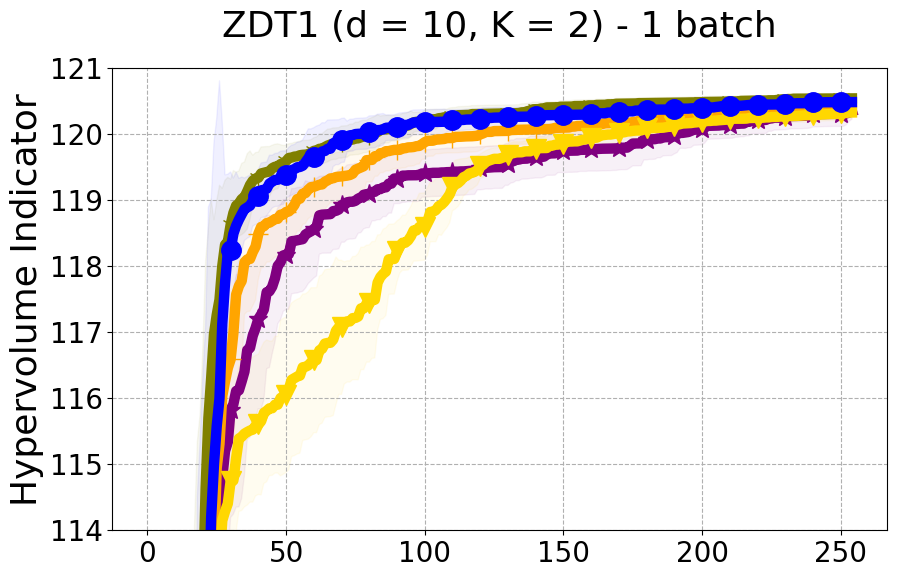}
     \end{subfigure}
     \hfill
     \begin{subfigure}
         \centering
         \includegraphics[width=0.23\textwidth]{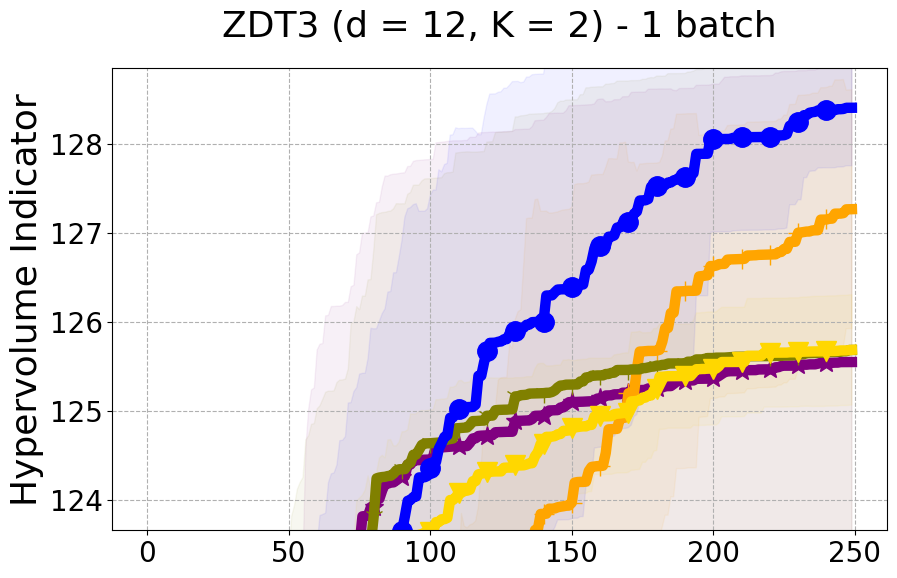}
     \end{subfigure}
     \hfill
     \begin{subfigure}
         \centering
         \includegraphics[width=0.23\textwidth]{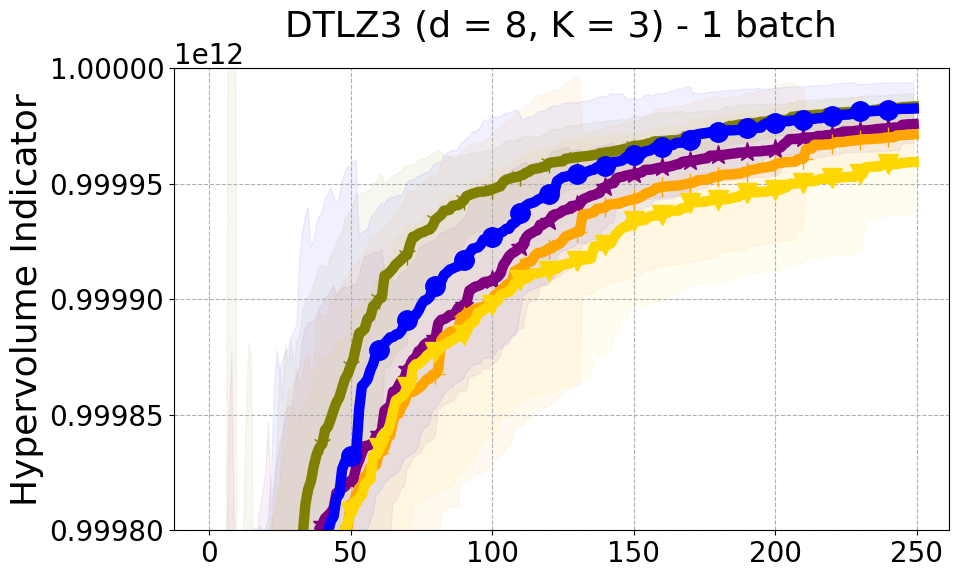}
     \end{subfigure}
      \centering
     \begin{subfigure}
         \centering
         \includegraphics[width=0.23\textwidth]{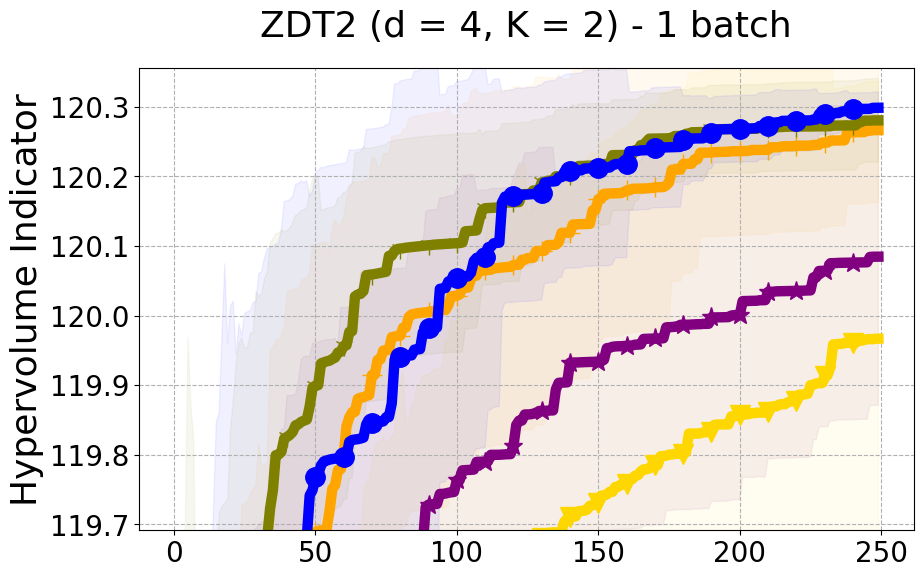}
     \end{subfigure}
    \begin{subfigure}
         \centering
         \includegraphics[width=0.98\textwidth]{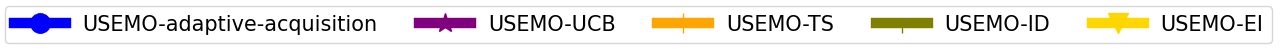}
     \end{subfigure}
     \vspace{-0.3cm}
        \caption{Acquisition functions ablation study results}
    \label{fig:acquisitions-ablation}
\end{figure*}

\subsection{Additional Diversity Metrics} \label{section:appendix-IGD}

We present examples of our results evaluated using the Inverted Generational Distance (IGD) metric \cite{coello2004study}. However, we contend that IGD may not be a suitable metric for assessing diversity. This is because IGD is heavily influenced by the number of points discovered, which can lead to misleading results. As an example, if one baseline uncovers only two points very close to the ideal Pareto front, it may achieve a heavily more favorable IGD value than a method finding 50 diverse points slightly further away. Clearly, the set of two points is not inherently more diverse. Furthermore, the IGD metric relies on a predefined ground truth Pareto front, which is often unavailable in black-box settings, rendering its use less meaningful. Moreover, the choice of the ideal Pareto front shape can impact whether IGD primarily reflects diversity or performance. Recent work \cite{pierrot2022multi} proposed an approach to evaluate diversity for cheap multi-objective problems. However, the proposed approach requires the user to manually define space descriptors. Depending on the descriptors' definition, the results provided by the metric will change leading to potentially inconsistent results, especially in the black-box expensive setting where it is harder to define the descriptors. To our understanding, there is no clear/known strategy to define the descriptors.

As additional metrics, we employ the Pareto Hypervolume (PHV) to evaluate the quality of the uncovered optimal Pareto front and use the Diversity of Pareto Front (DPF) metric to assess the algorithm's ability to select diverse designs. We believe that incorporating DPF alongside PHV provides a more comprehensive evaluation of the algorithm since these metrics offer complementary insights into different aspects of performance.

Here, we include two new metrics. First, we calculate the DPF value exclusively on the optimal discovered Pareto frontier as shown in Figure \ref{fig:appendix-dpf-on-pf}. Additionally, we include experiments where the IGD metric is computed as shown in Figure \ref{fig:appendix-IGD-results}. \looseness=-1

\subsection{Experimental Details}\label{subsect:appendix-experimental-details}
This section includes supplementary implementation elements and experimental details.

Table \ref{tab:problem-size} provides information regarding the dimensions, objectives, and reference points utilized for each benchmark problem in the primary hypervolume and DPF results. Additionally, Table \ref{tab:tab-gp} presents the hyperparameters of the GP surrogate model. In all experiments, a consistent definition of the GP surrogate model is employed, featuring a zero mean function and an anisotropic Matern 5/2 kernel.

\begin{table}[ht]
\centering
\begin{tabular}{c        c        c        c} 
 \toprule
 Problem Name & d & K & reference point  \\ 
 \midrule
 ZDT-1 & 25 & 2  & [11.0, 11.0]\\ 
 \midrule
 ZDT-2 & 4 & 2  & [11.0, 11.0]  \\
 \midrule
  ZDT-3 & 12 & 2  & [11.0, 11.0] \\ 
  \midrule
 SW-LLVM & 10 & 2  & [1000.0, 500.0]   \\  
 \midrule
 UAV & 6 & 2  & [1e9, 1000.0]   \\  
 \midrule
  Gear Train Design & 4 & 3 & [6.6764, 59.0, 0.4633]\\ 
 \midrule
 DTLZ-1 & 10 & 4  & [400.0, ..., 400.0]   \\  
 \midrule
  DTLZ-3 & 9 & 4 & [10000.0, ..., 10000.0]  \\ 
 \midrule
  DTLZ-5 & 12 & 6 & [10.0, ..., 10.0]  \\ 
  \bottomrule
\end{tabular}
\caption{The problem details for hyporvolume and diversity experiments}
\label{tab:problem-size}
\end{table}

\begin{table}[ht]
\centering
\begin{tabular}{c           c} 
 \toprule 
 Hyperparameter name & Hayperparameter value  \\ 
 \midrule
 initial l & $(1, \cdots, 1) \in \mathcal{R}^d$  \\  
 \midrule
 l range &  $(\sqrt{10^{-3}}, \sqrt{10^3})$  \\
 \midrule
  initial $\sigma_f$ & 1  \\ 
 \midrule
 $\sigma_f$ range & $(\sqrt{10^{-3}}, \sqrt{10^3})$   \\ 
 \midrule
  initial $\sigma_n$  & $10^{-2}$  \\ 
 \midrule
  $\sigma_n$ range & $10^{-2}$ \\ 
\bottomrule
\end{tabular}
\caption{Hyperparameters used in the GP} 
\label{tab:tab-gp}
\end{table}

The NSGA-II multi-objective optimization (MOO) solver adopts simulated binary crossover \cite{deb1995simulated} with a parameter of $\mu_c = 15$, and polynomial mutation \cite{deb1996combined} with a parameter of $\mu_m = 20$. It utilizes a population size of 100 and runs for a total of 200 generations to explore the Pareto front of acquisition function values. The initial population is derived from the best current samples identified using non-dominated sorting \cite{deb2002fast}. As for the baseline NSGA-II algorithm, it employs the same crossover and mutation parameters ($\mu_c$ and $\mu_m$), but the population size is adjusted to match the batch size, and the number of generations aligns with the number of algorithm iterations. These hyperparameters are as mentioned in \cite{konakovic2020diversity}

Table \ref{tab:runtimes} presents a comprehensive comparison of the runtime performance for all the methods discussed. The qEHVI and qPAREGO techniques have the capability to leverage GPUs for their execution. However, considering the limited accessibility of GPUs for many BO problems, we also provide an analysis of these methods when executed on CPUs. It should be noted that the runtime of qEHVI-CPU and qPAREGO-CPU is notably longer when compared to their counterparts. Therefore, for these two methods, we report the average runtime based on 4 runs, whereas for all other experiments, we report the runtime based on 25 runs. The experiments utilizing GPUs were conducted using four NVIDIA Quadro RTX 6000 GPUs, while the CPU experiments were performed on an AMD EPYC 7451 24-Core Processor.

\subsubsection{Practical Issues and Limitations of Existing Baselines}\label{issues_ehvi_dgemo}
The two state-of-the-art methods are DGEMO and qEHVI. The both have practical issues and limitations leading to their inapplicability in some experiments. Below we provide a detailed explanation of these limitations:
\begin{itemize}
    \item DGEMO: One algorithmic step of the DGEMO algorithm is to divide the points in the Pareto front into different regions (line 9, Algorithm 1 in \cite{konakovic2020diversity}). Depending on the number of objective functions, the algorithm used to divide the Pareto front into regions would be different. In the implementation of DGEMO, only the cases of 2 and 3 functions were implemented, any case beyond 3 objectives is not handled. So this limitation of the DGEMO approach prevents it from being applicable to any problem. The inability to handle more than 3 objectives was not explicitly discussed in the paper. Therefore, no solution was suggested for it and no implementation was provided for this case. Finding the appropriate algorithm for the splitting approach when the number of objectives exceeds 3 and implementing it is not straightforward, and a research investigation in itself.

\item qEHVI: qEHVI algorithm has the capability to use auto differentiation and batch selection parallelization. However, this leads to high memory consumption when the batch sizes increase or when the input dimension or the number of objective functions increases, leading to an out-of-memory issue even on (GPU) machines with 40 GB of memory. Unfortunately, we do not have access to higher memory machines. We believe that this is a limitation of the approach since most scientists and practitioners using BO algorithms might not necessarily have access to very high-memory GPU machines.

\end{itemize}

\subsection{Limitations}\label{limitations}
Our approach overcomes the adaptive acquisition function selection problem via a MAB approach. However, our algorithm requires the proper selection of the MAB algorithm hyperparameters, namely, the decay factor  and the probability of selection hyperparameter . We set these hyperparameters manually based on a previous study of their suitable value \cite{hoffman2011portfolio}. We leave the study of the adaptive selection of hyperparameters to future work. Another limitation of our work is its inability to handle high-dimensional search space. However, it is important to note that this work is not proposed in the context of high-dimensional BO, and it can be synergistically combined with effective high-dimensional BO approaches such as \cite{eriksson2019scalable} to handle high-dimensional search spaces.

\begin{figure*}[ht]
    \centering
    \begin{subfigure}
    \centering
     \begin{subfigure}
         \centering
         \includegraphics[width=0.24\textwidth]{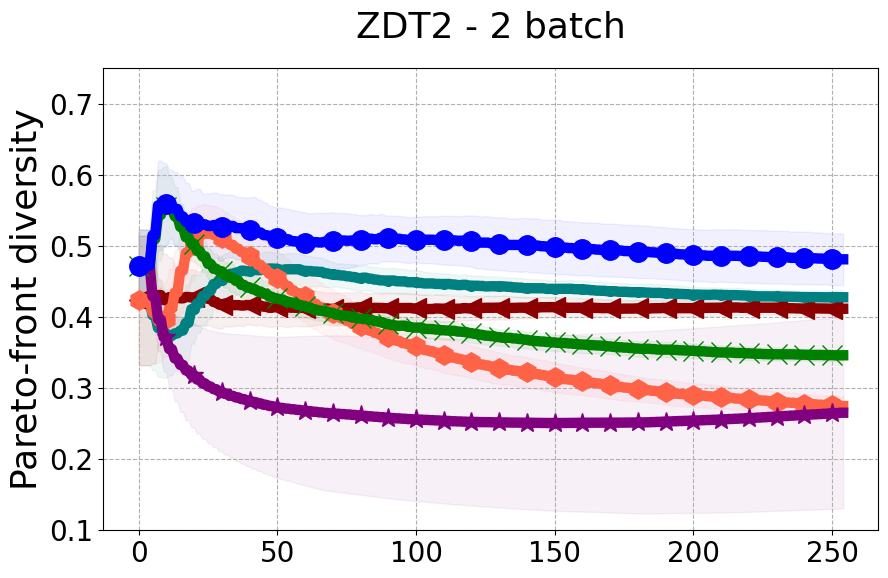}
     \end{subfigure}
     \hfill
     \begin{subfigure}
         \centering
         \includegraphics[width=0.24\textwidth]{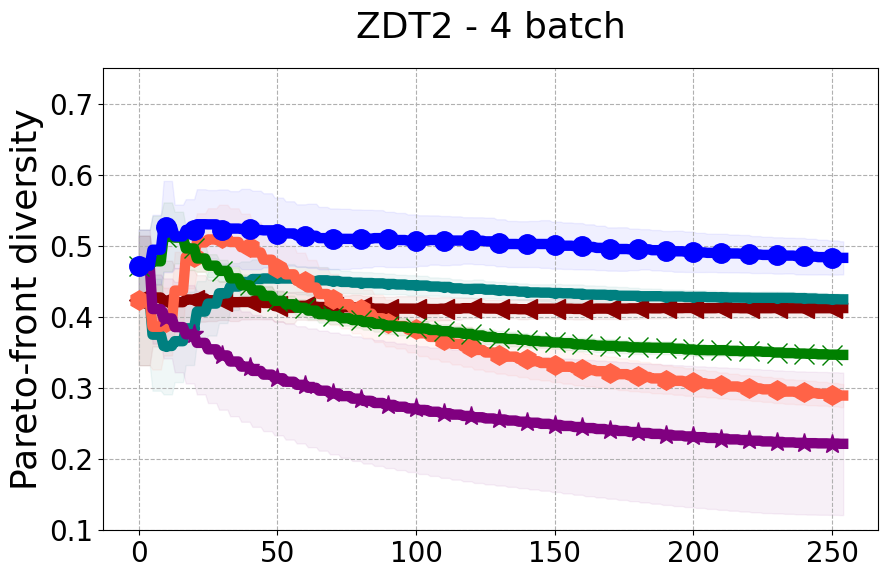}
     \end{subfigure}
     \hfill
     \begin{subfigure}
         \centering
         \includegraphics[width=0.24\textwidth]{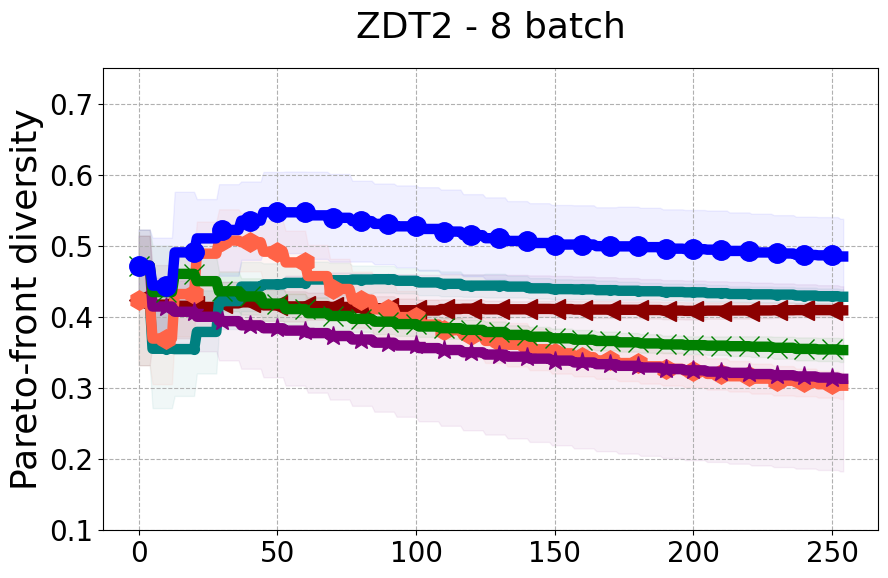}
     \end{subfigure}
     \hfill
     \begin{subfigure}
         \centering
         \includegraphics[width=0.24\textwidth]{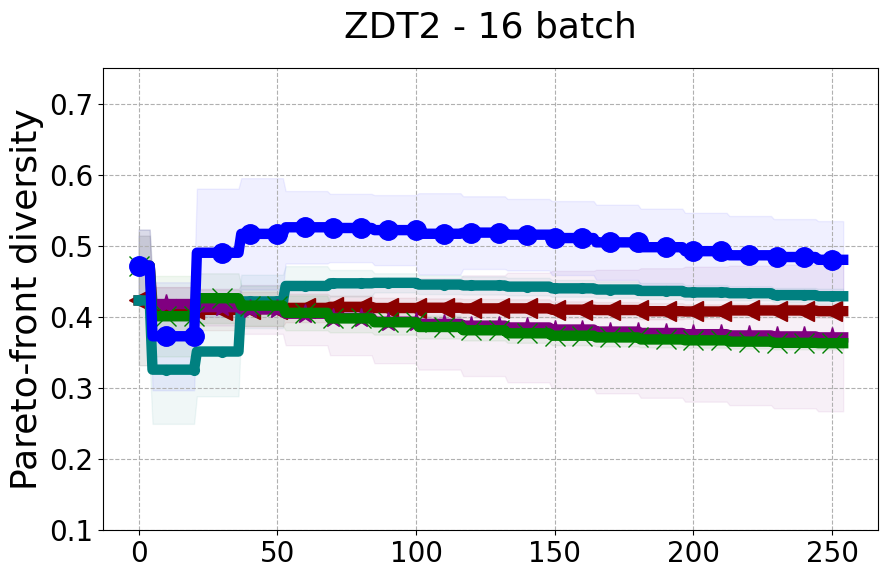}
     \end{subfigure}
        \label{fig:paper-div-zdt2}
    \end{subfigure}
    \centering
    \begin{subfigure}
    \centering
     \begin{subfigure}
         \centering
         \includegraphics[width=0.24\textwidth]{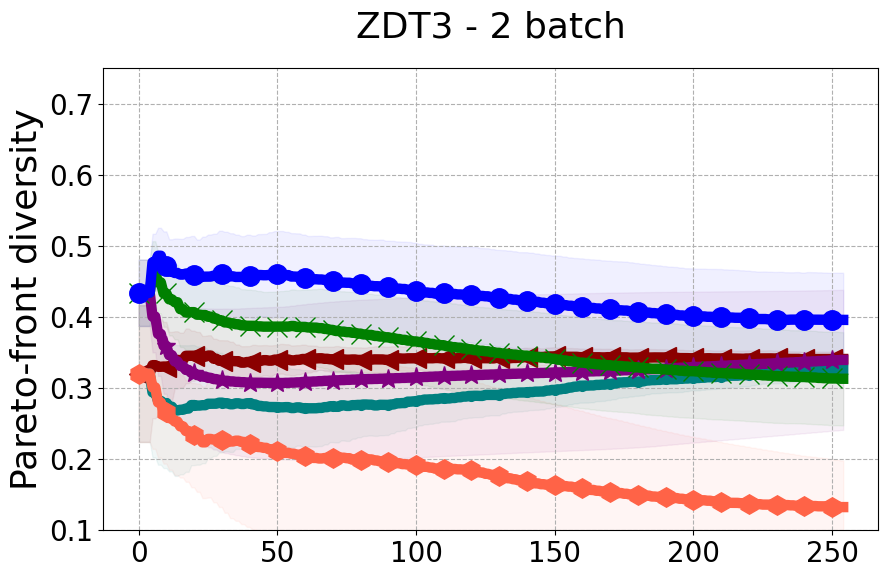}
     \end{subfigure}
     \hfill
     \begin{subfigure}
         \centering
         \includegraphics[width=0.24\textwidth]{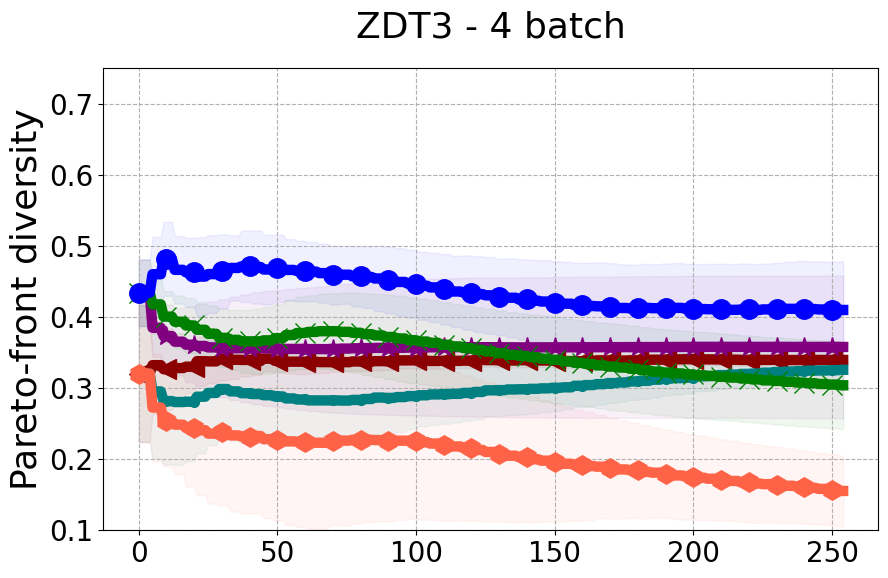}
     \end{subfigure}
     \hfill
     \begin{subfigure}
         \centering
         \includegraphics[width=0.24\textwidth]{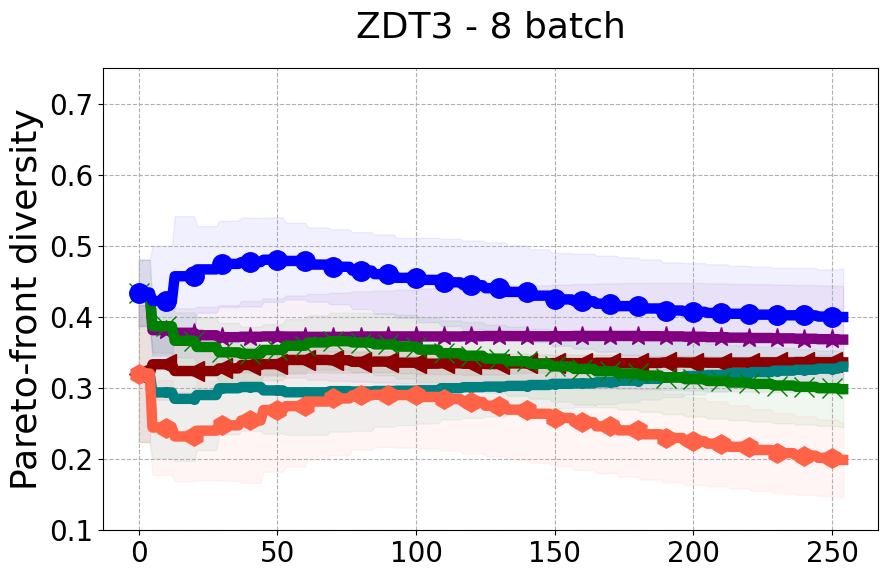}
     \end{subfigure}
     \hfill
     \begin{subfigure}
         \centering
         \includegraphics[width=0.24\textwidth]{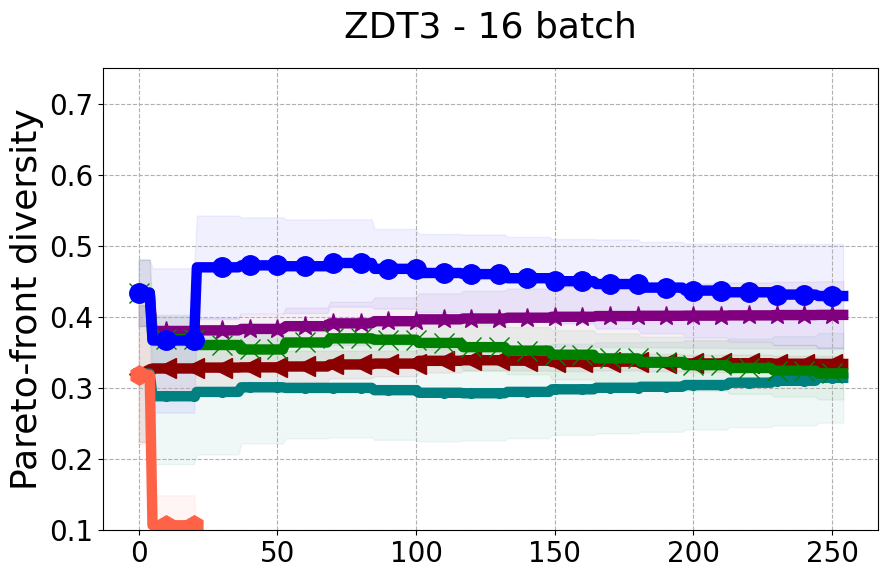}
     \end{subfigure}
    \label{fig:paper-div-zdt3}
     \end{subfigure}
    \centering
    \begin{subfigure}
     \centering
     \begin{subfigure}
         \centering
         \includegraphics[width=0.24\textwidth]{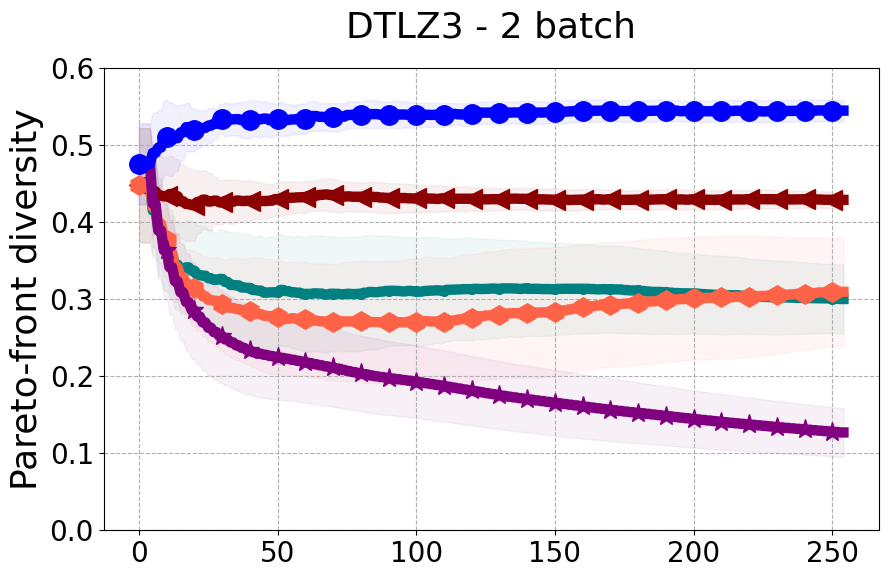}
     \end{subfigure}
     \hfill
     \begin{subfigure}
         \centering
         \includegraphics[width=0.24\textwidth]{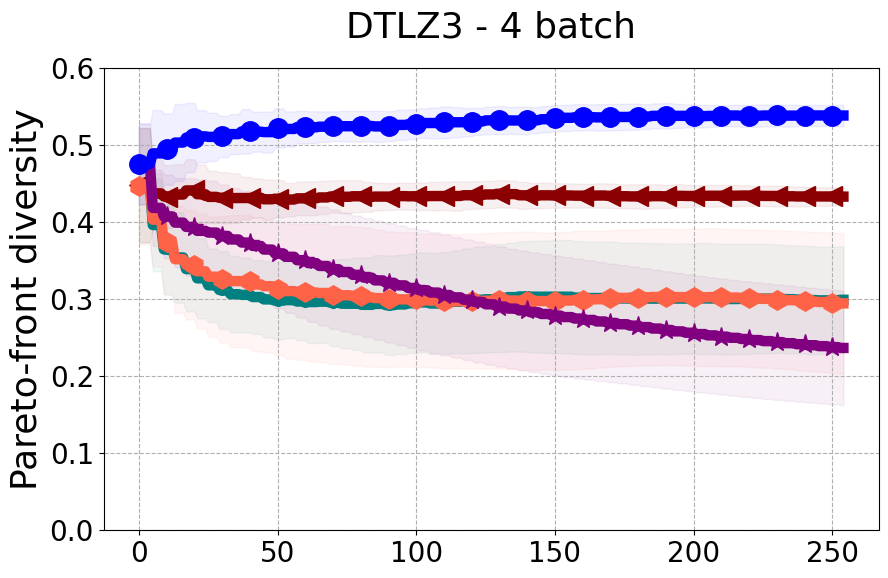}
     \end{subfigure}
     \hfill
     \begin{subfigure}
         \centering
         \includegraphics[width=0.24\textwidth]{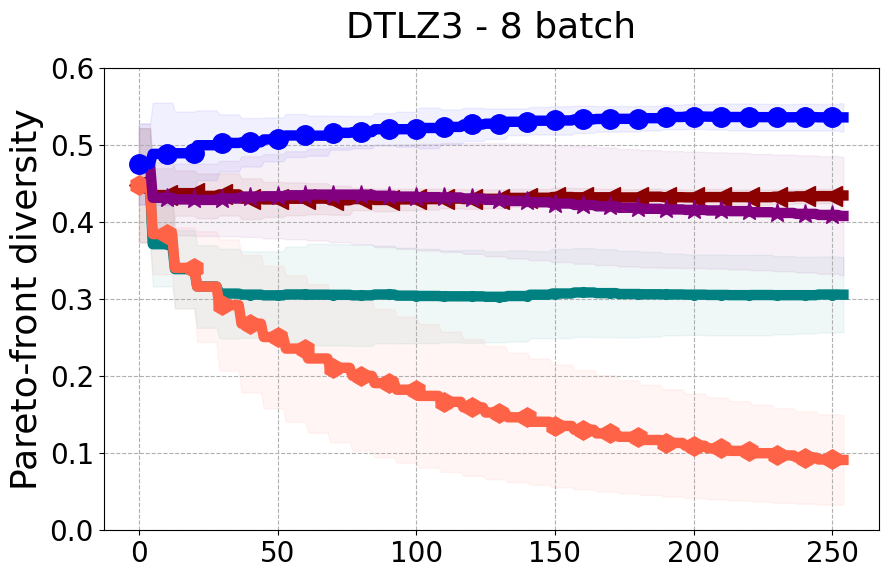}
     \end{subfigure}
     \hfill
     \begin{subfigure}
         \centering
         \includegraphics[width=0.24\textwidth]{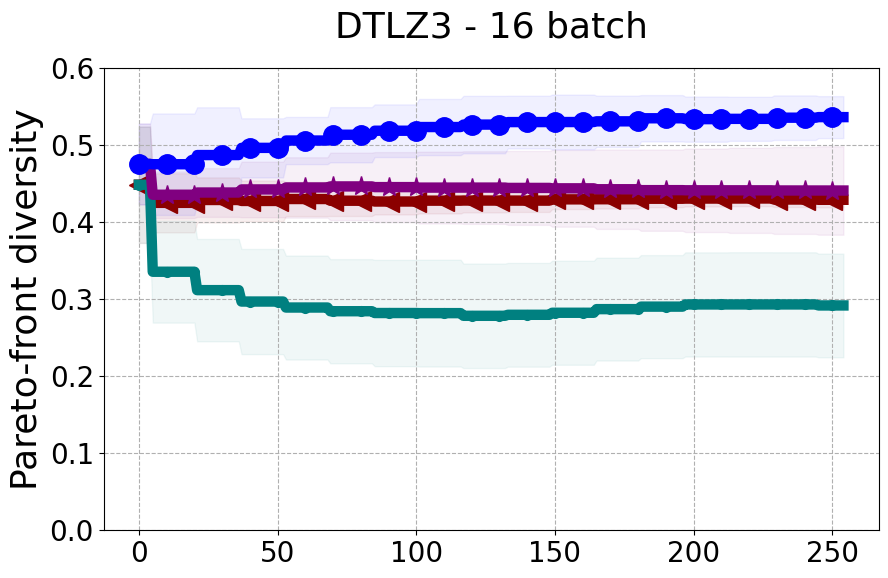}
     \end{subfigure}
    \label{fig:paper-div-dtlz3}
     \end{subfigure}
      \begin{subfigure}
         \centering
         \includegraphics[width=0.98\textwidth]{figures/diversity-legend.png}
     \end{subfigure}
     \caption{DPF results for all benchmarks}
     \label{fig:appendix-dpf}
\end{figure*}

\begin{figure*}[ht]
    \centering
    \begin{subfigure}
     \centering
     \begin{subfigure}
         \centering
         \includegraphics[width=0.24\textwidth]{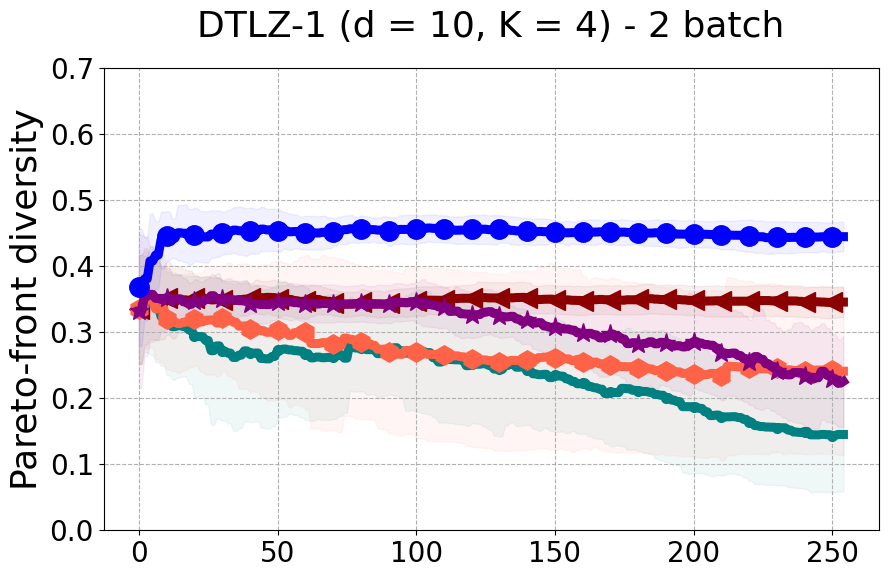}
     \end{subfigure}
     \hfill
     \begin{subfigure}
         \centering
         \includegraphics[width=0.24\textwidth]{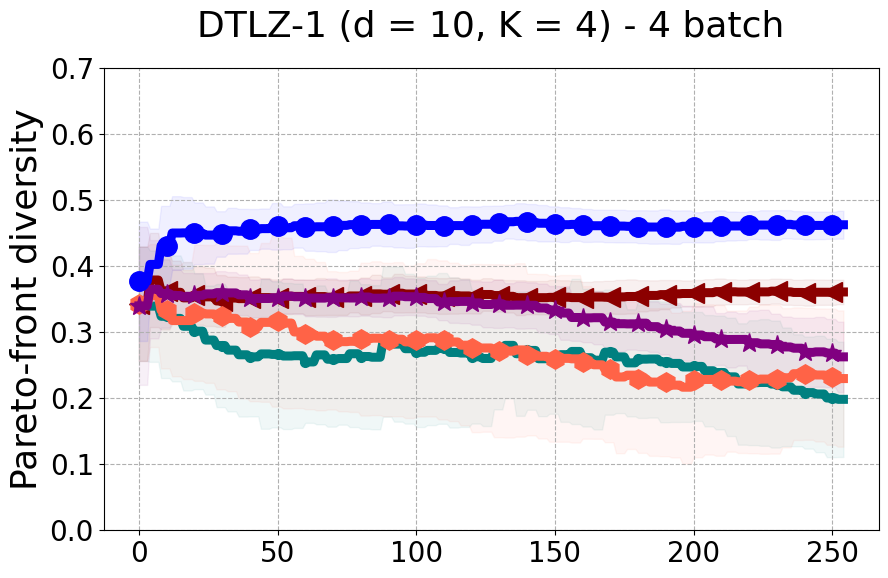}
     \end{subfigure}
     \hfill
     \begin{subfigure}
         \centering
         \includegraphics[width=0.24\textwidth]{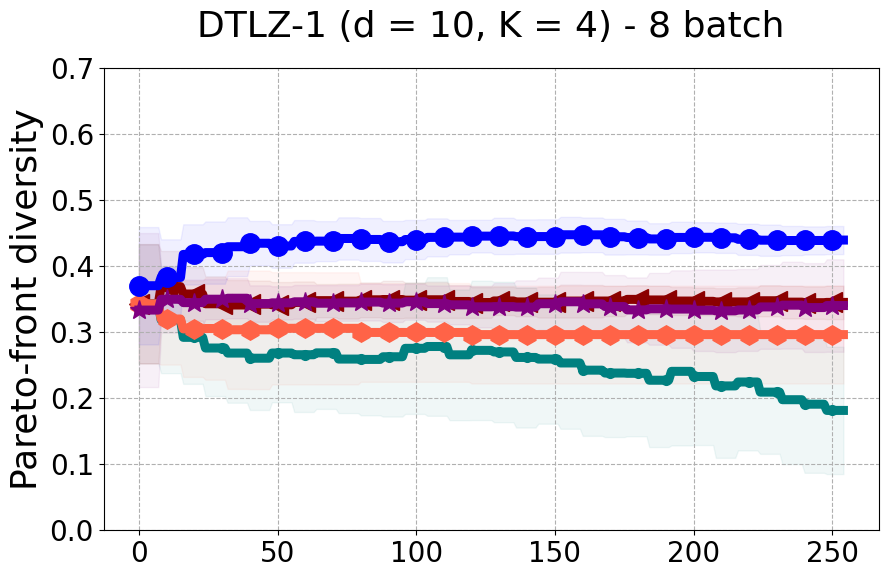}
     \end{subfigure}
     \hfill
     \begin{subfigure}
         \centering
         \includegraphics[width=0.24\textwidth]{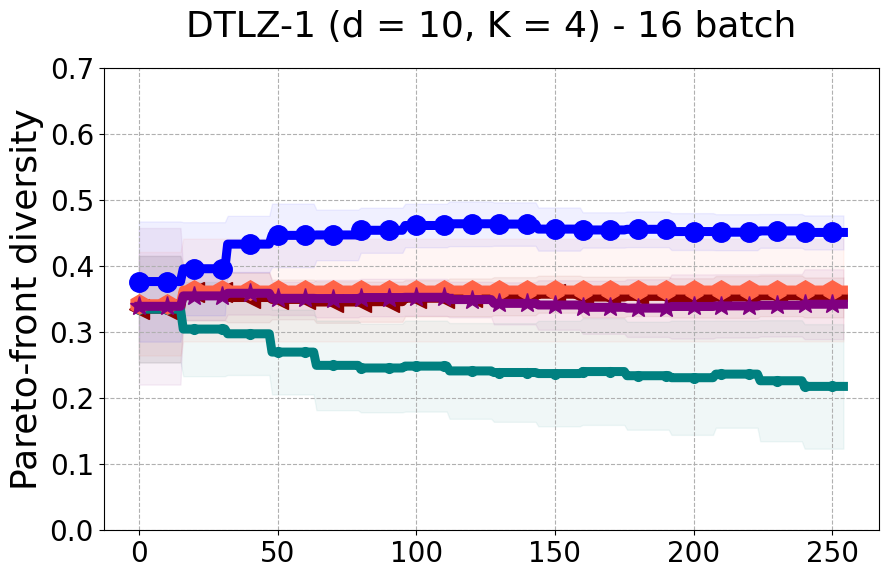}
     \end{subfigure}
     \end{subfigure}
     \centering
    \begin{subfigure}
     \centering
     \begin{subfigure}
         \centering
         \includegraphics[width=0.24\textwidth]{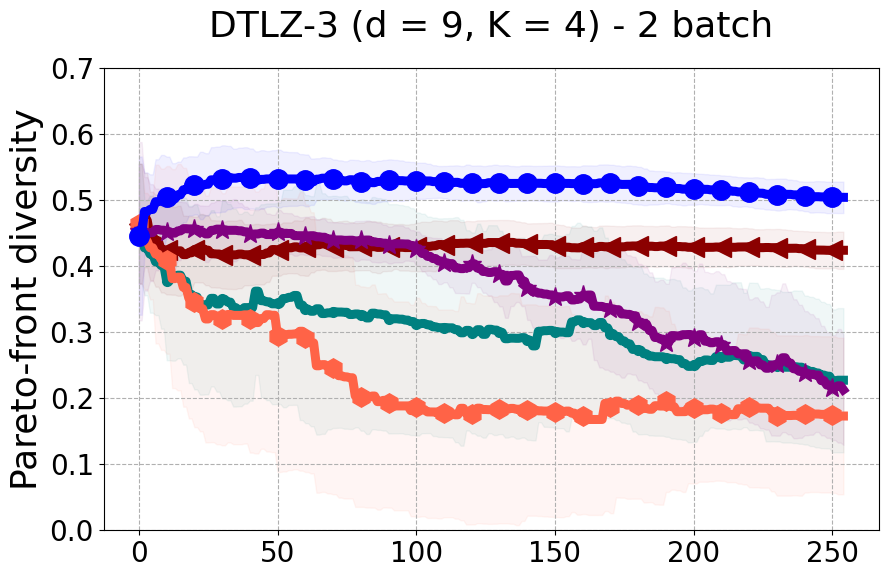}
     \end{subfigure}
     \hfill
     \begin{subfigure}
         \centering
         \includegraphics[width=0.24\textwidth]{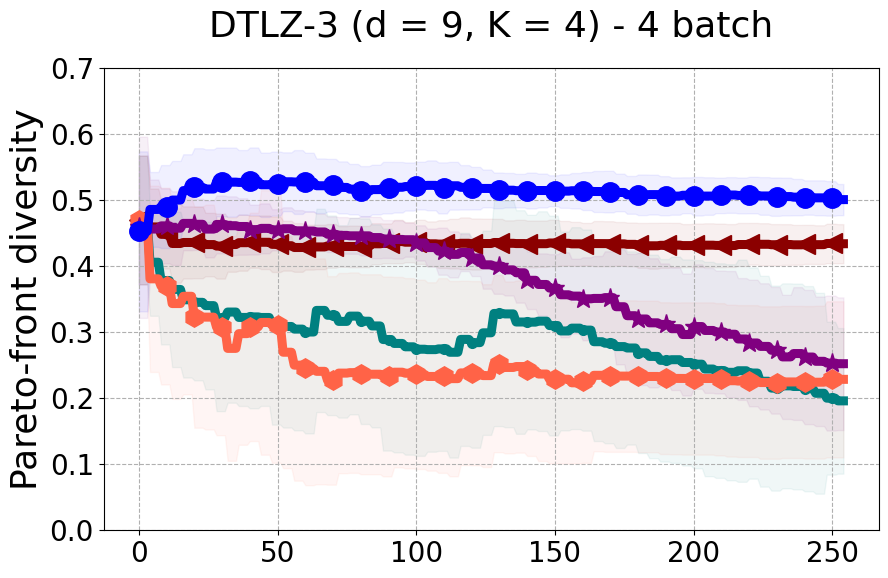}
     \end{subfigure}
     \hfill
     \begin{subfigure}
         \centering
         \includegraphics[width=0.24\textwidth]{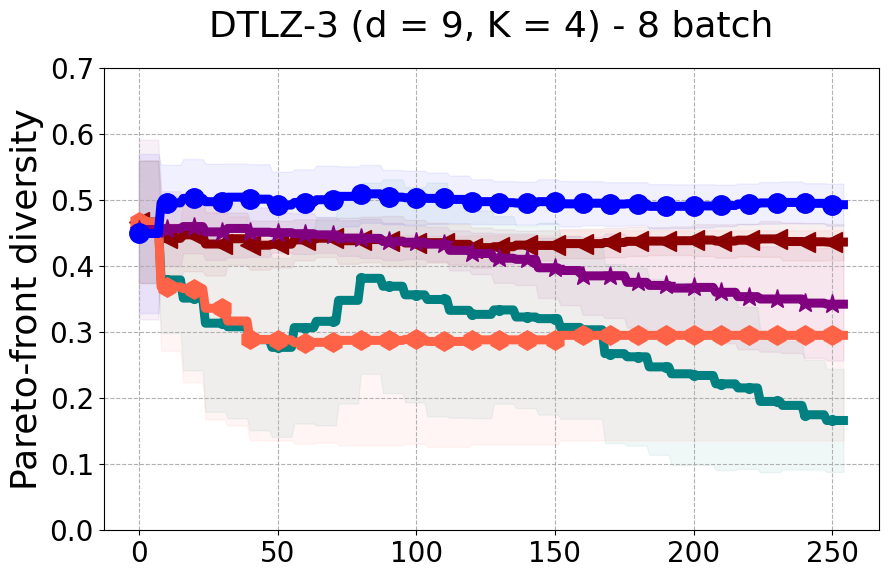}
     \end{subfigure}
     \hfill
     \begin{subfigure}
         \centering
         \includegraphics[width=0.24\textwidth]{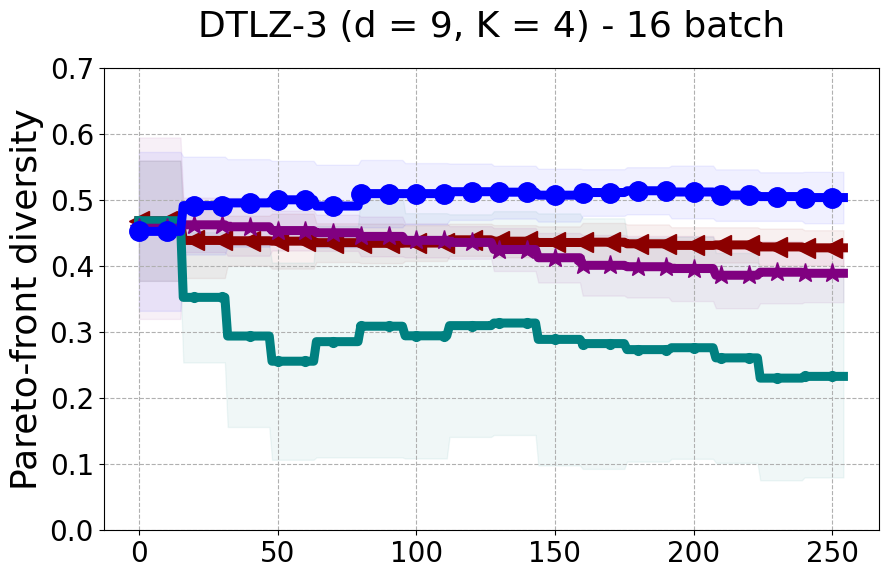}
     \end{subfigure}
     \end{subfigure}
     \centering
    \begin{subfigure}
     \centering
     \begin{subfigure}
         \centering
         \includegraphics[width=0.24\textwidth]{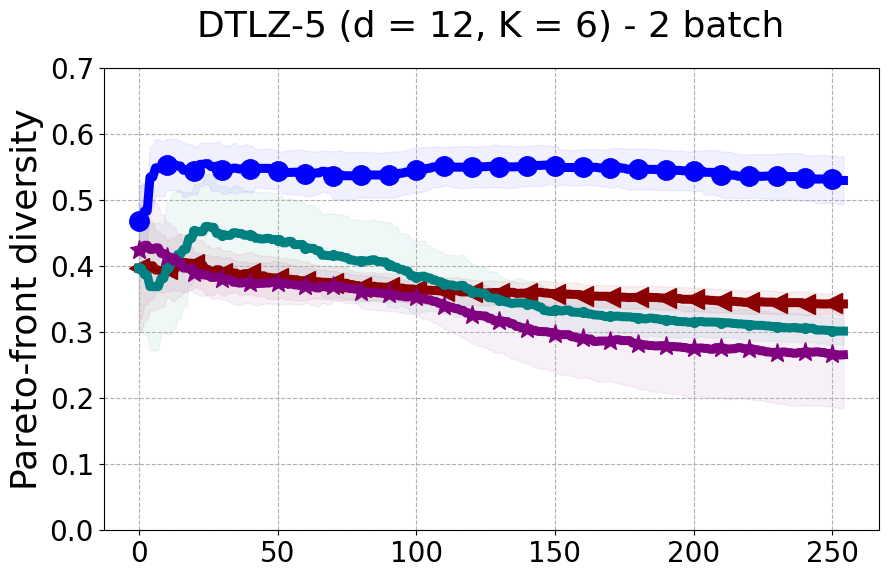}
     \end{subfigure}
     \hfill
     \begin{subfigure}
         \centering
         \includegraphics[width=0.24\textwidth]{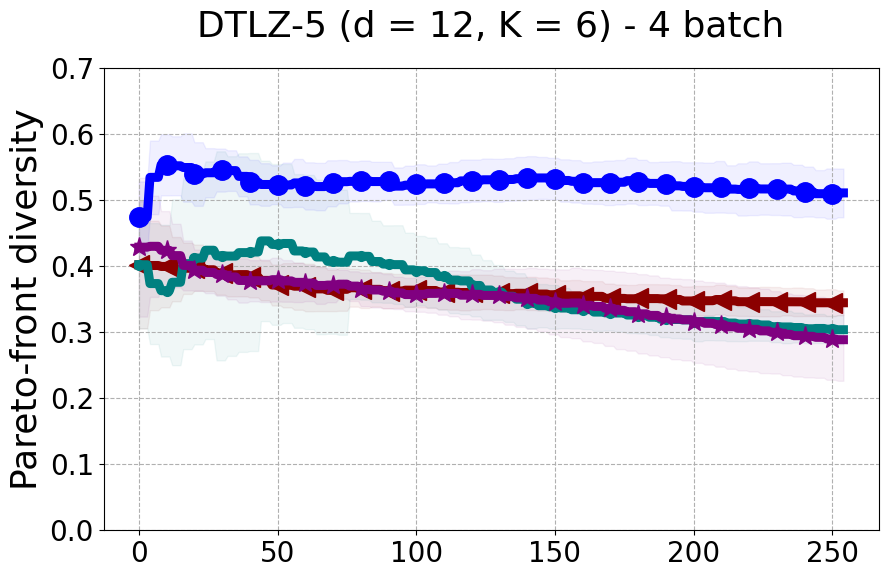}
     \end{subfigure}
     \hfill
     \begin{subfigure}
         \centering
         \includegraphics[width=0.24\textwidth]{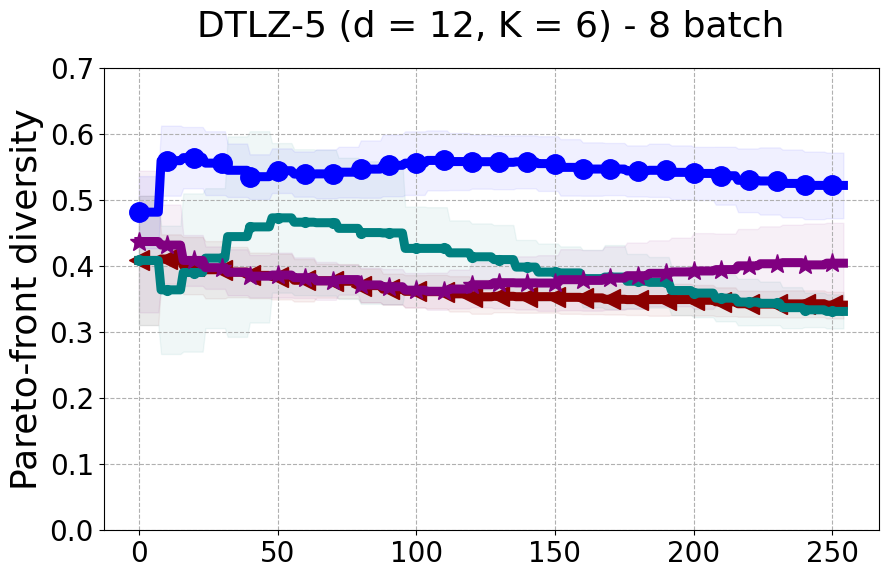}
     \end{subfigure}
     \hfill
     \begin{subfigure}
         \centering
         \includegraphics[width=0.24\textwidth]{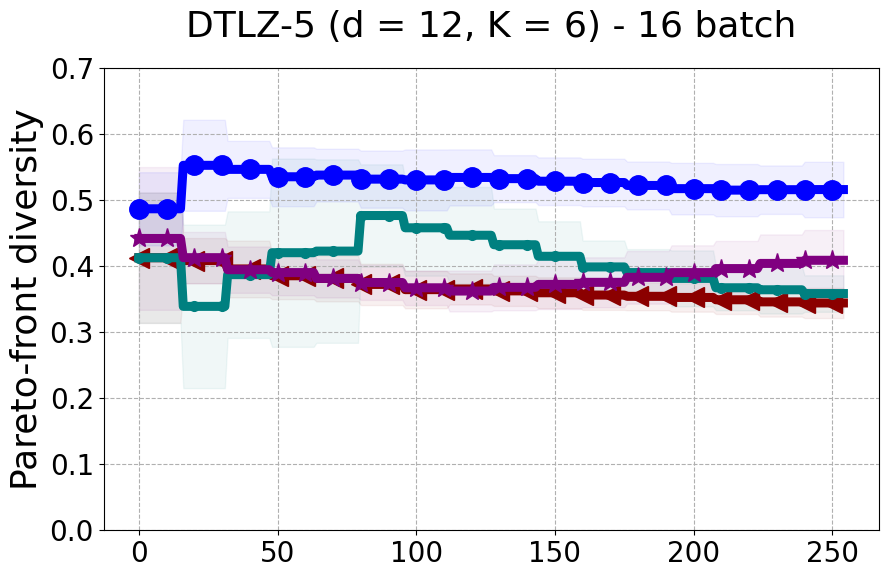}
     \end{subfigure}
     \end{subfigure}
      \begin{subfigure}
         \centering
         \includegraphics[width=0.98\textwidth]{figures/diversity-legend.png}
     \end{subfigure}
     \caption{DPF results calculated using the optimal Pareto front per algorithm iteration}
     \label{fig:appendix-dpf-on-pf}
\end{figure*}

\begin{figure*}[ht]
\centering
     \begin{subfigure}
         \centering
         \includegraphics[width=0.24\textwidth]{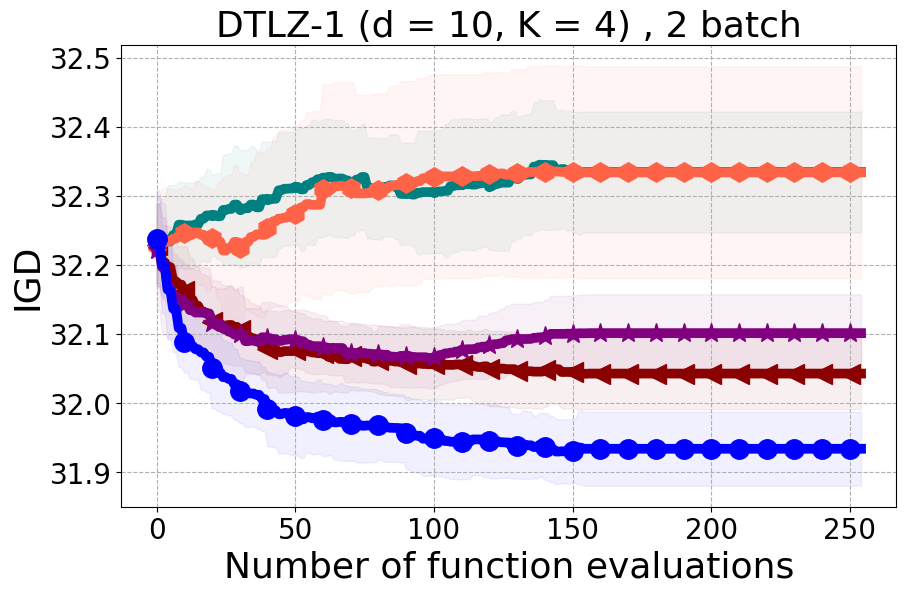}
     \end{subfigure}
     \hfill
     \begin{subfigure}
         \centering
         \includegraphics[width=0.24\textwidth]{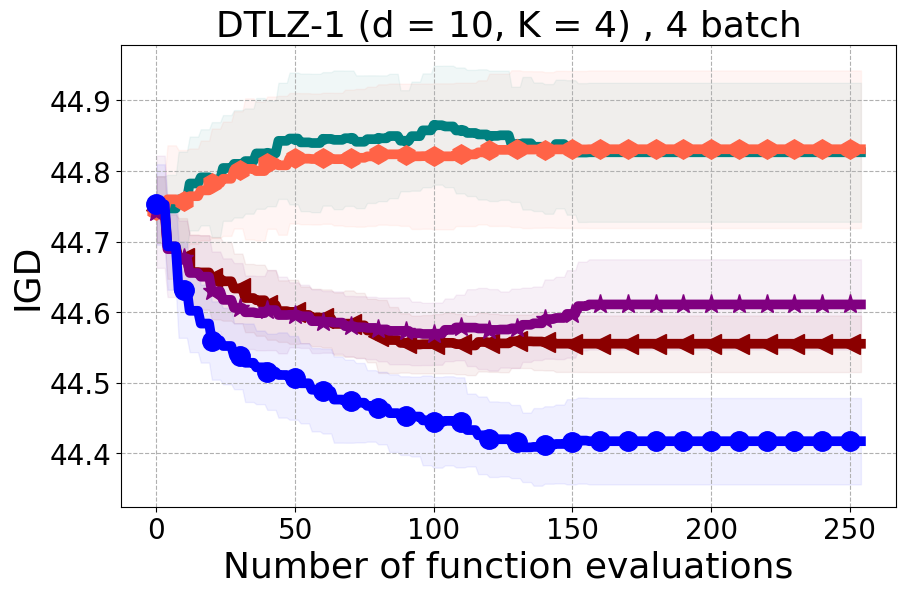}
     \end{subfigure}
     \hfill
     \begin{subfigure}
         \centering
         \includegraphics[width=0.24\textwidth]{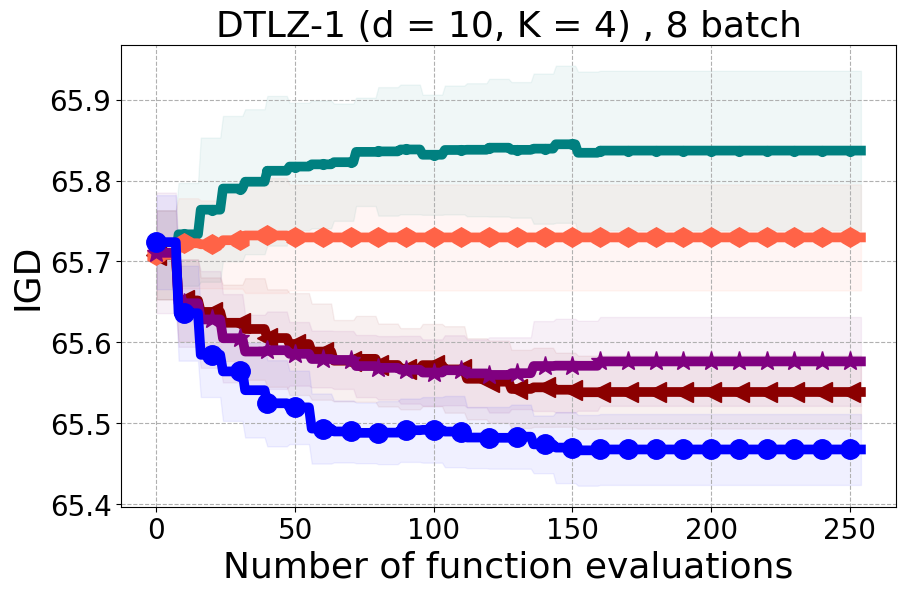}
     \end{subfigure}
     \hfill
     \begin{subfigure}
         \centering
         \includegraphics[width=0.24\textwidth]{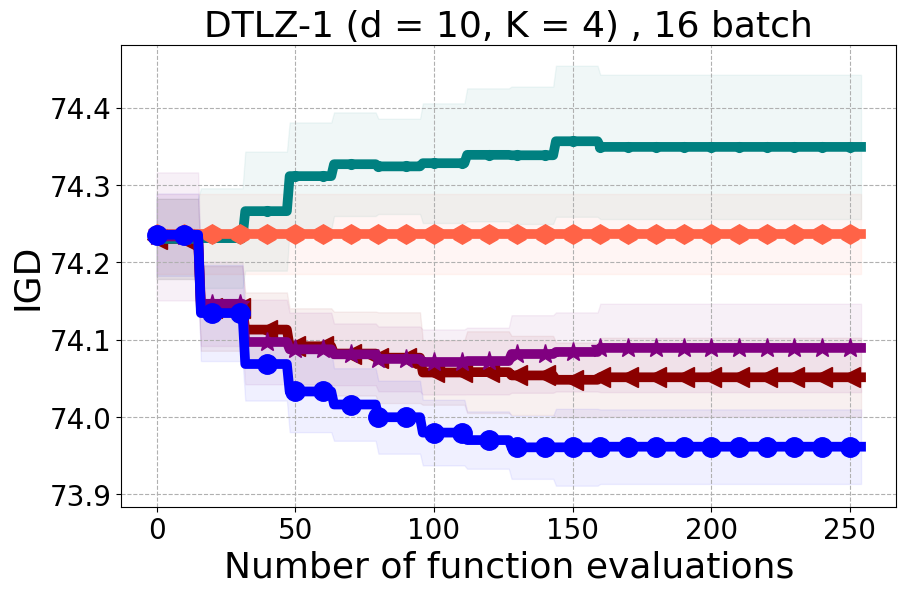}
     \end{subfigure}
        \label{fig:IGD-dtlz1}
    \centering
    \begin{subfigure}
         \centering
         \includegraphics[width=0.24\textwidth]{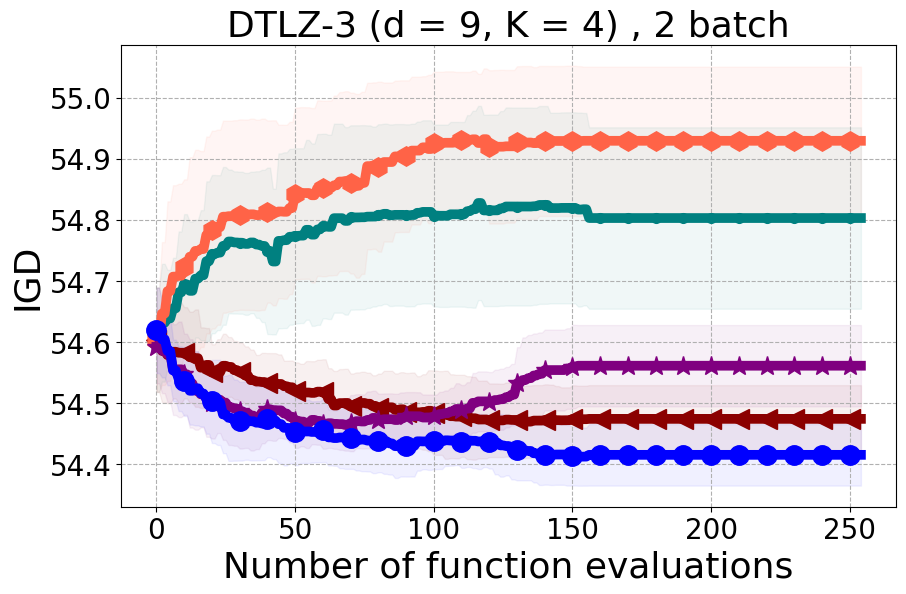}
     \end{subfigure}
     \hfill
     \begin{subfigure}
         \centering
         \includegraphics[width=0.24\textwidth]{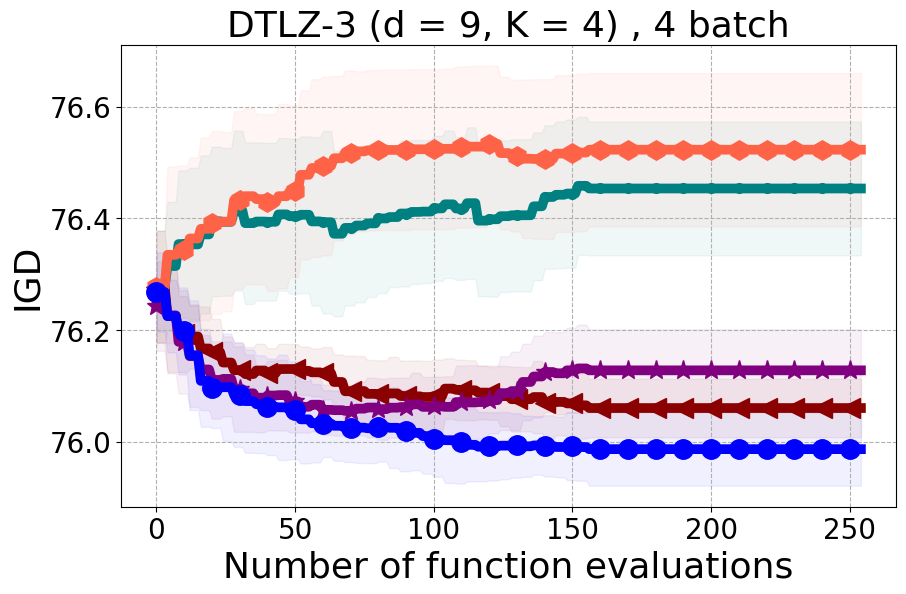}
     \end{subfigure}
     \hfill
     \begin{subfigure}
         \centering
         \includegraphics[width=0.24\textwidth]{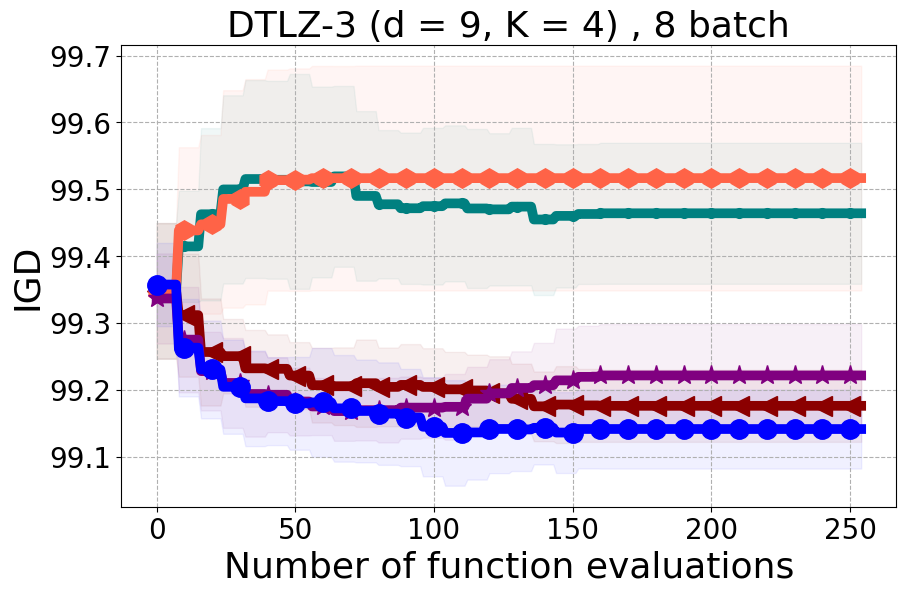}
     \end{subfigure}
     \hfill
     \begin{subfigure}
         \centering
         \includegraphics[width=0.24\textwidth]{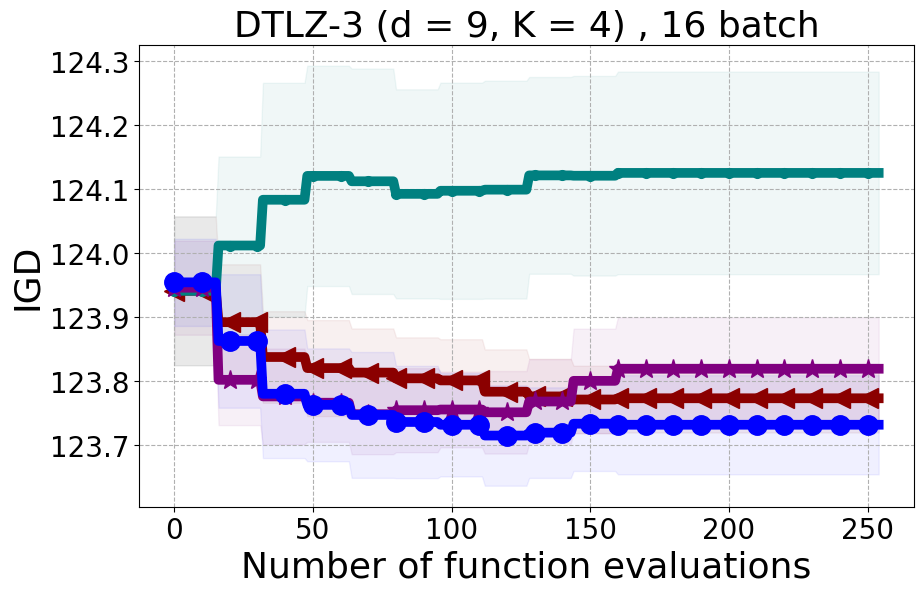}
     \end{subfigure}
        \label{fig:appendix-igd-dtlz3}
        \centering
    \begin{subfigure}
         \centering
         \includegraphics[width=0.24\textwidth]{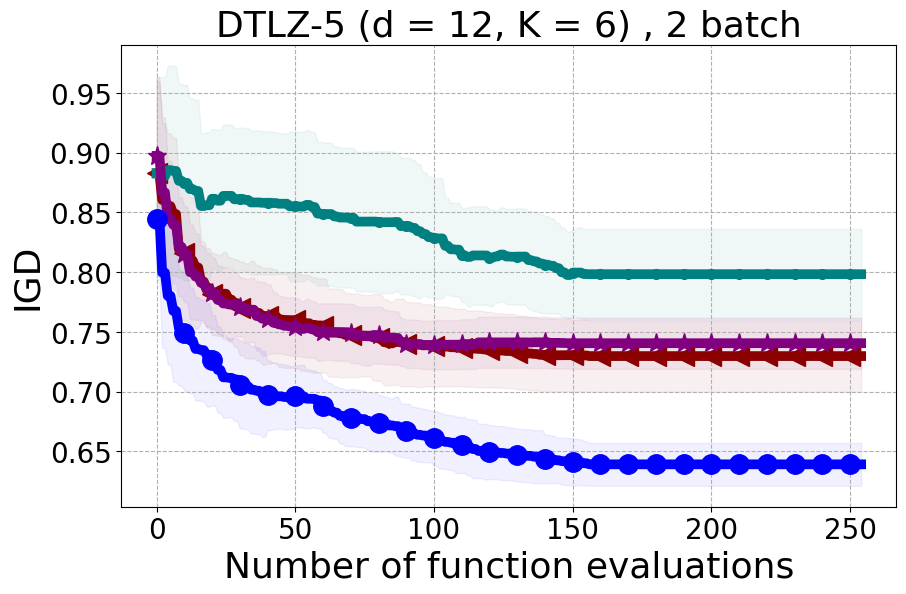}
     \end{subfigure}
     \hfill
     \begin{subfigure}
         \centering
         \includegraphics[width=0.24\textwidth]{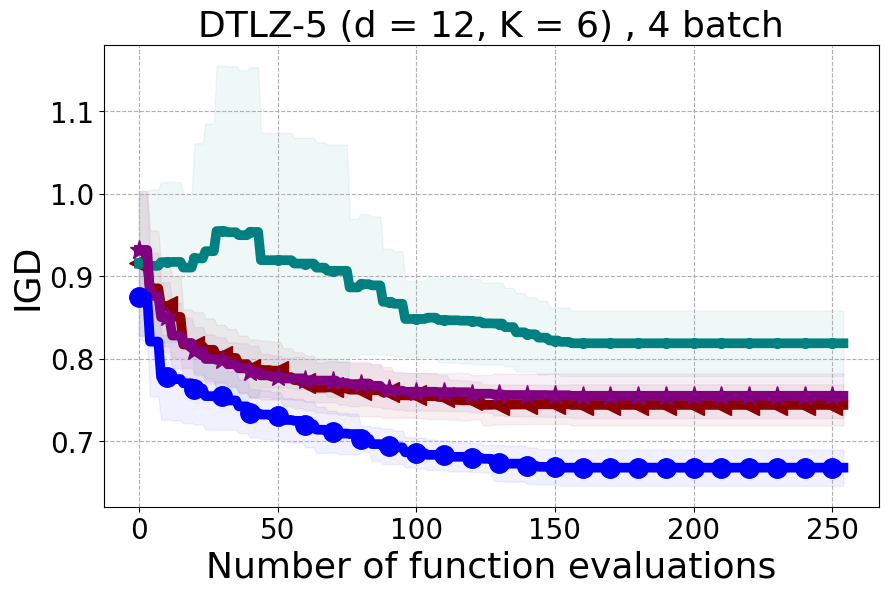}
     \end{subfigure}
     \hfill
     \begin{subfigure}
         \centering
         \includegraphics[width=0.24\textwidth]{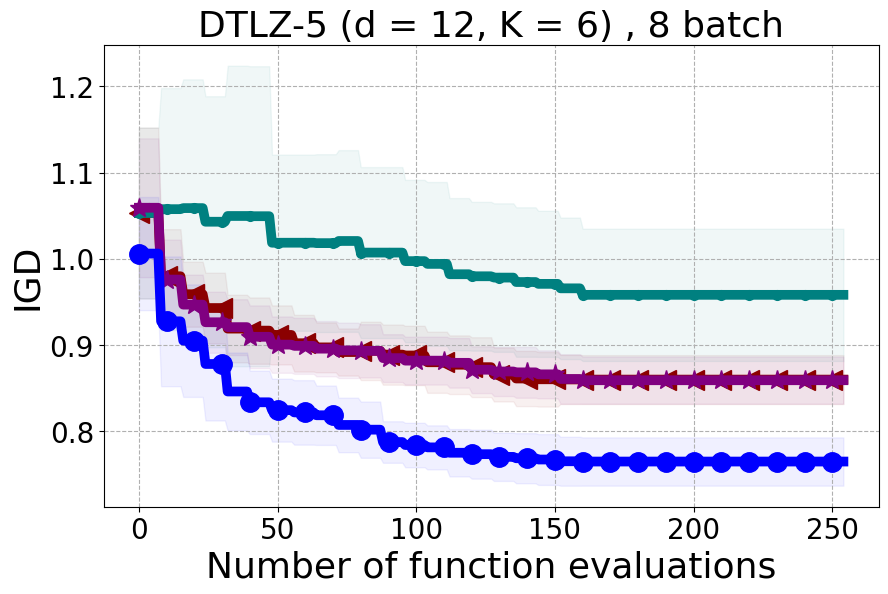}
     \end{subfigure}
     \hfill
     \begin{subfigure}
         \centering
         \includegraphics[width=0.24\textwidth]{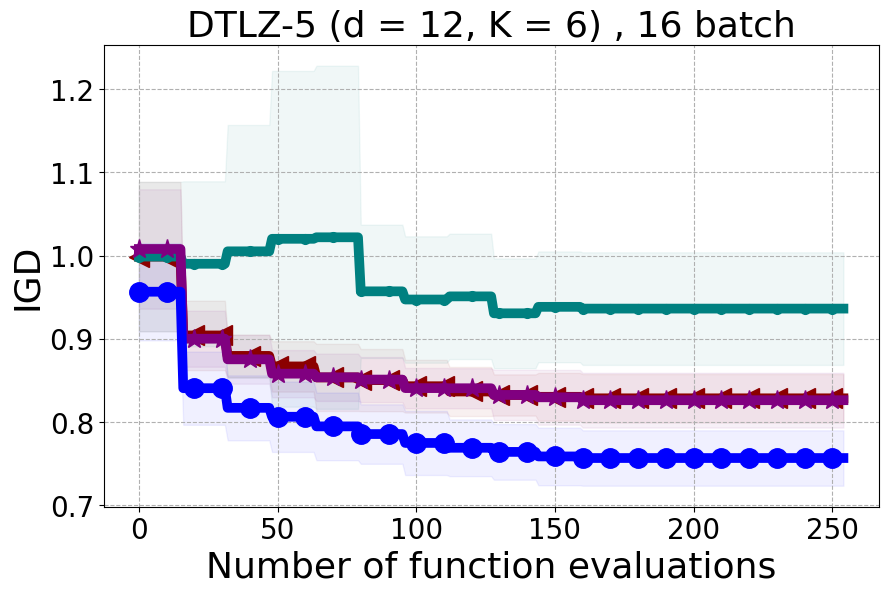}
     \end{subfigure}
        \label{fig:appendix-igd-dtlz5}
      \begin{subfigure}
         \centering
         \includegraphics[width=0.96\textwidth]{figures/diversity-legend.png}
     \end{subfigure}
     \caption{IGD results on several benchmarks; A lower IGD value indicates a solution of higher quality.}
    \label{fig:appendix-IGD-results}
\end{figure*}

\begin{figure*}[ht]
\centering
     \begin{subfigure}
         \centering
         \includegraphics[width=0.24\textwidth]{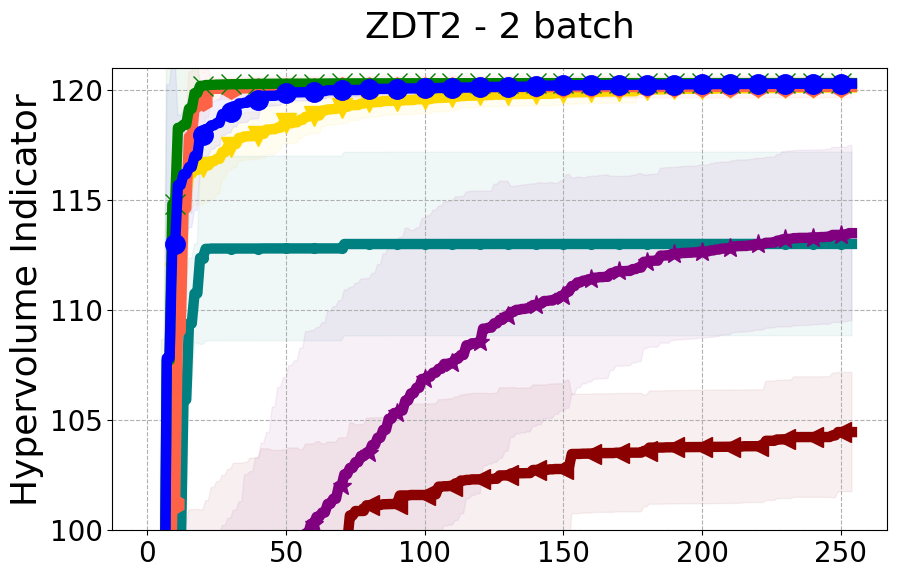}
     \end{subfigure}
     \hfill
     \begin{subfigure}
         \centering
         \includegraphics[width=0.24\textwidth]{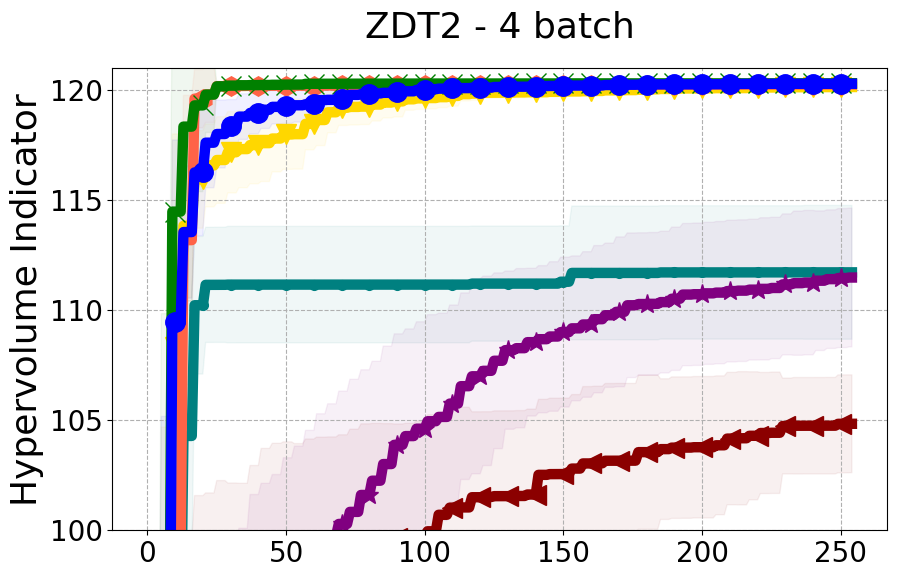}
     \end{subfigure}
     \hfill
     \begin{subfigure}
         \centering
         \includegraphics[width=0.24\textwidth]{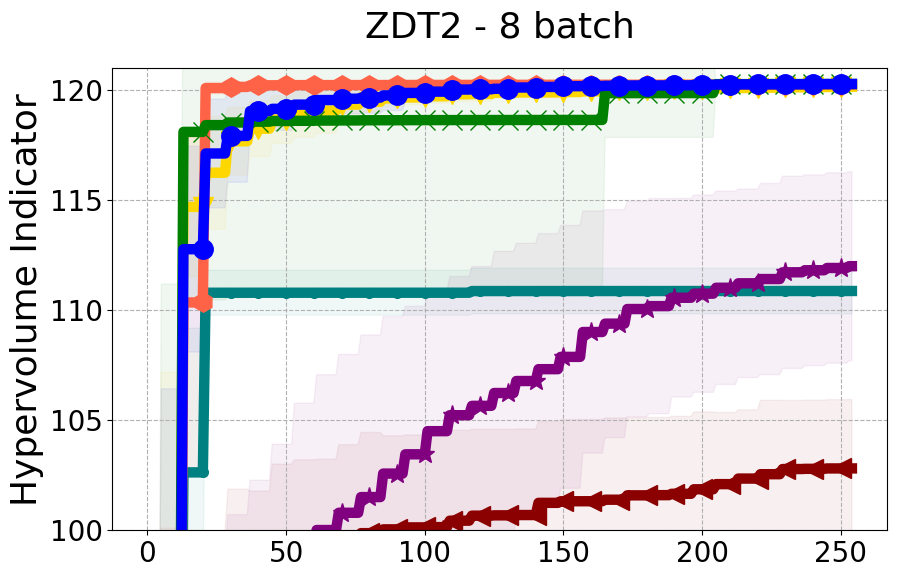}
     \end{subfigure}
     \hfill
     \begin{subfigure}
         \centering
         \includegraphics[width=0.24\textwidth]{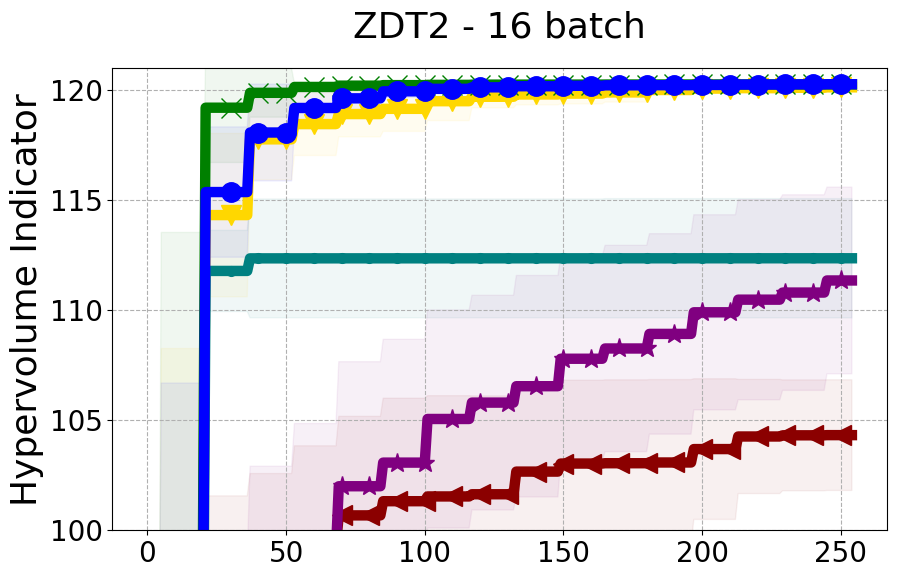}
     \end{subfigure}
        \label{fig:appendix-paper-zdt2}
    \centering
    \begin{subfigure}
         \centering
         \includegraphics[width=0.24\textwidth]{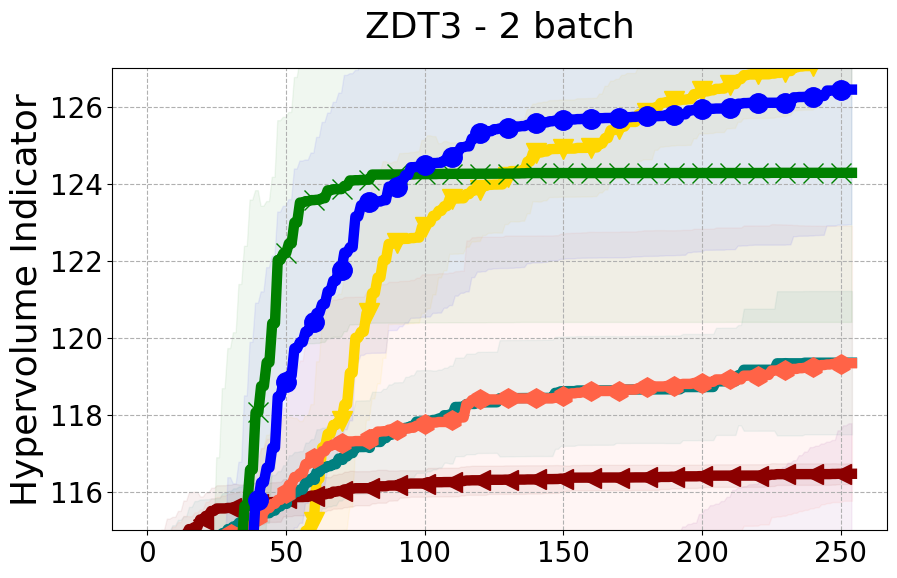}
     \end{subfigure}
     \hfill
     \begin{subfigure}
         \centering
         \includegraphics[width=0.24\textwidth]{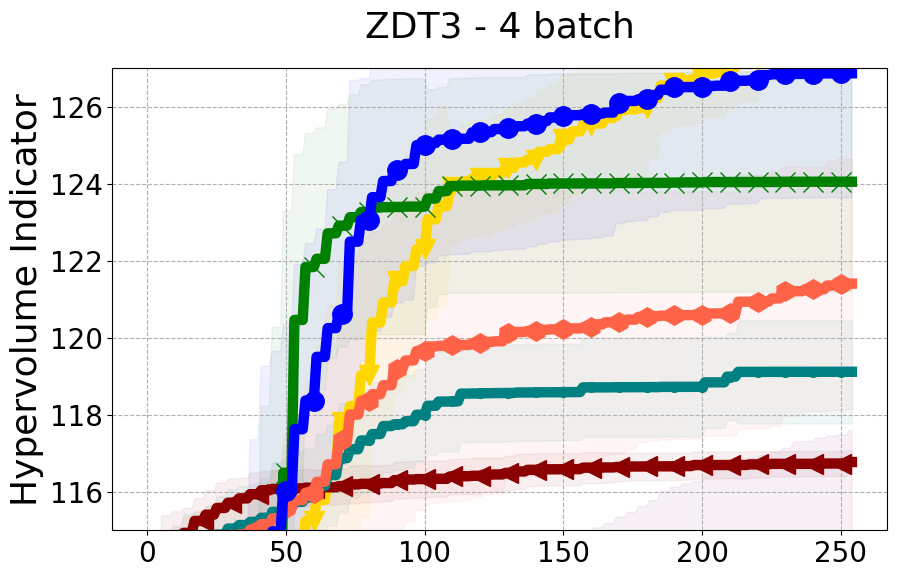}
     \end{subfigure}
     \hfill
     \begin{subfigure}
         \centering
         \includegraphics[width=0.24\textwidth]{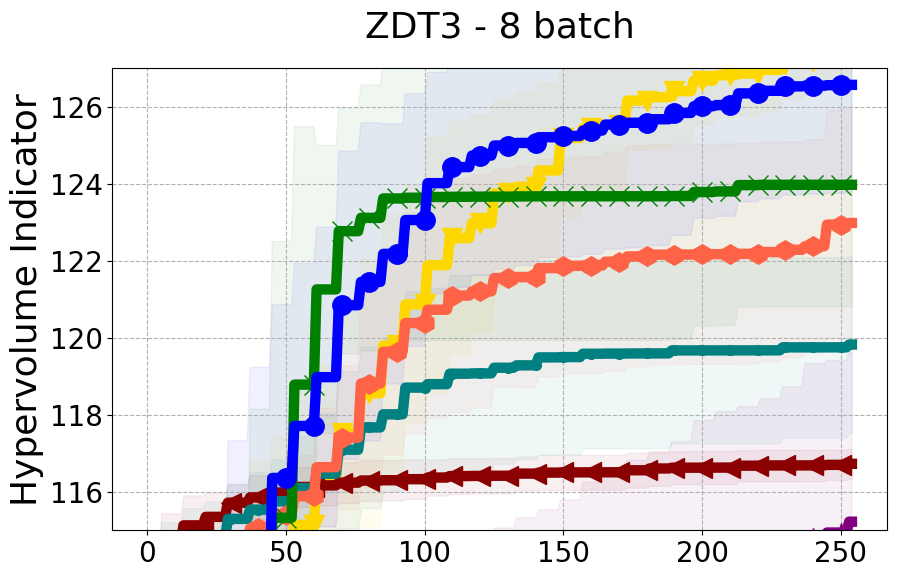}
     \end{subfigure}
     \hfill
     \begin{subfigure}
         \centering
         \includegraphics[width=0.24\textwidth]{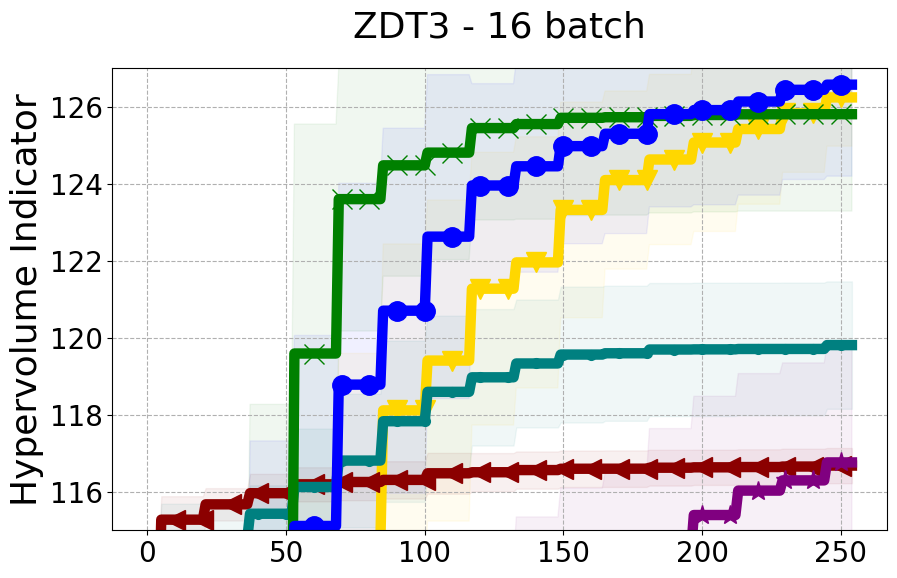}
     \end{subfigure}
        \label{fig:appendix-paper-zdt3}
        \centering
    \begin{subfigure}
         \centering
         \includegraphics[width=0.24\textwidth]{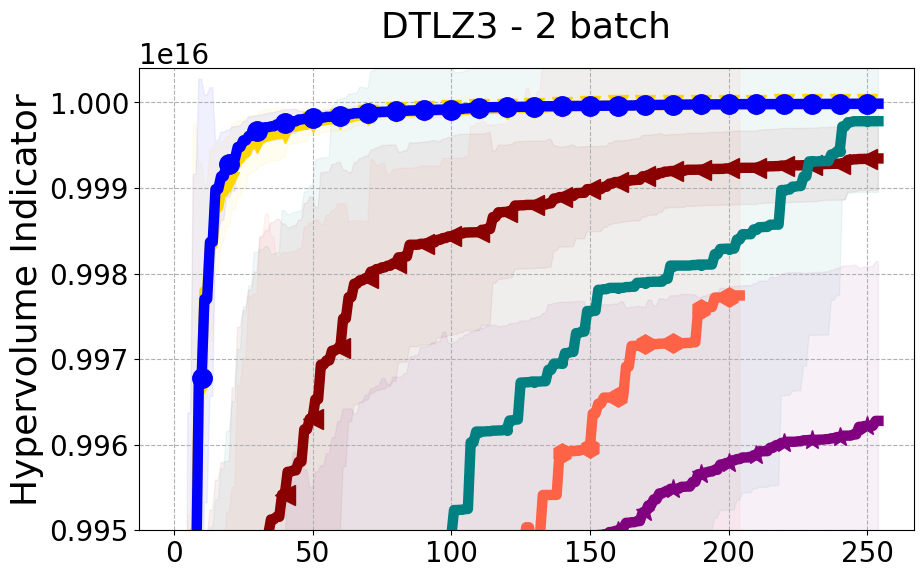}
     \end{subfigure}
     \hfill
     \begin{subfigure}
         \centering
         \includegraphics[width=0.24\textwidth]{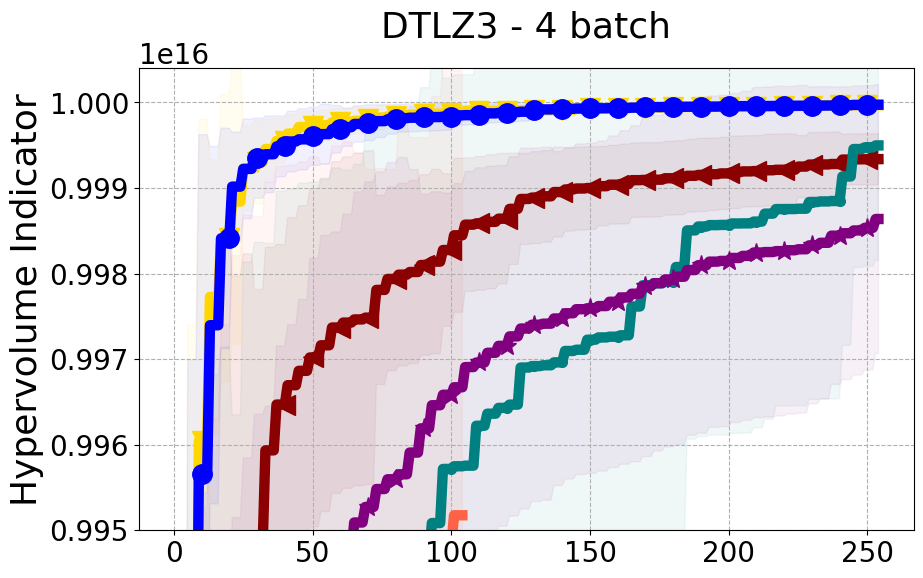}
     \end{subfigure}
     \hfill
     \begin{subfigure}
         \centering
         \includegraphics[width=0.24\textwidth]{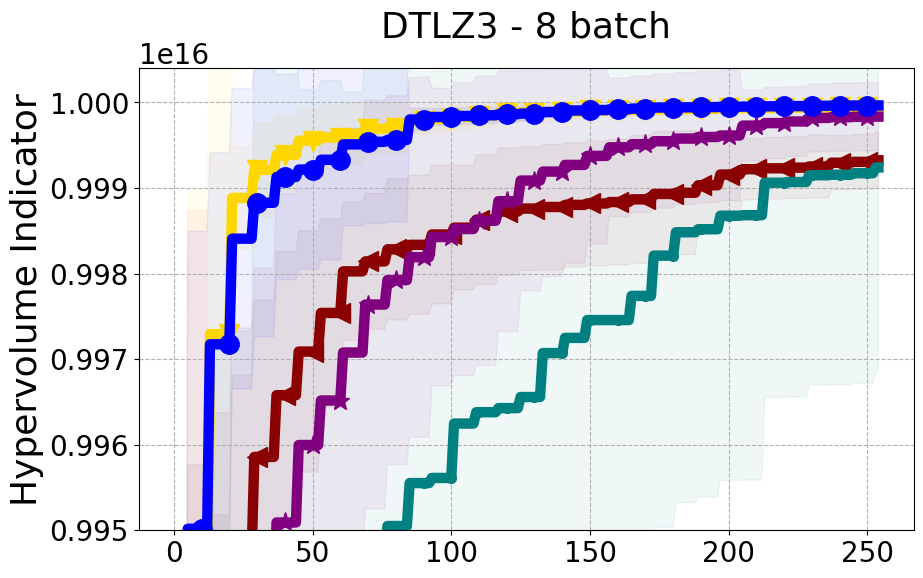}
     \end{subfigure}
     \hfill
     \begin{subfigure}
         \centering
         \includegraphics[width=0.24\textwidth]{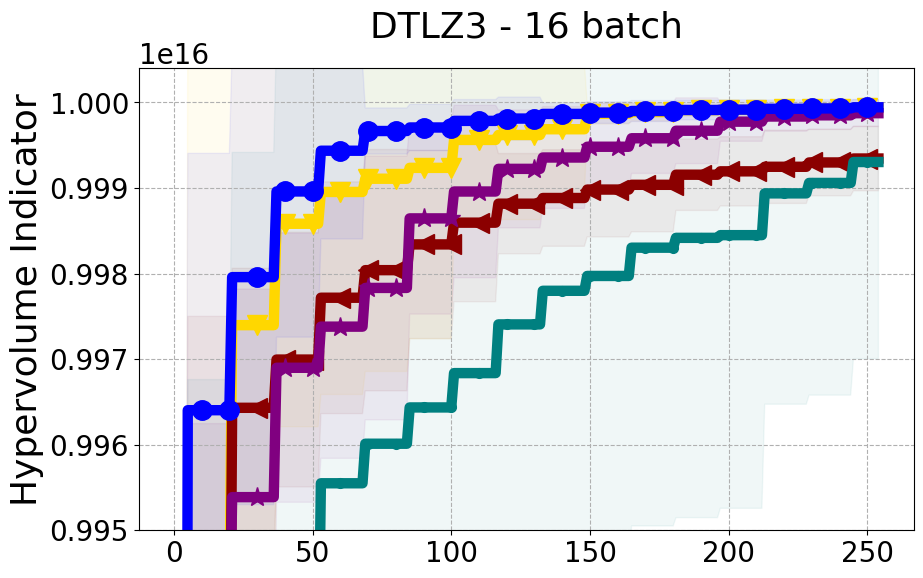}
     \end{subfigure}
        \label{fig:appendix-paper-dtlz33}
      \begin{subfigure}
         \centering
         \includegraphics[width=\textwidth]{figures/paper-legend.png}
     \end{subfigure}
     \caption{Average hypervolume results on additional benchmarks}
    \label{fig:appendix-main-results}
\end{figure*}

\begin{figure*}[ht]
    \begin{subfigure}
     \centering
     \begin{subfigure}
         \centering
         \includegraphics[width=0.24\textwidth]{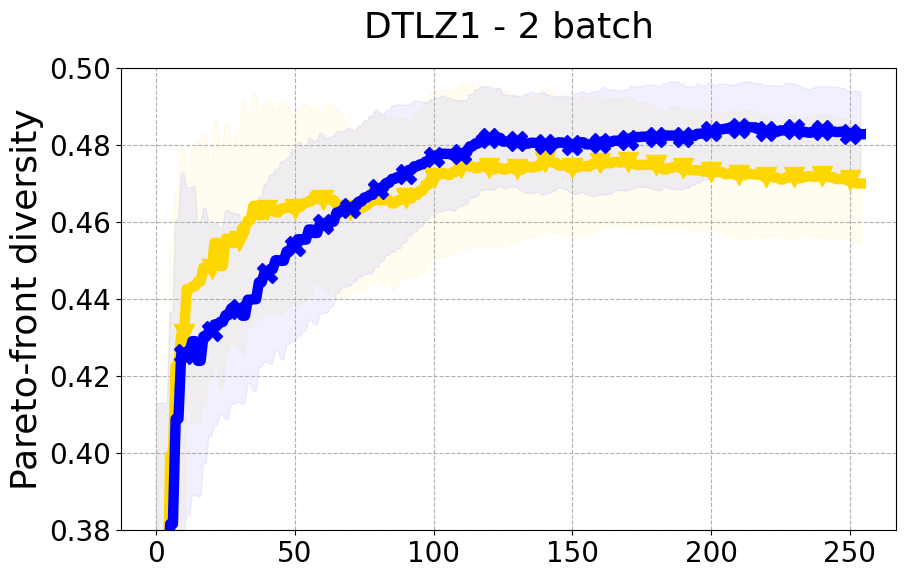}
     \end{subfigure}
     \hfill
     \begin{subfigure}
         \centering
         \includegraphics[width=0.24\textwidth]{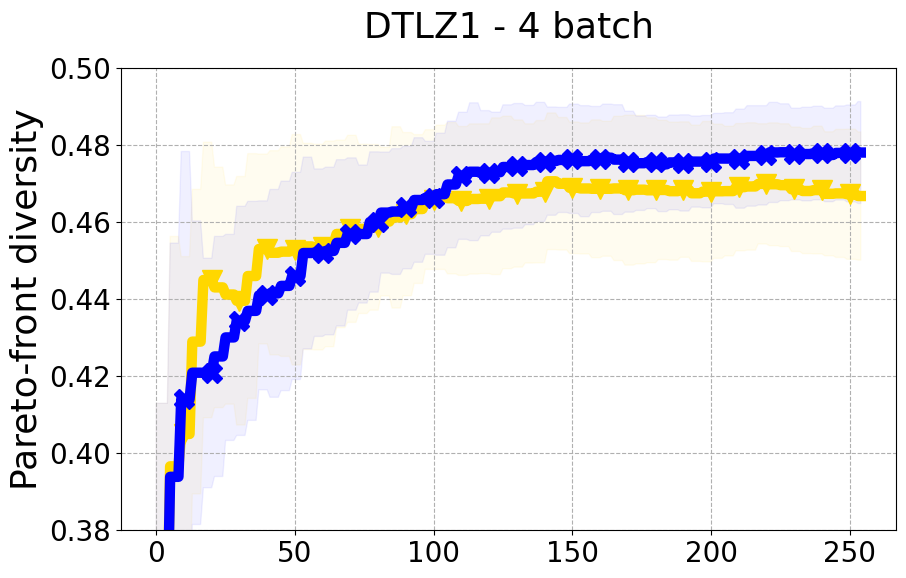}
     \end{subfigure}
     \hfill
     \begin{subfigure}
         \centering
         \includegraphics[width=0.24\textwidth]{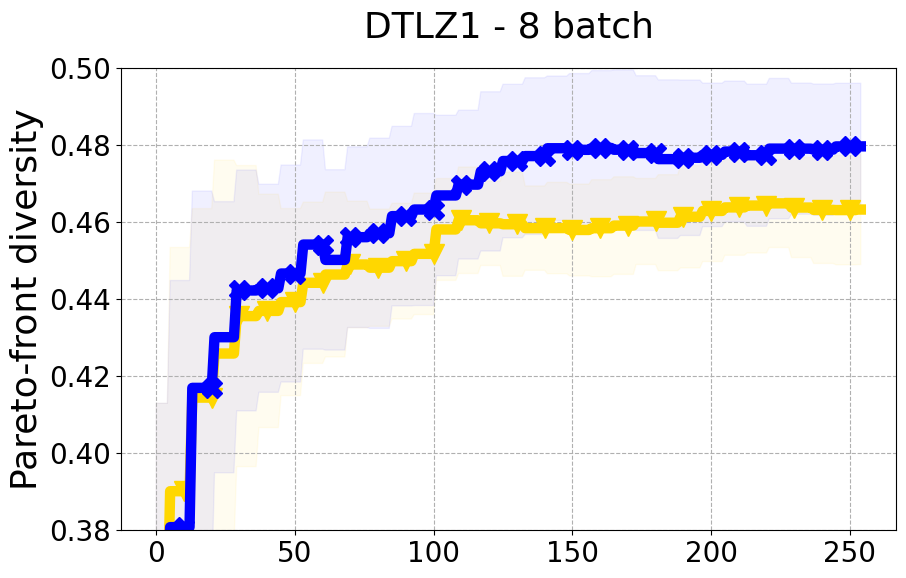}
     \end{subfigure}
     \hfill
     \begin{subfigure}
         \centering
         \includegraphics[width=0.24\textwidth]{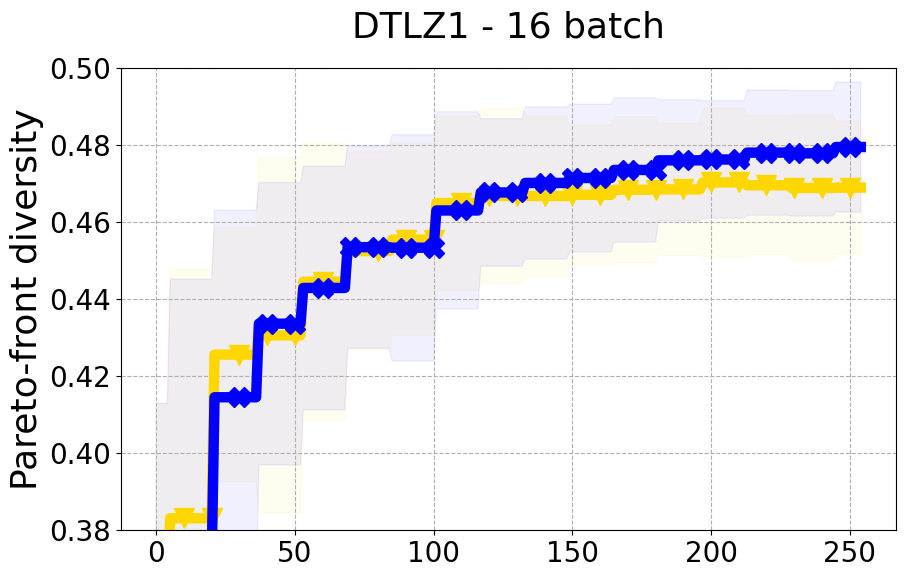}
     \end{subfigure}
        \centering
     \begin{subfigure}
         \centering
         \includegraphics[width=0.24\textwidth]{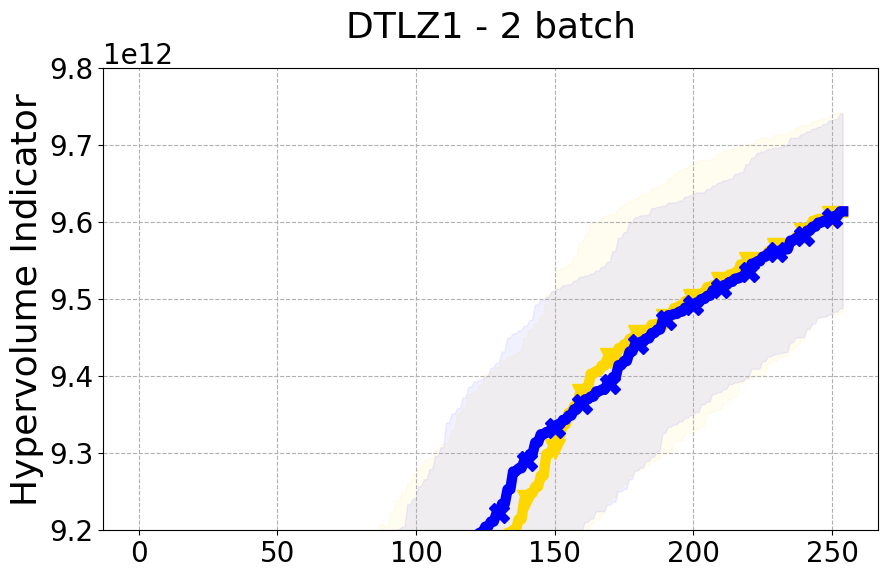}
     \end{subfigure}
     \hfill
     \begin{subfigure}
         \centering
         \includegraphics[width=0.24\textwidth]{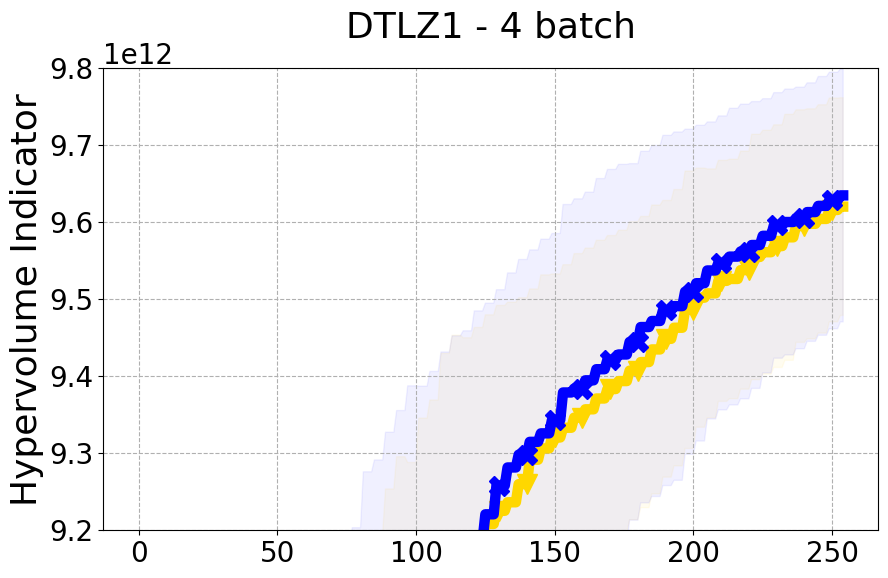}
     \end{subfigure}
     \hfill
     \begin{subfigure}
         \centering
         \includegraphics[width=0.24\textwidth]{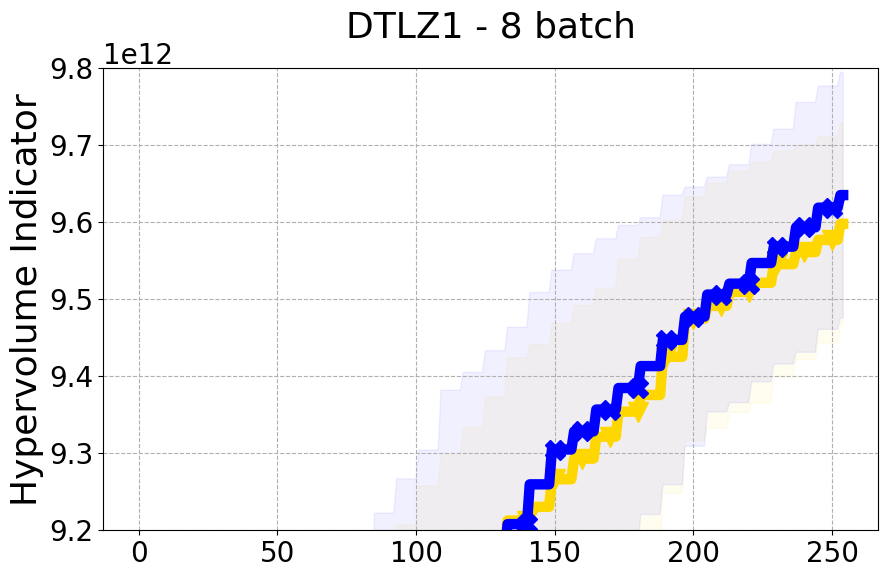}
     \end{subfigure}
     \hfill
     \begin{subfigure}
         \centering
         \includegraphics[width=0.24\textwidth]{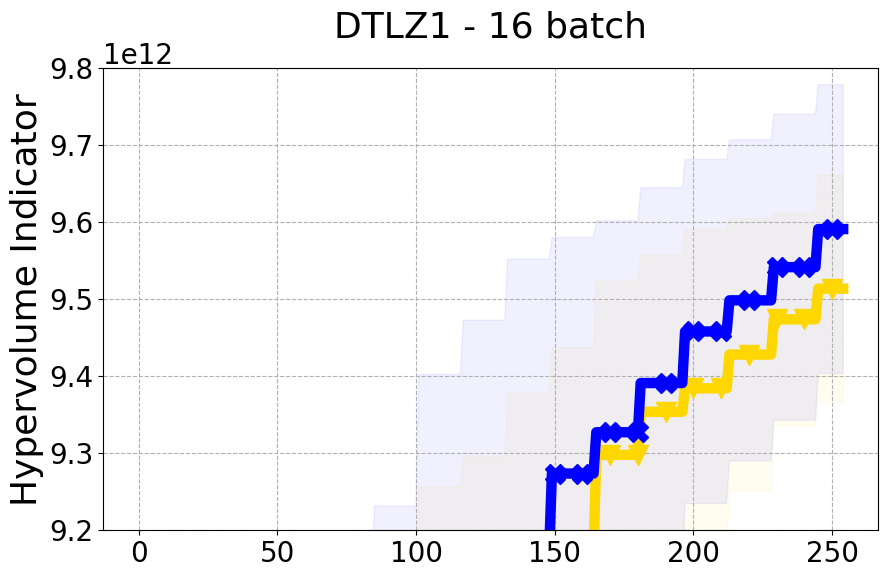}
     \end{subfigure}
     \begin{subfigure}
         \centering
         \includegraphics[width=0.4\textwidth]{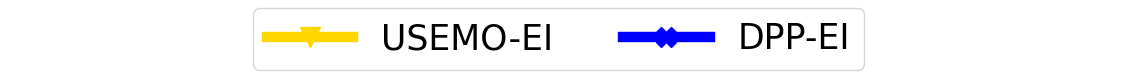}
     \end{subfigure}
     
        \caption{DPF and hypervolume results for DTLZ1 problem, d = 13, K = 5}
         \vspace{0.3cm}
        \label{fig:dpp-div-dtlz1}
    \end{subfigure}
     \centering
    \begin{subfigure}
         \centering
     \begin{subfigure}
         \centering
         \includegraphics[width=0.24\textwidth]{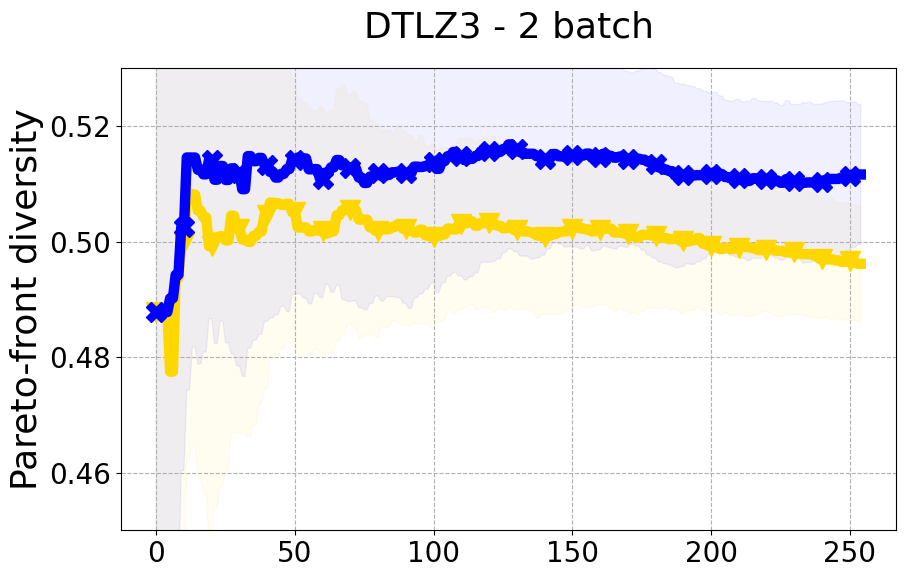}
     \end{subfigure}
     \hfill
     \begin{subfigure}
         \centering
         \includegraphics[width=0.24\textwidth]{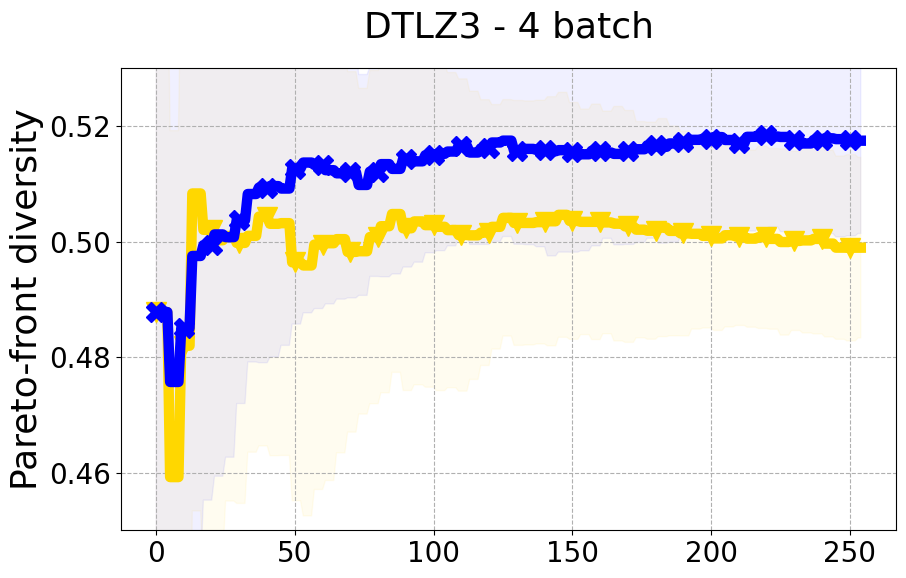}
     \end{subfigure}
     \hfill
     \begin{subfigure}
         \centering
         \includegraphics[width=0.24\textwidth]{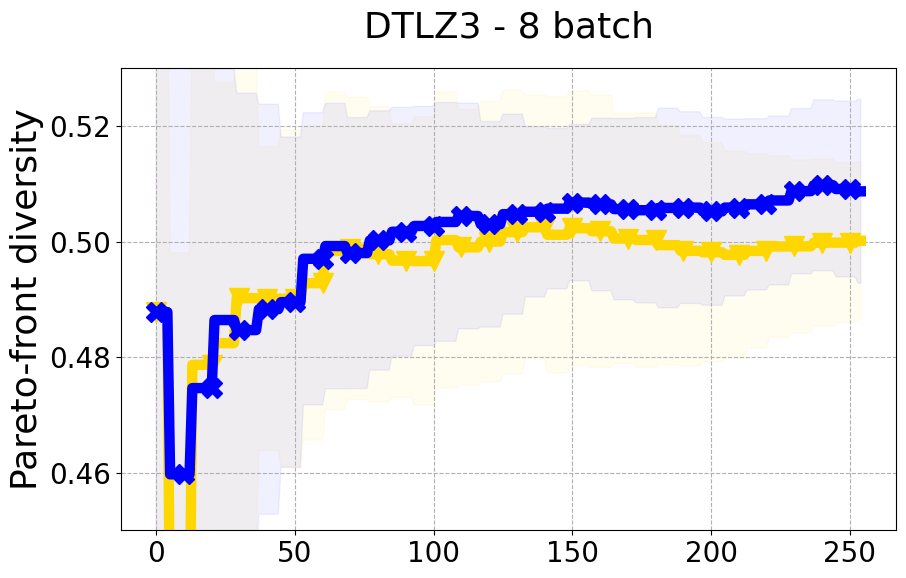}
     \end{subfigure}
     \hfill
     \begin{subfigure}
         \centering
         \includegraphics[width=0.24\textwidth]{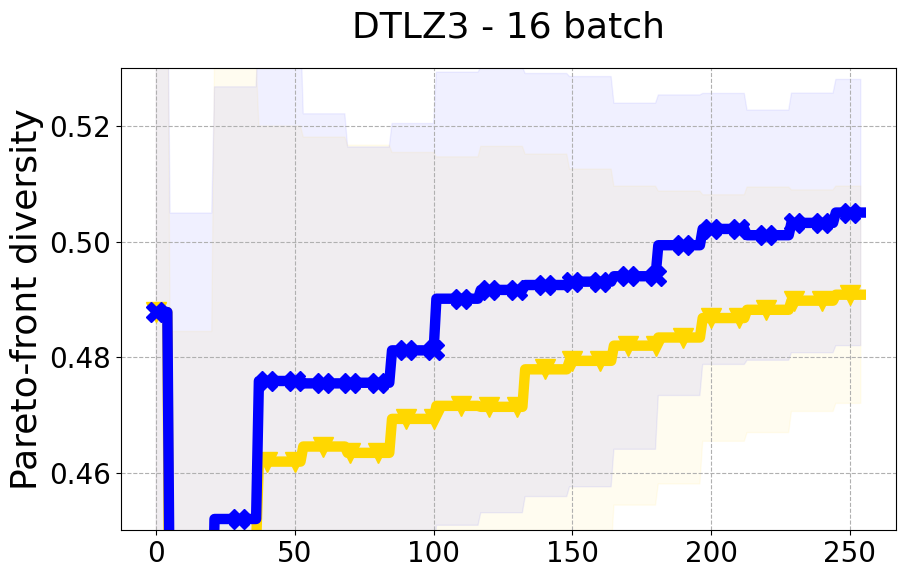}
     \end{subfigure}
        \centering
     \begin{subfigure}
         \centering
         \includegraphics[width=0.24\textwidth]{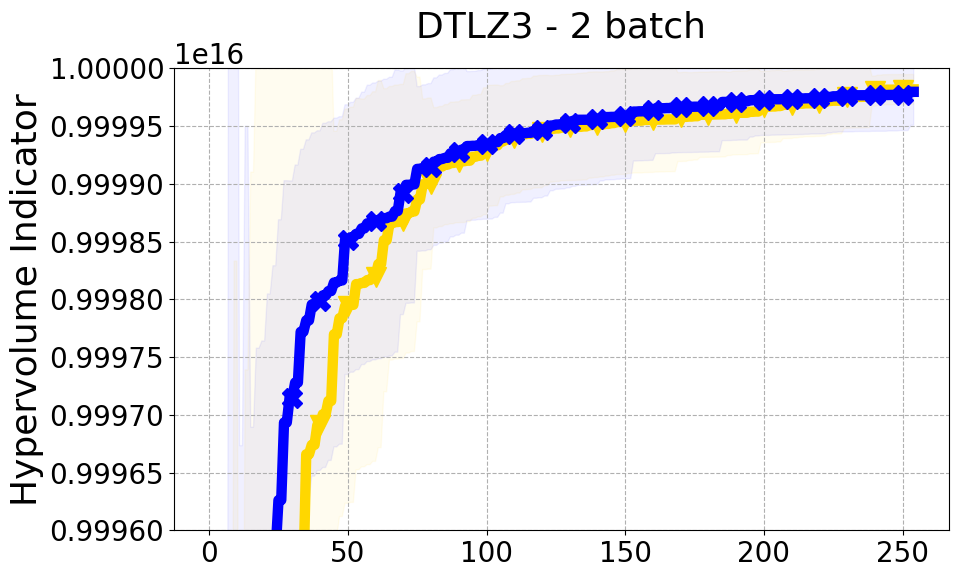}
     \end{subfigure}
     \hfill
     \begin{subfigure}
         \centering
         \includegraphics[width=0.24\textwidth]{figures/dpp/dtlz3/dtlz3-8-4-2batch-Hypervolumes-Plot_zoomed_plot.png}
     \end{subfigure}
     \hfill
     \begin{subfigure}
         \centering
         \includegraphics[width=0.24\textwidth]{figures/dpp/dtlz3/dtlz3-8-4-2batch-Hypervolumes-Plot_zoomed_plot.png}
     \end{subfigure}
     \hfill
     \begin{subfigure}
         \centering
         \includegraphics[width=0.24\textwidth]{figures/dpp/dtlz3/dtlz3-8-4-2batch-Hypervolumes-Plot_zoomed_plot.png}
     \end{subfigure}
     \begin{subfigure}
         \centering
         \includegraphics[width=0.4\textwidth]{figures/dpp-diversity-legend.png}
     \end{subfigure}
        \caption{DPF and hypervolume results - DTLZ3 problem, d = 8, K = 4}
        \label{fig:dpp-div-dtlz3}
        \vspace{0.3cm}
     \end{subfigure}
     \centering
     \begin{subfigure}
     \centering
     \begin{subfigure}
         \centering
         \includegraphics[width=0.24\textwidth]{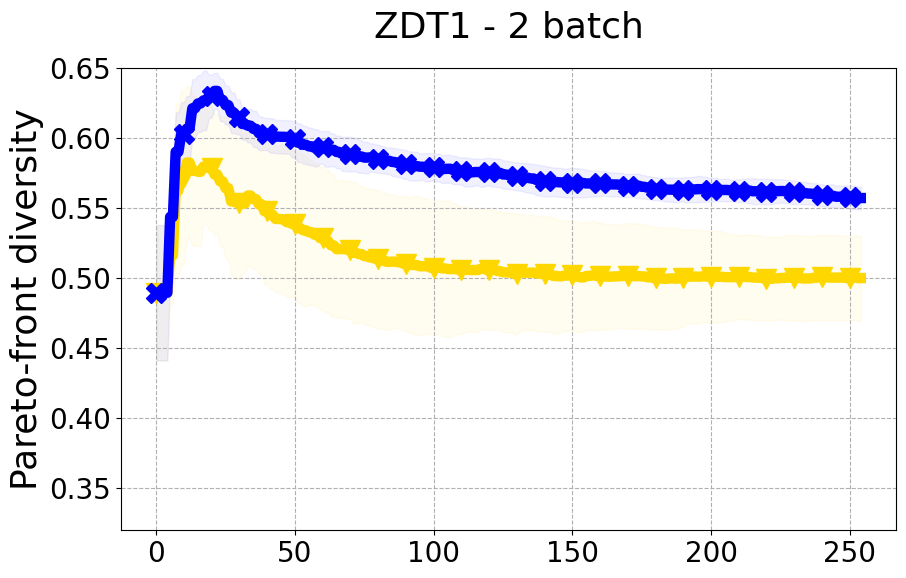}
     \end{subfigure}
     \hfill
     \begin{subfigure}
         \centering
         \includegraphics[width=0.24\textwidth]{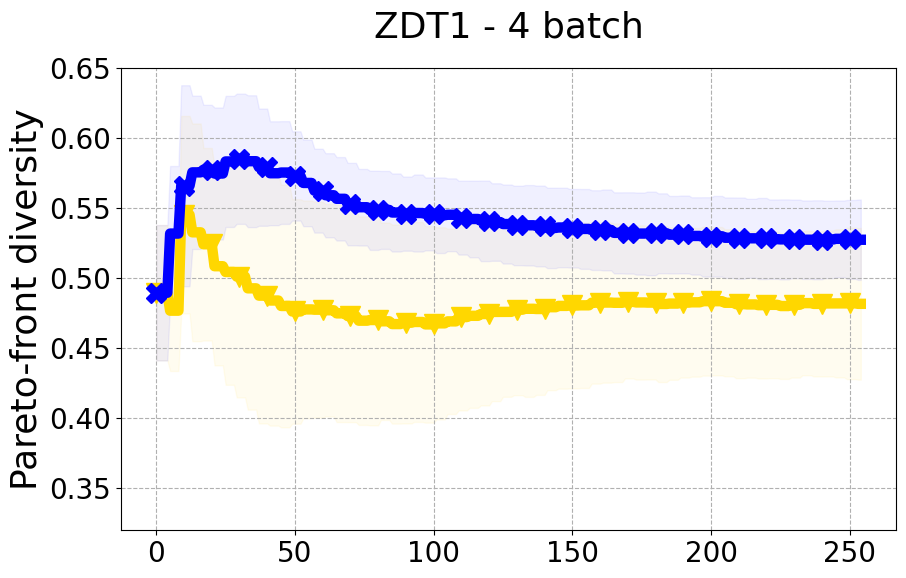}
     \end{subfigure}
     \hfill
     \begin{subfigure}
         \centering
         \includegraphics[width=0.24\textwidth]{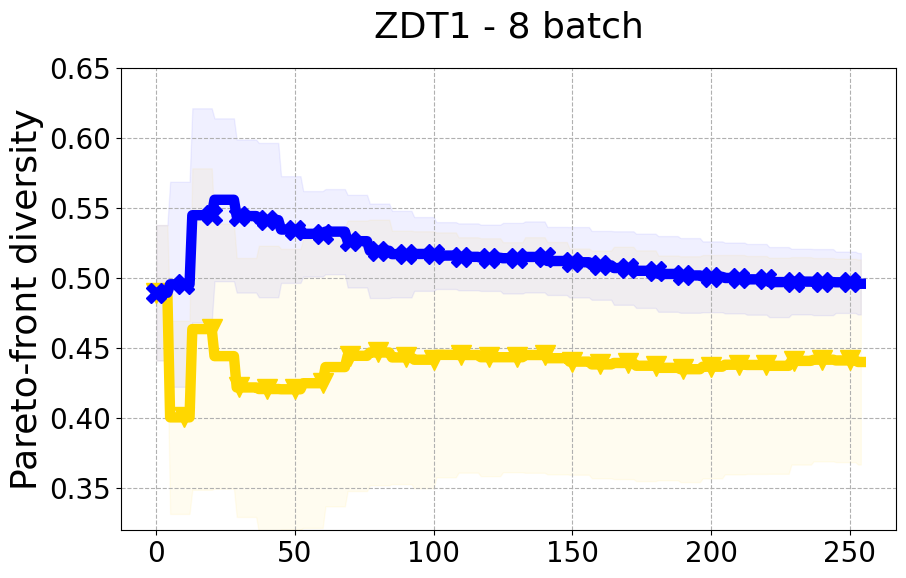}
     \end{subfigure}
     \hfill
     \begin{subfigure}
         \centering
         \includegraphics[width=0.24\textwidth]{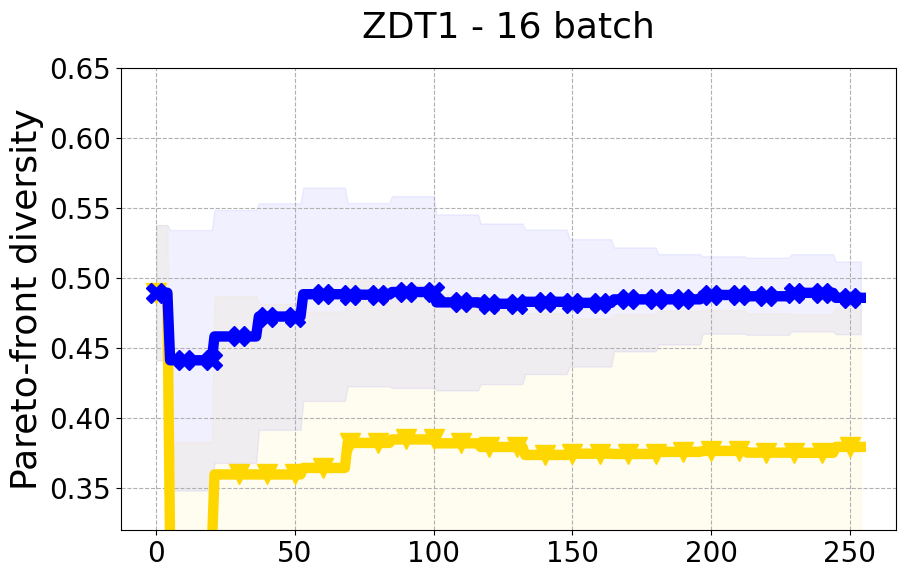}
     \end{subfigure}
        \centering
     \begin{subfigure}
         \centering
         \includegraphics[width=0.24\textwidth]{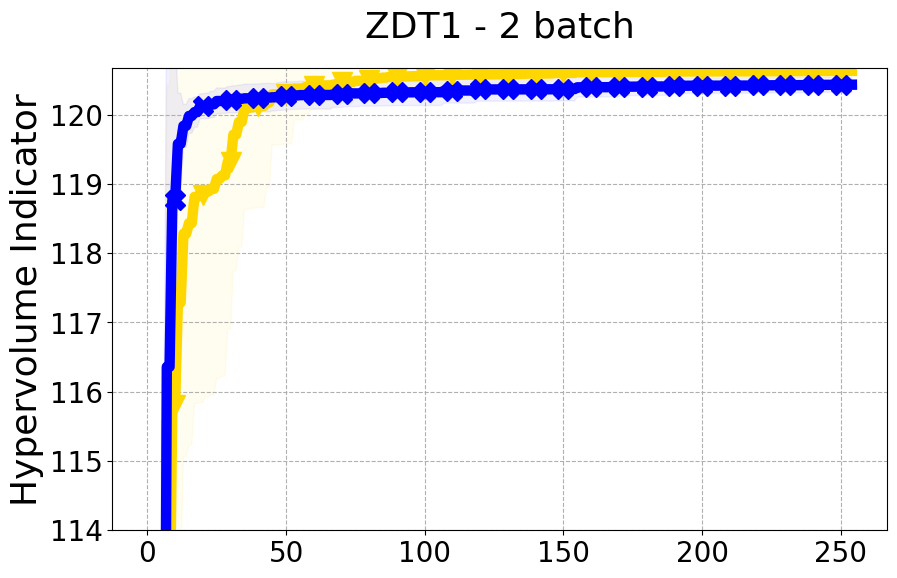}
     \end{subfigure}
     \hfill
     \begin{subfigure}
         \centering
         \includegraphics[width=0.24\textwidth]{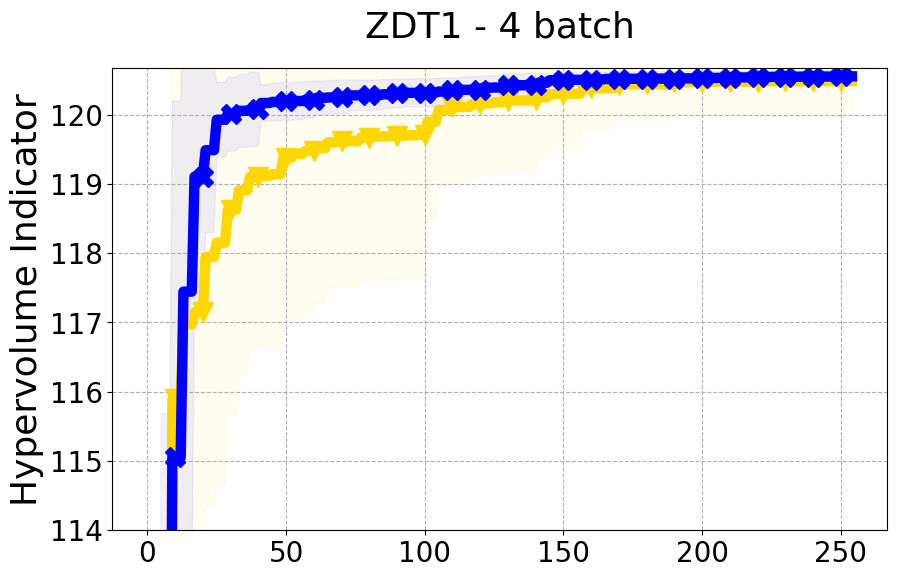}
     \end{subfigure}
     \hfill
     \begin{subfigure}
         \centering
         \includegraphics[width=0.24\textwidth]{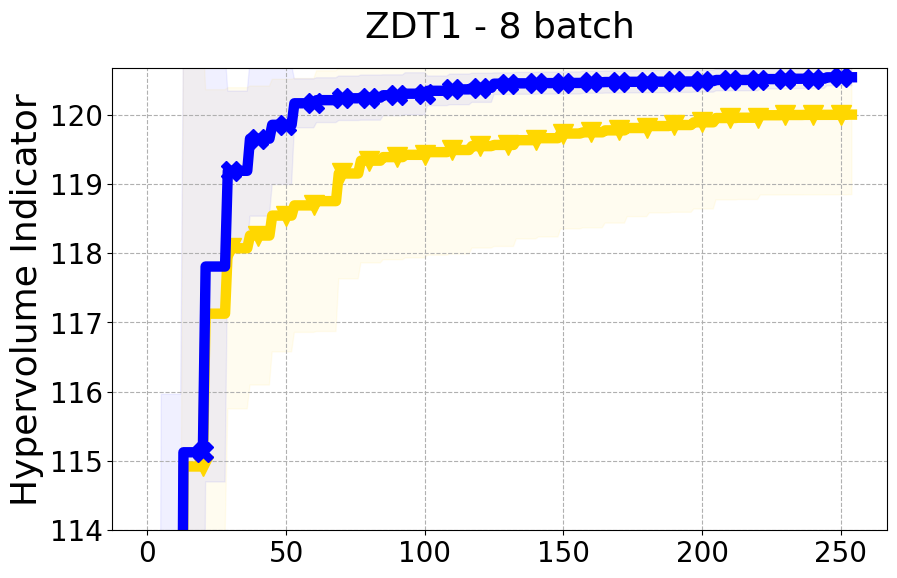}
     \end{subfigure}
     \hfill
     \begin{subfigure}
         \centering
         \includegraphics[width=0.24\textwidth]{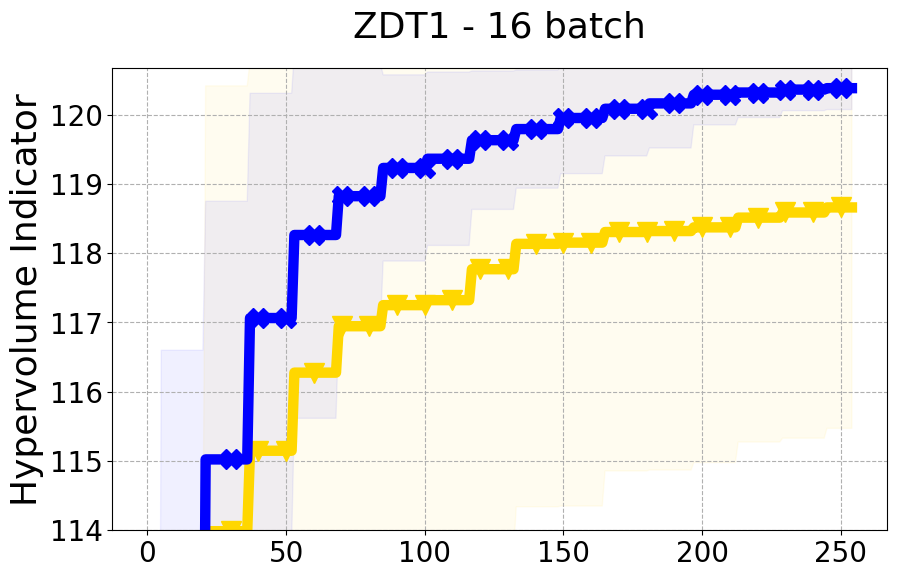}
     \end{subfigure}
     \begin{subfigure}
         \centering
         \includegraphics[width=0.4\textwidth]{figures/dpp-diversity-legend.png}
     \end{subfigure}
        \caption{DPF and hypervolume results - ZDT1 problem, d = 3, K = 2}
        \label{fig:dpp-div-zdt1}
    \end{subfigure}
\end{figure*}

\begin{figure*}[ht!]
     \centering
     \begin{subfigure}
         \centering
         \includegraphics[width=0.24\textwidth]{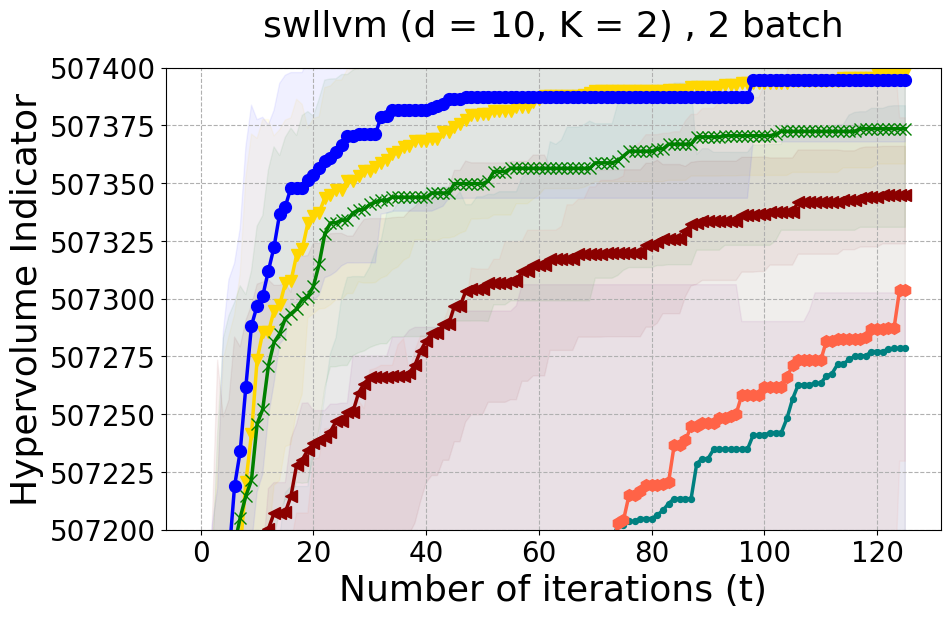}
     \end{subfigure}
     \hfill
     \begin{subfigure}
         \centering
         \includegraphics[width=0.24\textwidth]{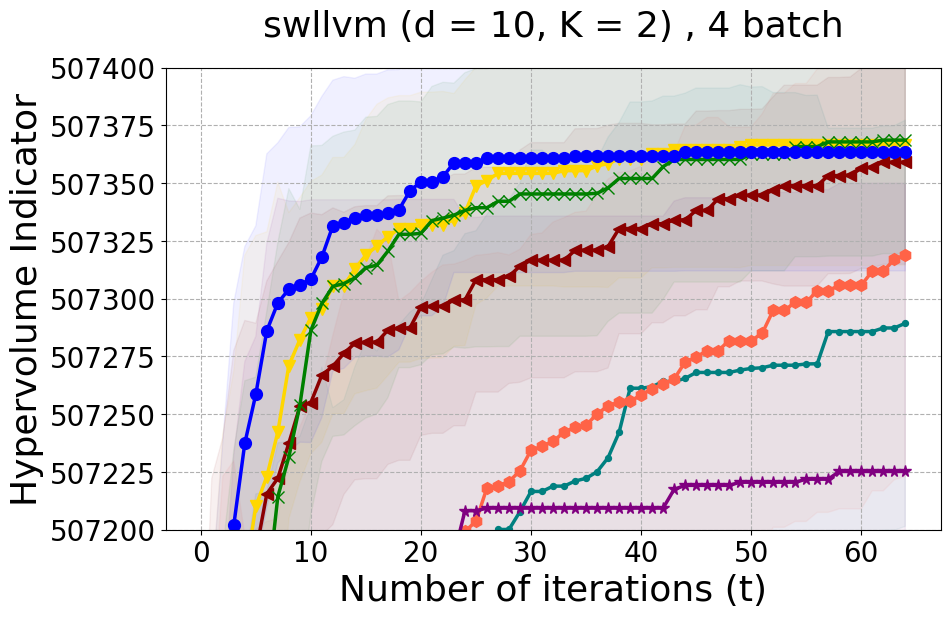}
     \end{subfigure}
     \hfill
     \begin{subfigure}
         \centering
         \includegraphics[width=0.24\textwidth]{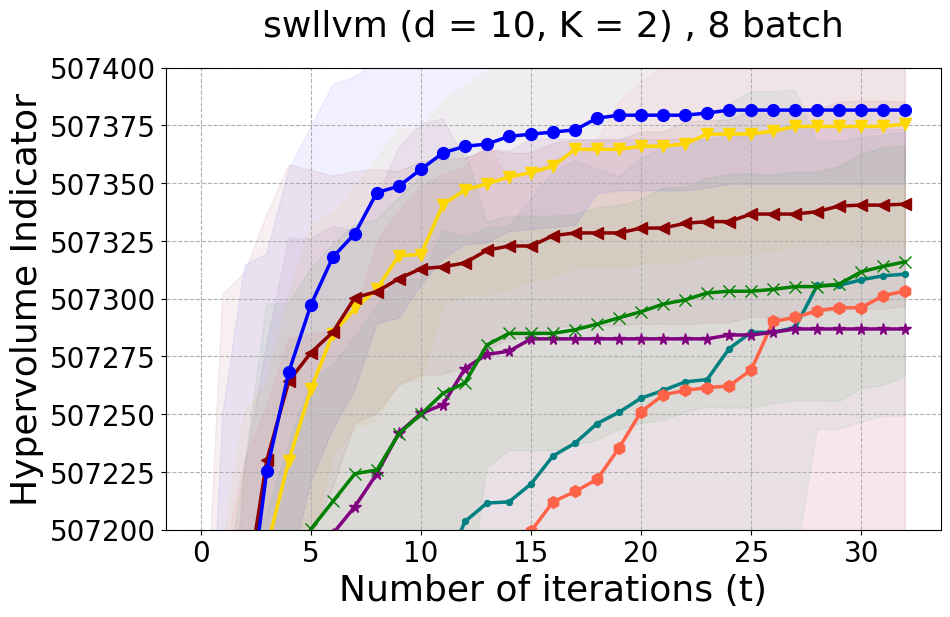}
     \end{subfigure}
     \centering
     \begin{subfigure}
         \centering
         \includegraphics[width=0.24\textwidth]{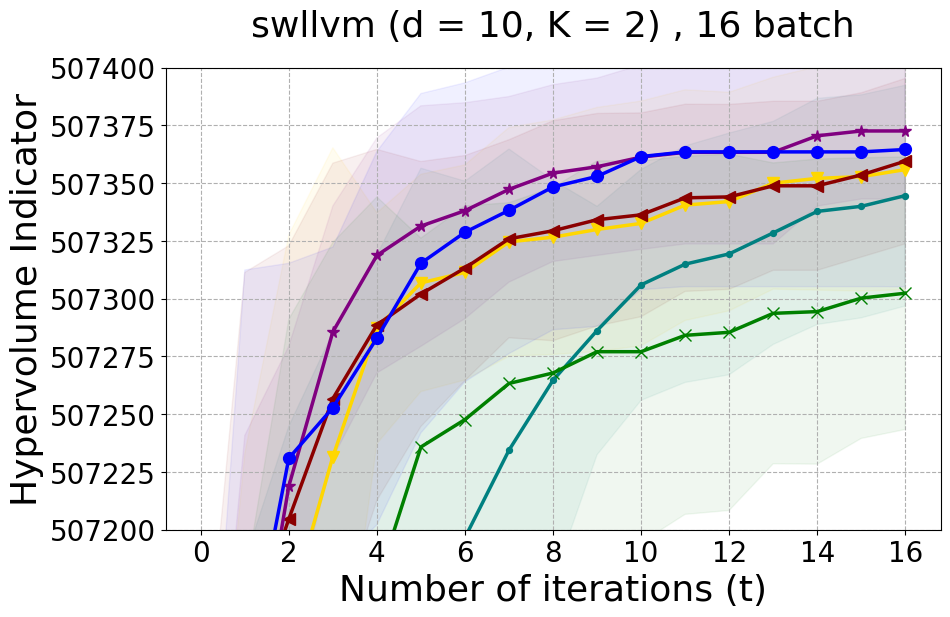}
     \end{subfigure}
     \hfill
     \begin{subfigure}
         \centering
         \includegraphics[width=0.24\textwidth]{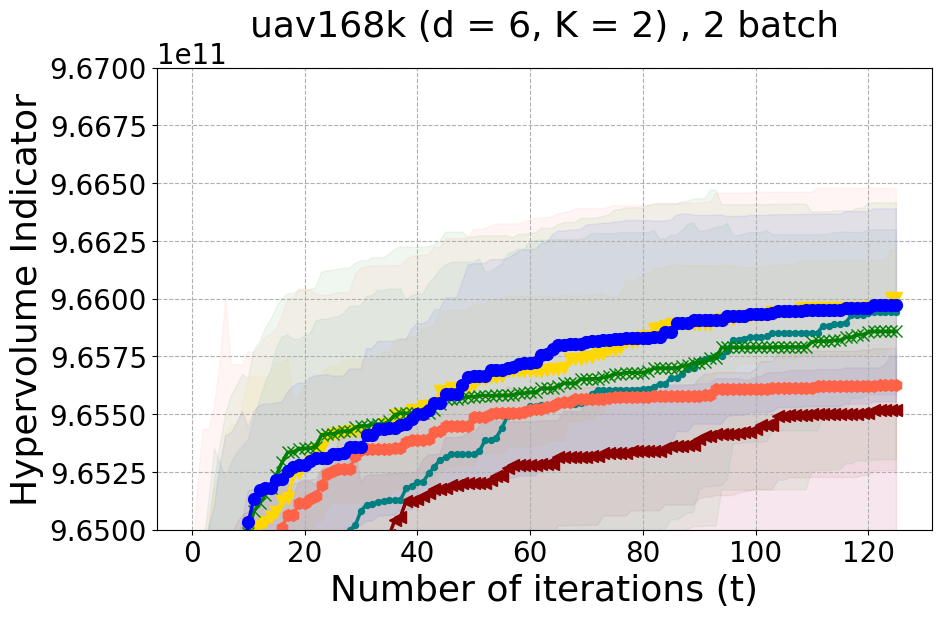}
     \end{subfigure}
     \hfill
     \begin{subfigure}
         \centering
         \includegraphics[width=0.24\textwidth]{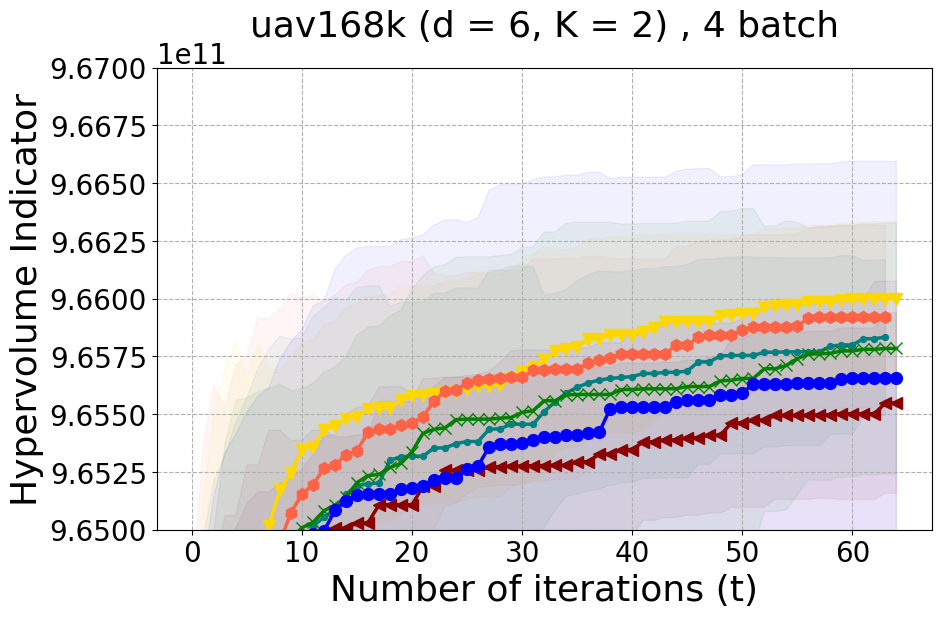}
     \end{subfigure}
     \hfill
     \begin{subfigure}
         \centering
         \includegraphics[width=0.24\textwidth]{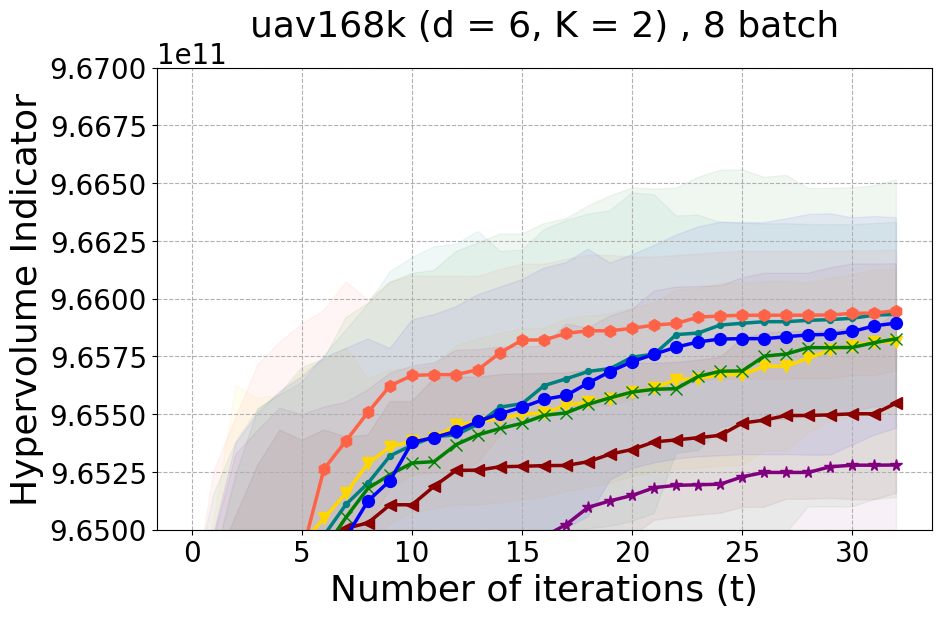}
     \end{subfigure}
      \centering
     \begin{subfigure}
         \centering
         \includegraphics[width=0.24\textwidth]{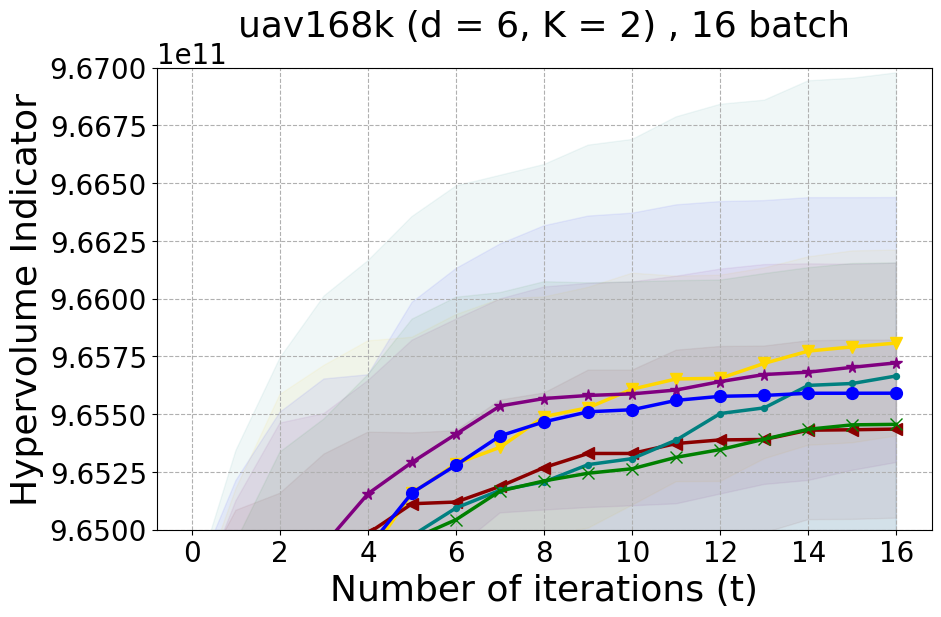}
     \end{subfigure}
    \begin{subfigure}
         \centering
         \includegraphics[width=0.98\textwidth]{figures/paper-legend.png}
     \end{subfigure}
        \caption{Additional real-world experiments hypervolume results}
    \label{fig:additional-real-world-hv}
\end{figure*}

\begin{figure*}[ht!]
     \centering
     \begin{subfigure}
         \centering
         \includegraphics[width=0.24\textwidth]{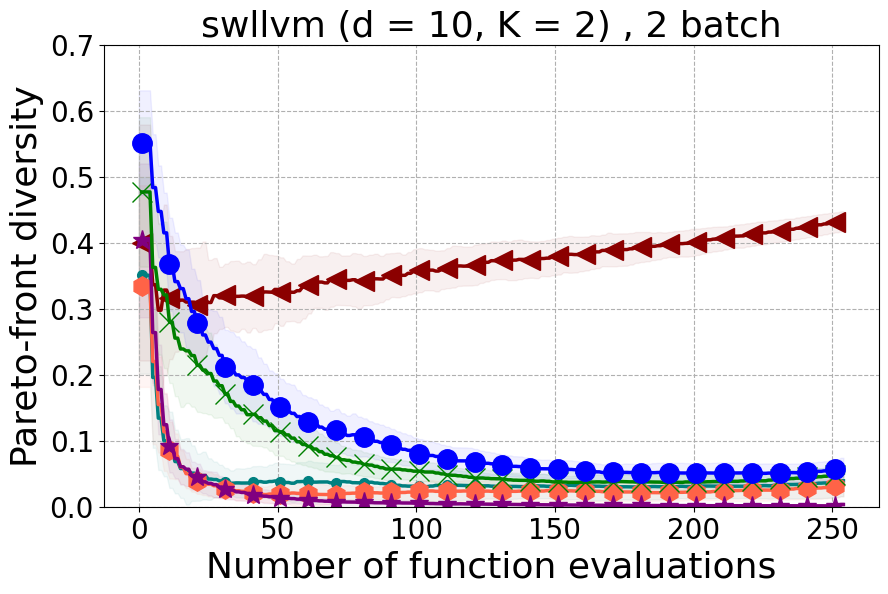}
     \end{subfigure}
     \hfill
     \begin{subfigure}
         \centering
         \includegraphics[width=0.24\textwidth]{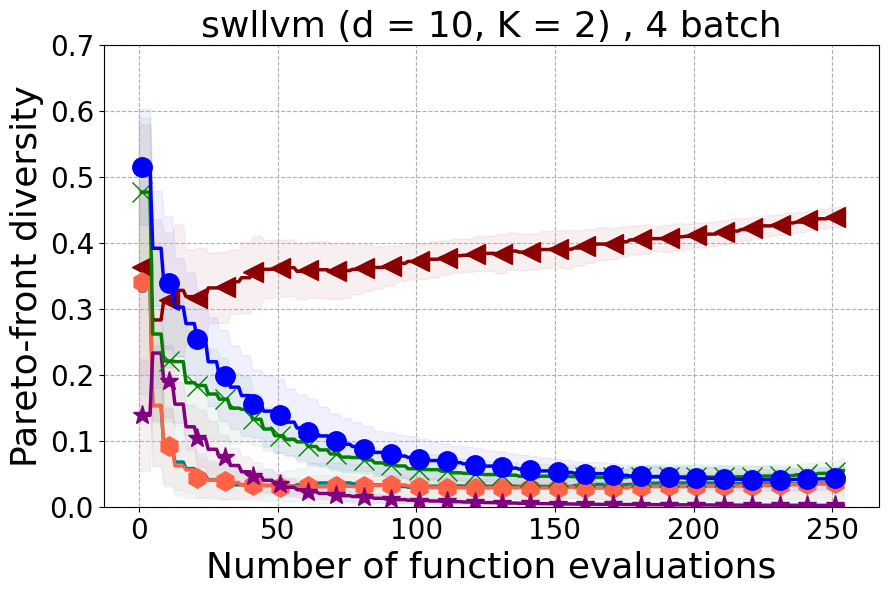}
     \end{subfigure}
     \hfill
     \begin{subfigure}
         \centering
         \includegraphics[width=0.24\textwidth]{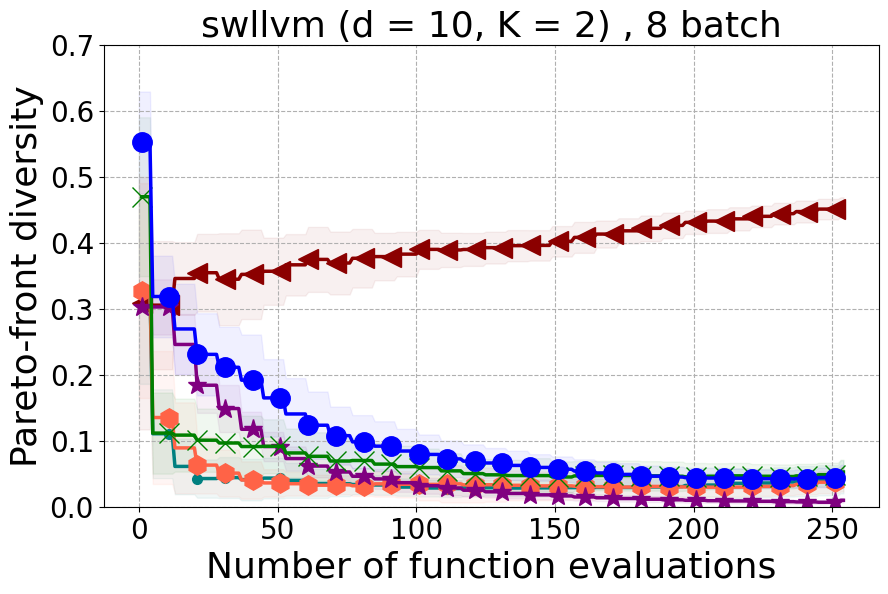}
     \end{subfigure}
     \centering
     \begin{subfigure}
         \centering
         \includegraphics[width=0.24\textwidth]{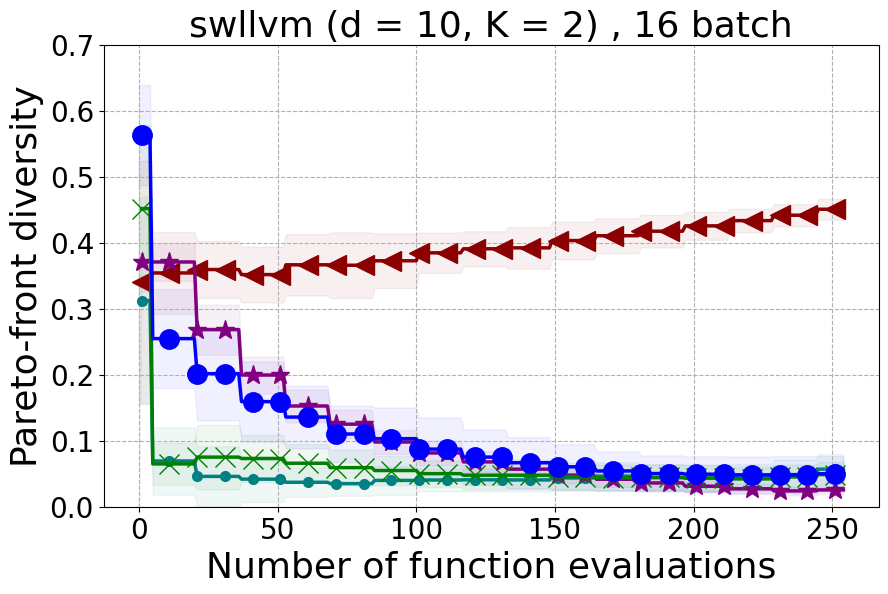}
     \end{subfigure}
     \hfill
     \begin{subfigure}
         \centering
         \includegraphics[width=0.24\textwidth]{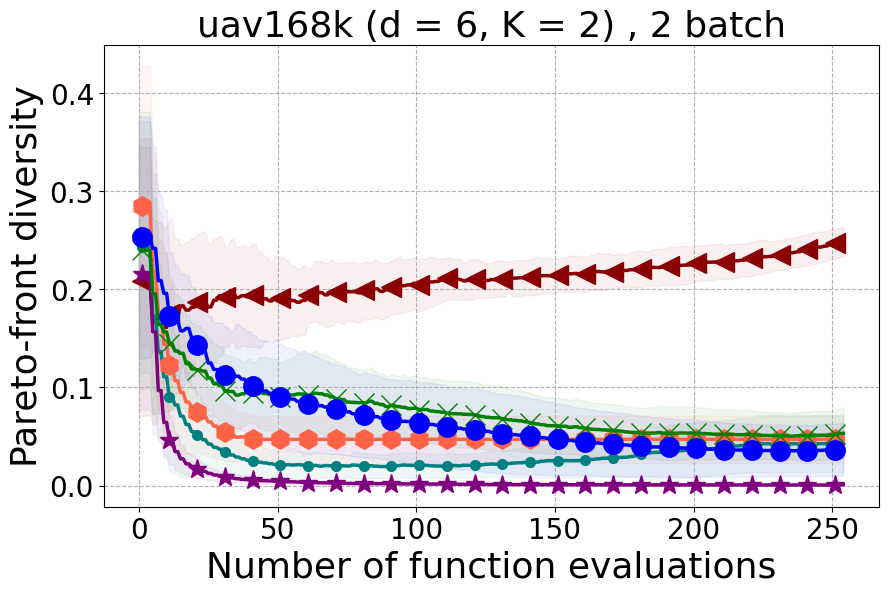}
     \end{subfigure}
     \hfill
     \begin{subfigure}
         \centering
         \includegraphics[width=0.24\textwidth]{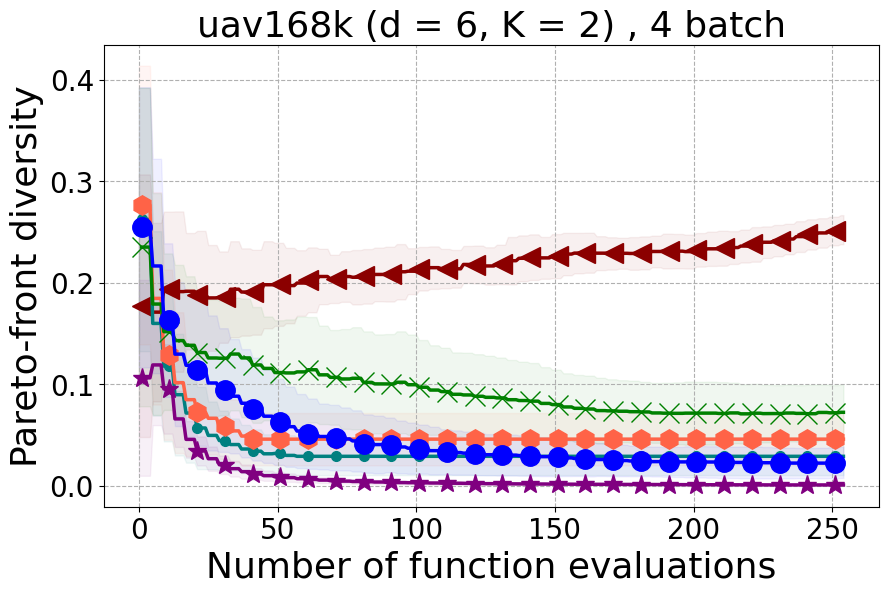}
     \end{subfigure}
     \hfill
     \begin{subfigure}
         \centering
         \includegraphics[width=0.24\textwidth]{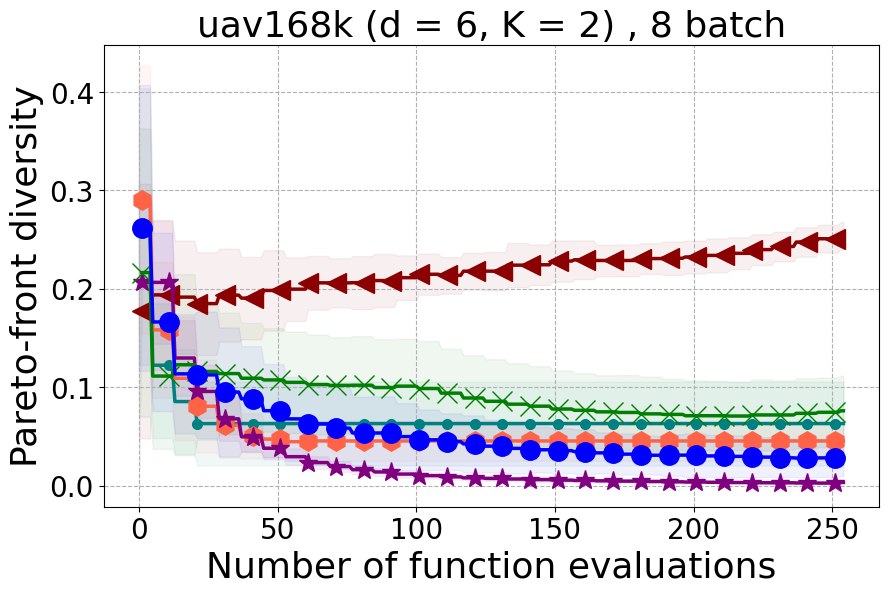}
     \end{subfigure}
      \centering
     \begin{subfigure}
         \centering
         \includegraphics[width=0.24\textwidth]{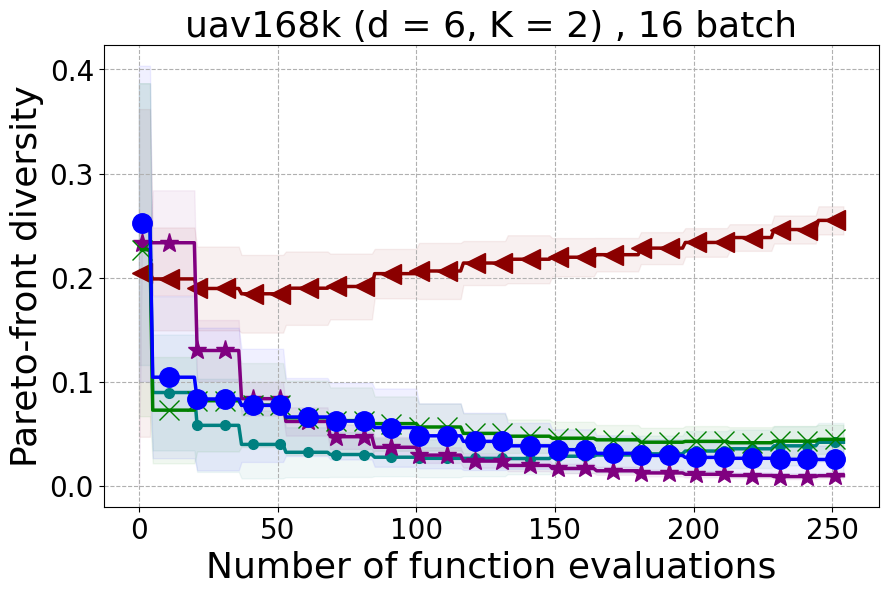}
     \end{subfigure}
    \begin{subfigure}
         \centering
         \includegraphics[width=0.98\textwidth]{figures/diversity-legend.png}
     \end{subfigure}
        \caption{Additional real-world experiments DPF results}
    \label{fig:additional-real-world-div}
\end{figure*}

\begin{figure*}[h!]
\begin{subfigure}{\includegraphics[width=0.24\textwidth]{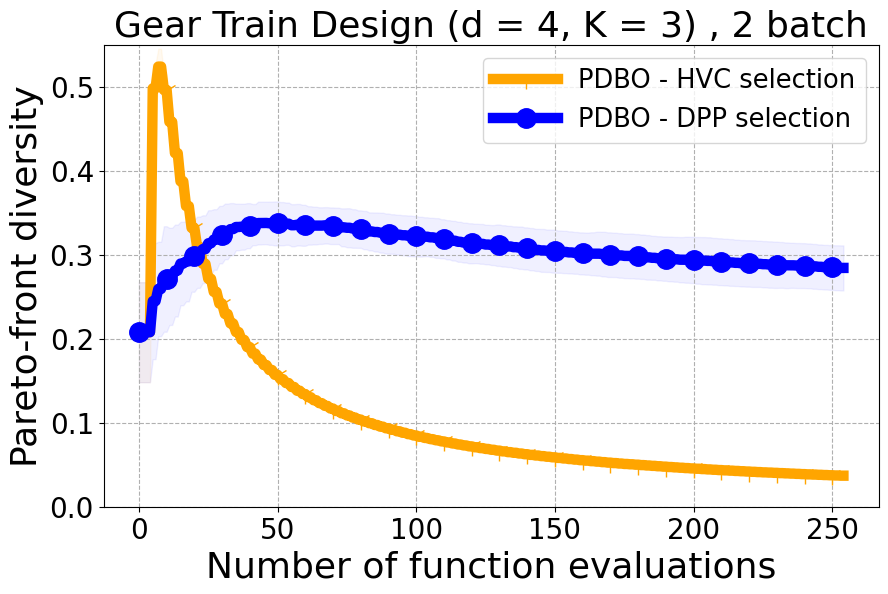}}
\end{subfigure}
\begin{subfigure}{\includegraphics[width=0.24\textwidth]{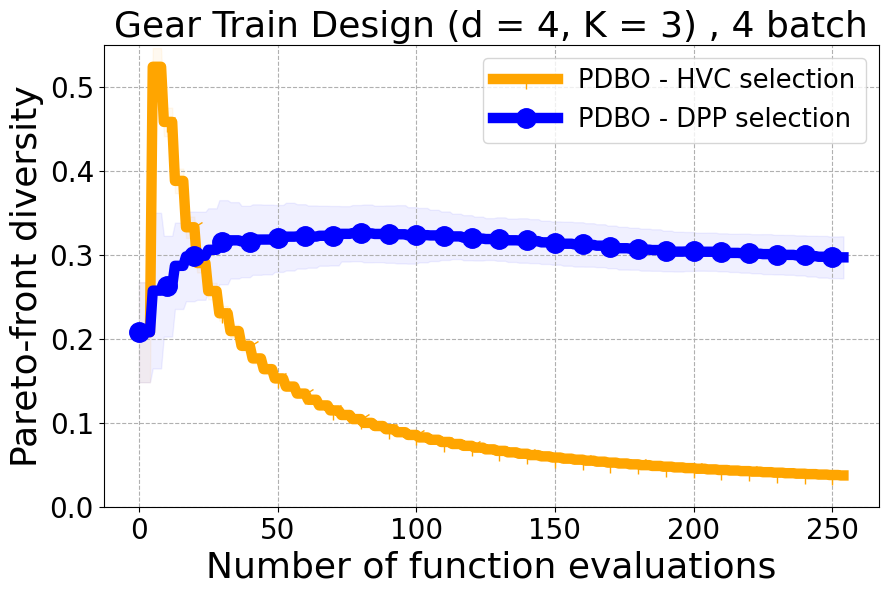}}
\end{subfigure}
\begin{subfigure}{\includegraphics[width=0.24\textwidth]{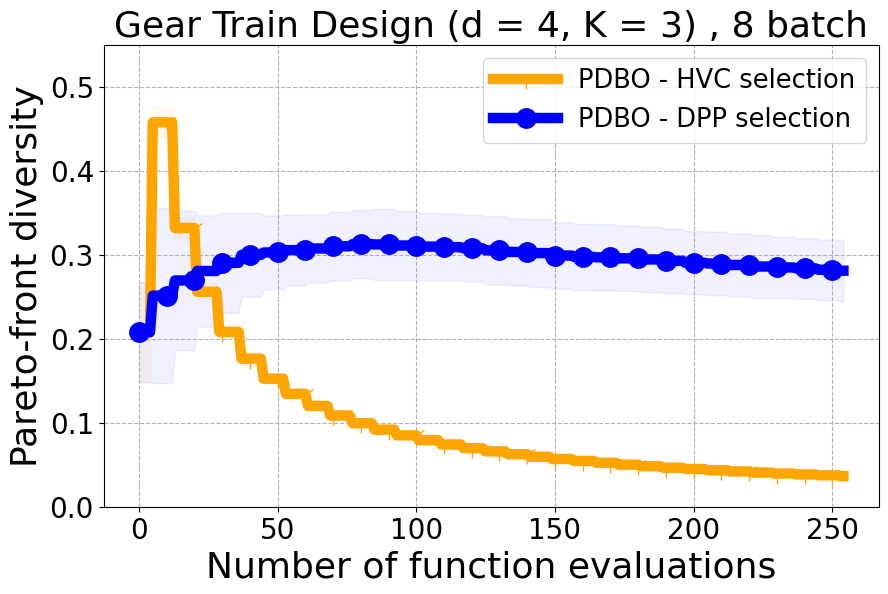}}
\end{subfigure}
\begin{subfigure}{\includegraphics[width=0.24\textwidth]{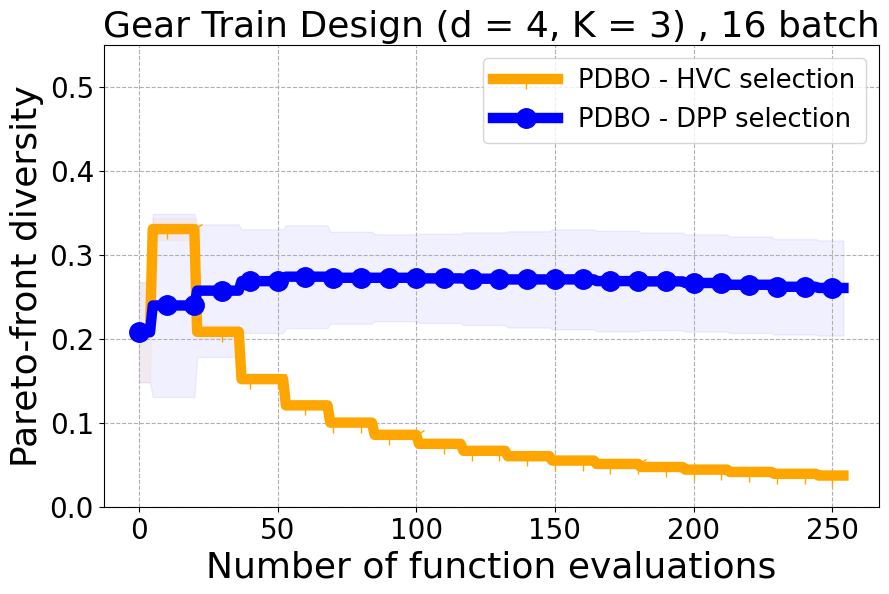}}\\
\end{subfigure}
\begin{subfigure}{\includegraphics[width=0.24\textwidth]{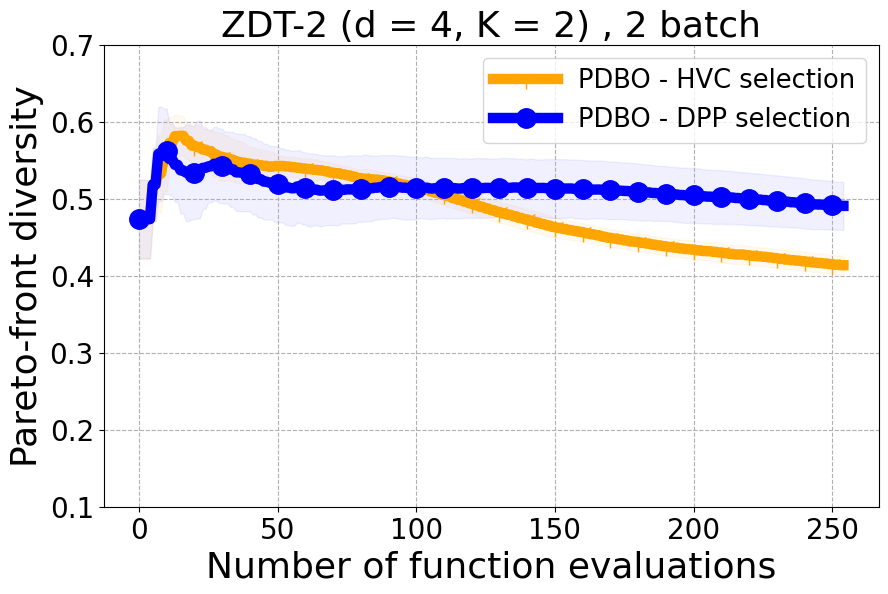}}
\end{subfigure}
\begin{subfigure}{\includegraphics[width=0.24\textwidth]{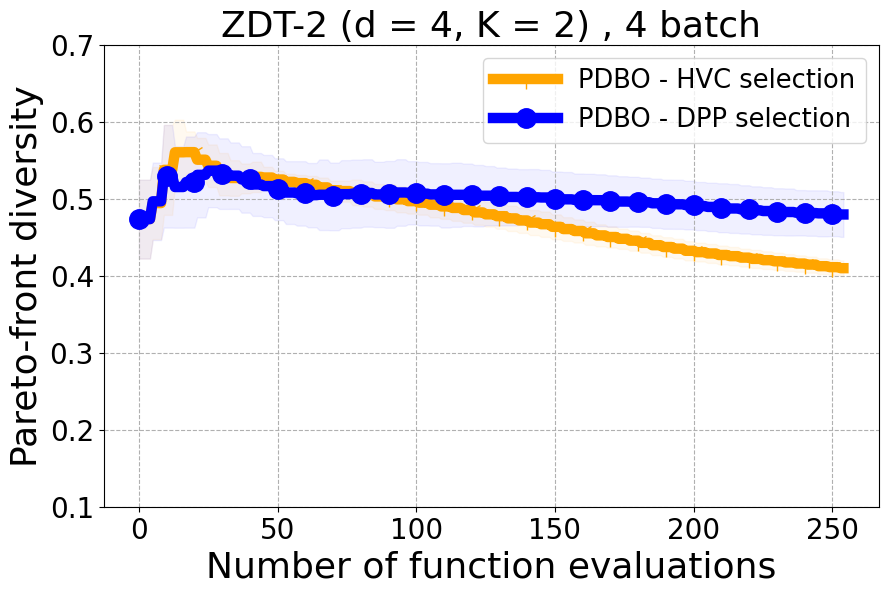}}
\end{subfigure}
\begin{subfigure}{\includegraphics[width=0.24\textwidth]{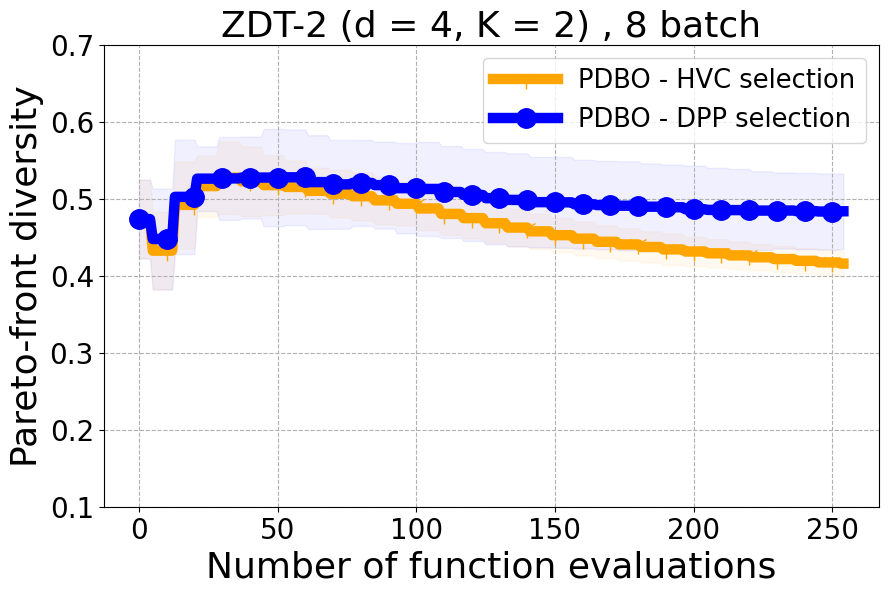}}
\end{subfigure}
\begin{subfigure}{\includegraphics[width=0.24\textwidth]{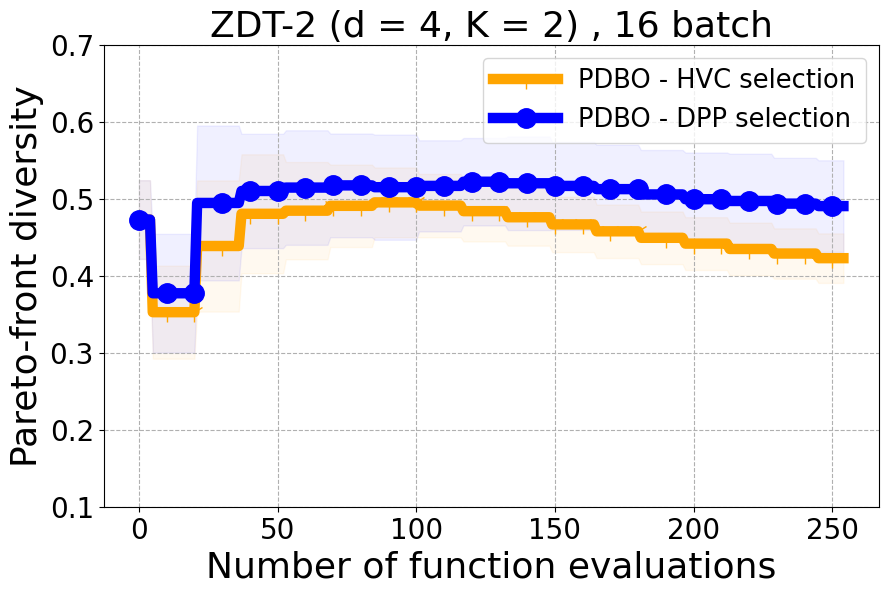}}
\end{subfigure}
\caption{Comparing DPF batch selection to greedy selection based on the highest individual hypervolume contribution}
\label{fig:rebuttal_hvc_vs_dpp}
\end{figure*}

 \begin{center}
\begin{table*}[ht]
\centering
\begin{tabular}{@{}cccccccccc@{}}
\toprule
\centering
\shortstack{Problem\\ Name} & \shortstack{Batch\\ Size} & PDBO      & DGEMO   & \multicolumn{2}{c}{\shortstack{qEHVI\\ CPU \hspace{0.35cm} GPU}}   & NSGA-II & \multicolumn{2}{c}{\shortstack{qPAREGO\\ CPU \hspace{0.35cm} GPU}}  & USEMO-EI \\ \midrule
\multirow{5}{*}{ZDT-1}  & 2    & 52.66   & 58.22 & 192.64 & 7.31 & 0.01 & 66.57 & 3.64 & 1.34   \\ \cmidrule(l){2-10} 
                   		& 4    & 57.73    & 57.93 & 510.11 & 16.17 & 0.01 & 113.00 & 5.94 & 1.47    \\ \cmidrule(l){2-10} 
                   		& 8    & 62.57   & 60.24 & 1781.94 & 53.67 & 0.01 & 226.20 & 11.93 & 1.55    \\ \cmidrule(l){2-10} 
                   		& 16  & 70.40   & 65.28 & NA & NA & 0.01 & 667.65 & 44.92 & 1.50    \\ \midrule
\multirow{5}{*}{ZDT-2}  & 2    & 52.37   & 51.47 & 48.22 & 2.95 & 0.01 & 16.58 & 1.37 & 0.89   \\ \cmidrule(l){2-10} 
                   		& 4    & 56.70   & 51.96 & 127.10 & 5.25 & 0.01 & 26.40 & 2.04 & 0.97    \\ \cmidrule(l){2-10} 
                   	    & 8    & 60.25   & 52.14 & 626.71 & 15.67 & 0.01 & 56.98 & 3.43 & 0.95    \\ \cmidrule(l){2-10} 
                   		& 16  & 64.91   & 51.94 & NA & NA & 0.01 & 163.08 & 9.05 & 0.91    \\ \midrule
\multirow{5}{*}{ZDT-3}  & 2    & 49.52   & 61.81 & 201.73 & 8.95 & 0.01 & 69.55 & 4.23 & 1.18   \\ \cmidrule(l){2-10} 
                   			 & 4    & 54.46  & 61.42 & 530.65 & 22.54 & 0.01 & 114.35 & 9.48 & 1.33    \\ \cmidrule(l){2-10} 
                   			 & 8    & 57.74   & 60.34 & 1871.59 & 68.70 & 0.01 & 247.07 & 18.39 & 1.33    \\ \cmidrule(l){2-10} 
                   			& 16  & 65.46   & 60.20 & NA & NA & 0.01 & 745.80 & 43.90 & 1.24    \\ \midrule
\multirow{5}{*}{Gear Train Design}  & 2    & 30.42   & 77.65 & 60.01 & 1.54 & 0.01 & 17.82 & 0.92 & 1.09    \\ \cmidrule(l){2-10} 
                  			& 4    & 34.51   & 77.15 & 316.89 & 5.62 & 0.01 & 43.32 & 2.19 & 1.08    \\ \cmidrule(l){2-10} 
                   			& 8    & 40.50   & 76.83 & 1534.78 & 37.66 & 0.01 & 200.46 & 7.05 & 1.08    \\ \cmidrule(l){2-10} 
                   			& 16  & 51.12    & 74.88 & NA & NA & 0.01 & 780.58 & 17.70 & 0.96    \\ \midrule
\multirow{5}{*}{DTLZ-1}  & 2    & 55.50   & NA & 267.98 & 25.99 & 0.02 & 112.70 & 7.43 & 2.28   \\ \cmidrule(l){2-10} 
                   			   & 4    & 60.86   & NA & NA & NA & 0.02 & 240.43 & 13.41 & 2.25   \\ \cmidrule(l){2-10} 
                   		           & 8    & 64.68   & NA  & NA & NA & 0.03 & 545.08 & 21.82 & 2.23    \\ \cmidrule(l){2-10} 
                   			   & 16  & 75.46    & NA & NA      & NA        & 0.05   & 1492.85 & 45.54   & 2.16    \\ \midrule
\multirow{5}{*}{DTLZ-3}    & 2    & 54.29    & NA & NA & NA & 0.02 & 111.90 & 9.20 & 2.07   \\ \cmidrule(l){2-10} 
                   			     & 4    & 58.44    & NA & NA & NA & 0.02 & 265.97 & 16.58 & 2.10   \\ \cmidrule(l){2-10} 
                   			     &8     & 63.90    & NA & NA & NA & 0.03 & 603.90 & 25.97 & 2.05    \\ \cmidrule(l){2-10} 
                  			     & 16  & 74.11     & NA   & NA      & NA        & 0.04  & 1676.98 & 50.30   & 2.02    \\ \midrule
\multirow{5}{*}{DTLZ-5}    & 2    & 242.78  & NA  & NA & NA & 1.15 & 146.71  & 6.89 & 4.86   \\ \cmidrule(l){2-10} 
                   			     & 4    & 244.08   & NA & NA & NA & 4.64 & 201.05 & 9.08 & 4.81   \\ \cmidrule(l){2-10} 
                   		             & 8    & 268.75   & NA & NA & NA & 1.59 & 368.43 & 12.69 & 4.41   \\ \cmidrule(l){2-10} 
                   			     & 16  & 303.58   & NA   & NA    & NA        & 2.06  & 935.33 & 29.46   & 4.20    \\ \midrule
\bottomrule
\end{tabular}
\caption{Average runtime (seconds per iteration) of each baseline}
\label{tab:runtimes}
\end{table*}
\end{center}

\end{document}